\documentclass[english]{article}
\usepackage[T1]{fontenc}
\usepackage[latin9]{inputenc}
\usepackage{geometry}
\usepackage{babel}
\usepackage{verbatim}
\usepackage{float}
\usepackage{bm}
\usepackage{amsmath}
\usepackage{amssymb}
\usepackage{graphicx}
\usepackage{hyperref}
\hypersetup{
 colorlinks,linkcolor=red,anchorcolor=blue,citecolor=blue}

\usepackage{multirow}

\makeatletter

\providecommand{\tabularnewline}{\\}
\floatstyle{ruled}
\newfloat{algorithm}{tbp}{loa}
\providecommand{\algorithmname}{Algorithm}
\floatname{algorithm}{\protect\algorithmname}

\usepackage{babel}

\usepackage{tabularx}

\usepackage{cite}\usepackage{amsthm}\usepackage{dsfont}\usepackage{array}\usepackage{mathrsfs}\usepackage{comment}\onecolumn

\usepackage{color}

\allowdisplaybreaks

\usepackage{enumitem}
\setlist[itemize]{leftmargin=2em}
\setlist[enumerate]{leftmargin=2em}

\usepackage{babel}
\usepackage{algorithm}
\usepackage{algorithmic}
\usepackage{arydshln}

\DeclareMathOperator{\ind}{\mathds{1}}  

\numberwithin{equation}{section}

\definecolor{yxc}{RGB}{255,0,0}
\definecolor{yjc}{RGB}{125,0,0}
\definecolor{cm}{RGB}{0,0,200}
\definecolor{kzw}{RGB}{0,150,0}
\newcommand{\cm}[1]{\textcolor{cm}{[CM: #1]}}

\usepackage{bbm}

\makeatother

\begin{document}
\theoremstyle{plain} \newtheorem{lemma}{\textbf{Lemma}} \newtheorem{prop}{\textbf{Proposition}}\newtheorem{theorem}{\textbf{Theorem}}\setcounter{theorem}{0}
\newtheorem{corollary}{\textbf{Corollary}}\newtheorem{claim}{\textbf{Claim}}
\newtheorem{assumption}{\textbf{Assumption}} \newtheorem{example}{\textbf{Example}}
\newtheorem{definition}{\textbf{Definition}} \newtheorem{fact}{\textbf{Fact}}\newtheorem{remark}{\textbf{Remark}}
\theoremstyle{definition}

\theoremstyle{remark}\newtheorem{condition}{\textbf{Condition}}\newtheorem{conjecture}{\textbf{Conjecture}}
\title{Inference and Uncertainty Quantification for \\
Noisy Matrix Completion\footnotetext{Author names are sorted alphabetically.}}
\author{Yuxin Chen\thanks{Department of Electrical Engineering, Princeton University, Princeton,
NJ 08544, USA; Email: \texttt{yuxin.chen@princeton.edu}.} \and Jianqing Fan\thanks{Department of Operations Research and Financial Engineering, Princeton
University, Princeton, NJ 08544, USA; Email: \texttt{\{jqfan, congm,
yulingy\}@princeton.edu}.} \and Cong Ma\footnotemark[2] \and Yuling Yan\footnotemark[2]}

\date{June 2019; \quad Revised: October 2019}

\maketitle
\begin{abstract}
Noisy matrix completion aims at estimating a low-rank matrix given
only partial and corrupted entries. Despite substantial progress in
designing efficient estimation algorithms, it remains largely unclear
how to assess the uncertainty of the obtained estimates and how to
perform statistical inference on the unknown matrix (e.g.~constructing
a valid and short confidence interval for an unseen entry).

This paper takes a step towards inference and uncertainty quantification
for noisy matrix completion. We develop a simple procedure to compensate
for the bias of the widely used convex and nonconvex estimators. The
resulting de-biased estimators admit nearly precise non-asymptotic distributional characterizations,
which in turn enable optimal construction of confidence intervals\,/\,regions
for, say, the missing entries and the low-rank factors. Our inferential
procedures do not rely on sample splitting, thus avoiding unnecessary
loss of data efficiency.  
As a byproduct, we obtain a sharp characterization of the estimation accuracy of our de-biased estimators, which, to the best of our knowledge, are the first tractable algorithms that provably achieve full statistical efficiency (including the preconstant). 
The analysis herein is built upon the intimate
link between convex and nonconvex optimization --- an appealing feature
recently discovered by \cite{chen2019noisy}.

\end{abstract}


\tableofcontents{}

\section{Introduction}

\subsection{Motivation: inference and uncertainty quantification?}

Low-rank matrix completion is concerned with recovering a low-rank
matrix, when only a small fraction of its entries are revealed to
us~\cite{srebro2004learning,ExactMC09,KesMonSew2010}. Tackling this
problem in large-scale applications is computationally challenging,
due to the intrinsic nonconvexity incurred by the low-rank structure.
To further complicate matters, another inevitable challenge stems
from the imperfectness of data acquisition mechanisms, wherein the
acquired samples are contaminated by a certain amount of noise.

Fortunately, if the entries of the unknown matrix are sufficiently
de-localized and randomly revealed, this problem may not be as hard
as it seems. Substantial progress has been made over the past several
years in designing computationally tractable algorithms --- including
both convex and nonconvex approaches --- that allow to fill in unseen
entries faithfully given only partial noisy samples~\cite{CanPla10,Negahban2012restricted,MR2906869,Se2010Noisy,chen2015fast,ma2017implicit,chen2019noisy}.
Nevertheless, modern decision making would often require one step
further. It not merely anticipates a faithful estimate, but also seeks
to quantify the uncertainty or ``confidence'' of the provided estimate,
ideally in a reasonably accurate fashion. For instance, given
an estimate returned by the convex approach, how to use it to compute
a short interval that is likely to contain a missing entry?

Conducting effective uncertainty quantification for noisy matrix completion
is, however, far from straightforward. For the most part, the state-of-the-art
matrix completion algorithms require solving highly complex optimization
problems, which often do not admit closed-form solutions. Of necessity,
it is generally very challenging to pin down the distributions of
the estimates returned by these algorithms. The lack of distributional
characterizations presents a major roadblock to performing valid, yet 
efficient, statistical inference on the unknown matrix of interest.

It is worth noting that a number of recent papers have been dedicated
to inference and uncertainty quantification for various high-dimensional problems in
high-dimensional statistics, including Lasso~\cite{zhang2014confidence,van2014asymptotically,javanmard2014confidence},
generalized linear models~\cite{van2014asymptotically,ning2017general,battey2018distributed}, graphical models~\cite{jankova2015confidence,ren2015asymptotic,ma2017inter}),
amongst others. Very little work, however, has looked into noisy matrix
completion along this direction. While non-asymptotic statistical
guarantees for noisy matrix completion have been derived in prior
theory, most, if not all, of the estimation error bounds are supplied
only at an order-wise level. Such order-wise error bounds either lose
a significant factor relative to the optimal guarantees, or come with
an unspecified (but often enormous) pre-constant. Viewed in this light,
a confidence region constructed directly based on such results is
bound to be overly conservative, resulting in substantial over-coverage.

\subsection{A glimpse of our contributions}

This paper takes a substantial step towards efficient inference and uncertainty
quantification for noisy matrix completion. Specifically, we develop
a simple procedure to compensate for the bias of the commonly used
convex and nonconvex estimators. The resulting de-biased estimators
admit nearly accurate non-asymptotic distributional guarantees. Such distributional
characterizations in turn allow us to reason about the uncertainty
of the obtained estimates vis-à-vis the unknown matrix. While details
of our main findings are postponed to Section~\ref{sec:procedures-and-theory},
we would like to immediately single out a few important merits of
the proposed inferential procedures and theory:
\begin{enumerate}
\item Our results enable two types of uncertainty assessment, namely, we
can construct (i) confidence intervals for each entry --- either
observed or missing --- of the unknown matrix; (ii) confidence regions
for the low-rank factors of interest (modulo some unavoidable global
ambiguity).
\item Despite the complicated statistical dependency, our procedure and
theory do not rely on sample splitting, thus avoiding the unnecessary
widening of confidence intervals\,/\,regions due to insufficient
data usage.
\item The confidence intervals\,/\,regions constructed based on the proposed
procedures are, in some sense, optimal.
\item We present a unified approach that accommodates both convex and nonconvex
estimators seamlessly.
\item As a byproduct, we characterize the Euclidean estimation errors of the proposed de-biased estimators. Such error bounds are sharp and match an oracle lower bound precisely (including the pre-constant). To the best of our knowledge, this is the first theory that demonstrates that a computationally feasible algorithm can achieve the statistical limit including the pre-constant. 
\end{enumerate}
All of this is built upon the intimate link between convex and nonconvex
estimators~\cite{chen2019noisy}, as well as the recent advances in
analyzing the stability of nonconvex optimization against random noise
\cite{ma2017implicit}.

\section{Models and notation\label{sec:setup}}

To cast the noisy matrix completion problem in concrete statistical
settings, we adopt a model commonly studied in the literature \cite{ExactMC09}. We also introduce some useful notation.

\paragraph{Ground truth.}
Denote
by $\bm{M}^{\star}\in\mathbb{R}^{n\times n}$ the unknown rank-$r$
matrix of interest,\footnote{We restrict our attention to squared matrices for simplicity of presentation.
Most findings extend immediately to the more general rectangular case
$\bm{M}^{\star}\in\mathbb{R}^{n_{1}\times n_{2}}$ with different
$n_{1}$ and $n_{2}$.} whose (compact) singular value decomposition (SVD) is given by
$\bm{M}^{\star}=\bm{U}^{\star}\bm{\Sigma}^{\star}\bm{V}^{\star\top}$.
We  set
\begin{equation}
\sigma_{\max}\triangleq\sigma_{1}(\bm{M}^{\star}),\quad\sigma_{\min}\triangleq\sigma_{r}(\bm{M}^{\star}),\quad\text{and}\quad\kappa\triangleq\sigma_{\max}/\sigma_{\min}, \label{eq:defn-sigma-max-sigma-min-kappa}
\end{equation}
where $\sigma_{i}(\bm{A})$ denotes
the $i$th largest singular value of a matrix $\bm{A}$. 
Further, we let $\bm{X}^{\star}\triangleq\bm{U}^{\star}\bm{\Sigma}^{\star1/2}\in\mathbb{R}^{n\times r}$
and $\bm{Y}^{\star}\triangleq\bm{V}^{\star}\bm{\Sigma}^{\star1/2}\in\mathbb{R}^{n\times r}$
stand for the \emph{balanced} low-rank factors of $\bm{M}^{\star}$,
which obey
\begin{equation}
\bm{X}^{\star\top}\bm{X}^{\star}=\bm{Y}^{\star\top}\bm{Y}^{\star}=\bm{\Sigma}^{\star}\qquad\text{and}\qquad\bm{M}^{\star}=\bm{X}^{\star}\bm{Y}^{\star\top}.\label{eq:defn-Xstar-Ystar}
\end{equation}

\paragraph{Observation models.} What we observe is a random subset of noisy entries of $\bm{M}^{\star}$;
more specifically, we observe
\begin{equation}
M_{ij}=M_{ij}^{\star}+E_{ij},\qquad E_{ij}\overset{\mathrm{i.i.d.}}{\sim}\mathcal{N}(0,\sigma^{2}),\qquad\text{for all }(i,j)\in\Omega,\label{eq:observation-model}
\end{equation}
where $\Omega\subseteq\{1,\cdots,n\}\times\{1,\cdots,n\}$ is a subset
of indices, and $E_{ij}$ denotes independently generated noise at
the location $(i,j)$. From now on, we assume the \emph{random sampling}
model where each index $(i,j)$ is included in~$\Omega$ independently
with probability $p$ (i.e.~data are missing uniformly at random). We shall use
$\mathcal{P}_{\Omega}(\cdot):\mathbb{R}^{n\times n}\mapsto\mathbb{R}^{n\times n}$
to represent the orthogonal projection onto the subspace of matrices
that vanish outside the index set $\Omega$.

\paragraph{Incoherence conditions.} 
 Clearly, not all  matrices can be reliably estimated from a highly incomplete set of measurements. To address this issue, we impose a standard incoherence
condition~\cite{ExactMC09,chen2015incoherence} on the singular subspaces
of $\bm{M}^{\star}$ (i.e.~$\bm{U}^{\star}$ and $\bm{V}^{\star}$):
\begin{equation}
\max\{\left\Vert \bm{U}^{\star}\right\Vert _{2,\infty},\left\Vert \bm{V}^{\star}\right\Vert _{2,\infty}\}\leq\sqrt{\mu r/n},\label{eq:incoherence-assumption-on-U}
\end{equation}
where $\mu$ is termed the incoherence parameter and $\|\bm{A}\|_{2,\infty}$ denotes the largest $\ell_{2}$ norm
of all rows in $\bm{A}$.  A small $\mu$ implies that the energy of  $\bm{U}^{\star}$ and $\bm{V}^{\star}$ are  reasonably spread out across all of their rows.

\paragraph{Asymptotic notation.} Here, $f(n) \lesssim h(n)$ (or $f(n)= O(h(n))$) means $|f(n)| \leq c_1 |h(n)|$ for some constant $c_1>0$,  $f(n) \gtrsim h(n)$  means $|f(n)| \geq c_2 |h(n)|$ for some constant $c_2>0$,  $f(n) \asymp h(n)$  means $c_2 |h(n)|\leq  |f(n)| \leq c_1 |h(n)|$ for some constants $c_1, c_2>0$, and $f(n) = o(h(n))$ means $\lim_{n\rightarrow \infty} f(n) / h(n) = 0$.  We write $f(n) \ll h(n)$ to indicate that $|f(n)| \leq c_1 |h(n)|$ for some  small constant $c_1>0$ (much smaller than 1), and use $f(n) \gg h(n)$ to indicate that $|f(n)| \geq c_2 |h(n)|$ for some large constant $c_2>0$ (much larger than 1).

\section{Inferential procedures and main results\label{sec:procedures-and-theory}}

The proposed inferential procedure lays its basis on two of the most
popular estimation paradigms --- convex relaxation and nonconvex
optimization --- designed for noisy matrix completion. Recognizing
the complicated bias of these two highly nonlinear estimators, we
shall first illustrate how to perform bias correction, followed by
a distributional theory that establishes the near-Gaussianity and
optimality of the proposed de-biased estimators.

\subsection{Background: convex and nonconvex estimators \label{subsec:Background}}

We first review in passing two tractable estimation algorithms that
are arguably the most widely used in practice. They serve as the starting
point for us to design inferential procedures for noisy low-rank matrix
completion. The readers familiar with this literature can proceed
directly to Section~\ref{subsec:debiasing}.

\paragraph{Convex relaxation.}

Recall that the rank function $\mathsf{rank}(\cdot)$ is highly nonconvex,
which often prevents us from computing a rank-constrained estimator
in polynomial time. For the sake of computational feasibility, prior
works suggest relaxing the rank function into its convex surrogate
\cite{fazel2002matrix,RecFazPar07}; for example, one can consider
the following penalized least-squares convex program
\begin{equation}
\underset{\bm{Z}\in\mathbb{R}^{n\times n}}{\textsf{minimize}}\qquad 
\frac{1}{2}\sum_{(i,j)\in\Omega}\left(Z_{ij}-M_{ij}\right)^{2}+\lambda\|\bm{Z}\|_{*},\label{eq:cvx}
\end{equation}
or using our notation $\mathcal{P}_{\Omega}$,
\begin{equation}
\underset{\bm{Z}\in\mathbb{R}^{n\times n}}{\textsf{minimize}}\qquad 
\frac{1}{2}\big\|\mathcal{P}_{\Omega}\big(\bm{Z}-\bm{M}\big)\big\|_{\mathrm{F}}^{2}+\lambda\|\bm{Z}\|_{*}.
\end{equation}
Here, $\|\cdot\|_{*}$ is the nuclear norm (the sum of singular values,
which is a convex surrogate of the rank function), and $\lambda>0$
is some regularization parameter. Under mild conditions, the solution
to the convex program~(\ref{eq:cvx}) provably attains near-optimal
estimation accuracy (in an order-wise sense), provided that a proper
regularization parameter $\lambda$ is adopted~\cite{chen2019noisy}.
\begin{algorithm}[t]
\caption{Gradient descent for solving the nonconvex problem~(\ref{eq:ncvx}).}

\label{alg:gd-mc-ncvx}\begin{algorithmic}

\STATE \textbf{{Suitable initialization}}: $\bm{X}^{0}$, $\bm{Y}^{0}$

\STATE \textbf{{Gradient updates}}: \textbf{for }$t=0,1,\ldots,t_{0}-1$
\textbf{do}

\STATE \vspace{-1em}
 \begin{subequations}\label{subeq:gradient_update_ncvx}
\begin{align}
\bm{X}^{t+1}= & \bm{X}^{t}-\frac{\eta}{p}\big[\mathcal{P}_{\Omega}(\bm{X}^{t}\bm{Y}^{t\top}-\bm{M})\bm{Y}^{t}+\lambda\bm{X}^{t}\big],\label{eq:gradient_update_ncvx_X}\\
\bm{Y}^{t+1}= & \bm{Y}^{t}-\frac{\eta}{p}\big[[\mathcal{P}_{\Omega}(\bm{X}^{t}\bm{Y}^{t\top}-\bm{M})]^{\top}\bm{X}^{t}+\lambda\bm{Y}^{t}\big],\label{eq:gradient_update_ncvx_Y}
\end{align}
where $\eta >0$ determines the step size or the learning rate.
\end{subequations}

\end{algorithmic}
\end{algorithm}

\paragraph{Nonconvex optimization.}

It is recognized that the convex approach, which typically relies
on solving a semidefinite program, is still computationally expensive
and not scalable to large dimensions. This motivates an alternative
route, which represents the matrix variable via two low-rank factors
$\bm{X},\bm{Y}\in\mathbb{R}^{n\times r}$ and attempts solving the
following nonconvex program directly
\begin{equation}
\underset{\bm{X},\bm{Y}\in\mathbb{R}^{n\times r}}{\textsf{minimize}}\qquad\frac{1}{2}\big\|\mathcal{P}_{\Omega}\big(\bm{X}\bm{Y}^{\top}-\bm{M}\big)\big\|_{\mathrm{F}}^{2}+\frac{\lambda}{2}\|\bm{X}\|_{\mathrm{F}}^{2}+\frac{\lambda}{2}\|\bm{Y}\|_{\mathrm{F}}^{2}.\label{eq:ncvx}
\end{equation}
Here, we choose a regularizer of the form $0.5\lambda(\|\bm{X}\|_{\mathrm{F}}^{2}+\|\bm{Y}\|_{\mathrm{F}}^{2})$
primarily to mimic the nuclear norm $\lambda\|\bm{Z}\|_{*}$ (see~\cite{srebro2005rank,mazumder2010spectral}). A variety of optimization
algorithms have been proposed to tackle the nonconvex program~(\ref{eq:ncvx})
or its variants~\cite{sun2016guaranteed,chen2015fast,ma2017implicit};
the readers are referred to~\cite{chi2018nonconvex} for a recent
overview. As a prominent example, a two-stage algorithm --- gradient
descent following suitable initialization --- provably enjoys fast
convergence and order-wise optimal statistical guarantees for a wide
range of scenarios~\cite{ma2017implicit,chen2019noisy,chen2019nonconvex}.
The current paper focuses on this simple yet powerful algorithm, as
documented in Algorithm~\ref{alg:gd-mc-ncvx} and detailed in Appendix~\ref{subsec:Algorithmic-details-nonconvex}.

\paragraph{Intimate connections between convex and nonconvex estimates.}

Denote by $\bm{Z}^{\mathsf{cvx}}$ any minimizer of the convex program
(\ref{eq:cvx}), and denote by $(\bm{X}^{\mathsf{ncvx}},\bm{Y}^{\mathsf{ncvx}})$
the estimate returned by Algorithm~\ref{alg:gd-mc-ncvx} aimed at
solving~(\ref{eq:ncvx}). As was recently shown in~\cite{chen2019noisy},
when the regularization parameter $\lambda$ is properly chosen, these
two estimates obey (see \eqref{eq:closedness-cvx-ncvx} in Appendix~\ref{subsec:Properties-of-approximate-solution} for a precise statement)
\begin{equation}
\bm{X}^{\mathsf{ncvx}}\bm{Y}^{\mathsf{ncvx}\top}\approx\bm{Z}^{\mathsf{cvx}}\approx\bm{Z}^{\mathsf{cvx},r}.\label{eq:proximity-cvx-ncvx}
\end{equation}
Here, $\bm{Z}^{\mathsf{cvx},r}\triangleq\mathcal{P}_{\text{rank-}r}(\bm{Z}^{\mathsf{cvx}})$
is the best rank-$r$ approximation of the convex estimate $\bm{Z}^{\mathsf{cvx}}$,
where $\mathcal{P}_{\text{rank-}r}(\bm{B})\triangleq\arg\min_{\bm{A}:\text{rank}(\bm{A})\leq r}\|\bm{A}-\bm{B}\|_{\mathrm{F}}$. 
In truth, the three matrices of interest in~(\ref{eq:proximity-cvx-ncvx})
are exceedingly close to, if not identical with, each other. This
salient feature paves the way for a unified treatment of convex and
nonconvex approaches: most inferential procedures and guarantees developed
for the nonconvex estimate can be readily transferred to perform inference
for the convex one, and vice versa.

\subsection{Constructing de-biased estimators \label{subsec:debiasing}}

\begin{table}[b]
{\renewcommand{\arraystretch}{1.3}
\caption{Notation used to unify the convex estimate $\bm{Z}^{\mathsf{cvx}}$
and the nonconvex estimate $(\bm{X}^{\mathsf{ncvx}},\bm{Y}^{\mathsf{ncvx}})$.
Here, $\bm{Z}^{\mathsf{cvx},r}=\mathcal{P}_{\text{rank-}r}(\bm{Z}^{\mathsf{cvx}})$
is the best rank-$r$ approximation of $\bm{Z}^{\mathsf{cvx}}$. See Appendix~\ref{sec:summary-estimators} for a complete summary. \label{tab:Notation-unified}}
\vspace{0.5em}
\begin{tabularx}{\textwidth}{c|X}

\hline

$\bm{Z}\in \mathbb{R}^{n\times n}$ & either $\bm{Z}^{\mathsf{cvx}}$
or $\bm{X}^{\mathsf{ncvx}}\bm{Y}^{\mathsf{ncvx}\top}$.

\tabularnewline \hline

$\bm{X},\bm{Y}\in\mathbb{R}^{n\times r}$ & for the nonconvex case,
we take $\bm{X}=\bm{X}^{\mathsf{ncvx}}$ and $\bm{Y}=\bm{Y}^{\mathsf{ncvx}}$;
for the convex case, let $\bm{X}=\bm{X}^{\mathsf{cvx}}$ and $\bm{Y}=\bm{Y}^{\mathsf{cvx}}$,
which are the \emph{balanced} low-rank factors of $\bm{Z}^{\mathsf{cvx},r}$
obeying $\bm{Z}^{\mathsf{cvx},r}=\bm{X}^{\mathsf{cvx}}\bm{Y}^{\mathsf{cvx}\top}$
and $\bm{X}^{\mathsf{cvx}\top}\bm{X}^{\mathsf{cvx}}=\vphantom{\frac{1}{2}}\bm{Y}^{\mathsf{cvx}\top}\bm{Y}^{\mathsf{cvx}}$.

\tabularnewline \hline

$\bm{M}^{\mathsf{d}}\in \mathbb{R}^{n\times n}$ & the proposed de-biased estimator as in~(\ref{eq:defn-Zd-cvx}).

\tabularnewline \hline

$\bm{X}^{\mathsf{d}}, \bm{Y}^{\mathsf{d}}\in \mathbb{R}^{n\times r}$ & the proposed de-shrunken estimator as in~(\ref{eq:defn-Xd-Yd}).

\tabularnewline \hline  \end{tabularx} }
\end{table}

We are now well equipped to describe how to construct new estimators based
on the convex estimate $\bm{Z}^{\mathsf{cvx}}$ and the nonconvex
estimate $(\bm{X}^{\mathsf{ncvx}},\bm{Y}^{\mathsf{ncvx}})$, so as
to enable statistical inference. Motivated by the proximity of the
convex and nonconvex estimates and for the sake of conciseness, we
shall abuse notation by using the shorthand $\bm{Z},\bm{X},\bm{Y}$ for both convex and nonconvex estimates; 
see Table~\ref{tab:Notation-unified} and Appendix~\ref{sec:summary-estimators} for precise definitions. This allows us to unify the presentation for both convex and nonconvex estimators. 

Given that both~(\ref{eq:cvx}) and~(\ref{eq:ncvx}) are regularized
least-squares problems, they behave effectively like shrinkage estimators,
indicating that the provided estimates  necessarily suffer from non-negligible
bias. In order to enable desired statistical inference, it is natural
to first correct the estimation bias.

\paragraph{A de-biased estimator for the matrix.}

A natural de-biasing strategy that immediately comes to mind is the
following simple linear transformation (recall the notation in Table~\ref{tab:Notation-unified}):
\begin{equation}
\bm{Z}^{0}\triangleq\bm{Z}-\frac{1}{p}\mathcal{P}_{\Omega}\big(\bm{Z}-\bm{M}\big)= \underset{\text{mean:}\,\bm{M}^{\star}}{\underbrace{\frac{1}{p}\mathcal{P}_{\Omega}\big(\bm{M}^{\star}\big)}} + \underset{\text{mean:}\,\bm{0}}{\underbrace{\frac{1}{p}\mathcal{P}_{\Omega}\big(\bm{E}\big) }} + \hspace{-0.5em} \underset{\text{mean:}\,\bm{0}\text{ (heuristically)}}{\underbrace{ \bm{Z}-\frac{1}{p}\mathcal{P}_{\Omega}\big(\bm{Z}\big) }}\hspace{-0.5em},
\end{equation}
where we identify $\mathcal{P}_{\Omega}(\bm{M})$ with $\mathcal{P}_{\Omega}(\bm{M}^{\star})+\mathcal{P}_{\Omega}\left(\bm{E}\right)$.
Heuristically, if $\Omega$ and $\bm{Z}$ are statistically independent,
then $\bm{Z}^{0}$ serves as an unbiased estimator of $\bm{M}^{\star}$,
i.e.~$\mathbb{E}[\bm{Z}^{0}]=\bm{M}^{\star}$; this arises since
the noise $\bm{E}$ has zero mean and $\mathbb{E}[\mathcal{P}_{\Omega}]=p\mathcal{I}$
under the uniform random sampling model, with $\mathcal{I}$ 
the identity operator. Despite its (near) unbiasedness nature at
a heuristic level, however, the matrix $\bm{Z}^{0}$ is typically
full-rank, with non-negligible energy spread across its entire spectrum.
This results in dramatically increased variability in the estimate,
which is undesirable for inferential purposes.

To remedy this issue, we propose to further project $\bm{Z}^{0}$
onto the set of rank-$r$ matrices, leading to the following de-biased estimator

\begin{equation}
\bm{M}^{\mathsf{d}}\triangleq\mathcal{P}_{\text{rank-}r}\Big[\bm{Z}-\frac{1}{p}\mathcal{P}_{\Omega}\left(\bm{Z}-\bm{M}\right)\Big],\label{eq:defn-Zd-cvx}
\end{equation}
where  $\mathcal{P}_{\text{rank-}r}(\bm{B})=\arg\min_{\bm{A}:\text{rank}(\bm{A})\leq r}\|\bm{A}-\bm{B}\|_{\mathrm{F}}$, and  $\bm{Z}$ can again be found in  Table~\ref{tab:Notation-unified}.
This projection step effectively suppresses the variability outside
the $r$-dimensional principal subspace. As we shall see shortly,
the proposed estimator~(\ref{eq:defn-Zd-cvx}) properly de-biases
the provided estimate~$\bm{Z}$, while optimally controlling the
extent of uncertainty.

\begin{remark} The estimator~(\ref{eq:defn-Zd-cvx}) can be viewed as performing one iteration of singular value projection (SVP) ~\cite{MekJaiDhi2009,ding2018leave} on the current estimate $\bm{Z}$.
\end{remark}

\begin{remark}
The estimator~(\ref{eq:defn-Zd-cvx}) also bears a similarity
to the de-biased estimator proposed by~\cite{xia2018confidence}
for low-rank trace regression; the disparity between them
shall be discussed in Section~\ref{sec:Prior-art}.
\end{remark}

\paragraph{An equivalent form: a de-shrunken estimator for the low-rank factors.}

It turns out that the de-biased estimator~(\ref{eq:defn-Zd-cvx})
admits another almost equivalent representation that offers further
insights. Specifically, we consider the following \emph{de-shrunken}
estimator for the low-rank factors
\begin{equation}
\bm{X}^{\mathrm{d}}\triangleq\bm{X}\Big(\bm{I}_{r}+\frac{\lambda}{p}\big(\bm{X}^{\top}\bm{X}\big)^{-1}\Big)^{1/2}\qquad\text{and}\qquad\bm{Y}^{\mathrm{d}}\triangleq\bm{Y}\Big(\bm{I}_{r}+\frac{\lambda}{p}\big(\bm{Y}^{\top}\bm{Y}\big)^{-1}\Big)^{1/2},\label{eq:defn-Xd-Yd}
\end{equation}
where we recall the definition of $\bm{X}$ and $\bm{Y}$ in Table~\ref{tab:Notation-unified}.  
To develop some intuition regarding why this is called a de-shrunken
estimator, let us look at a simple scenario where $\bm{U}\bm{\Sigma}\bm{V}^{\top}$
is the SVD of $\bm{X}\bm{Y}^{\top}$ and $\bm{X}=\bm{U}\bm{\Sigma}^{1/2}$,
$\bm{Y}=\bm{V}\bm{\Sigma}^{1/2}$. It is then self-evident that
\[
\bm{X}^{\mathsf{d}}=\bm{U}\bm{\Sigma}^{1/2}\Big(\bm{I}_{r}+\frac{\lambda}{p}\bm{\Sigma}^{-1}\Big)^{1/2}=\bm{U}\Big(\bm{\Sigma}+\frac{\lambda}{p}\bm{I}_{r}\Big)^{1/2}\qquad\text{and}\qquad\bm{Y}^{\mathsf{d}}=\bm{V}\Big(\bm{\Sigma}+\frac{\lambda}{p}\bm{I}_{r}\Big)^{1/2}.
\]
In words, $\bm{X}^{\mathsf{d}}$ and $\bm{Y}^{\mathsf{d}}$ are obtained
by de-shrinking the spectrum of $\bm{X}$ and $\bm{Y}$ properly.

As we shall formalize in Section~\ref{subsec:Approximate-equivalence},
the de-shrunken estimator~(\ref{eq:defn-Xd-Yd}) for the low-rank
factors is nearly equivalent to the de-biased estimator~(\ref{eq:defn-Zd-cvx})
for the whole matrix, in the sense that
\begin{equation}
\bm{M}^{\mathsf{d}}\approx\bm{X}^{\mathsf{d}}\bm{Y}^{\mathsf{d}\top}.\label{eq:equivalence-Zd-Xd-Yd}
\end{equation}
Therefore, $\bm{M}^{\mathrm{d}}$ can  be viewed as some sort
of de-shrunken estimator as well.

\subsection{Main results: distributional guarantees \label{subsec:Main-results}}

The proposed estimators  admit tractable distributional
characterizations in the large-$n$ regime, which facilitates the
construction of confidence regions for many quantities of interest.
In particular, this paper centers around two types of inferential
problems:
\begin{enumerate}
\item \emph{Each entry of the matrix} $\bm{M}^{\star}$: the entry can be
either missing (i.e.~predicting an unseen entry) or observed (i.e.~de-noising
an observed entry). For example, in the Netflix challenge, one would
like to infer a user's preference about any movie, given partially
revealed ratings~\cite{ExactMC09}. Mathematically, this seeks to
determine the distribution of
\begin{equation}
M_{ij}^{\mathrm{d}}-M_{ij}^{\star},\qquad\text{for all }1\leq i,j\leq n.\label{eq:entry}
\end{equation}
\item \emph{The low-rank factors} $\bm{X}^{\star},\bm{Y}^{\star}\in\mathbb{R}^{n\times r}$:
the low-rank factors often reveal critical information about the applications
of interest (e.g.~community memberships of each individual in the
community detection problem~\cite{abbe2017entrywise}, or angles between
each object and a global reference point in the angular synchronization
problem~\cite{singer2011angular}). Recognizing the global rotational
ambiguity issue,\footnote{For any $r\times r$ rotation matrix $\bm{H}$, we cannot distinguish
$(\bm{X}^{\star},\bm{Y}^{\star})$ from $(\bm{X}^{\star}\bm{H},\bm{Y}^{\star}\bm{H})$,
if only pairwise measurements are available. } we aim to pin down the distributions of $\bm{X}^{\mathrm{d}}$ and
$\bm{Y}^{\mathrm{d}}$ up to global rotational ambiguity. More precisely,
we intend to characterize the distributions of
\begin{equation}
\bm{X}^{\mathsf{d}}\bm{H}^{\mathsf{d}}-\bm{X}^{\star}\qquad\text{and}\qquad\bm{Y}^{\mathsf{d}}\bm{H}^{\mathsf{d}}-\bm{Y}^{\star}\label{eq:dist-low-rank-factors}
\end{equation}
for the global rotation matrix $\bm{H}^{\mathsf{d}}\in\mathbb{R}^{r\times r}$
that best ``aligns'' $(\bm{X}^{\mathsf{d}},\bm{Y}^{\mathsf{d}})$
and $(\bm{X}^{\star},\bm{Y}^{\star})$, i.e.
\begin{equation}
\bm{H}^{\mathsf{d}}\triangleq\arg\min_{\bm{R}\in\mathcal{O}^{r\times r}}\left\Vert \bm{X}^{\mathsf{d}}\bm{R}-\bm{X}^{\star}\right\Vert _{\mathrm{F}}^{2}+\left\Vert \bm{Y}^{\mathsf{d}}\bm{R}-\bm{Y}^{\star}\right\Vert _{\mathrm{F}}^{2}.
	\label{eq:defn-H-d}
\end{equation}
Here and below, $\mathcal{O}^{r\times r}$ denotes the set of orthonormal matrices
in $\mathbb{R}^{r\times r}$.
\end{enumerate}
Clearly, the above two inferential problems are tightly related: an
accurate distributional characterization for the low-rank factors
(\ref{eq:dist-low-rank-factors}) often results in a distributional
guarantee for the entries~(\ref{eq:entry}). As such, we shall begin
by presenting our distributional characterizations of the low-rank
factors. Here and throughout, $\bm{e}_{i}$ represents the $i$th
standard basis vector in $\mathbb{R}^{n}$.

\begin{theorem}[\textsf{{Distributional guarantees w.r.t.~low-rank factors}}]
\label{thm:low-rank-factor-master-bound-simple}
Suppose that the sample size and the noise obey
\begin{equation}
np\gtrsim\kappa^{8}\mu^{3}r^{2}\log^{3}n\qquad\text{and}\qquad\sigma / \sigma_{\min} \lesssim \sqrt{p/ (\kappa^{8}\mu n\log^2 n)} .\label{eq:requirement-low-rank}
\end{equation}
%
Then one has the following
decomposition 
\begin{subequations}\label{subeq:final-decomposition-low-rank}
\begin{align}
\bm{X}^{\mathsf{d}}\bm{H}^{\mathsf{d}}-\bm{X}^{\star} & =\bm{Z}_{\bm{X}}+\bm{\Psi}_{\bm{X}},\label{eq:final-decomposition-low-rank-X}\\
\bm{Y}^{\mathsf{d}}\bm{H}^{\mathsf{d}}-\bm{Y}^{\star} & =\bm{Z}_{\bm{Y}}+\bm{\Psi}_{\bm{Y}}.\label{eq:final-decomposition-low-rank-Y}
\end{align}
\end{subequations}
with $(\bm{X}^{\star},\bm{Y}^{\star})$ defined in \eqref{eq:defn-Xstar-Ystar}, $(\bm{X}^{\mathsf{d}}, \bm{Y}^{\mathsf{d}})$ defined in Table~\ref{tab:Notation-unified}, and $\bm{H}^{\mathrm{d}}$ defined in \eqref{eq:defn-H-d}. 	
Here, the rows of $\bm{Z}_{\bm{X}}\in\mathbb{R}^{n\times r}$
(resp.~$\bm{Z}_{\bm{Y}}\in\mathbb{R}^{n\times r}$) are 
independent and obey \begin{subequations}\label{subeq:normality-low-rank}
\begin{align}
\bm{Z}_{\bm{X}}^{\top}\bm{e}_{j}\overset{\mathrm{i.i.d.}}{\sim}\mathcal{N}\Big(\bm{0},\frac{\sigma^{2}}{p}\left(\bm{\Sigma}^{\star}\right)^{-1}\Big),\qquad & \mathrm{for}\quad1\leq j\leq n;\label{eq:normality-X}\\
\bm{Z}_{\bm{Y}}^{\top}\bm{e}_{j}\overset{\mathrm{i.i.d.}}{\sim}\mathcal{N}\Big(\bm{0},\frac{\sigma^{2}}{p}\left(\bm{\Sigma}^{\star}\right)^{-1}\Big),\qquad & \mathrm{for}\quad1\leq j\leq n.\label{eq:normality-Y}
\end{align}
\end{subequations}In addition, the residual matrices $\bm{\Psi}_{\bm{X}},\bm{\Psi}_{\bm{Y}}\in\mathbb{R}^{n\times r}$
satisfy, with probability at least $1-O(n^{-3})$, that
\begin{align}
	\max\big\{\left\Vert \bm{\Psi}_{\bm{X}}\right\Vert _{2,\infty},\left\Vert \bm{\Psi}_{\bm{Y}}\right\Vert _{2,\infty}\big\} = o\left(\frac{\sigma\sqrt{r}}{\sqrt{p\sigma_{\max}}}\right).\label{eq:thm-low-rank-residual-size-simple}
\end{align}
\end{theorem}

\begin{remark} A more complete version can be found in Theorem \ref{thm:low-rank-factor-master-bound}. \end{remark}

\begin{remark}\label{remark:pairwise-independence}Another interesting
feature --- which we shall make precise in the proof of this theorem --- is
that: for any given $1\leq i,j\leq n$, the two random vectors $\bm{Z}_{\bm{X}}^{\top}\bm{e}_{i}$
and $\bm{Z}_{\bm{Y}}^{\top}\bm{e}_{j}$ are nearly statistically independent.
This is crucial for deriving inferential guarantees for the entries
of the matrix. \end{remark}

Theorem~\ref{thm:low-rank-factor-master-bound-simple} is a non-asymptotic result. 
In words, Theorem~\ref{thm:low-rank-factor-master-bound-simple} decomposes
the estimation error $\bm{X}^{\mathsf{d}}\bm{H}^{\mathsf{d}}-\bm{X}^{\star}$
(resp.~$\bm{Y}^{\mathsf{d}}\bm{H}^{\mathsf{d}}-\bm{Y}^{\star}$)
into a Gaussian component~$\bm{Z}_{\bm{X}}$ (resp.~$\bm{Z}_{\bm{Y}}$) and a residual term~$\bm{\Psi}_{\bm{X}}$ (resp.~$\bm{\Psi}_{\bm{Y}}$).
If the sample size is sufficiently large  and the noise size is sufficiently small, 
then the residual terms are  much smaller in size compared to $\bm{Z}_{\bm{X}}$  and $\bm{Z}_{\bm{Y}}$. 
To see this, it is helpful to leverage the Gaussianity~(\ref{eq:normality-X})
and compute that: for each $1\leq j\leq n$, the $j$th row of $\bm{Z}_{\bm{X}}$ obeys
\[
\mathbb{E}\left[\bigl\Vert\bm{Z}_{\bm{X}}^{\top}\bm{e}_{j}\bigr\Vert_{2}^{2}\right]=\mathsf{Tr}\Big(\frac{\sigma^{2}}{p}\left(\bm{\Sigma}^{\star}\right)^{-1}\Big)\geq\frac{\sigma^{2}r}{p\sigma_{\max}};
\]
in other words, the typical size of the $j$th row of  $\bm{Z}_{\bm{X}}$
is no smaller than the order of $\sigma\sqrt{r/(p\sigma_{\max})}$.
In comparison, the size of each row of  $\bm{\Psi}_{\bm{X}}$
(see~(\ref{eq:thm-low-rank-residual-size-simple})) is much smaller than $\sigma\sqrt{r/(p\sigma_{\max})}$ (and hence smaller than the size of the corresponding row of $\bm{Z}_{\bm{X}}$) with
high probability, provided that \eqref{eq:requirement-low-rank} is satisfied.

Equipped with the above master decompositions of the low-rank factors
and Remark~\ref{remark:pairwise-independence}, we are ready to present
a similar decomposition for the entry $M_{ij}^{\mathsf{d}}-M_{ij}^{\star}$.

\begin{theorem}[\textsf{{Distributional guarantees w.r.t.~matrix entries}}]
\label{thm:entries-master-decomposition}
For each $1\leq i,j\leq n$, define the variance $v_{ij}^{\star}$ as
\begin{equation}
	v_{ij}^{\star}\triangleq \frac{\sigma^{2}}{p} \left(\left\Vert \bm{U}_{i,\cdot}^{\star}\right\Vert _{2}^{2}+\left\Vert \bm{V}_{j,\cdot}^{\star}\right\Vert _{2}^{2}\right),
	\label{eq:variance-entry}
\end{equation}
where $\bm{U}_{i,\cdot}^{\star}$ (resp.~$\bm{V}_{j,\cdot}^{\star}$)
denotes the $i$th (resp.~$j$th) row of $\bm{U}^{\star}$ (resp.~$\bm{V}^{\star}$). Suppose that
\begin{subequations}\label{subeq:entry-condition}
\begin{align}
np\gtrsim \kappa^{8}\mu^{3}r^{3}\log^{3}n, & \qquad\sigma\sqrt{(\kappa^{8}\mu rn\log^2 n)/p} \lesssim  \sigma_{\min}\qquad\text{and}\label{eq:requirement-entry}\\
\left\Vert \bm{U}_{i,\cdot}^{\star}\right\Vert _{2}+\left\Vert \bm{V}_{j,\cdot}^{\star}\right\Vert _{2} & \gtrsim \sqrt{\frac{r}{n}}\frac{\sigma}{\sigma_{\min}}\sqrt{\frac{\kappa^{6}\mu^{2}rn\log^{3}n}{p}}.\label{eq:entry-inference-lower-bound}
\end{align}
\end{subequations} 
Then the matrix $\bm{M}^{\mathrm{d}}$ defined in Table~\ref{tab:Notation-unified} satisfies
\begin{equation}
M_{ij}^{\mathsf{{d}}}-M_{ij}^{\star}=g_{ij}+\Delta_{ij},\label{eq:thm-entry-decomposition-simple}
\end{equation}
where $g_{ij}\sim\mathcal{N}(0,v_{ij}^{\star})$ and the residual
	obeys $|\Delta_{ij}| =o( \sqrt{v_{ij}^{\star}} )$ with 
	 probability exceeding $1-O(n^{-3})$. 
\end{theorem}

\begin{remark}[The symmetric case] 
In the symmetric case where the
noise $\bm{E}$, the truth $\bm{M}^{\star}$, and the sampling pattern
are all symmetric (i.e.~$\mathcal{P}_{\Omega}(\bm{E})=\big(\mathcal{P}_{\Omega}(\bm{E})\big)^{\top}$
and $\bm{M}^{\star}=\bm{M}^{\star\top}$), the variance $v_{ii}^{\star}$
(cf.~(\ref{eq:variance-entry})) for the diagonal entries has a different
formula; more specifically, it is straightforward to extend our theory
to show that
\[
v_{ii}^{\star}=\frac{4\sigma^{2}}{p}\left\Vert \bm{U}_{i,\cdot}^{\star}\right\Vert _{2}^{2}=\frac{2\sigma^{2}}{p}\left(\left\Vert \bm{U}_{i,\cdot}^{\star}\right\Vert _{2}^{2}+\left\Vert \bm{V}_{i,\cdot}^{\star}\right\Vert _{2}^{2}\right)\qquad\text{for the symmetric case}.
\]
This additional multiplicative factor of 2 arises since $\bm{Z}_{\bm{X}}^{\top}\bm{e}_{i}$
and $\bm{Z}_{\bm{Y}}^{\top}\bm{e}_{i}$ are identical (and hence not
independent) in this symmetric case. The variance formula for any
$v_{ij}^{\star}$ ($i\neq j$) remains unchanged. 
\end{remark}

Several remarks are in order. To begin with, we develop some intuition
regarding where the variance $v_{ij}^{\star}$ comes from.
By virtue of Theorem~\ref{thm:low-rank-factor-master-bound-simple}, one has the following Gaussian approximation 
\[
\bm{X}^{\mathrm{d}}\bm{H}^{\mathrm{d}}-\bm{X}^{\star}\approx\bm{Z}_{\bm{X}}\qquad\text{and}\qquad\bm{Y}^{\mathrm{d}}\bm{H}^{\mathrm{d}}-\bm{Y}^{\star}\approx\bm{Z}_{\bm{Y}} .
\]
Assuming that the first-order expansion is reasonably tight, one has
\begin{align}
M_{ij}^{\mathrm{d}}-M_{ij}^{\star} & =\left[\bm{X}^{\mathrm{d}}\bm{H}^{\mathrm{d}}\big(\bm{Y}^{\mathrm{d}}\bm{H}^{\mathrm{d}}\big)^{\top}-\bm{X}^{\star}\bm{Y}^{\star\top}\right]_{ij}\approx\bm{e}_{i}^{\top}\big(\bm{X}^{\mathrm{d}}\bm{H}^{\mathrm{d}}-\bm{X}^{\star}\big)\bm{Y}^{\star\top}\bm{e}_{j}+\bm{e}_{i}^{\top}\bm{X}^{\star}\big(\bm{Y}^{\mathrm{d}}\bm{H}^{\mathrm{d}}-\bm{Y}^{\star}\big)^{\top}\bm{e}_{j} \nonumber\\
	& \approx\bm{e}_{i}^{\top}\bm{Z}_{\bm{X}}\bm{Y}^{\star\top}\bm{e}_{j}+\bm{e}_{i}^{\top}\bm{X}^{\star}\bm{Z}_{\bm{Y}}^{\top}\bm{e}_{j}. \label{eq:Mij-leading-term}
\end{align}
According to Remark~\ref{remark:pairwise-independence}, $\bm{Z}_{\bm{X}}^{\top}\bm{e}_{i}$
and $\bm{Z}_{\bm{Y}}^{\top}\bm{e}_{j}$ are nearly independent. It
is thus straightforward to compute the variance of \eqref{eq:Mij-leading-term} as
\begin{align*}
\mathsf{Var}\left(M_{ij}^{\mathrm{d}}-M_{ij}^{\star}\right) & \overset{(\text{i})}{\approx}\mathsf{Var}\left(\bm{e}_{i}^{\top}\bm{Z}_{\bm{X}}\bm{Y}^{\star\top}\bm{e}_{j}\right)+\mathsf{Var}\left(\bm{e}_{i}^{\top}\bm{X}^{\star}\bm{Z}_{\bm{Y}}^{\top}\bm{e}_{j}\right) \\
 &\overset{(\text{ii})}{=}\frac{\sigma^{2}}{p}\left\{ \bm{e}_{j}^{\top}\bm{Y}^{\star}\left(\bm{\Sigma}^{\star}\right)^{-1}\bm{Y}^{\star\top}\bm{e}_{j}+\bm{e}_{i}^{\top}\bm{X}^{\star}\left(\bm{\Sigma}^{\star}\right)^{-1}\bm{X}^{\star\top}\bm{e}_{i}\right\} \overset{(\text{iii})}{=}\frac{\sigma^{2}}{p}\left(\left\Vert \bm{U}_{i,\cdot}^{\star}\right\Vert _{2}^{2}+\left\Vert \bm{V}_{j,\cdot}^{\star}\right\Vert _{2}^{2}\right)  = v_{ij}^{\star}.
\end{align*}
Here, (i) relies on \eqref{eq:Mij-leading-term} and the near independence between $\bm{Z}_{\bm{X}}^{\top}\bm{e}_{i}$
and $\bm{Z}_{\bm{Y}}^{\top}\bm{e}_{j}$; (ii) uses the variance formula
in Theorem~\ref{thm:low-rank-factor-master-bound-simple}; (iii)
arises from the definitions of $\bm{X}^{\star}$ and $\bm{Y}^{\star}$
(cf.~(\ref{eq:defn-Xstar-Ystar})). This computation explains (heuristically)
the variance formula $v_{ij}^{\star}$.



Given that Theorem~\ref{thm:entries-master-decomposition}
reveals the tightness of Gaussian approximation
under conditions \eqref{subeq:entry-condition}, it in turn allows  
 us to construct nearly
accurate confidence intervals for each matrix entry $M_{ij}^{\star}$.
This is formally summarized in the following corollary, the proof
of which is deferred to Appendix~\ref{sec:Proof-of-Corollary-confidence-interval}.
Here and throughout, we use $[a\pm b]$ to denote the interval $[a-b,a+b]$.

\begin{corollary}[\textsf{Confidence intervals for the entries $\{M^\star_{ij}\}$}]
\label{coro:confidence-interval}
Let $\bm{X}^{\mathrm{d}}$, $\bm{Y}^{\mathrm{d}}$ and $\bm{M}^{\mathrm{d}}$ be as defined in Table~\ref{tab:Notation-unified}. 
For any given $1\leq i,j\leq n$, suppose that \eqref{eq:requirement-entry} holds and that 
\begin{equation}
	\left\Vert \bm{U}_{i,\cdot}^{\star}\right\Vert _{2}+\left\Vert \bm{V}_{j,\cdot}^{\star}\right\Vert _{2}  \gtrsim \sqrt{\frac{r}{n}}\frac{\sigma}{\sigma_{\min}}\sqrt{\frac{\kappa^{10}\mu^{2}rn\log^{3}n}{p}}.
\end{equation}
Denote by
$\Phi(t)$ the CDF of a standard Gaussian random variable and by $\Phi^{-1}(\cdot)$
its inverse function. Let
\begin{equation}
	v_{ij}\triangleq  \frac{\sigma^{2}}{p} \left(\bm{X}_{i,\cdot}^{\mathsf{d}}\left(\bm{X}^{\mathsf{d}\top}\bm{X}^{\mathsf{d}}\right)^{-1}(\bm{X}_{i,\cdot}^{\mathsf{d}})^{\top}+\bm{Y}_{j,\cdot}^{\mathsf{d}}\left(\bm{Y}^{\mathsf{d}\top}\bm{Y}^{\mathsf{d}}\right)^{-1}(\bm{Y}_{j,\cdot}^{\mathsf{d}})^{\top}\right)\label{eq:empirical-variance-entry}
\end{equation}
be the empirical estimate of the theoretical variance $v_{ij}^{\star}$. Then one has
%
\begin{align*}
	& \sup_{0<\alpha<1}\Big|\mathbb{P}\Big\{ M_{ij}^{\star}\in\big[M_{ij}^{\mathsf{{d}}}\pm\Phi^{-1}\left(1-\alpha/2\right) \sqrt{v_{ij}}\big]\Big\} -(1-\alpha)\Big| = o(1). 
\end{align*}
\end{corollary}
In words, Corollary~\ref{coro:confidence-interval} tells us  that for any fixed significance level $0<\alpha<1$, the interval 
\begin{equation}\label{eq:entry-confidence-interval}
\big[M_{ij}^{\mathsf{d}}\pm\Phi^{-1}(1-\alpha/2)\sqrt{v_{ij}} \big]
\end{equation}
is a nearly accurate two-sided $(1-\alpha)$ confidence interval of $M_{ij}^{\star}$.

In addition, we remark that when $\|\bm{U}_{i,\cdot}^{\star}\|_{2}=\|\bm{V}_{j,\cdot}^{\star}\|_{2}=0$
(and hence $V_{ij}^{\star}=0$), the above Gaussian approximation
is completely off. In this case, one can still leverage Theorem~\ref{thm:low-rank-factor-master-bound-simple}
to show that
\begin{equation}
M_{ij}^{\mathsf{{d}}}-M_{ij}^{\star}=M_{ij}^{\mathsf{{d}}}\approx\bm{u}^{\top}\bm{v},\label{eq:chi-2-approx}
\end{equation}
where $\bm{u},\bm{v}\in\mathbb{R}^{r}$ are independent and identically
distributed according to $\mathcal{N}(\bm{0},\sigma^{2}(\bm{\Sigma}^{\star})^{-1}/p)$.
However, it is nontrivial to determine whether $\|\bm{U}_{i,\cdot}^{\star}\|_{2}+\|\bm{V}_{j,\cdot}^{\star}\|_{2}$
is vanishingly small or not based on the observed data, which makes
it challenging to conduct efficient inference for entries with small
(but \emph{a priori} unknown) $\|\bm{U}_{i,\cdot}^{\star}\|_{2}+\|\bm{V}_{j,\cdot}^{\star}\|_{2}$.

Last but not least, the careful readers might wonder how to interpret our conditions on the sample complexity and the signal-to-noise ratio. Take the case with $r, \mu, \kappa =O(1)$ for example: our conditions read
\begin{equation}
	n^2 p \gtrsim n \log^3 n; \qquad \sigma \sqrt{(n\log^2 n) / p} \lesssim \sigma_{\min}. 
\end{equation}
The first condition matches the minimal sample complexity limit (up to some logarithmic factor),  while the second one coincides with the regime (up to log factor) in which popular algorithms (like spectral methods or nonconvex algorithms) work better than a random guess~\cite{Se2010Noisy,chen2015fast,ma2017implicit}. The  take-away message is this: once we are able to compute a reasonable estimate in an overall $\ell_2$ sense,  then we can reinforce it to conduct entrywise inference in a
statistically efficient fashion.  
The discussion of the dependency on $r$ and $\kappa$ is deferred to Section~\ref{sec:discussion}.

\subsection{Lower bounds and optimality for inference}

It is natural to ask how well our inferential procedures perform compared
to other algorithms. Encouragingly, the de-biased estimator is optimal
in some sense; for instance, it nearly attains the minimum covariance
among all unbiased estimators. To formalize this claim, we shall
\begin{enumerate}
\item Quantify the performance of two  ideal estimators with the assistance of an oracle;
\item Demonstrate that the performance of our de-biased estimators is arbitrarily
close to that of the ideal estimators.
\end{enumerate}
In what follows, we denote by $\bm{X}_{i,\cdot}^{\star}$ (resp.~$\bm{Y}_{i,\cdot}^{\star}$)
the $i$th row of $\bm{X}^{\star}$ (resp.~$\bm{Y}^{\star}$).

\paragraph{An ideal estimator for $\bm{X}_{i,\cdot}^{\star}$ ($1\leq i\leq n$).}
Suppose that there is an oracle informing us of $\bm{Y}^{\star}$,
and that we observe the same set of data as in~(\ref{eq:observation-model}).
Under such an idealistic setting and for any given $1\leq i\leq n$, the following least-squares estimator
achieves the minimum covariance among all \emph{unbiased }estimators\emph{
}for the $i$th row $\bm{X}_{i,\cdot}^{\star}$ of $\bm{X}^{\star}$ (see e.g.~\cite[Theorem 3.7]{shao2003mathematical})  
\begin{equation}
\bm{X}_{i,\cdot}^{\mathsf{ideal}}\,\triangleq\,\arg\min_{\bm{u}\in\mathbb{R}^{1\times r}}\sum_{k:(i,k)\in\Omega}\left[M_{ik}-\bm{u}\big(\bm{Y}_{k,\cdot}^{\star}\big)^{\top}\right]^{2}.\label{eq:ideal-estimator-Xi}
\end{equation}
In other words, for any unbiased estimator $\bm{u}$ of $\bm{X}_{i,\cdot}^{\star}$
(conditional on $\Omega$), one has
\begin{equation}
\mathsf{Cov}\big(\bm{u}\,\big|\,\Omega\big)\,\succeq\,\mathsf{Cov}\big(\bm{X}_{i,\cdot}^{\mathsf{ideal}}\,\big|\,\Omega\big)=:\mathsf{CRLB}\big(\bm{X}_{i,\cdot}^{\star}\mid\Omega\big),\label{eq:CRLB-Xi}
\end{equation}
where $\mathsf{Cov}\big(\bm{X}_{i,\cdot}^{\mathsf{ideal}}\,\big|\,\Omega\big)$
is precisely the Cram\'er-Rao lower bound (conditional on $\Omega$)
under this ideal setting. As it turns out, with high probability,
this lower bound concentrates around $\sigma^{2}(\bm{\Sigma}^{\star})^{-1}/p$,
as stated in the following lemma. The proof is postponed to Appendix~\ref{sec:Proof-of-Lemma-optimal-low-rank}.

\begin{lemma}\label{lemma:optimal-low-rank-variance}Fix an arbitrarily
small constant $\varepsilon>0$. Suppose that $n^{2}p\geq C_{0}\varepsilon^{-2}\kappa^{4}\mu rn$
for some sufficiently large constant $C_{0}>0$ independent of $n$.
Then with probability at least $1-O(n^{-10})$, one has
\[
\mathsf{CRLB}\big(\bm{X}_{i,\cdot}^{\star}\mid\Omega\big)\succeq\left(1-\varepsilon\right)\frac{\sigma^{2}}{p}\left(\bm{\Sigma}^{\star}\right)^{-1}.
\]
\end{lemma}

Given that $\varepsilon$ can be an arbitrarily small constant, Lemma
\ref{lemma:optimal-low-rank-variance} uncovers that the covariance
of the de-shrunken estimator $\bm{X}_{i,\cdot}^{\mathsf{d}}$ (cf.~Theorem~\ref{thm:low-rank-factor-master-bound-simple})
matches that of the  ideal estimator $\bm{X}_{i,\cdot}^{\mathsf{ideal}}$,
thus achieving the Cram\'er-Rao lower bound with high probability.
The same conclusion applies to $\bm{Y}_{j,\cdot}^{\mathsf{d}}$ as
well.

\paragraph{An ideal estimator for $M_{ij}^{\star}$ ($1\leq i,j\leq n$).}
Suppose that there is another oracle  informing us of $\{\bm{X}_{k,\cdot}^{\star}\}_{k:k\neq i}$
and $\{\bm{Y}_{k,\cdot}^{\star}\}_{k:k\neq j}$; that is, everything
about $\bm{X}^{\star}$ except $\bm{X}_{i,\cdot}^{\star}$ and everything
about $\bm{Y}^{\star}$ except $\bm{Y}_{j,\cdot}^{\star}$. In addition,
we observe the same set of data as in~(\ref{eq:observation-model}),
except that we do not get to see $M_{ij}$.\footnote{The exclusion
of $M_{ij}$ is merely for ease of presentation. One can consider
the model where all $M_{ij}$ with $(i,j)\in\Omega$ are observed with a slightly more complicated argument.}
Under this idealistic model, the Cram\'er-Rao lower bound~\cite[Theorem 3.3]{shao2003mathematical}
for estimating $M_{ij}^{\star}=\bm{X}_{i,\cdot}^{\star}(\bm{Y}_{j,\cdot}^{\star})^{\top}$
can be computed as
\begin{align}
 & \mathsf{CRLB}\left(M_{ij}^{\star}\mid\Omega\right)\nonumber \\
 & \quad\triangleq\frac{\sigma^{2}}{p}\cdot\Big[\bm{Y}_{j,\cdot}^{\star}\Big(\frac{1}{p}\sum_{k:k\neq j,(i,k)\in\Omega}(\bm{Y}_{k,\cdot}^{\star})^{\top}\bm{Y}_{k,\cdot}^{\star}\Big)^{-1}(\bm{Y}_{j,\cdot}^{\star})^{\top}+\bm{X}_{i,\cdot}^{\star}\Big(\frac{1}{p}\sum_{k:k\neq i,(k,j)\in\Omega}(\bm{X}_{k,\cdot}^{\star})^{\top}\bm{X}_{k,\cdot}^{\star}\Big)^{-1}(\bm{X}_{i,\cdot}^{\star})^{\top}\Big].\label{eq:CRLB-Mij}
\end{align}
This means that any unbiased estimator of $M_{ij}^{\star}$ must have
variance no smaller than $\mathsf{CRLB}(M_{ij}^{\star}\mid\Omega)$.
This quantity admits a much simpler lower bound as follows, whose
proof can be found in Appendix~\ref{sec:Proof-of-Lemma-optimal-entry}.

\begin{lemma}\label{lemma:optimal-entry-variance}Fix an arbitrarily
small constant $\varepsilon>0$. Suppose that $n^{2}p\geq C_{0}\varepsilon^{-2}\kappa^{4}\mu rn\log n$
for some sufficiently large constant $C_{0}>0$ independent of $n$.
Then with probability at least $1-O(n^{-10})$,
\[
\mathsf{CRLB}\big(M_{ij}^{\star}\mid\Omega\big)\geq\left(1-\varepsilon\right)v_{ij}^{\star},
\]
where $v_{ij}^{\star}$ is defined in Theorem~\ref{thm:entries-master-decomposition}.
\end{lemma}

Similar to Lemma~\ref{lemma:optimal-low-rank-variance}, Lemma~\ref{lemma:optimal-entry-variance}
reveals that the variance of our de-biased estimator $M_{ij}^{\mathsf{d}}$
(cf.~Theorem~\ref{thm:entries-master-decomposition}) --- which
certainly does not have access to the side information provided by
the oracle --- is arbitrarily close to the  Cram\'er-Rao
lower bound aided by an oracle.

\bigskip

All in all, the above lower bounds demonstrate that the degrees of
uncertainty underlying our de-shrunken low-rank factors and de-biased
matrix are, in some sense, statistically minimal.

\subsection{Back to estimation: the de-biased estimator is optimal}

While the emphasis of the current paper is on inference,  we would nevertheless like to single out an important consequence that informs the estimation step. 
To be  specific, the decompositions and distributional guarantees derived in Theorem~\ref{thm:low-rank-factor-master-bound-simple} and Theorem~\ref{thm:entries-master-decomposition} allow us to track the estimation accuracy of $\bm{M}^{\mathsf{d}}$, as stated in the following theorem. The proof of this result is postponed to Appendix~\ref{sec:estimation-error}.

\begin{theorem}[\textsf{Estimation accuracy of $\bm{M}^{\mathsf{d}}$}]
\label{thm:estimation-error}
Let $\bm{M}^{\mathrm{d}}$ be the de-biased estimator as defined in Table~\ref{tab:Notation-unified}. Instate the conditions
in~(\ref{eq:requirement-entry}). Then with probability at least $1-O(n^{-3})$, one has 
\begin{align}
	\left\Vert \bm{M}^{\mathsf{d}}-\bm{M}^{\star}\right\Vert _{\mathrm{F}}^{2}=\frac{(2+o(1))nr\sigma^{2}}{p}.
	\label{eq:Md-estimation-error}
\end{align}
\end{theorem}
In stark contrast to prior statistical estimation guarantees (e.g.~\cite{CanPla10,Negahban2012restricted,MR2906869,chen2019noisy}), 
Theorem~\ref{thm:estimation-error} pins down the estimation error of the proposed de-biased estimator in a sharp manner (namely, even the pre-constant is fully determined).  Encouragingly, 
there is a sense in which the proposed de-biased estimator 
achieves the best possible statistical estimation accuracy, as revealed by the following result.  

\begin{theorem}[\textsf{An oracle lower bound on $\ell_2$ estimation errors}]
	Fix an arbitrarily small constant $\varepsilon>0$.  Suppose that $n^2 p \gtrsim \mu r n \log^2 n$, and that $r=o(n)$. Then with probability exceeding $1-O(n^{-10})$, any unbiased estimator $\widehat{\bm{M}}$ of $\bm{M}^{\star}$ obeys
	\begin{equation}
		\mathbb{E}\Big[ \big\| \widehat{\bm{M}} - \bm{M}^{\star} \big\|_{\mathrm{F}}^2 \mid \Omega \Big] \geq \frac{(1-\varepsilon)2 nr\sigma^2}{p}.  
	\end{equation}
\end{theorem}
\begin{proof}
	Intuitively, the term $2nr$ reflects approximately the underlying degrees of freedom in the true subspace $T^{\star}$ of interest (i.e.~the tangent space of the rank-$r$ matrices at the truth $\bm{M}^{\star}$), whereas the factor $1/p$ captures the effect due to sub-sampling.
	This result has already been established in~\cite[Section III.B]{CanPla10} (together with~\cite[Theorem 4.1]{ExactMC09}). We thus omit the proof for conciseness.  The key idea is to consider an oracle informing us of the true tangent space $T^{\star}$. 
\end{proof}

The implication of the above two theorems is remarkable: the de-biasing step  not merely facilitates uncertainty assessment,  but also proves crucial in minimizing the estimation errors. It achieves optimal statistical efficiency in terms of both the rate and the pre-constant. 
As far as we know, this is the first theory about a polynomial time algorithm that  matches the statistical limit in terms of the pre-constant. 
This intriguing finding is further corroborated by numerical experiments; see Section~\ref{sec:numerics} for details (in particular, Figure~\ref{fig:estimation-error}).

\subsection{Numerical experiments}
\label{sec:numerics}

\begin{table}[b]

\caption{Empirical coverage rates of $\bm{e}_{i}^{\top}\bm{X}^{\star}\bm{X}^{\star\top}\bm{e}_{j}$
for different $(r,p,\sigma)$'s over 200 Monte Carlo trials.\label{tab:cov_low_rank}}
\vspace{0.5em}
\centering
\begin{tabular}{|c|c|c|}
\hline 
$\vphantom{\sum_{i=1}^{K}}(r,p,\sigma)$ & $\mathsf{Mean}(\widehat{\mathsf{Cov}}_{\mathsf{L}})$ & $\mathsf{Std}(\widehat{\mathsf{Cov}}_{\mathsf{L}})$\tabularnewline
\hline 
$\vphantom{\frac{1}{2}}(2,0.2,10^{-6})$ & 0.9387 & 0.0197\tabularnewline
\hline 
$\vphantom{\frac{1}{2}}(2,0.2,10^{-3})$ & 0.9400 & 0.0193\tabularnewline
\hline 
$\vphantom{\frac{1}{2}}(2,0.4,10^{-6})$ & 0.9459 & 0.0161\tabularnewline
\hline 
$\vphantom{\frac{1}{2}}(2,0.4,10^{-3})$ & 0.9460 & 0.0162\tabularnewline
\hline 
$\vphantom{\frac{1}{2}}(5,0.2,10^{-6})$ & 0.9227 & 0.0244\tabularnewline
\hline 
$\vphantom{\frac{1}{2}}(5,0.2,10^{-3})$ & 0.9273 & 0.0226\tabularnewline
\hline 
$\vphantom{\frac{1}{2}}(5,0.4,10^{-6})$ & 0.9411 & 0.0173\tabularnewline
\hline 
$\vphantom{\frac{1}{2}}(5,0.4,10^{-3})$ & 0.9418 & 0.0171\tabularnewline
\hline 
\end{tabular}
\end{table}

We conduct numerical experiments on synthetic data to verify the distributional
characterizations provided in Theorem~\ref{thm:low-rank-factor-master-bound-simple}
and Theorem~\ref{thm:entries-master-decomposition}. Note that our
main results hold for the de-biased estimators built upon $\bm{Z}^{\mathsf{cvx}}$ and $\bm{X}^{\mathsf{ncvx}}\bm{Y}^{\mathsf{ncvx}\top}$.
As we will formalize shortly in Section~\ref{subsec:Approximate-equivalence},
these two de-biased estimators are extremely close to each other;
see also Figure~\ref{fig:equivalence} for experimental evidence.
Therefore, in order to save space, we use the de-biased estimator
built upon the convex estimate $\bm{Z}^{\mathsf{cvx}}$ throughout the experiments.

Fix the dimension $n=1000$ and the regularization parameter $\lambda=2.5\sigma\sqrt{np}$
throughout the experiments. We generate a rank-$r$
matrix $\bm{M}^{\star}=\bm{X}^{\star}\bm{Y}^{\star\top}$, where $\bm{X}^{\star},\bm{Y}^{\star}\in\mathbb{R}^{n\times r}$
are random orthonormal matrices and apply the proximal gradient method~\cite{parikh2014proximal} to solve the convex program
(\ref{eq:cvx}).

\begin{figure}[t]
\centering

\begin{tabular}{ccc}
\includegraphics[scale=0.3]{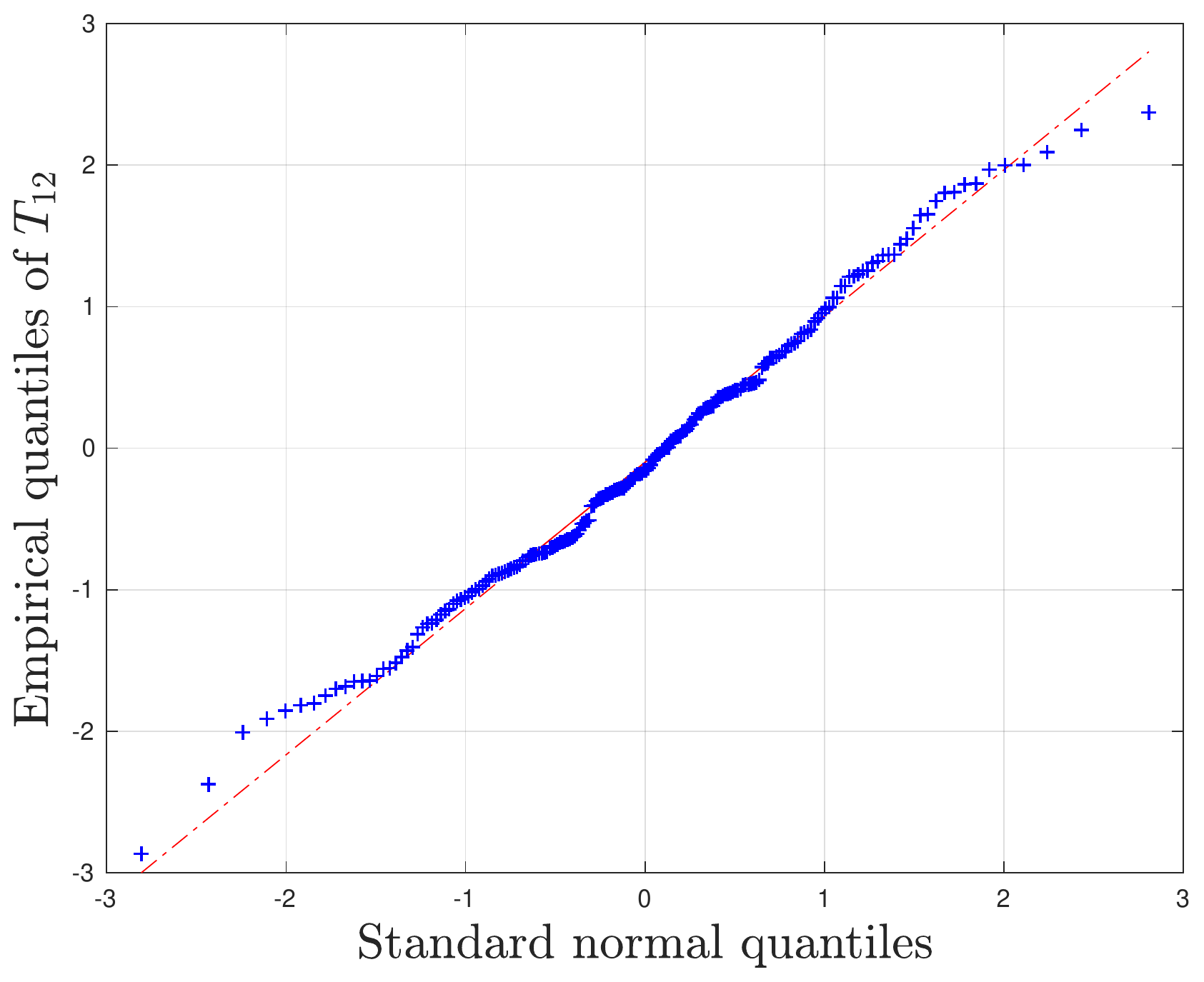} & \includegraphics[scale=0.3]{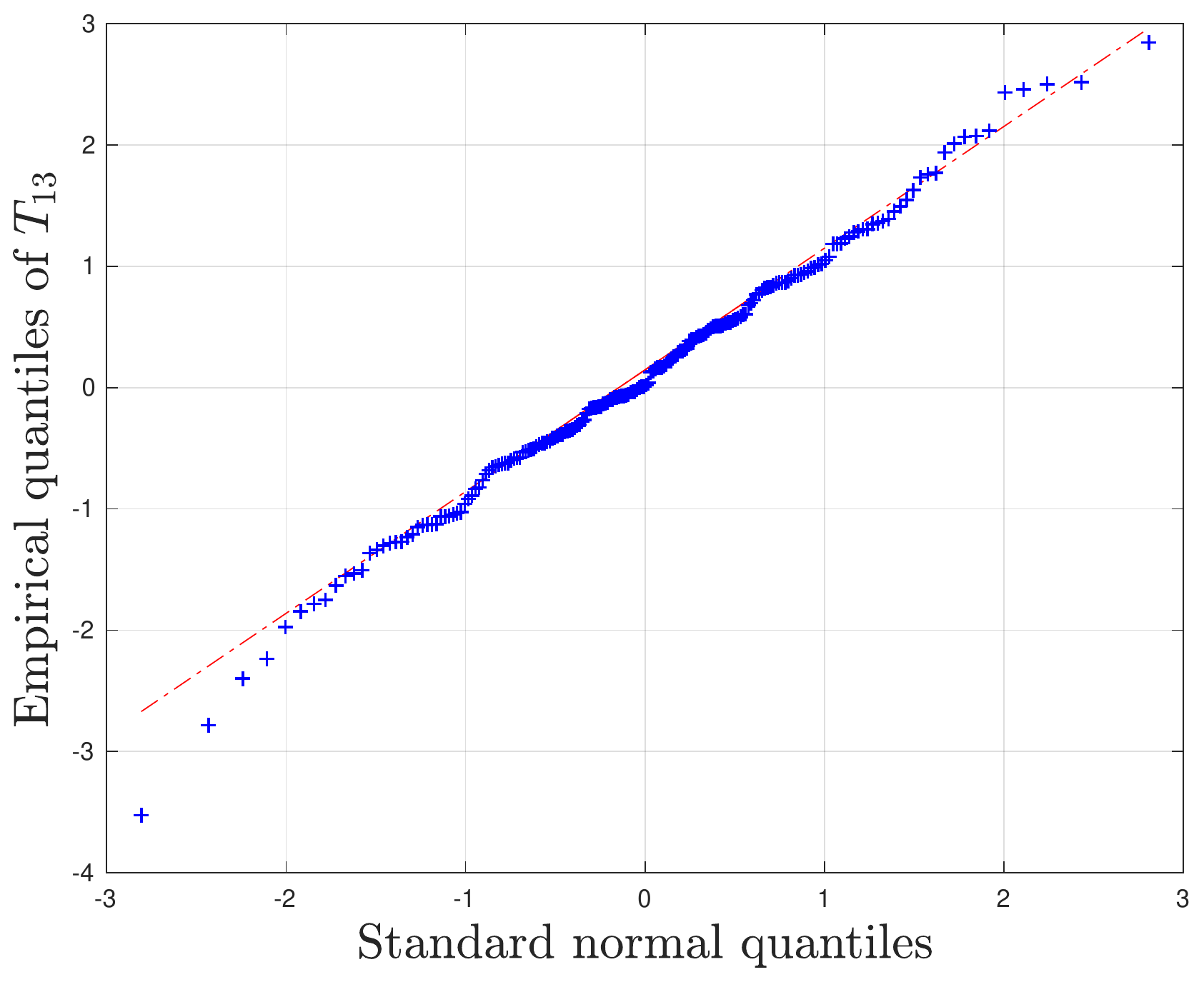} & \includegraphics[scale=0.3]{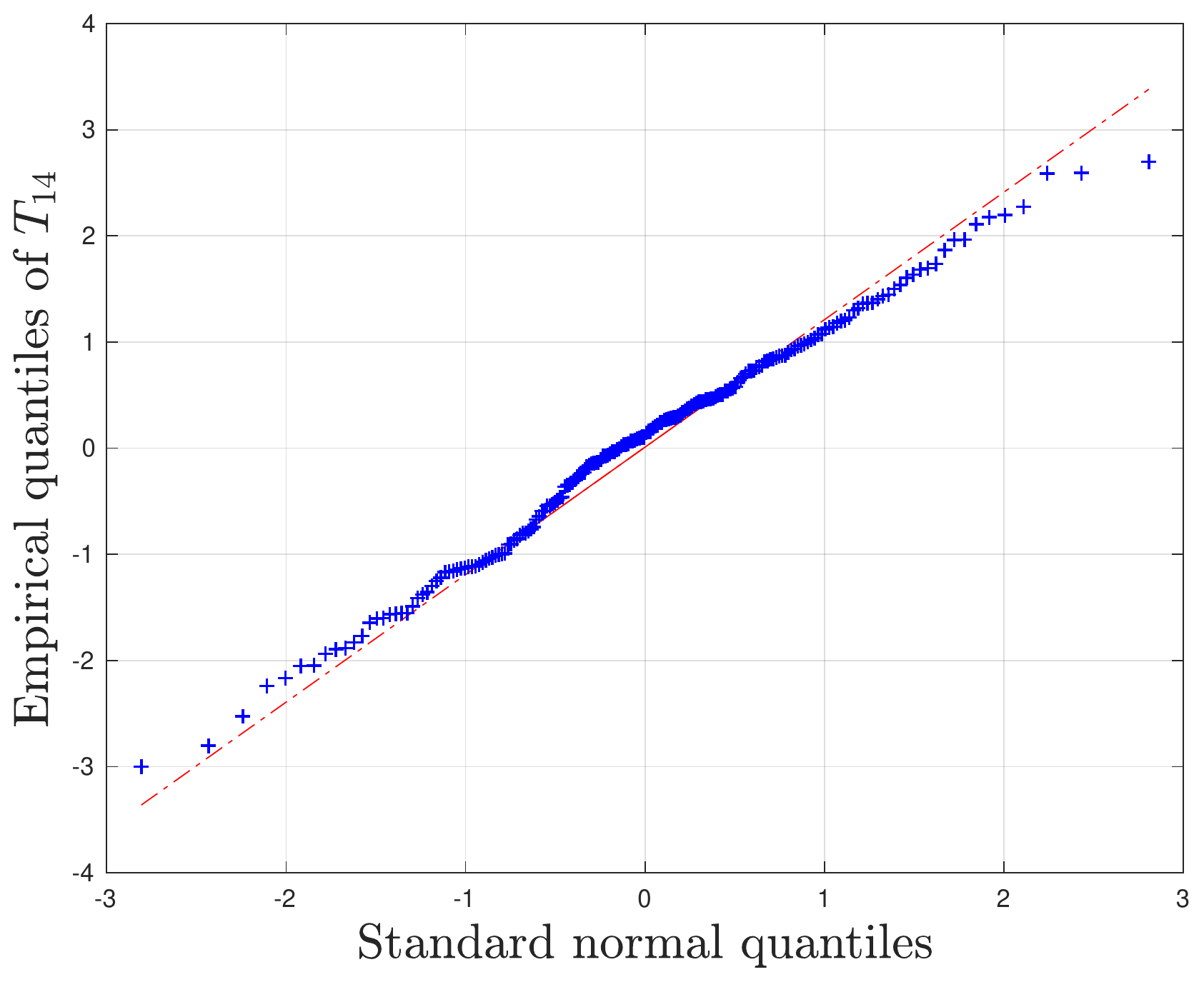}\tabularnewline
(a) & (b) & (c)\tabularnewline
\end{tabular}

\caption{Q-Q (quantile-quantile) plots of $T_{12}$, $T_{13}$ and $T_{14}$
vs.~the standard normal distribution in (a), (b) and (c), respectively. The results are reported over 200 independent trials for $r=5$, $p=0.4$ and $\sigma = 10^{-3}$.
\label{fig:low-rank}}
\end{figure}

We begin by checking the validity of Theorem~\ref{thm:low-rank-factor-master-bound-simple}.
Suppose that one is interested in estimating the inner product $\bm{e}_{i}^{\top}\bm{X}^{\star}\bm{X}^{\star\top}\bm{e}_{j}$
between $\bm{X}^{\star\top}\bm{e}_{i}$ and $\bm{X}^{\star\top}\bm{e}_{j}$
($i\neq j$). In the Netflix challenge, this might correspond to the
similarity between the $i$th user and the $j$th one. As a straightforward
consequence of Theorem~\ref{thm:low-rank-factor-master-bound-simple}, the
normalized estimation error
\begin{equation}\label{eq:defn-T-ij}
T_{ij}\triangleq\frac{1}{\sqrt{\rho_{ij}}}\left(\bm{e}_{i}^{\top}\bm{X}^{\mathsf{d}}\bm{X}^{\mathsf{d}\top}\bm{e}_{j}-\bm{e}_{i}^{\top}\bm{X}^{\star}\bm{X}^{\star\top}\bm{e}_{j}\right)
\end{equation}
is extremely close to a standard Gaussian random variable. Here,
similar to~(\ref{eq:empirical-variance-entry}), we let
\begin{equation}\label{eq:defn-rho-ij}
\rho_{ij}\triangleq\frac{\sigma^{2}}{p}\left\{ \bm{e}_{i}^{\top}\bm{X}^{\mathsf{d}}(\bm{X}^{\mathsf{d}\top}\bm{X}^{\mathsf{d}})^{-1}\bm{X}^{\mathsf{d}\top}\bm{e}_{i}+\bm{e}_{j}^{\top}\bm{X}^{\mathsf{d}}(\bm{X}^{\mathsf{d}\top}\bm{X}^{\mathsf{d}})^{-1}\bm{X}^{\mathsf{d}\top}\bm{e}_{j}\right\}
\end{equation}
be the empirical estimate of the theoretically predicted variance $\sigma^{2}(\|\bm{U}_{i,\cdot}^{\star}\|_{2}^{2}+\|\bm{U}_{j,\cdot}^{\star}\|_{2}^{2})/p$. As a result, a $95\%$ confidence interval of $\bm{e}_{i}^{\top}\bm{X}^{\star}\bm{X}^{\star\top}\bm{e}_{j}$ would be $[\bm{e}_{i}^{\top}\bm{X}^{\mathsf{d}}\bm{X}^{\mathsf{d}\top}\bm{e}_{j} \pm 1.96 \sqrt{\rho_{ij}}]$. For each $(i,j)$, we define $\widehat{\mathsf{Cov}}_{\mathsf{L},(i,j)}$
to be the empirical coverage rate of $\bm{e}_{i}^{\top}\bm{X}^{\star}\bm{X}^{\star\top}\bm{e}_{j}$
over $200$ Monte Carlo simulations. Correspondingly, denote by $\mathsf{Mean}(\widehat{\mathsf{Cov}}_{\mathsf{L}})$
(resp.~$\mathsf{Std}(\widehat{\mathsf{Cov}}_{\mathsf{L}})$) the
average (resp.~the standard deviation) of $\widehat{\mathsf{Cov}}_{\mathsf{L},(i,j)}$ over indices $1\leq i<j\leq n$.
Table~\ref{tab:cov_low_rank} collects the simulation results for
different values of $(r,p,\sigma)$. As can be seen, the  reported empirical
coverage rates are reasonably close to the nominal level $95\%$. In addition, Figure~\ref{fig:low-rank} depicts the Q-Q (quantile-quantile) plots
of $T_{12},T_{13}$ and $T_{14}$ vs.~the standard Gaussian random
variable over 200 Monte Carlo simulations for $r=5$, $p=0.4$ and $\sigma=10^{-3}$. It is clearly seen that
all of these are well approximated by a standard Gaussian random variable.

%
%
%

\begin{table}[b]

\caption{Empirical coverage rates of $M_{ij}^{\star}$ for different
$(r,p,\sigma)$'s over 200 Monte Carlo trials. \label{tab:cov_entry}}
\vspace{0.5em}

\centering
\begin{tabular}{|c|c|c|}
\hline 
$\vphantom{\sum_{i=1}^{K}}(r,p,\sigma)$ & $\mathsf{Mean}(\widehat{\mathsf{Cov}}_{\mathsf{E}})$ & $\mathsf{Std}(\widehat{\mathsf{Cov}}_{\mathsf{E}})$\tabularnewline
\hline 
$\vphantom{\frac{1}{2}}(2,0.2,10^{-6})$ & 0.9380 & 0.0200\tabularnewline
\hline 
$\vphantom{\frac{1}{2}}(2,0.2,10^{-3})$ & 0.9392 & 0.0196\tabularnewline
\hline 
$\vphantom{\frac{1}{2}}(2,0.4,10^{-6})$ & 0.9455 & 0.0164\tabularnewline
\hline 
$\vphantom{\frac{1}{2}}(2,0.4,10^{-3})$ & 0.9456 & 0.0164\tabularnewline
\hline 
$\vphantom{\frac{1}{2}}(5,0.2,10^{-6})$ & 0.9226 & 0.0247\tabularnewline
\hline 
$\vphantom{\frac{1}{2}}(5,0.2,10^{-3})$ & 0.9271 & 0.0228\tabularnewline
\hline 
$\vphantom{\frac{1}{2}}(5,0.4,10^{-6})$ & 0.9410 & 0.0173\tabularnewline
\hline 
$\vphantom{\frac{1}{2}}(5,0.4,10^{-3})$ & 0.9417 & 0.0172\tabularnewline
\hline 
\end{tabular}
\end{table}

\begin{figure}[t]
\centering

\begin{tabular}{ccc}
\includegraphics[scale=0.3]{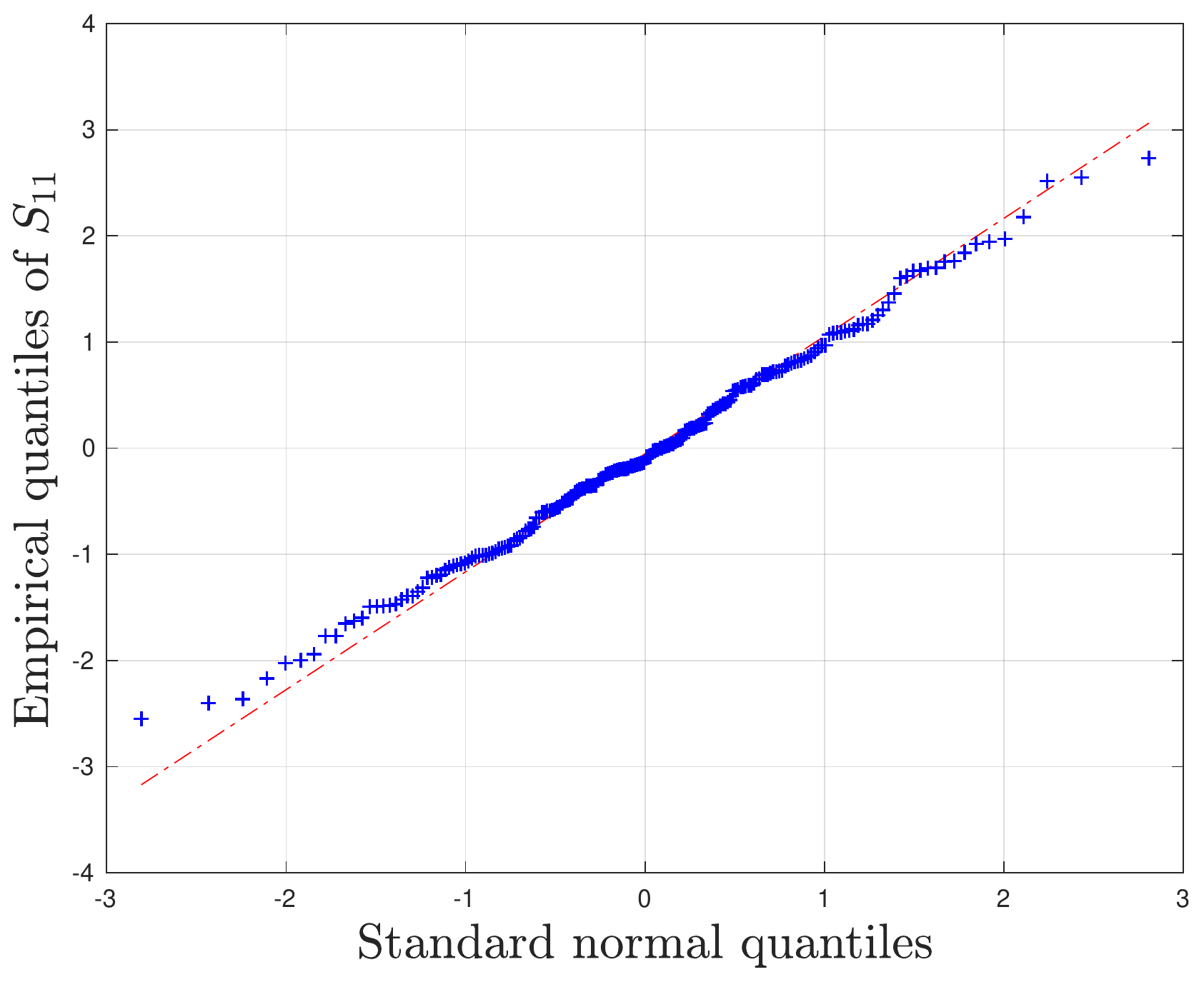} & \includegraphics[scale=0.3]{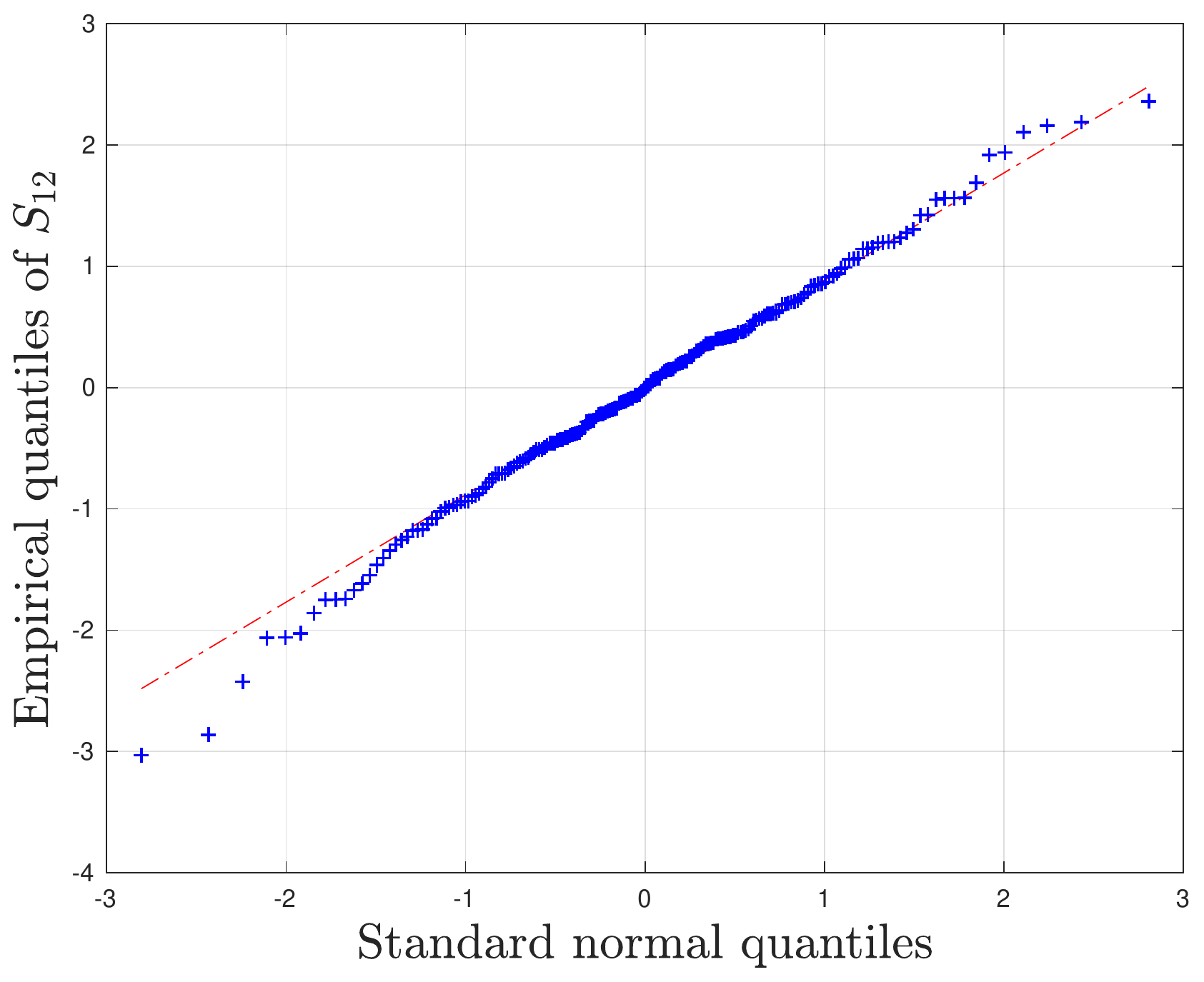} & \includegraphics[scale=0.3]{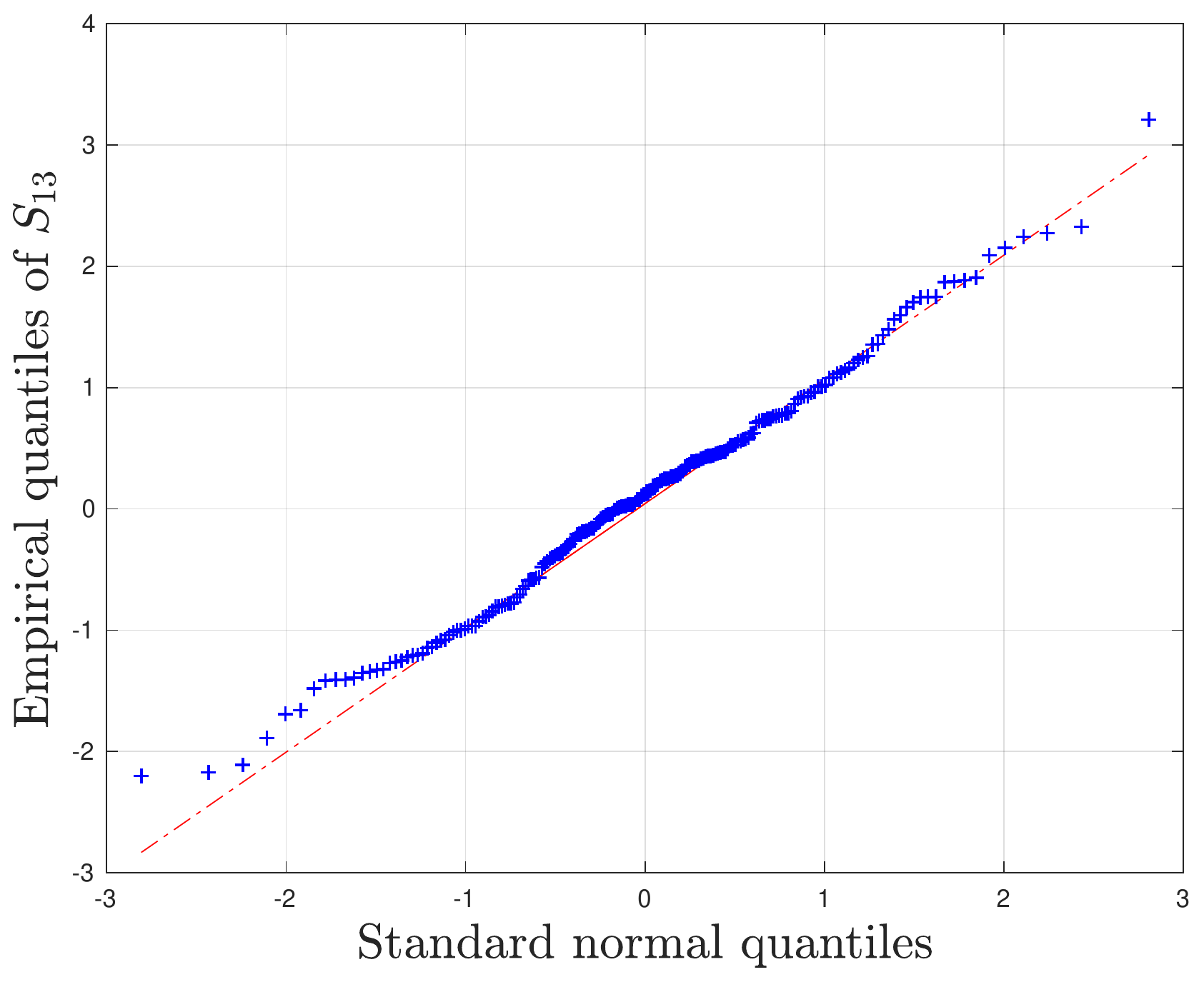}\tabularnewline
(a) & (b) & (c)\tabularnewline
\end{tabular}

\caption{Q-Q (quantile-quantile) plot of $S_{11}$, $S_{12}$ and $S_{13}$
vs.~the standard normal distribution in (a), (b) and (c) respectively. The results are reported over 200 independent trials for $r=5$, $p=0.4$ and $\sigma = 10^{-3}$.
\label{fig:entry}}
\end{figure}

Next, we turn to Theorem~\ref{thm:entries-master-decomposition}, namely
the distributional guarantee for the entries of the matrix. Denote
\begin{equation}\label{eq:defn-S-ij}
S_{ij}\triangleq\frac{1}{\sqrt{v_{ij}}}\left(M_{ij}^{\mathsf{d}}-M_{ij}^{\star}\right),
\end{equation}
where $v_{ij}$ is the empirical variance defined in~(\ref{eq:empirical-variance-entry}). In view of the $95\%$ confidence interval predicted by Corollary~\ref{coro:confidence-interval}, and similar to what have done for the low-rank components, for each $(i,j)$, we define $\widehat{\mathsf{Cov}}_{\mathsf{E},(i,j)}$
to be the empirical coverage rate of $M^\star_{ij}$
over $200$ Monte Carlo simulations. Correspondingly, denote by $\mathsf{Mean}(\widehat{\mathsf{Cov}}_{\mathsf{E}})$
(resp.~$\mathsf{Std}(\widehat{\mathsf{Cov}}_{\mathsf{E}})$) the
average (resp.~the standard deviation) of $\widehat{\mathsf{Cov}}_{\mathsf{E},(i,j)}$ over indices $1\leq i,j\leq n$.
As before, Table~\ref{tab:cov_entry} gathers the empirical coverage rates for $M^\star_{ij}$ and Figure~\ref{fig:entry} displays the Q-Q (quantile-quantile)
plots of $S_{11},S_{12}$ and $S_{13}$ vs.~the standard Gaussian random
variable over 200 Monte Carlo trials for $r=5$, $p=0.4$ and $\sigma = 10^{-3}$. It is evident that the distribution
of $S_{ij}$ matches that
of $\mathcal{N}(0,1)$ reasonably well.

\begin{figure}[t!]
\centering

\begin{tabular}{cc}
\quad\includegraphics[scale=0.35]{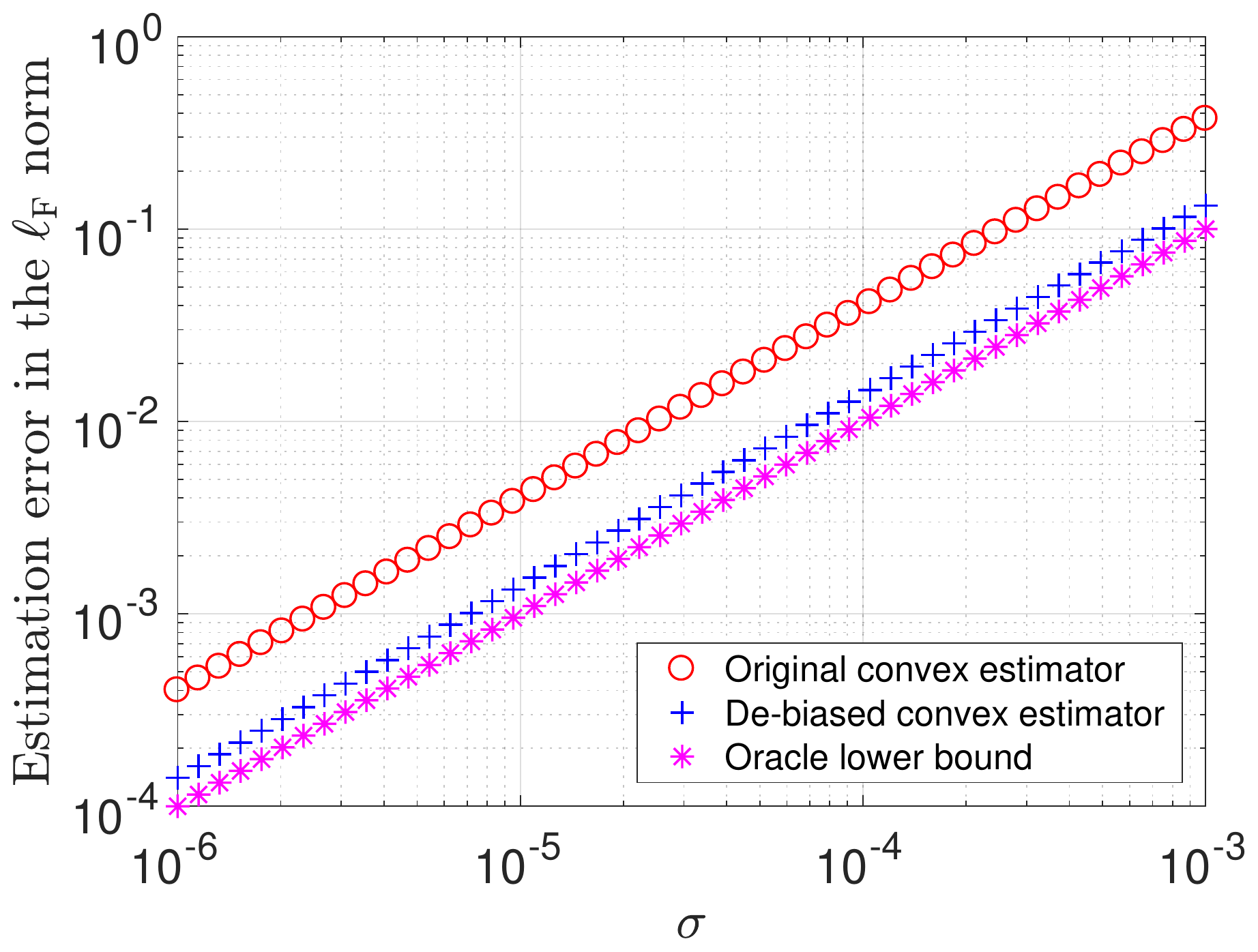}\quad & \quad\includegraphics[scale=0.35]{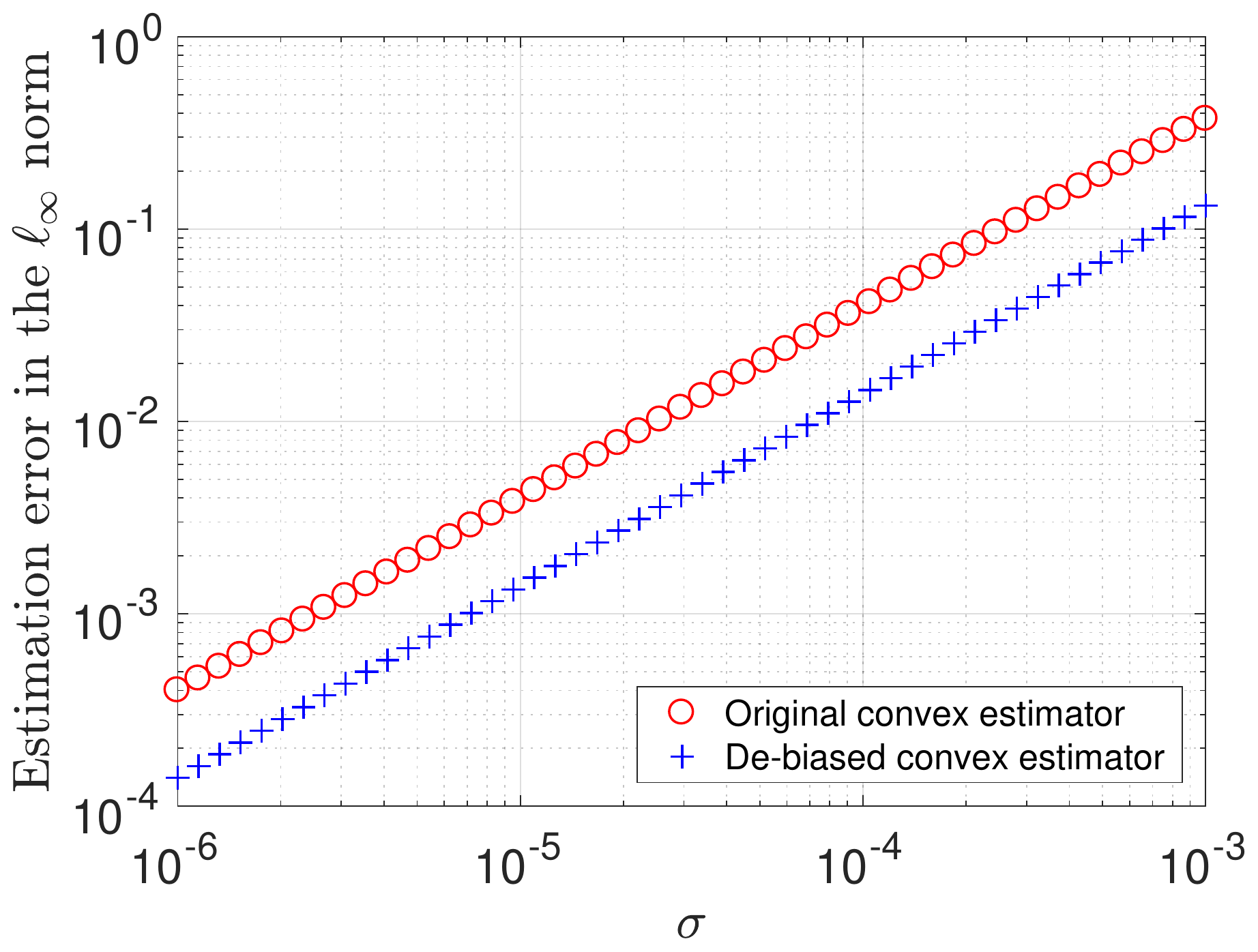}\quad\tabularnewline
(a) & (b)\tabularnewline
\end{tabular}

\caption{(a) Estimation error of $\bm{Z}^{\mathsf{cvx}}$ vs.~$\bm{M}^{\mathsf{d}}$
measured in the Frobenius norm. (b) Estimation error of $\bm{Z}^{\mathsf{cvx}}$
vs.~$\bm{M}^{\mathsf{d}}$ measured in the $\ell_{\infty}$ norm. The results are averaged over 20 independent trials for $r=5$, $p=0.2$ and $n=1000$.
\label{fig:estimation-error}}
\end{figure}


In addition to the tractable distributional guarantees,  the de-biased estimator
$\bm{M}^{\mathsf{d}}$ also exhibits superior estimation accuracy compared to the original estimator $\bm{Z}^{\mathsf{cvx}}$ (cf.~Theorem~\ref{thm:estimation-error}).
Figure~\ref{fig:estimation-error} reports the estimation error of
$\bm{M}^{\mathsf{d}}$ vs.~$\bm{Z}^{\mathsf{cvx}}$ measured in both
the Frobenius norm and in the $\ell_{\infty}$ norm across difference
noise levels. The results are averaged over 20 Monte Carlo simulations for $r=5$, $p=0.2$.
It can be seen that the errors of the de-biased estimator are uniformly
smaller than that of the original estimator and are much closer to the oracle lower bound.
As a result, we recommend using $\bm{M}^{\mathsf{d}}$ even for
the purpose of estimation. 

We conclude this section with experiments on real data. Similar to \cite{CanPla10}, we use the daily temperature data \cite{NCDC} for 1400 stations across the world in 2018, which results in a $1400 \times 365$ data matrix. Inspection on the singular values reveals that the data matrix is nearly  low-rank. We vary the observation probability $p$ from $0.5$ to $0.9$ and randomly subsample the data accordingly. Based on the observed temperatures, we then apply the proposed methodology to obtain $95\%$ confidence intervals for all the entries. Table~\ref{tab:real_data} reports the empirical coverage probabilities, the average length of the confidence intervals as well as the estimation error of both $\bm{Z}^\mathsf{cvx}$ and $\bm{M}^\mathsf{d}$ over 20 independent experiments. It can be seen that the average coverage probabilities are reasonably close to $95\%$ and the confidence intervals are also quite short. In addition, the estimation error of $\bm{M}^\mathsf{d}$ is smaller than that of $\bm{Z}^\mathsf{cvx}$, which corroborates our theoretical prediction. The discrepancy between the nominal coverage probability and the actual one might arise from the facts that (1) the underlying true temperature matrix is only approximately low-rank, and (2) the noise in the temperature might not be independent. 

\begin{table}[b]

\caption{Empirical coverage rates, average lengths of the confidence intervals of the entries as well as the estimation error vs.~observation probability $p$. The results are averaged over 20 Monte Carlo trials. \label{tab:real_data}}
\vspace{0.5em}
\centering

{\footnotesize
\centering
\begin{tabular}{|c|c|c|c|c|c|c|}
	\hline 
	\multirow{2}{*}{$p$} & \multicolumn{2}{c|}{${\mathsf{Coverage}}$} & \multicolumn{2}{c|}{CI Length} & \multicolumn{2}{c|}{$\Vert\widehat{\bm{Z}}-\bm{M}^\star\Vert_{\mathrm{F}}/\Vert\bm{M}^\star\Vert_{\mathrm{F}}$}\tabularnewline
	\cline{2-7} \cline{3-7} \cline{4-7} \cline{5-7} \cline{6-7} \cline{7-7} 
	& Mean & Std & Mean & Std & Convex $\bm{Z}^{\mathsf{cvx}}$& Debiased $\bm{M}^{\mathsf{d}}$\tabularnewline
	\hline 
	$0.5$ & $0.8265$ & $0.0016$ & $3.6698$ & $0.0209$ & $0.029$ & $0.028$\tabularnewline
	\hline 
	$0.6$ & $0.8268$ & $0.0011$ & $2.8774$ & $0.0098$ & $0.025$ & $0.023$\tabularnewline
	\hline 
	$0.7$ & $0.8431$ & $0.0006$ & $2.3426$ & $0.0054$ & $0.022$ & $0.019$\tabularnewline
	\hline 
	$0.8$ & $0.8725$ & $0.0003$ & $2.0234$ & $0.0052$ & $0.020$ & $0.015$\tabularnewline
	\hline 
	$0.9$ & $0.9093$ & $0.0003$ & $1.8296$ & $0.0072$ & $0.018$ & $0.011$\tabularnewline
	\hline 
\end{tabular}
}
\end{table}

\subsection{A bit of intuition}

We pause to develop some intuition behind the distributional guarantees
for the proposed estimators. Bearing in mind the intimate link between
convex and nonconvex optimization (cf.~(\ref{eq:proximity-cvx-ncvx})),
it suffices to concentrate on the nonconvex problem~(\ref{eq:ncvx}).
For the sake of clarity, we further restrict attention to the rank-1
positive semidefinite case where $\bm{M}^{\star}=\bm{x}^{\star}\bm{x}^{\star\top}$
and set $\lambda=0$, where one can focus on
\begin{equation}
\underset{\bm{x}\in\mathbb{R}^{n}}{\textsf{minimize}}\qquad f\left(\bm{x}\right)\triangleq\frac{1}{2}\big\|\mathcal{P}_{\Omega}\big(\bm{x}\bm{x}^{\top}-\bm{M}\big)\big\|_{\mathrm{F}}^{2}.\label{eq:ncvx-rank-1}
\end{equation}
Any optimizer $\widehat{\bm{x}}$ of~(\ref{eq:ncvx-rank-1}) would
necessarily satisfy the first-order optimality condition
\begin{equation}
\mathcal{P}_{\Omega}\big(\widehat{\bm{x}}\widehat{\bm{x}}^{\top}-\bm{M}\big)\widehat{\bm{x}}=\bm{0}.\label{eq:first-order-optimality-rank1}
\end{equation}
We shall also assume that $\widehat{\bm{x}}$ is a reasonably reliable
estimate obeying $\widehat{\bm{x}}\approx\bm{x}^{\star}$.

We begin with the no-missing-data case (i.e.~$p=1$), which already
conveys the key insight. The condition~(\ref{eq:first-order-optimality-rank1})
simplifies to
\begin{equation}
\widehat{\bm{x}}\widehat{\bm{x}}^{\top}\widehat{\bm{x}}-\bm{x}^{\star}\bm{x}^{\star\top}\widehat{\bm{x}}=\bm{E}\widehat{\bm{x}},
\end{equation}
which, through a little manipulation, leads to an equivalent decomposition:
\begin{align}
\|\widehat{\bm{x}}\|_{2}^{2}\left(\widehat{\bm{x}}-\bm{x}^{\star}\right)= & \underset{\text{approximately Gaussian}}{\underbrace{\vphantom{\bm{x}^{\star}\left(\bm{x}^{\star}-\widehat{\bm{x}}\right)^{\top}\bm{x}^{\star}}\bm{E}\widehat{\bm{x}}}}+\underset{\text{negligible first-order term}}{\underbrace{\bm{x}^{\star}\left(\bm{x}^{\star}-\widehat{\bm{x}}\right)^{\top}\bm{x}^{\star}}}+\underset{\text{second-order term}}{\underbrace{\bm{x}^{\star}\left(\bm{x}^{\star}-\widehat{\bm{x}}\right)^{\top}\left(\widehat{\bm{x}}-\bm{x}^{\star}\right)}}.\label{eq:first-order-condition-rank1-complete}
\end{align}
Then: (1) the third term of~(\ref{eq:first-order-condition-rank1-complete}),
which can be viewed as a second-order term (in the sense that it is
a quadratic term of $\widehat{\bm{x}}-\bm{x}^{\star}$), becomes vanishingly
small when $\widehat{\bm{x}}\approx\bm{x}^{\star}$; (2) while the
second term of~(\ref{eq:first-order-condition-rank1-complete}) looks
like a first-order term, it is natural to conjecture that $\widehat{\bm{x}}-\bm{x}^{\star}$
is sufficiently random and hence $(\widehat{\bm{x}}-\bm{x}^{\star})^{\top}\bm{x}^{\star}\ll\|\widehat{\bm{x}}-\bm{x}^{\star}\|_{2}\|\bm{x}^{\star}\|_{2}$
(i.e.~the estimation error is not aligned with $\bm{x}^{\star}$),
meaning that this term is also expected to be negligible compared
to a typical first-order term (e.g.~the term on the left-hand side of~\ref{eq:first-order-condition-rank1-complete}). In summary, these non-rigorous arguments
suggest that
\begin{align}
\|\widehat{\bm{x}}\|_{2}^{2}\left(\widehat{\bm{x}}-\bm{x}^{\star}\right)\,\approx\, & \bm{E}\widehat{\bm{x}}.
\end{align}
If one can be convinced that $\bm{E}$ and $\widehat{\bm{x}}$ are
only weakly dependent, then this means
\begin{align}
\widehat{\bm{x}}-\bm{x}^{\star}\,\approx\, & \frac{1}{\|\widehat{\bm{x}}\|_{2}^{2}}\bm{E}\widehat{\bm{x}}\,\approx\,\frac{1}{\|\bm{x}^{\star}\|_{2}^{2}}\bm{E}\bm{x}^{\star}\,\sim\,\mathcal{N}\Big(\bm{0},\frac{\sigma^{2}}{\|\bm{x}^{\star}\|_{2}^{2}}\bm{I}_{n}\Big).
\end{align}

Returning to the missing data scenario with $p<1$, everything is
based on the following approximation
\[
\mathcal{P}_{\Omega}\big(\widehat{\bm{x}}\widehat{\bm{x}}^{\top}-\bm{x}^{\star}\bm{x}^{\star\top}\big)\widehat{\bm{x}}\approx p\big(\widehat{\bm{x}}\widehat{\bm{x}}^{\top}-\bm{x}^{\star}\bm{x}^{\star\top}\big)\widehat{\bm{x}};
\]
this is certainly expected --- using standard concentration arguments
--- if we ``pretend'' that $\mathcal{P}_{\Omega}$ and $\widehat{\bm{x}}$
are statistically independent. With this approximation in mind, one
can translate~(\ref{eq:first-order-optimality-rank1}) into
\begin{equation}
p\left(\widehat{\bm{x}}\widehat{\bm{x}}^{\top}\widehat{\bm{x}}-\bm{x}^{\star}\bm{x}^{\star\top}\widehat{\bm{x}}\right)\approx\mathcal{P}_{\Omega}\left(\bm{E}\right)\widehat{\bm{x}}.
\end{equation}
Repeating the above argument then immediately yields
\begin{align}
\widehat{\bm{x}}-\bm{x}^{\star}\,\approx\, & \frac{1}{\|\widehat{\bm{x}}\|_{2}^{2}}\cdot\frac{1}{p}\mathcal{P}_{\Omega}\left(\bm{E}\right)\widehat{\bm{x}}\,\approx\,\frac{1}{p\|\bm{x}^{\star}\|_{2}^{2}}\cdot\mathcal{P}_{\Omega}\left(\bm{E}\right)\bm{x}^{\star}\overset{\text{approx.}}{\sim}\mathcal{N}\Big(\bm{0},\frac{\sigma^{2}}{p\|\bm{x}^{\star}\|_{2}^{2}}\bm{I}_{n}\Big).
\end{align}

The case with $\lambda>0$ can be intuitively understood in a very
similar way by first de-shrinking the estimate; we omit it here for
brevity. We note that these hand-waving arguments can all be made
rigorous, which is the main content of the proof.

\subsection{Inference based on spectral estimates?}

One would naturally be curious about whether there are other estimation
procedures that also enable reasonable statistical inference. While
this is beyond the scope of the current paper, we take a moment to
discuss one alternative: the spectral method, as pioneered by~\cite{KesMonSew2010,Se2010Noisy}
in the matrix completion problem. In a nutshell, this approach consists
in computing a rank-$r$ approximation to $\mathcal{P}_{\Omega}(\bm{M})/p$,
which is precisely the spectral initialization widely used in a two-stage
nonconvex algorithm (cf.~Algorithm~\ref{alg:gd-mc-ncvx})~\cite{Se2010Noisy,sun2016guaranteed,chen2015fast,chen2018asymmetry,ma2017implicit}.
While inference has not been, as far as we know, the focus of prior
work on spectral methods,\footnote{We note that inference from spectral estimates has been investigated
in other context beyond matrix completion (e.g.~the model without
missing data~\cite{xia2019data,fan2019asymptotic}).} the recent papers~\cite{abbe2017entrywise,ma2017implicit} hinted
at the possibility of characterizing the distribution of the spectral
estimate. Take a simple symmetric rank-1 case for example (i.e.~$\bm{M}^{\star}=\bm{x}^{\star}\bm{x}^{\star\top}$
with $\|\bm{x}^{\star}\|_{2}=1$): the leading eigenvector $\bm{u}^{\mathsf{spectral}}$
of $\mathcal{P}_{\Omega}(\bm{M})/p$ often admits the following approximation
(up to a global sign)
\begin{align*}
\bm{u}^{\mathsf{spectral}} & \approx\frac{1}{p}\mathcal{P}_{\Omega}(\bm{M})\bm{x}^{\star}.
\end{align*}
Expanding $\mathcal{P}_{\Omega}(\bm{M})=p\bm{x}^{\star}\bm{x}^{\star\top}+\mathcal{P}_{\Omega}(\bm{x}^{\star}\bm{x}^{\star\top})-p\bm{x}^{\star}\bm{x}^{\star\top}+\mathcal{P}_{\Omega}(\bm{E})$,
we arrive at
\begin{align*}
\bm{u}^{\mathsf{spectral}} & \approx\bm{x}^{\star}\bm{x}^{\star\top}\bm{x}^{\star}+\Big(\frac{1}{p}\mathcal{P}_{\Omega}(\bm{x}^{\star}\bm{x}^{\star\top})-\bm{x}^{\star}\bm{x}^{\star\top}\Big)\bm{x}^{\star}+\frac{1}{p}\mathcal{P}_{\Omega}(\bm{E})\bm{x}^{\star},
\end{align*}
which is equivalent to
\begin{align}
\bm{u}^{\mathsf{spectral}} & -\bm{x}^{\star}\approx\underset{\text{noise effect}}{\underbrace{\vphantom{\Big(\frac{1}{p}\mathcal{P}_{\Omega}(\bm{x}^{\star}\bm{x}^{\star\top})-\bm{x}^{\star}\bm{x}^{\star\top}\Big)\bm{x}^{\star}}\frac{1}{p}\mathcal{P}_{\Omega}(\bm{E})\bm{x}^{\star}}}+\underset{\text{effect of random sub-sampling}}{\underbrace{\Big(\frac{1}{p}\mathcal{P}_{\Omega}(\bm{x}^{\star}\bm{x}^{\star\top})-\bm{x}^{\star}\bm{x}^{\star\top}\Big)\bm{x}^{\star}}}.\label{eq:spectral-uncertainty}
\end{align}
In words, two major factors dictate the uncertainty of the spectral
estimate: (1) the additive Gaussian noise (cf.~the 1st term on the
right-hand side of~(\ref{eq:spectral-uncertainty})), and (2) random
sub-sampling (in particular, the randomness incurred by employing
the sub-sampled $\mathcal{P}_{\Omega}(\bm{x}^{\star}\bm{x}^{\star\top})/p$
to approximate the truth $\bm{x}^\star \bm{x}^{\star\top}$). Given that the random sub-sampling effect
cannot be ignored at all, the spectral estimates often suffer from
a much larger estimation error (and hence a higher degree of uncertainty)
compared to either the convex or the nonconvex estimates. In truth,
this random sub-sampling effect does not go away even when the noise
vanishes. Consequently, uncertainty quantification based on the spectral
estimates may not be the most desirable option.

\section{Prior art\label{sec:Prior-art}}

\paragraph{Matrix completion.}

Low-rank matrix completion, or more broadly, low-rank matrix recovery,
is a fundamental task that permeates through a wide spectrum of applications
in science, engineering, and finance (e.g.~\cite{rennie2005fast,so2007theory,chen2014robust,fan2018principal,chen2013exact,zhang2015accelerating,bai2006confidence,chen2016projected,fan2019robust,kneip2011factor,cai2016matrix,koltchinskii2015optimal,fan2013large,duchi2017solving,charisopoulos2019composite,davenport20141,sun2012calibrated,foygel2011concentration}).
A paper of this length is unable to review all papers motivating and
contributing to this enormous subject; interested readers are referred
to~\cite{davenport2016overview,chen2018harnessing} for extensive
discussions of motivating applications as well as the exciting recent development.

Numerous algorithms have been proposed to solve this problem efficiently,
with two paradigms being arguably the most widely used: convex relaxation
and nonconvex optimization. We briefly review the literature contributing
to these two paradigms.
\begin{itemize}
\item Convex relaxation was largely popularized by the seminal works~\cite{fazel2002matrix,RecFazPar07,ExactMC09}.
In the absence of noise, it has been shown that nuclear norm minimization,
which can be solved by semidefinite programming, achieves minimal
sample complexity under mild conditions~\cite{Gross2011recovering,recht2011simpler,chen2015incoherence}.
When the observed entries are further corrupted by noise, Candès and
Plan~\cite{CanPla10} provided the first theoretical guarantee regarding
the estimation accuracy of perhaps the most natural convex relaxation
algorithm. While the theory might be tight for certain adversarial
scenarios (as shown by the recent work~\cite{krahmer2019convex}),
it is loose by some large factor under the natural random noise model.
This statistical guarantee has been partially improved later on by
two papers~\cite{Negahban2012restricted,MR2906869} under proper modifications
to the convex program (e.g.~enforcing an additional spikiness constraint
\cite{Negahban2012restricted,klopp2014noisy}, or modifying the squared
loss~\cite{MR2906869}). Nevertheless, the error bounds provided in
these papers (and their follow-ups) remain suboptimal, unless the
typical size of the noise is sufficiently large. Our recent work~\cite{chen2019noisy}
establishes near-optimal statistical guarantees --- when the estimation
errors are measured by the Frobenius norm, the spectral norm, and
the $\ell_{2,\infty}$ norm --- for a wide range of noise levels
when $r=O(1)$. All of these estimation guarantees, however, come
with a hidden and likely large pre-constant, which do not serve the
inferential purpose well.
\item Nonconvex optimization algorithms, as pioneered by~\cite{KesMonSew2010,srebro2004learning},
become increasingly more popular for solving various low-rank factorization
problems, due to their appealing computational complexities~\cite{jain2013low,candes2014wirtinger,ChenCandes15solving,tu2016low,sun2016guaranteed,zheng2016convergence,chen2018gradient,wang2016unified,chen2019nonconvex}.
For instance, the gradient-based nonconvex methods have been analyzed
for noisy matrix completion~\cite{Se2010Noisy,chen2015fast,ma2017implicit,chen2019noisy},
which are shown to achieve near-optimal statistical accuracy and linear-time
convergence guarantees all at once. Going beyond gradient methods,
we note that other nonconvex methods (e.g.~\cite{rennie2005fast,jain2010guaranteed,wen2012solving,jain2013low,fornasier2011low,vandereycken2013low,lai2013improved,hardt2014understanding,jin2016provable,rohde2011estimation,wei2016guarantees,ding2018leave,zhao2015nonconvex-estimation,zhang2018primal,molybogno2019,charisopoulos2019low})
and landscape properties~\cite{ge2016matrix,chen2017memory,ge2017no,zhang2019sharp,zhang2018much,shapiro2018matrix}
have been largely explored as well. The interested readers are referred
to~\cite{chi2018nonconvex} for an in-depth discussion. One limitation,
however, is that the theoretical guarantees provided for nonconvex
algorithms often exhibit sub-optimal dependency in the rank $r$ of
the unknown matrix; for instance, most theory requires a sample complexity
of at least $nr^{2}$ (in fact, often much larger than $nr^{2}$).
This is outperformed by the convex relaxation approach.
\end{itemize}
Despite these recent developments, very little work has investigated
statistical inference for noisy matrix completion. While~\cite{carpentier2018adaptive, carpentier2016constructing,carpentier2015signal,carpentier2015uncertainty}
discussed the construction of ``honest'' confidence regions, the
volume of these regions is dependent on some (possibly huge) hidden
constants, thus resulting in over-coverage. Perhaps the closest to
our paper is the recent work~\cite{xia2018confidence}, which investigated
inference for low-rank trace regression. Employing a closely related de-biased
estimator with sample splitting, the paper~\cite{xia2018confidence}
established asymptotic normality of a certain projected distance between
the estimate and the truth. The result therein, however, requires
a sampling mechanism obeying the restricted isometry property (e.g.~i.i.d.~Gaussian
designs), which fails to hold for matrix completion. Also, our approach
does not require sample splitting --- a technique that is convenient
for analysis but conservative in constructing confidence regions.
Another work by Cai et al.~\cite{cai2016geometric} developed a unified
approach to provide inference guarantees for linear inverse problems
including low-rank matrix estimation. Their results, however, require
the sample size to exceed the total dimension $n^{2}$ even under
the Gaussian design. Finally, a recent line of work~\cite{mak2017active}
explored uncertainty quantification under the Bayesian setting, hypothesizing
on a special prior regarding the true matrix. This departs drastically
from the scenario considered herein.

\paragraph{Inference in high-dimensional problems.}

Inference in high-dimensional sparse regression has received much
attention in the last few years~\cite{wasserman2009high,zhang2014confidence,belloni2011inference,
van2014asymptotically,javanmard2014hypothesis,dezeure2015high,
cai2017confidence,ning2017general,neykov2018unified,lee2016exact,
lockhart2014significance,meinshausen2009p,dezeure2017high,zhang2017simultaneous,battey2018distributed}.
Our inferential approach is partly inspired by the recent developments
on this topic, particularly with regard to the de-biased\,/\,de-sparsified
estimators proposed for Lasso. More specifically, recognizing the
non-negligible bias of the Lasso estimate
\begin{equation}
\widehat{\bm{\beta}}\triangleq\arg\min_{\bm{\beta}}\,\,\frac{1}{2}\|\bm{y}-\bm{X}\bm{\beta}\|_{2}^{2}+\lambda\|\bm{\beta}\|_{1},\label{eq:Lasso}
\end{equation}
A line of work~\cite{zhang2014confidence,van2014asymptotically,javanmard2014confidence}
came up with a linear transformation of $\widehat{\bm{\beta}}$ of
the form
\begin{equation}
\bm{\beta}^{\mathrm{d}}\triangleq\widehat{\bm{\beta}}+\bm{L}\bm{X}^{\top}\big(\bm{y}-\bm{X}\widehat{\bm{\beta}}\big),\label{eq:debiased-lasso}
\end{equation}
where $\bm{L}$ is some matrix to be designed, and $\bm{X}^{\top}\big(\bm{y}-\bm{X}\widehat{\bm{\beta}}\big)$
corresponds to the negative gradient of the squared loss at $\widehat{\bm{\beta}}$,
or equivalently, the (scaled) sub-gradient of the $\ell_{1}$ norm
at $\widehat{\bm{\beta}}$. If $\bm{L}$ is properly chosen, then
$\bm{\beta}^{\mathrm{d}}$ is able to correct the bias of this nonlinear
estimator $\bm{\beta}$, while controlling the degree of uncertainty. Many follow-up
papers have investigated the design of $\bm{L}$ as well as the resulting
inferential guarantees~\cite{zhang2014confidence,van2014asymptotically,javanmard2014confidence,javanmard2015biasing}.

Interestingly, our de-biased estimator (\ref{eq:defn-Zd-cvx}) for
matrix completion admits a very similar form as (\ref{eq:debiased-lasso}).
To see this, recall that our de-biased estimator is given by
\[
\bm{M}^{\mathsf{d}}=\mathcal{P}_{\text{rank-}r}\big(\bm{Z}-\tfrac{1}{p}\mathcal{P}_{\Omega}\big(\bm{Z}\big)+\tfrac{1}{p}\mathcal{P}_{\Omega}\big(\bm{M}\big)\big),
\]
where $\bm{Z}$ can be either $\bm{Z}^{\mathsf{cvx}}$ or $\bm{X}^{\mathsf{ncvx}}\bm{Y}^{\mathsf{ncvx}\top}$ (see Table \ref{tab:Notation-unified}). Let $T$ be
the tangent space of the set of rank-$r$ matrices at $\bm{Z}^{\mathsf{cvx},r}$
(resp.~$\bm{X}^{\mathsf{ncvx}}\bm{Y}^{\mathsf{ncvx}\top}$) in the
convex (resp.~nonconvex) case, and $\mathcal{P}_{T}$ be the projection
operator onto $T$. Somewhat surprisingly, replacing $\mathcal{P}_{\text{rank-}r}$
by $\mathcal{P}_{T}$ does not affect the de-biased estimator by much,
in the sense that
\begin{equation}
\bm{M}^{\mathsf{d}}\approx\mathcal{P}_{T}\big(\bm{Z}-\tfrac{1}{p}\mathcal{P}_{\Omega}\big(\bm{Z}\big)+\tfrac{1}{p}\mathcal{P}_{\Omega}\big(\bm{M}\big)\big).
\end{equation}
In addition, recognizing that $\bm{Z}$ almost lies within the tangent
space $T$,\footnote{More precisely, if $\bm{Z}=\bm{X}^{\mathsf{ncvx}}\bm{Y}^{\mathsf{ncvx}\top}$,
then $\bm{Z}\in T$; if $\bm{Z}=\bm{Z}^{\mathsf{cvx}}$, one has $\mathcal{P}_{T}(\bm{Z})\approx\bm{Z}$.}one can rewrite
\begin{equation}
\bm{M}^{\mathsf{d}}\approx\bm{Z}-\tfrac{1}{p}\mathcal{P}_{T}\mathcal{P}_{\Omega}\big(\bm{Z}-\bm{M}\big),\label{eq:debiased-estimator-T}
\end{equation}
a fact to be made precise in Section~\ref{subsec:Approximate-equivalence}.
This bears a striking resemblance to the de-biasing approach developed
for Lasso --- the term $\mathcal{P}_{\Omega}\big(\bm{Z}-\bm{M}^{\star}\big)$
represents the gradient of the squared loss $0.5\|\mathcal{P}_{\Omega}\big(\bm{Z}-\bm{M}\big)\|_{\mathrm{F}}^{2}$
(or equivalently, the negative sub-gradient of the nuclear norm) at
$\bm{Z}$, and $\mathcal{P}_{T}$ is the linear operator we pick.
To the best of our knowledge, no de-biasing approach --- with rigorous
theoretical guarantees and without sample splitting --- has been
proposed and analyzed for matrix completion in prior literature.
In addition, we note that our de-biased estimator for matrix completion achieves full statistical efficiency in terms of both the rates and the pre-constant; in comparison, the commonly used de-biased estimators for sparse linear regression typically fall short of achieving the best possible estimation accuracy, unless additional  thresholding procedures are enforced.  

Finally, de-biased estimators have been put forward to tackle
other high-dimensional problems, including but not limited to generalized
linear models~\cite{van2014asymptotically,ning2017general}, graphical
models~\cite{jankova2015confidence,ren2015asymptotic,ma2017inter,jankova2017honest},
sparse PCA~\cite{jankova2018biased}, treatment effects estimation~\cite{chernozhukov2018double,athey2018approximate}. These are beyond
the scope of the current paper.

\section{Architecture of the proof}

This section outlines the main steps for establishing Theorem~\ref{thm:low-rank-factor-master-bound-simple}
and Theorem~\ref{thm:entries-master-decomposition}. Before starting,
we introduce some useful notation. For convenience of presentation,
we insert the factor $1/p$ into~(\ref{eq:ncvx}) and redefine the
nonconvex loss function as
\begin{equation}
f\left(\bm{X},\bm{Y}\right)\triangleq\frac{1}{2p}\left\Vert \mathcal{P}_{\Omega}\left(\bm{X}\bm{Y}^{\top}-\bm{M}\right)\right\Vert _{\mathrm{F}}^{2}+\frac{\lambda}{2p}\left\Vert \bm{X}\right\Vert _{\mathrm{F}}^{2}+\frac{\lambda}{2p}\left\Vert \bm{Y}\right\Vert _{\mathrm{F}}^{2}.\label{eq:ncvx-new}
\end{equation}
In addition, for each $1\leq j,k\leq n$, we define the indicator $\delta_{jk}\triangleq\ind\{(j,k)\in\Omega\}$,
which is a Bernoulli random variable with mean $p$.

We also note that Theorem~\ref{thm:low-rank-factor-master-bound-simple} (resp.~Theorem~\ref{thm:entries-master-decomposition}) is subsumed by Theorem~\ref{thm:low-rank-factor-master-bound} (resp.~Theorem~\ref{thm:entries-master-decomposition-augment}). As a result, we shall focus on establishing Theorem \ref{thm:low-rank-factor-master-bound} (resp.~Theorem~\ref{thm:entries-master-decomposition-augment}) when it comes to estimating low-rank factors (resp.~the entries of the matrix). 

\begin{theorem}
\label{thm:low-rank-factor-master-bound}
Suppose that the sample complexity meets $n^{2}p\geq C\kappa^{4}\mu^{2}r^{2}n\log^{3}n$
for some sufficiently large constant $C>0$ and the noise obeys $\sigma\sqrt{(\kappa^{4}\mu rn\log n)/p}\leq c\sigma_{\min}$
for some sufficiently small constant $c>0$. Then the decomposition in Theorem \ref{thm:low-rank-factor-master-bound-simple} remains valid, except that the residual matrices $\bm{\Psi}_{\bm{X}},\bm{\Psi}_{\bm{Y}}\in\mathbb{R}^{n\times r}$
satisfy, with probability at least $1-O(n^{-3})$, that
\begin{equation}
\max\big\{\left\Vert \bm{\Psi}_{\bm{X}}\right\Vert _{2,\infty},\left\Vert \bm{\Psi}_{\bm{Y}}\right\Vert _{2,\infty}\big\}\lesssim\frac{\sigma}{\sqrt{p\sigma_{\min}}}\left(\,\frac{\sigma}{\sigma_{\min}}\sqrt{\frac{\kappa^{7}\mu rn\log n}{p}}+\sqrt{\frac{\kappa^{7}\mu^{3}r^{3}\log^{2}n}{np}}\,\right).\label{eq:thm-low-rank-residual-size}
\end{equation}
\end{theorem}
\begin{theorem}\label{thm:entries-master-decomposition-augment}Instate
the assumptions of Theorem~\ref{thm:low-rank-factor-master-bound}. 
Recall the definition of $v_{ij}^{\star}$ in \eqref{eq:variance-entry}. 
Then one has the following decomposition
\begin{equation}
M_{ij}^{\mathsf{{d}}}-M_{ij}^{\star}=g_{ij}+\Delta_{ij},  \label{eq:thm-entry-decomposition}
\end{equation}
where $g_{ij}\sim\mathcal{N}(0,v_{ij}^{\star})$ and the residual
obeys --- with probability exceeding $1-O(n^{-10})$ --- that
\[
\left|\Delta_{ij}\right|\lesssim\left(\left\Vert \bm{U}_{i,\cdot}^{\star}\right\Vert _{2}+\left\Vert \bm{V}_{j,\cdot}^{\star}\right\Vert _{2}\right)\frac{\sigma}{\sqrt{p}}\left(\frac{\sigma}{\sigma_{\min}}\sqrt{\frac{\kappa^{8}\mu rn\log n}{p}}+\sqrt{\frac{\kappa^{8}\mu^{3}r^{3}\log^{2}n}{np}}\right)+\left(\frac{\sigma}{\sqrt{\sigma_{\min}}}\sqrt{\frac{\kappa^{3}\mu r\log n}{p}}\right)^{2}.
\]
\end{theorem}

\subsection{Near equivalence between convex and nonconvex estimators \label{subsec:Approximate-equivalence}}

Note that Theorem~\ref{thm:low-rank-factor-master-bound} and Theorem~\ref{thm:entries-master-decomposition-augment}
are concerned with the de-biased estimators built upon
both convex and nonconvex estimates.
At first glance, one needs to establish theoretical guarantees for
each of them separately. Fortunately, as  alluded to previously (cf.~(\ref{eq:proximity-cvx-ncvx})), the
convex 
and nonconvex estimates 
are extremely close --- a fact that has been established in \cite{chen2019noisy}.
The proximity of these two estimates naturally extends to the de-biased estimators constructed
based on them. As a result,
it suffices to concentrate on proving the theorems for any of these
 estimators; the claims for the other one follow immediately.

The following key lemma formalizes this argument, 
which will be established in Appendix~\ref{sec:Proof-of-Lemma-connection}
(see also Figure~\ref{fig:equivalence} for numerical evidence). Before
continuing, we remind the readers of the key notation (see Appendix~\ref{sec:summary-estimators} for precise definitions): 
\begin{itemize}
\item $(\bm{X}^{\mathsf{ncvx}},\bm{Y}^{\mathsf{ncvx}})$: an approximate
solution to the nonconvex problem~(\ref{eq:ncvx}) (see Appendix~\ref{subsec:Algorithmic-details-nonconvex}); 
\item $\bm{M}^{\mathsf{cvx,d}},\bm{X}^{\mathsf{cvx,d}},\bm{Y}^{\mathsf{cvx,d}}$:
	the de-biased estimators built upon the convex optimizer $\bm{Z}^{\mathsf{cvx}}$;
\item $\bm{M}^{\mathsf{ncvx,d}},\bm{X}^{\mathsf{ncvx,d}},\bm{Y}^{\mathsf{ncvx,d}}$:
	the de-biased estimators built upon the nonconvex estimate $(\bm{X}^{\mathsf{ncvx}}, \bm{Y}^{\mathsf{ncvx}})$.
\end{itemize}
Our proximity result is this:
\begin{lemma}
	\label{lemma:connection-de-bias-shrunken}
Suppose
that the sample size obeys $n^{2}p\geq C\kappa^{4}\mu^{2}r^{2}n\log^{3}n$
for some sufficiently large constant $C>0$ and the noise satisfies
$\sigma\sqrt{(\kappa^{4}\mu nr\log n)/p}\leq c\sigma_{\min}$ for
some sufficiently small constant $c>0$. Set $\lambda=C_{\lambda}\sigma\sqrt{np}$
with some large enough constant $C_{\lambda}>0$.
\begin{enumerate}
\item With probability at least $1-O(n^{-10})$, one has \begin{subequations}
\begin{align}
\max\left\{ \left\Vert \bm{M}^{\mathsf{cvx,d}}-\bm{X}^{\mathsf{ncvx,d}}\bm{Y}^{\mathsf{ncvx,d}\top}\right\Vert _{\mathrm{F}},\left\Vert \bm{M}^{\mathsf{ncvx,d}}-\bm{X}^{\mathsf{ncvx,d}}\bm{Y}^{\mathsf{ncvx,d}\top}\right\Vert _{\mathrm{F}}\right\}  & \lesssim\frac{1}{n^{4}}\cdot\sigma\sqrt{\frac{n}{p}},\label{eq:matrix-equivalence}\\
\min_{\bm{R}\in\mathcal{O}^{r\times r}}\sqrt{\left\Vert \bm{X}^{\mathsf{cvx,d}}\bm{R}-\bm{X}^{\mathsf{ncvx,d}}\right\Vert _{\mathrm{F}}^{2}+\left\Vert \bm{Y}^{\mathsf{cvx,d}}\bm{R}-\bm{Y}^{\mathsf{ncvx,d}}\right\Vert _{\mathrm{F}}^{2}} & \lesssim\frac{1}{n^{3}}\cdot\frac{\sigma}{\sqrt{\sigma_{\min}}}\sqrt{\frac{n}{p}},\label{eq:low-rank-factor-equivalence}
\end{align}
\end{subequations}where $\mathcal{O}^{r\times r}$ is the set of
$r\times r$ rotation matrices. 
\item With probability exceeding $1-O(n^{-10})$, one has
\begin{equation}
\left\Vert \bm{M}^{\mathsf{ncvx,d}}-\left[\bm{X}^{\mathsf{ncvx}}\bm{Y}^{\mathsf{ncvx}\top}-p^{-1}\mathcal{P}_{T}\mathcal{P}_{\Omega}\left(\bm{X}^{\mathsf{ncvx}}\bm{Y}^{\mathsf{ncvx}\top}-\bm{M}\right)\right]\right\Vert _{\mathrm{F}}\lesssim\frac{1}{n^{4}}\cdot\sigma\sqrt{\frac{n}{p}},\label{eq:linearize-equivalence}
\end{equation}
where $T$ is the tangent space of the set of rank-$r$ matrices at
$\bm{X}^{\mathsf{ncvx}}\bm{Y}^{\mathsf{ncvx}\top}$. The same holds
true if we replace $\bm{X}^{\mathsf{ncvx}}\bm{Y}^{\mathsf{ncvx}\top}$
with $\bm{Z}^{\mathsf{cvx}}$  and
		replace $T$ with the tangent space at $\bm{Z}^{\mathsf{cvx},r}=\mathcal{P}_{\mathrm{rank}\text{-}r}(\bm{Z}^{\mathsf{cvx}})$ .
\end{enumerate}
\end{lemma}

In short, the first part of Lemma~\ref{lemma:connection-de-bias-shrunken} tells us that 
\begin{subequations}
\begin{align}
	\bm{M}^{\mathsf{cvx,d}} &\approx \bm{X}^{\mathsf{ncvx,d}}\bm{Y}^{\mathsf{ncvx,d}\top} \approx \bm{M}^{\mathsf{ncvx,d}}  
	, \\
	\big( \bm{X}^{\mathsf{cvx,d}}, \bm{Y}^{\mathsf{cvx,d}} \big)  &\approx \big( \bm{X}^{\mathsf{ncvx,d}}, \bm{Y}^{\mathsf{ncvx,d}} \big)  \qquad\qquad (\text{up to global rotation}),
\end{align}
\end{subequations}
whereas the second part of Lemma~\ref{lemma:connection-de-bias-shrunken}
justifies that the proposed de-biased estimator is closely approximated
by a linearized version (cf.~(\ref{eq:debiased-estimator-T})). Note that this
linearized form bears a resemblance to the de-biased estimators developed
for sparse linear regression \cite{zhang2014confidence,van2014asymptotically,javanmard2014confidence}.

\begin{figure}

\centering

\includegraphics[scale=0.35]{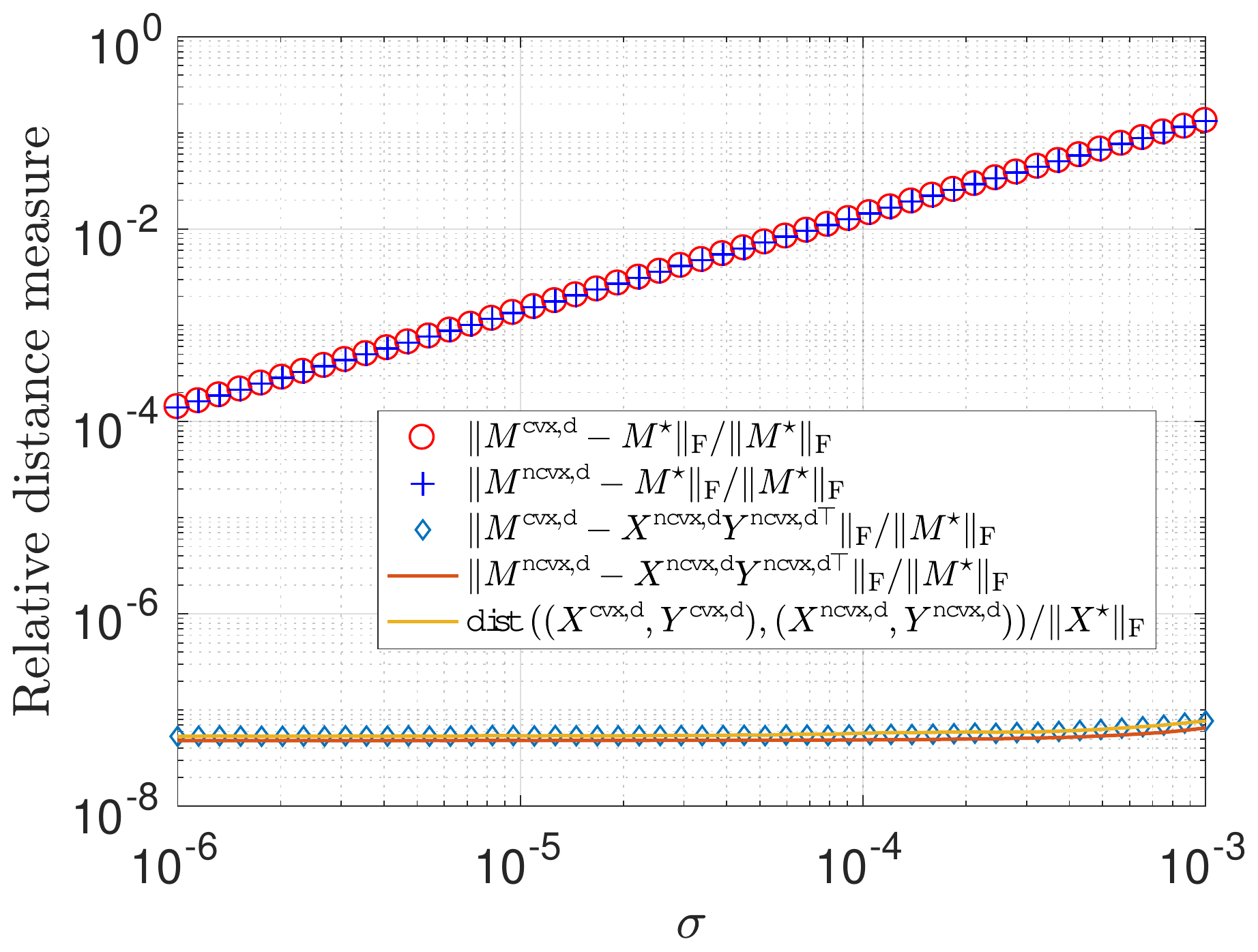}\caption{The relative estimation errors of $\bm{M}^{\mathsf{cvx,d}}$ and $\bm{M}^{\mathsf{ncvx,d}}$
and related quantities in Lemma~\ref{lemma:connection-de-bias-shrunken}
vs.~the standard deviation $\sigma$ of the noise. Here, $\mathsf{dist}((\bm{X}^{\mathsf{cvx,d}},\bm{Y}^{\mathsf{cvx,d}}),(\bm{X}^{\mathsf{ncvx,d}},\bm{Y}^{\mathsf{ncvx,d}}))$
is defined to be the left-hand side of~(\ref{eq:low-rank-factor-equivalence}).
The results, which are averaged over 20 trials, are reported for $n=1000$, $r=5$, $p=0.2$, and $\lambda=5\sigma\sqrt{np}$.
As can be seen, the difference between 	$\bm{M}^{\mathsf{cvx,d}}$, $\bm{M}^{\mathsf{ncvx,d}}$ and $\bm{X}^{\mathsf{ncvx,d}}\bm{Y}^{\mathsf{ncvx,d}\top}$, as well as the distance $\mathsf{dist}((\bm{X}^{\mathsf{cvx,d}},\bm{Y}^{\mathsf{cvx,d}}),(\bm{X}^{\mathsf{ncvx,d}},\bm{Y}^{\mathsf{ncvx,d}}))$, 
are all significantly smaller than the estimation errors. 
\label{fig:equivalence}}

\end{figure}

With Lemma~\ref{lemma:connection-de-bias-shrunken} in place, we
shall, from now on, focus on proving the main theorems for the nonconvex estimators, viz.
\begin{enumerate}
\item establishing Theorem~\ref{thm:low-rank-factor-master-bound} for
the de-shrunken low-rank factors $(\bm{X}^{\mathsf{ncvx,d}},\bm{Y}^{\mathsf{ncvx,d}})$;
\item establishing Theorem~\ref{thm:entries-master-decomposition-augment} for
the de-biased matrix estimator $\bm{X}^{\mathsf{ncvx,d}}\bm{Y}^{\mathsf{ncvx,d}\top}$.
\end{enumerate}
To simplify the presentation hereafter, we shall use the following notation
throughout the rest of this section:
\begin{itemize}
	\item $(\bm{X},\bm{Y})$: the nonconvex estimate $(\bm{X}^{\mathsf{ncvx}},\bm{Y}^{\mathsf{ncvx}})$;
\item $(\bm{X}^{\mathsf{d}},\bm{Y}^{\mathsf{d}})$: the de-shrunken estimate
	defined in~(\ref{eq:defn-Xd-Yd}) based on $(\bm{X},\bm{Y})=(\bm{X}^{\mathsf{ncvx}},\bm{Y}^{\mathsf{ncvx}})$;
\item $\bm{M}^{\mathsf{d}}\triangleq \bm{X}^{\mathsf{d}}\bm{Y}^{\mathsf{d}\top}$.
\end{itemize}

\subsection{A precise characterization of the de-shrunken low-rank factors}

We start with a precise characterization of the de-shrunken low-rank
factors $\bm{X}^{\mathsf{d}}$ and $\bm{Y}^{\mathsf{d}}$, which paves
the way for demonstrating both Theorem~\ref{thm:low-rank-factor-master-bound}
and Theorem~\ref{thm:entries-master-decomposition-augment}.

\begin{lemma}[\textsf{Decompositions of low-rank factors}]\label{lemma:decomp-low-rank}Denote
\begin{equation}
\bm{A}\triangleq\frac{1}{p}\mathcal{P}_{\Omega}\left(\bm{X}\bm{Y}^{\top}-\bm{X}^{\star}\bm{Y}^{\star\top}\right)-\left(\bm{X}\bm{Y}^{\top}-\bm{X}^{\star}\bm{Y}^{\star\top}\right).\label{eq:defn-A}
\end{equation}
One has the following decompositions for $\bm{X}^{\mathsf{d}}$ and
$\bm{Y}^{\mathsf{d}}$ \begin{subequations}\label{subeq:decomposition-X-d-Y-d}
\begin{align}
\bm{X}^{\mathsf{d}} & =\frac{1}{p}\mathcal{P}_{\Omega}\left(\bm{E}\right)\bm{Y}^{\mathsf{d}}\left(\bm{Y}^{\mathsf{d}\top}\bm{Y}^{\mathsf{d}}\right)^{-1}+\bm{X}^{\star}\bm{Y}^{\star\top}\bm{Y}^{\mathsf{d}}\left(\bm{Y}^{\mathsf{d}\top}\bm{Y}^{\mathsf{d}}\right)^{-1}-\bm{A}\bm{Y}^{\mathsf{d}}\left(\bm{Y}^{\mathsf{d}\top}\bm{Y}^{\mathsf{d}}\right)^{-1}\nonumber \\
 & \quad+\nabla_{\bm{X}}f\left(\bm{X},\bm{Y}\right)\Big(\bm{I}_{r}+\frac{\lambda}{p}\left(\bm{Y}^{\top}\bm{Y}\right)^{-1}\Big)^{1/2}\left(\bm{Y}^{\mathsf{d}\top}\bm{Y}^{\mathsf{d}}\right)^{-1}+\bm{X}\bm{\Delta}_{\mathsf{balancing}};\label{eq:X-d-identity}\\
\bm{Y}^{\mathsf{d}} & =\frac{1}{p}\left[\mathcal{P}_{\Omega}\left(\bm{E}\right)\right]^{\top}\bm{X}^{\mathsf{d}}\left(\bm{X}^{\mathsf{d}\top}\bm{X}^{\mathsf{d}}\right)^{-1}+\bm{Y}^{\star}\bm{X}^{\star\top}\bm{X}^{\mathsf{d}}\left(\bm{X}^{\mathsf{d}\top}\bm{X}^{\mathsf{d}}\right)^{-1}-\bm{A}^{\top}\bm{X}^{\mathsf{d}}\left(\bm{X}^{\mathsf{d}\top}\bm{X}^{\mathsf{d}}\right)^{-1}\nonumber \\
 & \quad+\nabla_{\bm{Y}}f\left(\bm{X},\bm{Y}\right)\Big(\bm{I}_{r}+\frac{\lambda}{p}\left(\bm{X}^{\top}\bm{X}\right)^{-1}\Big)^{1/2}\left(\bm{X}^{\mathsf{d}\top}\bm{X}^{\mathsf{d}}\right)^{-1}-\bm{Y}\bm{\Delta}_{\mathsf{balancing}}.\label{eq:Y-d-identity-1}
\end{align}
\end{subequations}Here, we denote
\begin{equation}
\bm{\Delta}_{\mathsf{balancing}}\triangleq\Big(\bm{I}_{r}+\frac{\lambda}{p}\left(\bm{X}^{\top}\bm{X}\right)^{-1}\Big)^{1/2}-\Big(\bm{I}_{r}+\frac{\lambda}{p}\left(\bm{Y}^{\top}\bm{Y}\right)^{-1}\Big)^{1/2},\label{eq:defn-balance-diff}
\end{equation}
which measures the imbalance between the low-rank factors $\bm{X}$ and
$\bm{Y}$.\end{lemma}\begin{proof}The claims follow from straightforward
algebraic manipulations; see Appendix~\ref{subsec:Proof-of-Lemma-decomposition-low-rank}.
\end{proof}

We make a few observations regarding Lemma~\ref{lemma:decomp-low-rank}.
Take the decomposition of $\bm{X}^{\mathsf{d}}$~(\ref{eq:X-d-identity})
as an example:
\begin{itemize}
\item First, the term $\bm{A}\bm{Y}^{\mathsf{d}}(\bm{Y}^{\mathsf{d}\top}\bm{Y}^{\mathsf{d}})^{-1}$
vanishes when we have full observations, i.e.~$p=1$. Second, the
terms involving $\nabla_{\bm{X}}f(\bm{X},\bm{Y})$ and $\bm{\Delta}_{\mathsf{balancing}}$
are both zero if $(\bm{X},\bm{Y})$ is an exact stationary point of
$f(\cdot,\cdot)$; to see this, it is not hard to verify that any
stationary point of $f(\cdot,\cdot)$ necessarily satisfies $\bm{X}^{\top}\bm{X}=\bm{Y}^{\top}\bm{Y}$,
which in turn implies $\bm{\Delta}_{\mathsf{balancing}}=\bm{0}$.
Consequently, the last three terms in~(\ref{eq:X-d-identity}) are
expected to be small when $p$ is sufficiently large and $(\bm{X},\bm{Y})$
is near a stationary point.
\item Turning to the first two terms in~(\ref{eq:X-d-identity}), we note
that the second term of~(\ref{eq:X-d-identity}) is close to $\bm{X}^{\star}$
(up to rotation) if $\bm{Y}^{\mathrm{d}}$ is a nearly accurate approximation
to $\bm{Y}^{\star}$. In comparison, the first term $\mathcal{P}_{\Omega}(\bm{E})\bm{Y}^{\mathsf{d}}(\bm{Y}^{\mathsf{d}\top}\bm{Y}^{\mathsf{d}})^{-1}/p$
has to do with a collection of Gaussian random variables, which accounts
for the main uncertainty term.
\end{itemize}
We shall make precise these arguments in subsequent subsections.

\subsection{Taking global rotation into account}

In order to invoke the decompositions of $\bm{X}^{\mathsf{d}}$ and $\bm{Y}^{\mathsf{d}}$
(cf.~Lemma~\ref{lemma:decomp-low-rank}) to characterize the estimation
errors, we still need to incorporate the (unrecoverable) rotation
matrix. From now on, we shall focus primarily on the factor $\bm{X}^{\mathsf{d}}$.
The claims on the other factor $\bm{Y}^{\mathsf{d}}$ can be easily
obtained via symmetry.

Denote
\begin{equation}
\overline{\bm{X}}^{\mathsf{d}}\triangleq\bm{X}^{\mathsf{d}}\bm{H}^{\mathsf{d}}\qquad\text{and}\qquad\overline{\bm{Y}}^{\mathsf{d}}\triangleq\bm{Y}^{\mathsf{d}}\bm{H}^{\mathsf{d}},\label{eq:defn-Xbar-d-Ybar-d}
\end{equation}
where we recall that $\bm{H}^{\mathsf{d}}$ is the rotation matrix
that best aligns $(\bm{X}^{\mathsf{d}},\bm{Y}^{\mathsf{d}})$ and
$(\bm{X}^{\star},\bm{Y}^{\star})$ (see~(\ref{eq:defn-H-d})). Substituting
the identity
\[
\bm{Y}^{\mathsf{d}}\left(\bm{Y}^{\mathsf{d}\top}\bm{Y}^{\mathsf{d}}\right)^{-1}\bm{H}^{\mathsf{d}}=\bm{Y}^{\mathsf{d}}\bm{H}^{\mathsf{d}}\left(\bm{H}^{\mathsf{d}\top}\bm{Y}^{\mathsf{d}\top}\bm{Y}^{\mathsf{d}}\bm{H}^{\mathsf{d}}\right)^{-1}=\overline{\bm{Y}}^{\mathsf{d}}\big(\overline{\bm{Y}}^{\mathsf{d}\top}\overline{\bm{Y}}^{\mathsf{d}}\big)^{-1}
\]
into the decomposition~(\ref{eq:X-d-identity}), we arrive at 
\begin{align}
\bm{X}^{\mathsf{d}}\bm{H}^{\mathsf{d}}-\bm{X}^{\star} & =\frac{1}{p}\mathcal{P}_{\Omega}\left(\bm{E}\right)\overline{\bm{Y}}^{\mathsf{d}}\big(\overline{\bm{Y}}^{\mathsf{d}\top}\overline{\bm{Y}}^{\mathsf{d}}\big)^{-1}+\bm{X}^{\star}\left[\bm{Y}^{\star\top}\overline{\bm{Y}}^{\mathsf{d}}\big(\overline{\bm{Y}}^{\mathsf{d}\top}\overline{\bm{Y}}^{\mathsf{d}}\big)^{-1}-\bm{I}_{r}\right]-\bm{A}\overline{\bm{Y}}^{\mathsf{d}}\big(\overline{\bm{Y}}^{\mathsf{d}\top}\overline{\bm{Y}}^{\mathsf{d}}\big)^{-1}\nonumber \\
 & \quad+\nabla_{\bm{X}}f\left(\bm{X},\bm{Y}\right)\Big(\bm{I}_{r}+\frac{\lambda}{p}\left(\bm{Y}^{\top}\bm{Y}\right)^{-1}\Big)^{1/2}\left(\bm{Y}^{\mathsf{d}\top}\bm{Y}^{\mathsf{d}}\right)^{-1}\bm{H}^{\mathsf{d}}+\bm{X}\bm{\Delta}_{\mathsf{balancing}}\bm{H}^{\mathsf{d}}\nonumber \\
 & =\frac{1}{p}\mathcal{P}_{\Omega}\left(\bm{E}\right)\bm{Y}^{\star}\left(\bm{Y}^{\star\top}\bm{Y}^{\star}\right)^{-1}+\bm{\Phi}_{\bm{X}}.\label{eq:X-d-H-d-X-star-decomposition}
\end{align}
Here, the term $\bm{\Phi}_{\bm{X}}\in\mathbb{R}^{n\times r}$ is defined to
be
\begin{align}
\bm{\Phi}_{\bm{X}} & \triangleq\underbrace{\frac{1}{p}\mathcal{P}_{\Omega}\left(\bm{E}\right)\left[\overline{\bm{Y}}^{\mathsf{d}}\big(\overline{\bm{Y}}^{\mathsf{d}\top}\overline{\bm{Y}}^{\mathsf{d}}\big)^{-1}-\bm{Y}^{\star}\left(\bm{Y}^{\star\top}\bm{Y}^{\star}\right)^{-1}\right]}_{:=\bm{\Phi}_{1}}+\underbrace{\vphantom{\frac{1}{p}\mathcal{P}_{\Omega}\left(\bm{E}\right)\left[\overline{\bm{Y}}^{\mathsf{d}}\big(\overline{\bm{Y}}^{\mathsf{d}\top}\overline{\bm{Y}}^{\mathsf{d}}\big)^{-1}-\bm{Y}^{\star}\left(\bm{Y}^{\star\top}\bm{Y}^{\star}\right)^{-1}\right]}\bm{X}^{\star}\left[\bm{Y}^{\star\top}\overline{\bm{Y}}^{\mathsf{d}}\big(\overline{\bm{Y}}^{\mathsf{d}\top}\overline{\bm{Y}}^{\mathsf{d}}\big)^{-1}-\bm{I}_{r}\right]}_{:=\bm{\Phi}_{2}}\nonumber \\
 & \quad\underbrace{-\vphantom{\Big(\bm{I}_{r}+\frac{\lambda}{p}\left(\bm{Y}^{\top}\bm{Y}\right)^{-1}\Big)^{1/2}}\bm{A}\overline{\bm{Y}}^{\mathsf{d}}\big(\overline{\bm{Y}}^{\mathsf{d}\top}\overline{\bm{Y}}^{\mathsf{d}}\big)^{-1}}_{:=\bm{\Phi}_{3}}+\underbrace{\nabla_{\bm{X}}f\left(\bm{X},\bm{Y}\right)\Big(\bm{I}_{r}+\frac{\lambda}{p}\left(\bm{Y}^{\top}\bm{Y}\right)^{-1}\Big)^{1/2}\left(\bm{Y}^{\mathsf{d}\top}\bm{Y}^{\mathsf{d}}\right)^{-1}\bm{H}^{\mathsf{d}}+\bm{X}\bm{\Delta}_{\mathsf{balancing}}\bm{H}^{\mathsf{d}}}_{:=\bm{{\Phi}}_{4}},\label{eq:defn-PhiX}
\end{align}
where $\bm{A}$ is defined in~(\ref{eq:defn-A}). To establish Theorem~\ref{thm:low-rank-factor-master-bound},
it remains to (1) demonstrate that $\bm{\Phi}_{\bm{X}}$ has small
$\ell_{2,\infty}$ norm, and (2) show that $\mathcal{P}_{\Omega}(\bm{E})\bm{Y}^{\star}(\bm{Y}^{\star\top}\bm{Y}^{\star})^{-1}/p$
is approximately a Gaussian random matrix. These two steps constitute
the main content of the next subsection.

\subsection{Key lemmas for establishing Theorem~\ref{thm:low-rank-factor-master-bound} \label{subsec:Key-lemmas}}

We now state five key lemmas. Taking these collectively and substituting
them into~(\ref{eq:X-d-H-d-X-star-decomposition}) immediately establish
Theorem~\ref{thm:low-rank-factor-master-bound}.

We shall start by controlling the term~$\bm{\Phi}_{1}$ as defined
in~(\ref{eq:defn-PhiX}).

\begin{lemma}[\textsf{Negligibility of }$\bm{\Phi}_{1}$] \label{lem:Phi1-two-infty}Suppose
that the sample complexity obeys $n^{2}p\geq C\kappa^{4}\mu^{2}r^{2}n\log^{3}n$
for some sufficiently large constant $C>0$ and the noise satisfies $\sigma\sqrt{(\kappa^{4}\mu rn\log n)/p}\leq c\sigma_{\min}$ for some sufficiently small constant $c<0$.
Then with probability at least $1-O(n^{-10})$, we have
\[
\left\Vert \bm{{\Phi}}_{1}\right\Vert _{2,\infty}\lesssim\frac{\sigma}{\sqrt{p\sigma_{\min}}}\cdot\frac{{\sigma}}{\sigma_{\min}}\sqrt{{\frac{{\kappa^{3}\mu rn\log n}}{p}}}.
\]

\end{lemma}

\begin{proof}Fix any $1\leq j\leq n$. If the de-shrunken estimate
$\overline{\bm{Y}}^{\mathsf{d}}$ were independent of the randomness
in the $j$th row of the matrix, i.e.~$\bm{e}_{j}^{\top}\mathcal{P}_{\Omega}(\bm{E})$,
then $\|\bm{e}_{j}^{\top}\bm{\Phi}_{1}\|_{2}$ would be well controlled.
This hypothesis is certainly false, as $\overline{\bm{Y}}^{\mathsf{d}}$
clearly depends on $\bm{e}_{j}^{\top}\mathcal{P}_{\Omega}(\bm{E})$.
Nevertheless, by exploiting the leave-one-out technique recently used
in \cite{el2013robust,el2015impact,abbe2017entrywise,ma2017implicit,chen2017spectral,chen2019noisy,chen2019nonconvex,ding2018leave},
one can properly decouple the dependency and establish the desired
bound. See Appendix~\ref{subsec:Proof-of-Lemma-Phi1}.\end{proof}

The next lemma controls the size of $\|\bm{\Phi}_{2}\|_{2,\infty}$.
In essence, the term $\bm{\Phi}_{2}$ measures the difference between
the estimate $\overline{\bm{Y}}^{\mathsf{d}}$ and the true signal
$\bm{Y}^{\star}$; the closer these two are, the smaller $\|\bm{\Phi}_{2}\|_{2,\infty}$
should be. See Appendix~\ref{subsec:Proof-of-Lemma-Phi2} for the
proof of the following result.

\begin{lemma}[\textsf{Negligibility of }$\bm{\Phi}_{2}$]\label{lem:Phi2-two-infty}Suppose
that the sample complexity obeys $n^{2}p\geq C\kappa^{4}\mu^{2}r^{2}n\log^{3}n$
for some sufficiently large constant $C>0$ and the noise satisfies $\sigma\sqrt{(\kappa^{4}\mu rn\log n)/p}\leq c\sigma_{\min}$ for some sufficiently small constant $c<0$.
Then with probability exceeding $1-O(n^{-10})$, one has
\[
\left\Vert \bm{\Phi}_{2}\right\Vert _{2,\infty}\lesssim\frac{\sigma}{\sqrt{p\sigma_{\min}}}\left(\kappa\frac{\sigma}{\sigma_{\min}}\sqrt{\frac{\kappa^{7}\mu rn}{p}}+\sqrt{\frac{\kappa^{7}\mu^{3}r^{3}\log n}{np}}\right).
\]

\end{lemma}

Moving on to $\bm{\Phi}_{3}$ and $\bm{\Phi}_{4}$, one has the following
lemmas.

\begin{lemma}[\textsf{Negligibility of }$\bm{\Phi}_{3}$]\label{lem:Phi3-two-infty}Suppose
that the sample complexity obeys $n^{2}p\geq C\kappa^{4}\mu^{2}r^{2}n\log^{3}n$
for some sufficiently large constant $C>0$ and the noise satisfies $\sigma\sqrt{(\kappa^{4}\mu rn\log n)/p}\leq c\sigma_{\min}$ for some sufficiently small constant $c<0$.
Then with probability exceeding $1-O(n^{-10})$, we have
\[
\left\Vert \bm{{\Phi}}_{3}\right\Vert _{2,\infty}\lesssim\frac{\sigma}{\sqrt{p\sigma_{\min}}}\sqrt{\frac{\kappa^{5}\mu^{3}r^{3}\log^{2}n}{np}}.
\]

\end{lemma}

\begin{proof}It is straightforward to check that when $p=1$, one
has $\|\bm{\Phi}_{3}\|_{2,\infty}=\|\bm{A}\|=0$, where we recall
the definition of $\bm{A}$ in~(\ref{eq:defn-A}). Therefore, one
expects $\|\bm{\Phi}_{3}\|_{2,\infty}$ to be small when $p$ is sufficiently
large. See Appendix~\ref{subsec:Proof-of-Lemma-Phi3}.\end{proof}

\begin{lemma}[\textsf{Negligibility of }$\bm{\Phi}_{4}$]\label{lem:Phi4-two-infty}Suppose
that the sample complexity obeys $n^{2}p\geq C\kappa^{4}\mu^{2}r^{2}n\log^{3}n$
for some sufficiently large constant $C>0$ and the noise satisfies $\sigma\sqrt{(\kappa^{4}\mu rn\log n)/p}\leq c\sigma_{\min}$ for some sufficiently small constant $c<0$.
Then with probability at least $1-O(n^{-10})$, one has
\[
\left\Vert \bm{{\Phi}}_{4}\right\Vert _{2,\infty}\lesssim\frac{\sigma}{\sqrt{p\sigma_{\min}}}\cdot\frac{1}{n^{4}}.
\]

\end{lemma}

\begin{proof}It is easily seen that the size of $\bm{\Phi}_{4}$
depends on how close $(\bm{X},\bm{Y})$ is to a stationary point of
$f(\cdot,\cdot)$. For instance, in the extreme case where $(\bm{X},\bm{Y})$
is an exact stationary point, then one would have $\bm{\Phi}_{4}=\bm{0}$.
See Appendix~\ref{subsec:Proof-of-Lemma-Phi4}.\end{proof}

The last lemma asserts that $\mathcal{P}_{\Omega}(\bm{E})\bm{Y}^{\star}(\bm{Y}^{\star\top}\bm{Y}^{\star})^{-1}/p$
is, in some sense, close to a zero-mean Gaussian random matrix with
the desired covariance.

\begin{lemma}[\textsf{Approximate Gaussianity of }$\mathcal{P}_{\Omega}(\bm{E})\bm{Y}^{\star}(\bm{Y}^{\star\top}\bm{Y}^{\star})^{-1}/p$]\label{lemma:approx-gaussian-leading-term}Suppose
that the sample size obeys $n^{2}p\geq C\kappa^{2}\mu rn\log^{3}n$ for some sufficiently large constant $C>0$.
Then one has the decomposition
\[
\frac{1}{p}\mathcal{P}_{\Omega}\left(\bm{E}\right)\bm{Y}^{\star}\left(\bm{Y}^{\star\top}\bm{Y}^{\star}\right)^{-1}=\bm{Z}_{\bm{X}}+\bm{\Delta}_{\bm{X}},
\]
where each row of $\bm{Z}_{\bm{X}}\in\mathbb{R}^{n\times r}$ is independent
and identically distributed according to
\[
\bm{Z}_{\bm{X}}^{\top}\bm{e}_{j}\overset{\mathrm{i.i.d}}{\sim}\mathcal{N}\Big(\bm{0},\frac{\sigma^{2}}{p}\left(\bm{\Sigma}^{\star}\right)^{-1}\Big),\qquad\mathrm{for}\quad1\leq j\leq n.
\]
In addition, with probability at least $1-O(n^{-10})$, the remaining
term $\bm{\Delta}_{\bm{X}}\in\mathbb{R}^{n\times r}$ obeys
\[
\left\Vert \bm{\Delta}_{\bm{X}}\right\Vert _{2,\infty}\lesssim\frac{\sigma}{\sqrt{p\sigma_{\min}}} \sqrt{\frac{\kappa^{2}\mu r^{2}\log^{2}n}{np}}.
\]
 \end{lemma}

\begin{proof}Fix any $1\leq j\leq n$. The $j$th row, namely $\bm{e}_{j}^{\top}[\mathcal{P}_{\Omega}(\bm{E})\bm{Y}^{\star}(\bm{Y}^{\star\top}\bm{Y}^{\star})^{-1}/p]$
is conditionally Gaussian in the sense that
\[
\bm{e}_{j}^{\top}\left[\frac{1}{p}\mathcal{P}_{\Omega}\left(\bm{E}\right)\bm{Y}^{\star}\left(\bm{Y}^{\star\top}\bm{Y}^{\star}\right)^{-1}\right]\,\Big|\,\Omega\ \sim\ \mathcal{N}\Big(\bm{0},\frac{\sigma^{2}}{p}\Big(\,\frac{1}{p}\sum_{k=1}^{n}\delta_{jk}\left(\bm{\Sigma}^{\star}\right)^{-1}\left(\bm{Y}_{k,\cdot}^{\star}\right)^{\top}\bm{Y}_{k,\cdot}^{\star}\left(\bm{\Sigma}^{\star}\right)^{-1}\Big)\Big),
\]
where we recall that $\delta_{jk}=\ind\{(j,k)\in \Omega\}$. 
Recognize that the conditional covariance matrix concentrates sharply
around its expectation, i.e.~$\sigma^{2}(\bm{\Sigma}^{\star})^{-1}/p$,
which is the covariance matrix of $\bm{Z}_{\bm{X}}^{\top}\bm{e}_{j}$
that we are after. Hence, one can expect that $\mathcal{P}_{\Omega}(\bm{E})\bm{Y}^{\star}(\bm{Y}^{\star\top}\bm{Y}^{\star})^{-1}/p$
is, marginally, not too far from a Gaussian random matrix. This argument can be carried
out formally; see Appendix~\ref{subsec:Proof-of-Lemma-approx-gaussian-low-rank}.
\end{proof}


\subsection{From low-rank factors to matrix entries (Proof of Theorem~\ref{thm:entries-master-decomposition-augment})}

We now turn attention to inference on the matrix entries, by establishing Theorem~\ref{thm:entries-master-decomposition-augment}. Towards this, we first make
the following observation: for any $1\leq i,j\leq n$, 
\begin{align}
M_{ij}^{\mathsf{{d}}}-M_{ij}^{\star} & =\bm{{e}}_{i}^{\top}\text{\ensuremath{\overline{{\bm{{X}}}}}}^{\mathsf{{d}}}\overline{{\bm{{Y}}}}^{\mathsf{{d}}\top}\bm{{e}}_{j}-\bm{{e}}_{i}^{\top}\bm{{X}}^{\star}\bm{{Y}}^{\star\top}\bm{{e}}_{j}\nonumber \\
 & =\bm{e}_{i}^{\top}\big(\overline{\bm{X}}^{\mathsf{d}}-\bm{X}^{\star}\big)\bm{Y}^{\star\top}\bm{e}_{j}+\bm{e}_{i}^{\top}\bm{X}^{\star}\big(\overline{\bm{Y}}^{\mathsf{d}}-\bm{Y}^{\star}\big)^{\top}\bm{e}_{j}+\bm{e}_{i}^{\top}\big(\overline{\bm{X}}^{\mathsf{d}}-\bm{X}^{\star}\big)\big(\overline{\bm{Y}}^{\mathsf{d}}-\bm{Y}^{\star}\big)^{\top}\bm{e}_{j}.\label{eq:entry-1st-decomp}
\end{align}
One can readily apply the decompositions in Theorem~\ref{thm:low-rank-factor-master-bound}
to obtain
\begin{align}
\bm{e}_{i}^{\top}\big(\overline{\bm{X}}^{\mathsf{d}}-\bm{X}^{\star}\big)\bm{Y}^{\star\top}\bm{e}_{j} & =\bm{e}_{i}^{\top}\bm{Z}_{\bm{X}}\bm{Y}^{\star\top}\bm{e}_{j}+\bm{e}_{i}^{\top}\bm{\Psi}_{\bm{X}}\bm{Y}^{\star\top}\bm{e}_{j},\label{eq:entry-2nd-decomp}\\
\bm{e}_{i}^{\top}\bm{X}^{\star}\big(\overline{\bm{Y}}^{\mathsf{d}}-\bm{Y}^{\star}\big)^{\top}\bm{e}_{j} & =\bm{e}_{i}^{\top}\bm{X}^{\star}\bm{Z}_{\bm{Y}}^{\top}\bm{e}_{j}+\bm{e}_{i}^{\top}\bm{X}^{\star}\bm{\Psi}_{\bm{Y}}^{\top}\bm{e}_{j}.
\end{align}
Take the preceding three identities collectively to reach
\begin{align*}
M_{ij}^{\mathsf{{d}}}-M_{ij}^{\star} & =\underbrace{\bm{e}_{i}^{\top}\bm{Z}_{\bm{X}}\bm{Y}^{\star\top}\bm{e}_{j}+\bm{e}_{i}^{\top}\bm{X}^{\star}\bm{Z}_{\bm{Y}}^{\top}\bm{e}_{j}}_{:=\Theta_{ij}}\\
 & \quad+\underbrace{\bm{e}_{i}^{\top}\bm{\Psi}_{\bm{X}}\bm{Y}^{\star\top}\bm{e}_{j}+\bm{e}_{i}^{\top}\bm{X}^{\star}\bm{\Psi}_{\bm{Y}}^{\top}\bm{e}_{j}+\bm{e}_{i}^{\top}\big(\overline{\bm{X}}^{\mathsf{d}}-\bm{X}^{\star}\big)\big(\overline{\bm{Y}}^{\mathsf{d}}-\bm{Y}^{\star}\big)^{\top}\bm{e}_{j}}_{:=\Lambda_{ij}}.
\end{align*}
Following the same route as in Section~\ref{subsec:Key-lemmas},
one can verify that $\Theta_{ij}=\bm{e}_{i}^{\top}\bm{Z}_{\bm{X}}\bm{Y}^{\star\top}\bm{e}_{j}+\bm{e}_{i}^{\top}\bm{X}^{\star}\bm{Z}_{\bm{Y}}^{\top}\bm{e}_{j}$
is approximately Gaussian, whereas the residual term $\Lambda_{ij}$
is small in magnitude. These claims are formally stated in the next
two lemmas, with the proofs deferred to Appendix~\ref{sec:Analysis-of-entry}.

\begin{lemma}[$\mathsf{Negligibility\;of}$ $\Lambda_{ij}$]\label{lemma:entry-residual}
Suppose
that the sample complexity obeys $n^{2}p\geq C\kappa^{4}\mu^{2}r^{2}n\log^{3}n$
for some sufficiently large constant $C>0$ and the noise satisfies $\sigma\sqrt{(\kappa^{4}\mu rn\log n)/p}\leq c\sigma_{\min}$ for some sufficiently small constant $c<0$.
Then with probability exceeding $1-O(n^{-10})$, one has
\[
\left|\Lambda_{ij}\right|\lesssim\left(\left\Vert \bm{U}_{i,\cdot}^{\star}\right\Vert _{2}+\left\Vert \bm{V}_{j,\cdot}^{\star}\right\Vert _{2}\right)\frac{\sigma}{\sqrt{p}}\left(\frac{\sigma}{\sigma_{\min}}\sqrt{\frac{\kappa^{8}\mu rn\log n}{p}}+\sqrt{\frac{\kappa^{8}\mu^{3}r^{3}\log^{2}n}{np}}\right)+\left(\frac{\sigma}{\sqrt{\sigma_{\min}}}\sqrt{\frac{\kappa^{3}\mu r\log n}{p}}\right)^{2}.
\]

\end{lemma}\begin{lemma}[$\mathsf{Approximate\;Gaussianity\;of}$
$\Theta_{ij}$] \label{lemma:entry-normal} Suppose that $np\geq C\kappa^{2}\mu r^{2}\log^{2}n$ for some sufficiently large constant $C>0$.
Then we have the decomposition
\[
\Theta_{ij}=\bm{e}_{i}^{\top}\bm{Z}_{\bm{X}}\bm{Y}^{\star\top}\bm{e}_{j}+\bm{e}_{i}^{\top}\bm{X}^{\star}\bm{Z}_{\bm{Y}}^{\top}\bm{e}_{j}=g_{ij}+\theta_{ij},
\]
where $g_{ij}\sim\mathcal{N}(0,v_{ij}^{\star})$ and the remaining
term $\theta_{ij}$ satisfies --- with probability exceeding $1-O(n^{-10})$
--- that
\[
\left|\theta_{ij}\right|\lesssim\frac{\sigma}{\sqrt{p}} \sqrt{\frac{\kappa^{2}\mu r\log n}{np}}\min\left\{ \left\Vert \bm{U}_{i,\cdot}^{\star}\right\Vert _{2},\left\Vert \bm{V}_{j,\cdot}^{\star}\right\Vert _{2}\right\} .
\]
\end{lemma}

Putting the above two lemmas together immediately establishes Theorem~\ref{thm:entries-master-decomposition-augment} and hence Theorem~\ref{thm:entries-master-decomposition}.

\section{Discussion}
\label{sec:discussion}

The present paper makes progress towards inference and uncertainty
quantification for noisy matrix completion, by developing simple de-biased
estimators that admit tractable and accurate distributional characterizations.
While we have achieved some early success in accomplishing this, our results
are likely sub-optimal in the following aspects:
\begin{itemize}
\item \emph{Dependency on the rank and the condition number. }To enable
valid inference, our sample complexity (cf.~(\ref{eq:requirement-low-rank})
and (\ref{eq:requirement-entry})) scales sub-optimally with the rank
$r$ and the condition number $\kappa$. The sub-optimality can be
understood through comparisons with the sample size requirement $O(nr\log^{2}n)$
in the noise-free settings, which is independent of $\kappa$ and
matches the information limit (up to some $\log$ factor). Improving
such dependency calls for more refined analysis techniques. 
\item \emph{Detection of the size of the entries. }On one hand, when the
size of the entry $M_{ij}^{\star}$ is moderately large (cf.~(\ref{eq:entry-inference-lower-bound})),
Corollary~\ref{coro:confidence-interval} allows us to construct
a valid confidence interval for it. On the other hand, when $\|\bm{U}_{i,\cdot}^{\star}\|_{2}+\|\bm{V}_{j,\cdot}^{\star}\|_{2}$
vanishes, Theorem~\ref{thm:low-rank-factor-master-bound} tells us
that the estimation error $M_{ij}^{\mathsf{d}}-M_{ij}^{\star}$ is
better approximated by the inner product of two independent Gaussian
random vectors. It remains to be seen how to determine whether $\|\bm{U}_{i,\cdot}^{\star}\|_{2}+\|\bm{V}_{j,\cdot}^{\star}\|_{2}$
is too small. 
\item \emph{Low signal-to-noise (SNR) regime. }Our theory operates under
	the moderate-to-high SNR regime, where $\sigma^2_{\min}/\sigma^2$ (which is proportional to the SNR)
is required to exceed the order of $n/p$; see the conditions
in Theorem~\ref{thm:low-rank-factor-master-bound}. It is unclear
whether the connection between the convex and the nonconvex estimators
hold in the low SNR regime. How to conduct inference in such a scenario
is an important future direction.
\end{itemize}
In addition, our investigation has been dedicated to a natural random
model, which by no means covers the most general settings of practical
interest. There are numerous possible extensions that merit future
investigation:
\begin{itemize}
\item \emph{Approximate low-rank structure}. Our current theory is built
upon the exact low-rank structure of $\bm{{M}}^{\star}$. Realistically,
the matrix of interest is often only \emph{approximately} low-rank.
It is of great interest to study how to carry out statistical inference
under such imperfect structural assumptions. 
\item \emph{More general sampling patterns. }This paper operates under the
uniform random sampling assumption, which might sometimes be off in
practical situations. It would be interesting to investigate whether
our results in this paper can extend to more general non-uniform sampling
patterns (e.g.~\cite{Negahban2012restricted}). 
\item \emph{Extensions to robust PCA, sparse PCA, and 1-bit matrix completion.
}A variety of important extensions of matrix completion have been
explored in prior literature, including but not limited to the case
with sparse outliers (i.e.~robust PCA~\cite{CanLiMaWri09,chandrasekaran2011rank}),
the case where the matrix of interest is simultaneously sparse and
low-rank (i.e.~sparse PCA~\cite{zou2006sparse,cai2013sparse}), and
the case where only finite-bit observations are available (i.e.~1-bit
matrix completion~\cite{davenport20141,cai2013max}). Performing valid
uncertainty assessment for these scenarios requires non-trivial extensions
of the link between convex and nonconvex optimization.
\item \emph{Other loss functions. }In the estimation stage, one might sometimes
prefer other loss functions beyond the penalized squared loss. This
might arise from either statistical considerations (e.g.~employing
a penalized Poisson log-likelihood to accommodate Poisson noise~\cite{cao2016poisson}),
or computational concerns (e.g.~adopting a non-smooth loss to improve
convergence~\cite{charisopoulos2019low}). It would be of fundamental
importance to develop a unified inferential framework that covers
a broader family of loss functions.
\end{itemize}

\section*{Acknowledgements}

Y.~Chen is supported in part by the AFOSR YIP award FA9550-19-1-0030,
by the ONR grant N00014-19-1-2120, by the ARO grant W911NF-18-1-0303, by the NSF grants CCF-1907661 and IIS-1900140, 
and by the Princeton SEAS innovation award. J.~Fan is supported in
part by NSF grants DMS-1662139 and DMS-1712591, ONR grant N00014-19-1-2120,
and NIH grant 2R01-GM072611-13. C.~Ma is supported in part by Hudson River Trading AI Labs (HAIL) Fellowship. This work was done in part while Y.~Chen
was visiting the Kavli Institute for Theoretical Physics (supported
in part by the NSF grant PHY-1748958). We thank Weijie Su for helpful
discussions.

\appendix

\section{Preliminaries \label{subsec:Preliminaries}}

In this section, we gather several notation and preliminary facts
that are useful throughout the analysis. All the proofs, if needed,
are deferred to Appendix~\ref{sec:Proofs-in-Section-prelim}.

\subsection{Algorithmic details of nonconvex optimization\label{subsec:Algorithmic-details-nonconvex}}

To begin with, we make precise the algorithm used to minimize the
nonconvex loss function~(\ref{eq:ncvx-new}). Specifically, we describe
the following details that are crucial for us to implement Algorithm~\ref{alg:gd-mc-ncvx}:
\begin{itemize}
\item Set the initial point to be $(\bm{X}^{0},\bm{Y}^{0})=(\bm{X}^{\star},\bm{Y}^{\star})$ or  the spectral initialization as in \cite{ma2017implicit,chen2019nonconvex};
\item Set the stepsize $\eta\asymp1/(n^{6}\kappa^{3}\sigma_{\max})$;
\item Set the maximum number of iterations to be $t_{0}\asymp n^{23}$;
\item The returned estimate is  $(\bm{X}^{\mathsf{ncvx}},\bm{Y}^{\mathsf{ncvx}})\triangleq(\bm{X}^{t_{\star}},\bm{Y}^{t_{\star}})$,
where 
\begin{equation}
t_{\star}\,\triangleq\,\min\,\left\{ 0\leq t\leq t_{0}\,\Big|\,\left\Vert \nabla f\left(\bm{X}^{t},\bm{Y}^{t}\right)\right\Vert _{\mathrm{F}}\leq\frac{1}{n^{5}}\frac{\lambda}{p}\sqrt{\sigma_{\min}}\right\} .\label{eq:stopping_criterion}
\end{equation}
In words, we run gradient descent in Algorithm~\ref{alg:gd-mc-ncvx}
until we reach a point whose gradient is exceedingly small.
\end{itemize}

\begin{remark}[Spectral initialization]
Many of the preliminary facts below were established for the case $(\bm{X}^{0},\bm{Y}^{0})=(\bm{X}^{\star},\bm{Y}^{\star})$ \cite{chen2019noisy}, which is certainly 
not implementable
in practice, however, it serves as a good proxy for studying
the convex estimator. Fortunately, the same theoretical guarantees
stated in Appendix~\ref{subsec:Properties-of-approximate-solution}
can be readily established for spectral initialization using almost the same arguments adopted in \cite{ma2017implicit,chen2019nonconvex,chen2019noisy}. We omit this part mainly for the sake of brevity.  \end{remark}

To facilitate analysis, we introduce a set of auxiliary nonconvex
loss functions. For any $1\leq j\leq n$, define 
\begin{equation}
f^{(j)}\left(\bm{X},\bm{Y}\right)\triangleq\frac{1}{2p}\left\Vert \mathcal{P}_{\Omega_{-j},\cdot}\left(\bm{X}\bm{Y}^{\top}-\bm{M}\right)\right\Vert _{\mathrm{F}}^{2}+\frac{1}{2}\left\Vert \mathcal{P}_{j,\cdot}\left(\bm{X}\bm{Y}^{\top}-\bm{M}\right)\right\Vert _{\mathrm{F}}^{2}+\frac{\lambda}{2p}\left\Vert \bm{X}\right\Vert _{\mathrm{F}}^{2}+\frac{\lambda}{2p}\left\Vert \bm{Y}\right\Vert _{\mathrm{F}}^{2},\label{eq:ncvx-loo}
\end{equation}
where $\mathcal{P}_{\Omega_{-j,\cdot}}:\mathbb{R}^{n\times n}\mapsto\mathbb{R}^{n\times n}$
(resp.~$\mathcal{P}_{j,\cdot}(\cdot)$) denotes the orthogonal projection
onto the subspace of matrices that vanish outside of $\{(i,k)\in\Omega\,|\,i\neq j\}$
(resp.~$\{(i,k)\,|\,i=j\}$). Let 
\begin{equation}
	\label{eq:nonconvex-LOO-estimate}
	(\bm{X}^{(j)},\bm{Y}^{(j)})=(\bm{X}^{t_{\star},(j)},\bm{Y}^{t_{\star},(j)}) 
\end{equation}
be the nonconvex estimate returned by this auxiliary algorithm
(i.e.~Algorithm~\ref{alg:gd-mc-ncvx-loo}), which 
serves as an approximate solution to~(\ref{eq:ncvx-loo}).

\begin{algorithm}
\caption{Gradient descent for solving the auxiliary nonconvex problem (\ref{eq:ncvx-loo}).}

\label{alg:gd-mc-ncvx-loo}\begin{algorithmic}

\STATE \textbf{{Suitable initialization}}: $(\bm{X}^{0,(j)},\bm{Y}^{0,(j)})=(\bm{X}^{\star},\bm{Y}^{\star})$

\STATE \textbf{{Gradient updates}}: \textbf{for }$t=0,1,\ldots,t_{\star}-1$
\textbf{do}

\STATE \vspace{-1em}
 \begin{subequations}\label{subeq:gradient_update_ncvx-loo}
\begin{align}
\bm{X}^{t+1,(j)}= & \bm{X}^{t,(j)}-\eta\nabla_{\bm{X}}f^{(j)}\big(\bm{X}^{t,(j)},\bm{Y}^{t,(j)}\big);\label{eq:gradient_update_ncvx_X-loo}\\
\bm{Y}^{t+1,(j)}= & \bm{Y}^{t,(j)}-\eta\nabla_{\bm{Y}}f^{(j)}\big(\bm{X}^{t,(j)},\bm{Y}^{t,(j)}\big).\label{eq:gradient_update_ncvx_Y-loo}
\end{align}
\end{subequations}

\end{algorithmic}
\end{algorithm}

\subsection{Properties of approximate nonconvex solutions \label{subsec:Properties-of-approximate-solution}}

This subsection gathers the properties of the (approximate) nonconvex
solutions. Throughout this subsection, we use the shorthand
\begin{equation}
	(\bm{X},\bm{Y}) = (\bm{X}^{\mathsf{ncvx}},\bm{Y}^{\mathsf{ncvx}}) 
\end{equation}
and recall the definition of $(\bm{X}^{(j)}, \bm{Y}^{(j)})$ in \eqref{eq:nonconvex-LOO-estimate}. 
The regularization parameter is chosen to satisfy
\begin{equation}\label{eq:lambda-condition}
\lambda\asymp\sigma\sqrt{np}.
\end{equation}
To further simplify the presentation, we introduce $\bm{F}^{\star}$, $\bm{F}$,
$\bm{F}^{\mathrm{d}}$, $\bm{F}^{\mathrm{d},(j)}\in\mathbb{R}^{2n\times r}$
as follows
\begin{equation}
\bm{F}^{\star}\triangleq\left[\begin{array}{c}
\bm{X}^{\star}\\
\bm{Y}^{\star}
\end{array}\right];\quad\bm{F}\triangleq\left[\begin{array}{c}
\bm{X}\\
\bm{Y}
\end{array}\right];\quad\bm{F}^{\mathsf{d}}\triangleq\left[\begin{array}{c}
\bm{X}^{\mathsf{d}}\\
\bm{Y}^{\mathsf{d}}
\end{array}\right];\quad\bm{F}^{\mathsf{d},(j)}\triangleq\left[\begin{array}{c}
\bm{X}^{\mathsf{d},(j)}\\
\bm{Y}^{\mathsf{d},(j)}
\end{array}\right],\label{eq:defn-F}
\end{equation}
and define \begin{subequations}
\begin{align}
\bm{H} & \triangleq\arg\min_{\bm{R}\in\mathcal{O}^{r\times r}}\left\Vert \bm{F}\bm{R}-\bm{F}^{\star}\right\Vert _{\mathrm{F}}^{2}=\arg\min_{\bm{R}\in\mathcal{O}^{r\times r}}\left\{ \left\Vert \bm{X}\bm{R}-\bm{X}^{\star}\right\Vert _{\mathrm{F}}^{2}+\left\Vert \bm{Y}\bm{R}-\bm{Y}^{\star}\right\Vert _{\mathrm{F}}^{2}\right\} ,\label{eq:defn-H}\\
\bm{H}^{(j)} & \triangleq\arg\min_{\bm{R}\in\mathcal{O}^{r\times r}}\bigl\Vert\bm{F}^{(j)}\bm{R}-\bm{F}^{\star}\bigr\Vert_{\mathrm{F}}^{2}=\arg\min_{\bm{R}\in\mathcal{O}^{r\times r}}\left\{ \bigl\Vert\bm{X}^{(j)}\bm{R}-\bm{X}^{\star}\bigr\Vert_{\mathrm{F}}^{2}+\bigl\Vert\bm{Y}^{(j)}\bm{R}-\bm{Y}^{\star}\bigr\Vert_{\mathrm{F}}^{2}\right\} ,\\
\bm{R}^{(j)} & \triangleq\arg\min_{\bm{R}\in\mathcal{O}^{r\times r}}\bigl\Vert\bm{F}^{(j)}\bm{R}-\bm{F}\bm{H}\bigr\Vert_{\mathrm{F}}^{2}=\arg\min_{\bm{R}\in\mathcal{O}^{r\times r}}\left\{ \bigl\Vert\bm{X}^{(j)}\bm{R}-\bm{X}\bm{H}\bigr\Vert_{\mathrm{F}}^{2}+\bigl\Vert\bm{Y}^{(j)}\bm{R}-\bm{Y}\bm{H}\bigr\Vert_{\mathrm{F}}^{2}\right\} ,\\
\bm{H}^{\mathsf{d},(j)} & \triangleq\arg\min_{\bm{R}\in\mathcal{O}^{r\times r}}\bigl\Vert\bm{F}^{\mathsf{d},(j)}\bm{R}-\bm{F}^{\star}\bigr\Vert_{\mathrm{F}}^{2}=\arg\min_{\bm{R}\in\mathcal{O}^{r\times r}}\left\{ \bigl\Vert\bm{X}^{\mathsf{d},(j)}\bm{R}-\bm{X}^{\star}\bigr\Vert_{\mathrm{F}}^{2}+\bigl\Vert\bm{Y}^{\mathsf{d},(j)}\bm{R}-\bm{Y}^{\star}\bigr\Vert_{\mathrm{F}}^{2}\right\} .
\end{align}
\end{subequations}
The claims stated below hold under the sample complexity and the noise condition presumed in \cite[Theorem 1]{chen2019noisy} (see also Theorem~\ref{thm:low-rank-factor-master-bound} in the current manuscript)
\[
n^2 p\gg \kappa^4\mu^2r^2n\log^3 n \quad \text{and} \quad \sigma\sqrt{\frac{n}{p}}\ll\frac{\sigma_{\min}}{\sqrt{\kappa^4\mu r\log n}}.
\]
\begin{enumerate}
\item The first set of facts is related to $(\bm{X},\bm{Y})$. In view of
\cite{chen2019noisy}, $\bm{F}$ is a faithful estimate\footnote{Technically,
the statements in \cite[Lemma 5]{chen2019noisy} are for $\eta\asymp1/(n\kappa^{3}\sigma_{\max})$
and $t_{0}\asymp n^{18}$. Nevertheless, inspecting their proofs reveals
that the claims continue to hold for our choices $\eta\asymp1/(n^{6}\kappa^{3}\sigma_{\max})$
and $t_{0}\asymp n^{23}$. } of $\bm{F}^{\star}$, in the sense that\begin{subequations}\label{subeq:F-quality}
\begin{align}
\left\Vert \bm{F}\bm{H}-\bm{F}^{\star}\right\Vert _{\mathrm{F}} & \lesssim\frac{\sigma}{\sigma_{\min}}\sqrt{\frac{n}{p}}\left\Vert \bm{X}^{\star}\right\Vert _{\mathrm{F}},\label{eq:F-fro-quality}\\
\left\Vert \bm{F}\bm{H}-\bm{F}^{\star}\right\Vert  & \lesssim\frac{\sigma}{\sigma_{\min}}\sqrt{\frac{n}{p}}\left\Vert \bm{X}^{\star}\right\Vert ,\label{eq:F-op-quality}\\
\left\Vert \bm{F}\bm{H}-\bm{F}^{\star}\right\Vert _{\mathrm{2,\infty}} & \lesssim\kappa\frac{\sigma}{\sigma_{\min}}\sqrt{\frac{n\log n}{p}}\left\Vert \bm{F}^{\star}\right\Vert _{2,\infty}\label{eq:F-2-inf-quality}
\end{align}
\end{subequations}hold with probability exceeding $1-O(n^{-10})$.
In addition, on the same high-probability event, one has
\begin{align}
\left\Vert \nabla f\left(\bm{X},\bm{Y}\right)\right\Vert _{\mathrm{F}} & \leq\frac{1}{n^{5}}\frac{\lambda}{p}\sqrt{\sigma_{\min}};\label{eq:small-gradient}\\
\left\Vert \bm{X}^{\top}\bm{X}-\bm{Y}^{\top}\bm{Y}\right\Vert _{\mathrm{F}} & \leq\frac{1}{n^{5}}\frac{\sigma}{\sigma_{\min}}\sqrt{\frac{n}{p}}\sigma_{\max}\leq\frac{1}{n^{5}}\sigma_{\max};\label{eq:X-Y-balance}\\
\max\left\{ \left\Vert \bm{Z}^{\mathsf{cvx}}-\bm{X}\bm{Y}^{\top}\right\Vert _{\mathrm{F}},\left\Vert \bm{Z}^{\mathsf{cvx},r}-\bm{X}\bm{Y}^{\top}\right\Vert _{\mathrm{F}}\right\}  & \lesssim\frac{\kappa^{2}}{n^{5}}\frac{\lambda}{p}.\label{eq:closedness-cvx-ncvx}
\end{align}
In words, the first claim ensures that $(\bm{X},\bm{Y})$ is an approximate
stationary point of $f(\cdot,\cdot)$; the second bound tells us that
$(\bm{X},\bm{Y})$ is nearly \emph{balanced}, in the sense that $\bm{X}^{\top}\bm{X}\approx\bm{Y}^{\top}\bm{Y}$;
the last one formalizes the proximity between the convex solution
and the nonconvex one; see also~(\ref{eq:proximity-cvx-ncvx}).
\item We move on to the properties of the de-shrunken estimator $(\bm{X}^{\mathsf{d}},\bm{Y}^{\mathrm{d}})$, which is defined in~(\ref{eq:defn-Xd-Yd}).
Specifically, we can show that (see~Appendix~\ref{sec:Proofs-in-Section-prelim})
\begin{subequations}\label{subeq:F-d-property}
\begin{align}
\left\Vert \bm{F}^{\mathsf{d}}\bm{H}-\bm{F}^{\star}\right\Vert  & \lesssim\frac{\sigma}{\sigma_{\min}}\sqrt{\frac{n}{p}}\left\Vert \bm{X}^{\star}\right\Vert, \label{eq:F-d-op-quality-H}\\
\left\Vert \bm{F}^{\mathsf{d}}\bm{H}^{\mathsf{d}}-\bm{F}^{\star}\right\Vert  & \lesssim\kappa\frac{\sigma}{\sigma_{\min}}\sqrt{\frac{n}{p}}\left\Vert \bm{X}^{\star}\right\Vert, \label{eq:F-d-op-quality-H-d}\\
\left\Vert \bm{F}^{\mathsf{d}}\bm{H}^{\mathsf{d}}-\bm{F}^{\star}\right\Vert_{\mathrm{F}}  & \lesssim\frac{\sigma}{\sigma_{\min}}\sqrt{\frac{n}{p}}\left\Vert \bm{X}^{\star}\right\Vert_{\mathrm{F}}, \label{eq:F-d-fro-quality-H-d}\\
\left\Vert \bm{F}^{\mathsf{d}}\bm{H}^{\mathsf{d}}-\bm{F}^{\star}\right\Vert _{2,\infty} & \lesssim\kappa\frac{\sigma}{\sigma_{\min}}\sqrt{\frac{n\log n}{p}}\left\Vert \bm{F}^{\star}\right\Vert _{2,\infty},\label{eq:F-d-2-infty}\\
\left\Vert \bm{X}^{\mathsf{d}\top}\bm{X}^{\mathsf{d}}-\bm{Y}^{\mathsf{d}\top}\bm{Y}^{\mathsf{d}}\right\Vert  & \lesssim\frac{\kappa}{n^{5}}\frac{\sigma}{\sigma_{\min}}\sqrt{\frac{n}{p}}\sigma_{\max}\label{eq:X-d-Y-d-balance}
\end{align}
\end{subequations}hold with probability at least $1-O(n^{-10})$.
\item As has been shown in \cite{chen2019noisy}, the leave-one-out auxiliary
point $(\bm{X}^{(j)},\bm{Y}^{(j)})$ satisfies \begin{subequations}\label{subeq:F-j-properties}
\begin{align}
\bigl\Vert\bm{F}^{(j)}\bm{R}^{(j)}-\bm{F}\bm{H}\bigr\Vert_{\mathrm{F}} & \lesssim\frac{\sigma}{\sigma_{\min}}\sqrt{\frac{n\log n}{p}}\left\Vert \bm{F}^{\star}\right\Vert _{2,\infty},\label{eq:F-j-F-best-rotate}\\
\bigl\Vert\bm{F}^{(j)}\bm{H}^{(j)}-\bm{F}\bm{H}\bigr\Vert_{\mathrm{F}} & \lesssim\kappa\frac{\sigma}{\sigma_{\min}}\sqrt{\frac{n\log n}{p}}\left\Vert \bm{F}^{\star}\right\Vert _{2,\infty},\label{eq:F-j-F-dist}\\
\bigl\Vert\bm{F}^{(j)}\bm{H}^{(j)}-\bm{F}^{\star}\bigr\Vert & \lesssim\frac{\sigma}{\sigma_{\min}}\sqrt{\frac{n}{p}}\left\Vert \bm{X}^{\star}\right\Vert, \label{eq:F-j-F-star-dist}\\
\bigl\Vert\bm{{F}}^{(j)}\bm{{R}}^{(j)}-\bm{{F}}^{\star}\bigr\Vert_{2,\infty} & \lesssim\kappa\frac{\sigma}{\sigma_{\min}}\sqrt{\frac{n\log n}{p}}\left\Vert \bm{{F}}^{\star}\right\Vert _{2,\infty}\label{eq:F-j-F-star-2-infty-dist}
\end{align}
\end{subequations}with probability exceeding $1-O(n^{-10})$.
\item Parallel to the transition from $(\bm{X},\bm{Y})$ to $(\bm{X}^{\mathsf{d}},\bm{Y}^{\mathsf{d}})$,
we set 
\begin{equation}
\bm{X}^{\mathrm{d},(j)}\triangleq\bm{X}^{(j)}\Big(\bm{I}_{r}+\frac{\lambda}{p}\big(\bm{X}^{(j)\top}\bm{X}^{(j)}\big)^{-1}\Big)^{1/2}\quad\text{and}\quad\bm{Y}^{\mathrm{d},(j)}\triangleq\bm{Y}^{(j)}\Big(\bm{I}_{r}+\frac{\lambda}{p}\big(\bm{Y}^{(j)\top}\bm{Y}^{(j)}\big)^{-1}\Big)^{1/2}\label{eq:defn-xd-1}
\end{equation}
to be the de-shrunken estimators of $\bm{X}^{(j)}$ and $\bm{Y}^{(j)}$,
respectively. We shall demonstrate in Appendix~\ref{sec:Proofs-in-Section-prelim}
that, with probability at least $1-O(n^{-10})$, \begin{subequations}\label{subeq:F-d-j-property}
\begin{align}
\bigl\Vert\bm{F}^{\mathsf{d},(j)}\bm{H}^{\mathsf{d},(j)}-\bm{F}^{\star}\bigr\Vert & \lesssim\kappa\frac{\sigma}{\sigma_{\min}}\sqrt{\frac{n}{p}}\left\Vert \bm{X}^{\star}\right\Vert ,\label{eq:F-d-j-op}\\
\bigl\Vert\bm{F}^{\mathsf{d},(j)}\bm{H}^{\mathsf{d},(j)}-\bm{F}^{\star}\bigr\Vert_{2,\infty} & \lesssim\kappa\frac{\sigma}{\sigma_{\min}}\sqrt{\frac{n\log n}{p}}\left\Vert \bm{F}^{\star}\right\Vert _{2,\infty},\label{eq:F-d-j-2-infty}\\
\bigl\Vert\bm{F}^{\mathsf{d},(j)}\bm{H}^{\mathsf{d},(j)}-\bm{F}^{\mathsf{d}}\bm{H}^{\mathsf{d}}\bigr\Vert & \lesssim\kappa\frac{\sigma}{\sigma_{\min}}\sqrt{\frac{n\log n}{p}}\left\Vert \bm{F}^{\star}\right\Vert _{2,\infty}.\label{eq:F-d-j-F-d-dist}
\end{align}
\end{subequations}
\end{enumerate}
In addition to these four sets of claims, we have the following immediate
consequence of the incoherence condition~(\ref{eq:incoherence-assumption-on-U})
\begin{equation}
\left\Vert \bm{F}^{\star}\right\Vert _{2,\infty}=\max\bigl\{\left\Vert \bm{X}^{\star}\right\Vert _{2,\infty},\left\Vert \bm{Y}^{\star}\right\Vert _{2,\infty}\bigr\}\leq\sqrt{\mu r\sigma_{\max}/n}.\label{eq:incoherence-X}
\end{equation}

Moreover, recall that $\bm{A}=(1/p)\cdot\mathcal{P}_{\Omega}\left(\bm{X}\bm{Y}^{\top}-\bm{X}^{\star}\bm{Y}^{\star\top}\right)-\left(\bm{X}\bm{Y}^{\top}-\bm{X}^{\star}\bm{Y}^{\star\top}\right)$ (cf.~(\ref{eq:defn-A})). We obtain from the proof of \cite[Lemma 8]{chen2019noisy} that 
\begin{equation}\label{eq:A-norm}
\left\Vert \bm{A}\right\Vert \lesssim\sigma\sqrt{\frac{n}{p}}\cdot\sqrt{\frac{\kappa^{4}\mu^{2}r^{2}\log n}{np}}.
\end{equation}

Last but not least, we list a few simple but useful results:
the nonconvex solution $\bm{F}$ satisfies
\begin{equation}\label{eq:F-norm-upper-bound}
\sigma_{r}(\bm{F}) \geq 0.5\sqrt{\sigma_{\min}},\quad\left\Vert \bm{{F}}\right\Vert \leq2\left\Vert \bm{{X}}^{\star}\right\Vert ,\quad\left\Vert \bm{{F}}\right\Vert _{\mathrm{{F}}}\leq2\left\Vert \bm{{X}}^{\star}\right\Vert_{\mathrm{F}} ,\quad\left\Vert \bm{{F}}\right\Vert _{2,\infty}\leq2\left\Vert \bm{{F}}^{\star}\right\Vert _{2,\infty}.
\end{equation}
The same holds true if we replace $\bm{F}$ by either $\bm{F}^{\mathsf{d}}$, $\bm{F}^{(j)}$ $\bm{F}^{\mathsf{d},(j)}$ or their corresponding low-rank factors. Here $j$ can vary from 1 to $n$.

\section{Summary of the proposed estimators}
\label{sec:summary-estimators}


Let $\bm{Z}^{\mathsf{cvx}}$ be the minimizer of the convex program~\eqref{eq:cvx}, and let $(\bm{X}^{\mathsf{ncvx}},\bm{Y}^{\mathsf{ncvx}})$
be the solution returned by the Algorithm~\ref{alg:gd-mc-ncvx} (with
algorithmic details specified in Appendix~\ref{subsec:Algorithmic-details-nonconvex}).
Recall that $\bm{Z}^{\mathsf{cvx},r}$ is the best rank-$r$ approximation
of $\bm{Z}^{\mathsf{cvx}}$, viz.
\[
\bm{Z}^{\mathsf{cvx},r}=\underset{\bm{B}:\,\text{rank}(\bm{B})\leq r}{\arg\min}\|\bm{B}-\bm{Z}^{\mathsf{cvx}}\|_{\mathrm{F}}.
\]
In addition, we let the matrix estimate obtained by the nonconvex
algorithm be $\bm{Z}^{\mathsf{ncvx}}\triangleq\bm{X}^{\mathsf{ncvx}}\bm{Y}^{\mathsf{ncvx}\top}$.
We further denote by $(\bm{X}^{\mathsf{cvx}},\bm{Y}^{\mathsf{cvx}})$
the estimate of low-rank factors obtained by convex relaxation; more
specifically, we set $(\bm{X}^{\mathsf{cvx}},\bm{Y}^{\mathsf{cvx}})$
to be the balanced rank-$r$ factorization of $\bm{Z}^{\mathsf{cvx},r}$
obeying $\bm{X}^{\mathsf{cvx}}\bm{Y}^{\mathsf{cvx}\top}=\bm{Z}^{\mathsf{cvx},r}$
and $\bm{X}^{\mathsf{cvx}\top}\bm{X}^{\mathsf{cvx}}=\bm{Y}^{\mathsf{cvx}\top}\bm{Y}^{\mathsf{cvx}}$.
With these notations in place, our de-biased and de-shrunken estimators
can be summarized as follows.

\begin{itemize}
\item De-biased matrix estimators:
\begin{subequations}
\begin{align}
	\bm{M}^{\mathsf{cvx},\mathrm{d}} & \triangleq\mathcal{P}_{\text{rank-}r}\Big[\bm{Z}^{\mathsf{cvx}}-\frac{1}{p}\mathcal{P}_{\Omega}\big(\bm{Z}^{\mathsf{cvx}}-\bm{M}\big)\Big], \label{eq:M-cvx-d-formula}\\
\bm{M}^{\mathsf{ncvx},\mathrm{d}} & \triangleq\mathcal{P}_{\text{rank-}r}\Big[\bm{X}^{\mathsf{ncvx}}\bm{Y}^{\mathsf{ncvx}\top}-\frac{1}{p}\mathcal{P}_{\Omega}\big(\bm{X}^{\mathsf{ncvx}}\bm{Y}^{\mathsf{ncvx}\top}-\bm{M}\big)\Big].
\end{align}
\end{subequations}
\item De-shrunken estimators for low-rank factors:
\begin{subequations}
\begin{align}
	\bm{X}^{\mathsf{ncvx},\mathrm{d}} & \triangleq\bm{X}^{\mathsf{ncvx}}\Big(\bm{I}_{r}+\frac{\lambda}{p}\big(\bm{X}^{\mathsf{ncvx}\top}\bm{X}^{\mathsf{ncvx}}\big)^{-1}\Big)^{1/2}, \label{eq:X-ncvx-d-defn}\\
\bm{Y}^{\mathsf{ncvx},\mathrm{d}} & \triangleq\bm{Y}^{\mathsf{ncvx}}\Big(\bm{I}_{r}+\frac{\lambda}{p}\big(\bm{Y}^{\mathsf{ncvx}\top}\bm{Y}^{\mathsf{ncvx}}\big)^{-1}\Big)^{1/2}, \label{eq:Y-ncvx-d-defn}\\
\bm{X}^{\mathsf{cvx},\mathrm{d}} & \triangleq\bm{X}^{\mathsf{cvx}}\Big(\bm{I}_{r}+\frac{\lambda}{p}\big(\bm{X}^{\mathsf{cvx}\top}\bm{X}^{\mathsf{cvx}}\big)^{-1}\Big)^{1/2},\\
\bm{Y}^{\mathsf{cvx},\mathrm{d}} & \triangleq\bm{Y}^{\mathsf{cvx}}\Big(\bm{I}_{r}+\frac{\lambda}{p}\big(\bm{Y}^{\mathsf{cvx}\top}\bm{Y}^{\mathsf{cvx}}\big)^{-1}\Big)^{1/2}.
\end{align}
\end{subequations}
\end{itemize}

\section{Proof of Lemma \ref{lemma:connection-de-bias-shrunken} \label{sec:Proof-of-Lemma-connection}}

Throughout this section, let $\bm{U}\bm{\Sigma}\bm{V}^{\top}$ be
the rank-$r$ SVD of the nonconvex estimate $\bm{X}^{\mathsf{ncvx}}\bm{Y}^{\mathsf{ncvx}\top}$
and $T$ the tangent space of the set of rank-$r$ matrices at $\bm{X}^{\mathsf{ncvx}}\bm{Y}^{\mathsf{ncvx}\top}$.
Correspondingly, we denote by $\mathcal{P}_{T}$ the projection operator
onto the tangent space $T$, and let $\mathcal{P}_{T^{\perp}} = \mathcal{I} - \mathcal{P}_{T}$, where $\mathcal{I}$ is the identity operator.

\subsection{Proof of the inequality~(\ref{eq:matrix-equivalence})}
\label{sec:Proof-of-Lemma-connection-1}

In essence, we intend
to justify that $\bm{M}^{\mathsf{cvx,d}}$, $\bm{M}^{\mathsf{ncvx,d}}$
and $\bm{X}^{\mathsf{ncvx,d}}\bm{Y}^{\mathsf{ncvx,d}\top}$ are all
very close to $\bm{U}(\bm{\Sigma}+\frac{\lambda}{p}\bm{I}_{r})\bm{V}^{\top}$.

Recall from the definition of the de-biased estimator $\bm{M}^{\mathsf{cvx,d}}$
(cf.~(\ref{eq:M-cvx-d-formula})) that 
\begin{equation}
\bm{M}^{\mathsf{cvx,d}}=\mathcal{P}_{\text{rank-}r}\Big[\bm{Z}^{\mathsf{cvx}}-\frac{1}{p}\mathcal{P}_{\Omega}\left(\bm{Z}^{\mathsf{cvx}}-\bm{M}\right)\Big].\label{eq:connection-1}
\end{equation}
Replacing $\bm{Z}^{\mathsf{cvx}}$ by $\bm{X}^{\mathsf{ncvx}}\bm{Y}^{\mathsf{ncvx}\top}$
results in 
\begin{equation}
\bm{Z}^{\mathsf{cvx}}-\frac{1}{p}\mathcal{P}_{\Omega}\left(\bm{Z}^{\mathsf{cvx}}-\bm{M}\right)=\bm{X}^{\mathsf{ncvx}}\bm{Y}^{\mathsf{ncvx}\top}-\frac{1}{p}\mathcal{P}_{\Omega}\left(\bm{X}^{\mathsf{ncvx}}\bm{Y}^{\mathsf{ncvx}\top}-\bm{M}\right)+\bm{\Delta}_{\bm{Z}},
	\label{eq:connection-2}
\end{equation}
where we denote 
\[
\bm{\Delta}_{\bm{Z}}\triangleq\left(\bm{Z}^{\mathsf{cvx}}-\bm{X}^{\mathsf{ncvx}}\bm{Y}^{\mathsf{ncvx}\top}\right)+\frac{1}{p}\mathcal{P}_{\Omega}\left(\bm{X}^{\mathsf{ncvx}}\bm{Y}^{\mathsf{ncvx}\top}-\bm{Z}^{\mathsf{cvx}}\right).
\]
Apply the proximity bound~(\ref{eq:closedness-cvx-ncvx}) to obtain (recall that in~(\ref{eq:closedness-cvx-ncvx}), one has $(\bm{X},\bm{Y})=(\bm{X}^{\mathsf{ncvx}},\bm{Y}^{\mathsf{ncvx}})$)
\begin{align}
\left\Vert \bm{\Delta}_{\bm{Z}}\right\Vert _{\mathrm{{F}}} & \leq\left\Vert \bm{Z}^{\mathsf{cvx}}-\bm{X}^{\mathsf{ncvx}}\bm{Y}^{\mathsf{ncvx}\top}\right\Vert _{\mathrm{F}}+\frac{1}{p}\left\Vert \bm{Z}^{\mathsf{cvx}}-\bm{X}^{\mathsf{ncvx}}\bm{Y}^{\mathsf{ncvx}\top}\right\Vert _{\mathrm{F}}\nonumber \\
 & \leq\frac{2}{p}\left\Vert \bm{Z}^{\mathsf{cvx}}-\bm{X}^{\mathsf{ncvx}}\bm{Y}^{\mathsf{ncvx}\top}\right\Vert _{\mathrm{F}}\lesssim\frac{\kappa^{2}}{n^{5}p}\frac{{\lambda}}{p}\leq \frac{\lambda}{8p},\label{eq:Delta-Z-bound}
\end{align}
as long as $n^{5}p\gg\kappa^{2}$. In addition, in view of \cite[Claim 2]{chen2019noisy},
one has the decomposition 
\begin{equation}
	\mathcal{P}_{\Omega}\left(\bm{X}^{\mathsf{ncvx}}\bm{Y}^{\mathsf{ncvx}\top}-\bm{M}\right)=-\lambda\bm{U}\bm{V}^{\top}+\bm{R},
	\label{eq:connection-3}
\end{equation}
where $\bm{R}\in\mathbb{R}^{n\times n}$ is a residual matrix obeying
\begin{equation}
\left\Vert \mathcal{P}_{T}\left(\bm{R}\right)\right\Vert _{\mathrm{F}}\lesssim\kappa\frac{p}{\sqrt{\sigma_{\min}}}\left\Vert \nabla f\left(\bm{X},\bm{Y}\right)\right\Vert _{\mathrm{F}}\lesssim\frac{{\kappa}}{n^{5}}\lambda\leq\frac{\lambda}{8} \qquad \text{and} \qquad \left\Vert \mathcal{P}_{T^{\perp}}\left(\bm{R}\right)\right\Vert \leq\frac{\lambda}{2}\label{eq:R-T-T-perp-bound}
\end{equation}
with probability exceeding $1-O(n^{-10})$. Here we utilize the small-gradient
condition $\|\nabla f(\bm{X},\bm{Y})\|_{\mathrm{F}}\leq\frac{1}{n^{5}}\frac{\lambda}{p}\sqrt{\sigma_{\min}}$
(cf.~(\ref{eq:small-gradient})). Take~(\ref{eq:connection-1}),
(\ref{eq:connection-2}) and~(\ref{eq:connection-3}) collectively
to reach 
\begin{align}
\bm{M}^{\mathsf{cvx,d}} & =\mathcal{P}_{\text{rank-}r}\left[\bm{X}^{\mathsf{ncvx}}\bm{Y}^{\mathsf{ncvx}\top}+\frac{\lambda}{p}\bm{U}\bm{V}^{\top}-\frac{1}{p}\bm{R}+\bm{\Delta}_{\bm{Z}}\right] \nonumber\\
 & =\mathcal{P}_{\text{rank-}r}\left[\bm{U}\Big(\bm{\Sigma}+\frac{\lambda}{p}\bm{I}_{r}\Big)\bm{V}^{\top}+\bm{\Delta}_{\bm{Z}}-\frac{1}{p}\bm{R}\right] \nonumber\\
	& =\mathcal{P}_{\text{rank-}r}\Big[\underbrace{\bm{U}\Big(\bm{\Sigma}+\frac{\lambda}{p}\bm{I}_{r}\Big)\bm{V}^{\top}+\mathcal{P}_{T^{\perp}}\Big(\bm{\Delta}_{\bm{Z}}-\frac{1}{p}\bm{R}\Big)}_{:=\bm{C}}+\underbrace{\mathcal{P}_{T}\Big(\bm{\Delta}_{\bm{Z}}-\frac{1}{p}\bm{R}\Big)}_{:=\bm{\Delta}}\Big],  \label{eq:M-cvx-d-decomposition-all}
\end{align}
where the middle line follows since $\bm{U}\bm{\Sigma}\bm{V}^{\top}$ is defined to be the SVD of 
$\bm{X}^{\mathsf{ncvx}}\bm{Y}^{\mathsf{ncvx}\top}$. 

We view $\bm{\Delta}$ as a perturbation and intend to apply Lemma~\ref{lemma:SVD-pert} to control $\|\bm{M}^{\mathsf{cvx,d}} - \bm{U}(\bm{\Sigma}+({\lambda}/{p})\bm{I}_{r})\bm{V}^{\top}\|_\mathrm{F}$. First, notice that the $r$th largest singular value obeys 
$
\sigma_r(\bm{U}(\bm{\Sigma}+\frac{\lambda}{p}\bm{I}_{r})\bm{V}^{\top})\geq \frac{\lambda}{p}
$, and that 
\begin{equation}
\Big\Vert \mathcal{P}_{T^\perp}\Big(\bm{\Delta}_{\bm{Z}}-\frac{1}{p}\bm{R}\Big)\Big\Vert \leq \left\Vert \bm{\Delta}_{\bm{Z}}\right\Vert _{\mathrm{F}}+\frac{1}{p}\left\Vert \mathcal{P}_{T^\perp}\left(\bm{R}\right)\right\Vert _{\mathrm{F}}\leq\frac{5\lambda}{8p}\label{eq:inequality-1},
\end{equation}
where the last inequality results from~(\ref{eq:Delta-Z-bound}) and
(\ref{eq:R-T-T-perp-bound}). Combining the above two bounds with the fact that $\bm{U}(\bm{\Sigma}+\frac{\lambda}{p}\bm{I}_{r})\bm{V}^{\top}$ and $\mathcal{P}_{T}(\bm{\Delta}_{\bm{Z}}-\frac{1}{p}\bm{R})$ are orthogonal to each other, we arrive at the conclusion that $\bm{U}(\bm{\Sigma}+\frac{\lambda}{p}\bm{I}_{r})\bm{V}^{\top}$ is the top-$r$ SVD of $\bm{C}$ and
\begin{subequations}
\begin{align}
\sigma_i\left(\bm{C}\right) &= \sigma_i\Big(\bm{U}\Big(\bm{\Sigma}+\frac{\lambda}{p}\bm{I}_{r}\Big)\bm{V}^{\top}\Big),\qquad\text{for }1\leq i \leq r; \label{eq:singular-value-equality}\\
\sigma_{r+1}\left(\bm{C}\right) &= \left\Vert \mathcal{P}_{T^\perp}\Big(\bm{\Delta}_{\bm{Z}}-\frac{1}{p}\bm{R}\Big)\right\Vert\label{eq:singular-value-equality-2}.
\end{align}
\end{subequations}

Second, let $\hat{\bm{U}}\hat{\bm{\Sigma}}\hat{\bm{V}}^\top$ be the top-$r$ SVD of $\bm{C} +\bm{\Delta}$. By definition, one has $\hat{\bm{U}}\hat{\bm{\Sigma}}\hat{\bm{V}}^\top = \bm{M}^{\mathsf{cvx,d}}$. We are left with checking the two conditions in Lemma~\ref{lemma:SVD-pert}. To begin with, the perturbation term $\bm{\Delta}$ obeys 
\begin{align}
\left\Vert \bm{\Delta}\right\Vert_{\mathrm{F}}&\leq\left\Vert \bm{\Delta}_{\bm{Z}}\right\Vert _{\mathrm{F}}+\frac{1}{p}\left\Vert \mathcal{P}_{T}\left(\bm{R}\right)\right\Vert_{\mathrm{F}} 
\overset{(\text{i})}{\lesssim}\frac{\kappa^{2}}{n^{5}p}\frac{{\lambda}}{p}+\frac{{\kappa}}{n^{5}}\frac{\lambda}{p} \overset{(\text{ii})}{\leq}\frac{{1}}{2n^{4}}\frac{{\lambda}}{p} \label{eq:inequality-2},
\end{align}
where (i) comes from~(\ref{eq:Delta-Z-bound}) and
(\ref{eq:R-T-T-perp-bound}) and the last inequality (ii) arises since $np\gg\kappa^{2}.$ Clearly, the size of the perturbation is much smaller than $\lambda / p$ and hence  $\|\bm{C}\|$ (cf.~(\ref{eq:singular-value-equality})). In addition, 
\begin{align*}
\sigma_{r+1}\left(\bm{C}+\bm{\Delta}\right) & \leq\sigma_{r+1}\left(\bm{C}\right)+\left\Vert \bm{\Delta}\right\Vert 
	=\left\Vert \mathcal{P}_{T^{\perp}}\Big(\bm{\Delta}_{\bm{Z}}-\frac{1}{p}\bm{R}\Big)\right\Vert +\left\Vert \bm{\Delta}\right\Vert _{\mathrm{F}}\\
 & \leq\frac{5\lambda}{8p}+\frac{1}{2n^{4}}\frac{\lambda}{p}\leq\frac{3\lambda}{4p},
\end{align*}
where the equality depends on~(\ref{eq:singular-value-equality-2})
and the last line results from~(\ref{eq:inequality-1}) and~(\ref{eq:inequality-2}).
Consequently,
\begin{align*}
\sigma_{r}\left(\bm{C}\right)-\sigma_{r+1}\left(\bm{C}+\bm{\Delta}\right) & \geq\sigma_{r}\left(\bm{U}\Big(\bm{\Sigma}+\frac{\lambda}{p}\bm{I}_{r}\Big)\bm{V}^{\top}\right)-\frac{3\lambda}{4p}
	\geq\sigma_{r}\left(\bm{\Sigma}\right)+\frac{\lambda}{4p}\geq\frac{\sigma_{\min}}{2}.
\end{align*}
Here the first relation arises from~(\ref{eq:singular-value-equality})
and the second holds since $\sigma_{r}(\bm{\Sigma})\geq\sigma_{\min}/2$, a simple consequence of~(\ref{eq:F-norm-upper-bound}).
We are now ready to apply Lemma~\ref{lemma:SVD-pert} to obtain 
\begin{align*}
\Big\Vert \bm{M}^{\mathsf{cvx,d}}-\bm{U}\Big(\bm{\Sigma}+\frac{\lambda}{p}\bm{I}_{r}\Big)\bm{V}^{\top}\Big\Vert _{\mathrm{F}} & \leq\left(\frac{12\left\Vert \bm{\Sigma}+({\lambda}/{p})\bm{I}_{r}\right\Vert }{\sigma_{\min}/2}+1\right)\left\Vert \bm{\Delta}\right\Vert _{\mathrm{F}}\lesssim\kappa\left\Vert \bm{\Delta}\right\Vert _{\mathrm{F}},
\end{align*}
where we have used the fact that $\|\bm{\Sigma}+(\lambda/p)\bm{I}_{r}\|\lesssim\sigma_{\max}$, which also can be derived from~(\ref{eq:F-norm-upper-bound}).
The above bound combined with~(\ref{eq:inequality-2}) yields 
\[
\Big\Vert \bm{M}^{\mathsf{cvx,d}}-\bm{U}\Big(\bm{\Sigma}+\frac{\lambda}{p}\bm{I}_{r}\Big)\bm{V}^{\top}\Big\Vert _{\mathrm{F}}\lesssim\frac{\kappa^{3}}{n^{5}p}\frac{\lambda}{p}+\frac{\kappa^{2}}{n^{5}}\frac{\lambda}{p}\leq\frac{1}{2n^{4}}\frac{\lambda}{p}
\]
as long as $np\gg\kappa^{3}$. We remark that by setting
$\bm{\Delta}_{\bm{Z}}=\bm{0}$, one also obtains the bound on $\bm{M}^{\mathsf{ncvx,d}}$,
i.e. 
\begin{equation}
\Big\Vert \bm{M}^{\mathsf{ncvx,d}}-\bm{U}\Big(\bm{\Sigma}+\frac{\lambda}{p}\bm{I}_{r}\Big)\bm{V}^{\top}\Big\Vert _{\mathrm{F}}\leq\frac{{1}}{2n^{4}}\frac{{\lambda}}{p}.\label{eq:Z-d-ncvx-U-sigma-V}
\end{equation}

We move on to investigating $\|\bm{X}^{\mathsf{ncvx,d}}\bm{Y}^{\mathsf{ncvx,d}\top}-\bm{U}(\bm{\Sigma}+\frac{\lambda}{p}\bm{I}_{r})\bm{V}^{\top}\|$,
for which we have the following claim. \begin{claim}\label{claim:connection-ncvx-ref}One
has 
\begin{equation}
\left\Vert \bm{X}^{\mathsf{ncvx,d}}\bm{Y}^{\mathsf{ncvx,d}\top}-\bm{U}\Big(\bm{\Sigma}+\frac{\lambda}{p}\bm{I}_{r}\Big)\bm{V}^{\top}\right\Vert \leq\frac{{1}}{2n^{4}}\frac{{\lambda}}{p}.\label{eq:X-d-Y-d-U-sigma-V}
\end{equation}
\end{claim}Taking the above three bounds collectively and recognizing
that $\lambda\lesssim\sigma\sqrt{np}$ yield the advertised bound~(\ref{eq:matrix-equivalence}).

\begin{proof}[Proof of Claim~\ref{claim:connection-ncvx-ref}]Utilize
\cite[Claim 3]{chen2019noisy} to see that 
\begin{equation}
\bm{X}^{\mathsf{ncvx}}=\bm{U}\bm{\Sigma}^{1/2}\bm{Q}\qquad\text{and}\qquad\bm{Y}^{\mathsf{ncvx}}=\bm{V}\bm{\Sigma}^{1/2}\bm{Q}^{-\top}\label{eq:X-Y-ncvx-near-rotation}
\end{equation}
hold for some invertible matrix $\bm{Q}\in\mathbb{R}^{r\times r}$
with SVD $\bm{U}_{\bm{Q}}\bm{\Sigma}_{\bm{Q}}\bm{V}_{\bm{Q}}^{\top}$ obeying
\begin{equation}
\big\Vert \bm{\Sigma}_{\bm{Q}}-\bm{\Sigma}_{\bm{Q}}^{-1} \big\Vert _{\mathrm{F}}
	\leq 8\sqrt{\kappa}\frac{p}{\lambda\sqrt{\sigma_{\min}}}\big\Vert \nabla f\left(\bm{X}^{\mathsf{ncvx}},\bm{Y}^{\mathsf{ncvx}}\right)\big\Vert _{\mathrm{F}}\leq\frac{8\sqrt{{\kappa}}}{n^{5}}.
	\label{eq:Q-spectrum}
\end{equation}
	The last inequality is the small-gradient condition (see~(\ref{eq:small-gradient}), in which $(\bm{X},\bm{Y}) = (\bm{X}^{\mathsf{ncvx}}, \bm{Y}^{\mathsf{ncvx}}))$.
Employ the definitions for $\bm{X}^{\mathsf{ncvx,d}}$ and $\bm{Y}^{\mathsf{ncvx,d}}$
(cf.~\eqref{eq:X-ncvx-d-defn} and \eqref{eq:Y-ncvx-d-defn}) to see that 
\begin{align}
\bm{X}^{\mathsf{ncvx,d}}\bm{Y}^{\mathsf{ncvx,d}\top} & =\bm{X}^{\mathsf{ncvx}}\Big(\bm{I}_{r}+\frac{\lambda}{p}\left(\bm{X}^{\mathsf{ncvx}\top}\bm{X}^{\mathsf{ncvx}}\right)^{-1}\Big)^{1/2}\Big(\bm{I}_{r}+\frac{\lambda}{p}\left(\bm{Y}^{\mathsf{ncvx}\top}\bm{Y}^{\mathsf{ncvx}}\right)^{-1}\Big)^{1/2}\bm{Y}^{\mathsf{ncvx}\top}\nonumber \\
 & =\bm{X}^{\mathsf{ncvx}}\Big(\bm{I}_{r}+\frac{\lambda}{p}\left(\bm{X}^{\mathsf{ncvx}\top}\bm{X}^{\mathsf{ncvx}}\right)^{-1}\Big)^{1/2}\Big(\bm{I}_{r}+\frac{\lambda}{p}\left(\bm{X}^{\mathsf{ncvx}\top}\bm{X}^{\mathsf{ncvx}}\right)^{-1}\Big)^{1/2}\bm{Y}^{\mathsf{ncvx}\top}\nonumber \\
	& \quad-\bm{X}^{\mathsf{ncvx}}\Big(\bm{I}_{r}+\frac{\lambda}{p}\left(\bm{X}^{\mathsf{ncvx}\top}\bm{X}^{\mathsf{ncvx}}\right)^{-1} \Big) \bm{\Delta}_{\mathsf{balancing}}\bm{Y}^{\mathsf{ncvx}\top}\nonumber \\
 & =\underbrace{\bm{X}^{\mathsf{ncvx}}\Big(\bm{I}_{r}+\frac{\lambda}{p}\left(\bm{X}^{\mathsf{ncvx}\top}\bm{X}^{\mathsf{ncvx}}\right)^{-1}\Big)\bm{Y}^{\mathsf{ncvx}\top}}_{:=\bm{A}_{1}}-\underbrace{\bm{X}^{\mathsf{ncvx}}\Big(\bm{I}_{r}+\frac{\lambda}{p}\left(\bm{X}^{\mathsf{ncvx}\top}\bm{X}^{\mathsf{ncvx}}\right)^{-1} \Big)  \bm{\Delta}_{\mathsf{balancing}}\bm{Y}^{\mathsf{ncvx}\top}}_{:=\bm{A}_{2}}.\label{eq:X-d-Y-d-decomposition}
\end{align}
Here we denote 
\[
\bm{\Delta}_{\mathsf{balancing}}\triangleq\Big(\bm{I}_{r}+\frac{\lambda}{p}\left(\bm{X}^{\mathsf{ncvx}\top}\bm{X}^{\mathsf{ncvx}}\right)^{-1}\Big)^{1/2}-\Big(\bm{I}_{r}+\frac{\lambda}{p}\left(\bm{Y}^{\mathsf{ncvx}\top}\bm{Y}^{\mathsf{ncvx}}\right)^{-1}\Big)^{1/2}.
\]
It then boils down to showing that (i) $\bm{{A}}_{1}$ is very close to $\bm{U}(\bm{\Sigma}+\frac{\lambda}{p}\bm{I}_{r})\bm{V}^{\top}$,
and (ii) $\bm{{A}}_{2}$ is small in size.

First, recall that $\bm{X}^{\mathsf{ncvx}}\bm{Y}^{\mathsf{ncvx}\top}=\bm{{U}}\bm{{\Sigma}}\bm{{V}}^{\top}$, 
which combined with~(\ref{eq:X-Y-ncvx-near-rotation}) gives 
\begin{align}
\left\Vert \bm{{A}}_{1}-\bm{U}\Big(\bm{\Sigma}+\frac{\lambda}{p}\bm{I}_{r}\Big)\bm{V}^{\top}\right\Vert  & =\frac{{\lambda}}{p}\left\Vert \bm{X}^{\mathsf{ncvx}}\left(\bm{X}^{\mathsf{ncvx}\top}\bm{X}^{\mathsf{ncvx}}\right)^{-1}\bm{Y}^{\mathsf{ncvx}\top}-\bm{{U}}\bm{{V}}^{\top}\right\Vert \nonumber \\
 & =\frac{{\lambda}}{p}\left\Vert \bm{{U}}\bm{{\Sigma}}^{-1/2}\bm{{Q}}^{-\top}\bm{{Q}}^{-1}\bm{{\Sigma}}^{1/2}\bm{{V}}^{\top}-\bm{{U}}\bm{{V}}^{\top}\right\Vert \nonumber \\
 & =\frac{{\lambda}}{p}\left\Vert \bm{{\Sigma}}^{-1/2}\left(\bm{{Q}}^{-\top}\bm{{Q}}^{-1}-\bm{{I}}_{r}\right)\bm{{\Sigma}}^{1/2}\right\Vert \leq\sqrt{{\kappa}}\frac{{\lambda}}{p}\big\Vert \bm{{Q}}^{-\top}\bm{{Q}}^{-1}-\bm{{I}}_{r}\big\Vert \nonumber \\
 & = \sqrt{{\kappa}}\frac{{\lambda}}{p}\big\Vert \bm{{\Sigma}}_{\bm{{Q}}}^{-2}-\bm{{I}}_{r}\big\Vert \leq\sqrt{{\kappa}}\frac{{\lambda}}{p}\big\Vert \bm{{\Sigma}}_{\bm{{Q}}}^{-1}\big\Vert \cdot \big\Vert \bm{{\Sigma}}_{\bm{{Q}}}^{-1}-\bm{{\Sigma}}_{\bm{{Q}}}\big\Vert _{\mathrm{{F}}}\nonumber \\
 & \lesssim\kappa\frac{{\lambda}}{p}\frac{{1}}{n^{5}}.\label{eq:A-1-U-Sigma-V-near}
\end{align}
Here, the last inequality comes from~(\ref{eq:Q-spectrum}) and its
immediate consequence that $\Vert\bm{{\Sigma}}_{\bm{{Q}}}\Vert\asymp\Vert\bm{{\Sigma}}_{\bm{{Q}}}^{-1}\Vert\asymp1$.

Second, apply the perturbation bound for matrix square roots (cf.~Lemma~\ref{lemma:matrix-sqrt})
to obtain 
\begin{align}
\left\Vert \bm{\Delta}_{\mathsf{balancing}}\right\Vert  & \lesssim \frac{ \frac{\lambda}{p} \big\Vert \left(\bm{X}^{\mathsf{ncvx}\top}\bm{X}^{\mathsf{ncvx}}\right)^{-1}-\left(\bm{Y}^{\mathsf{ncvx}\top}\bm{Y}^{\mathsf{ncvx}}\right)^{-1}\big\Vert }{\lambda_{\min}\left[\left(\bm{I}_{r}+\frac{\lambda}{p}\left(\bm{X}^{\mathsf{ncvx}\top}\bm{X}^{\mathsf{ncvx}}\right)^{-1}\right)^{1/2}\right]+\lambda_{\min}\left[\left(\bm{I}_{r}+\frac{\lambda}{p}\left(\bm{Y}^{\mathsf{ncvx}\top}\bm{Y}^{\mathsf{ncvx}}\right)^{-1}\right)^{1/2}\right]}\nonumber \\
 & \overset{(\text{i})}{\lesssim}\frac{\lambda}{p}\left\Vert \left(\bm{X}^{\mathsf{ncvx}\top}\bm{X}^{\mathsf{ncvx}}\right)^{-1}-\left(\bm{Y}^{\mathsf{ncvx}\top}\bm{Y}^{\mathsf{ncvx}}\right)^{-1}\right\Vert \nonumber \\
 & \leq\frac{\lambda}{p}\left\Vert \left(\bm{X}^{\mathsf{ncvx}\top}\bm{X}^{\mathsf{ncvx}}\right)^{-1}\right\Vert \left\Vert \bm{X}^{\mathsf{ncvx}\top}\bm{X}^{\mathsf{ncvx}}-\bm{Y}^{\mathsf{ncvx}\top}\bm{Y}^{\mathsf{ncvx}}\right\Vert _{\mathrm{F}}\left\Vert \left(\bm{Y}^{\mathsf{ncvx}\top}\bm{Y}^{\mathsf{ncvx}}\right)^{-1}\right\Vert \nonumber \\
 & \overset{(\text{ii})}{\lesssim}\frac{{1}}{n^{5}}\frac{{\lambda}}{p}\frac{\kappa}{\sigma_{\min}}.\label{eq:debias-correction-close}
\end{align}
Here, the inequality (i) depends on the facts that 
\[
\lambda_{\min}\left[\left(\bm{I}_{r}+\frac{\lambda}{p}\left(\bm{X}^{\mathsf{ncvx}\top}\bm{X}^{\mathsf{ncvx}}\right)^{-1}\right)^{1/2}\right]\geq 1\qquad\text{and}\qquad\lambda_{\min}\left[\left(\bm{I}_{r}+\frac{\lambda}{p}\left(\bm{Y}^{\mathsf{ncvx}\top}\bm{Y}^{\mathsf{ncvx}}\right)^{-1}\right)^{1/2}\right]\geq 1,
\]
whereas the inequality (ii) holds because of the facts that $\|(\bm{X}^{\mathsf{ncvx}\top}\bm{X}^{\mathsf{ncvx}})^{-1}\|\lesssim1/\sigma_{\min}$,
$\|(\bm{Y}^{\mathsf{ncvx}\top}\bm{Y}^{\mathsf{ncvx}})^{-1}\|\lesssim1/\sigma_{\min}$
and the balancedness condition~(\ref{eq:X-Y-balance})
\[
\left\Vert \bm{X}^{\mathsf{ncvx}\top}\bm{X}^{\mathsf{ncvx}}-\bm{Y}^{\mathsf{ncvx}\top}\bm{Y}^{\mathsf{ncvx}}\right\Vert _{\mathrm{F}}\leq\frac{1}{n^{5}}\sigma_{\max}.
\]
As a result, the operator norm of $\bm{{A}}_{2}$ is bounded by 
\begin{align}
\left\Vert \bm{{A}}_{2}\right\Vert  & \leq\left\Vert \bm{X}^{\mathsf{ncvx}}\Big(\bm{I}_{r}+\frac{\lambda}{p}\left(\bm{X}^{\mathsf{ncvx}\top}\bm{X}^{\mathsf{ncvx}}\right)^{-1}\Big)\right\Vert \left\Vert \bm{\Delta}_{\mathsf{balancing}}\right\Vert \left\Vert \bm{Y}^{\mathsf{ncvx}}\right\Vert \nonumber \\
 & \lesssim\sqrt{{\sigma_{\max}}}\cdot\frac{{1}}{n^{5}}\frac{{\lambda}}{p}\frac{\kappa}{\sigma_{\min}}\cdot\sqrt{{\sigma_{\max}}}\asymp\frac{\lambda}{p}\frac{\kappa^{2}}{n^{5}}.\label{eq:A-2-small}
\end{align}

Take~(\ref{eq:X-d-Y-d-decomposition}),~(\ref{eq:A-1-U-Sigma-V-near})
and~(\ref{eq:A-2-small}) collectively to arrive at 
\[
\left\Vert \bm{X}^{\mathsf{ncvx,d}}\bm{Y}^{\mathsf{ncvx,d}\top}-\bm{U}\Big(\bm{\Sigma}+\frac{\lambda}{p}\bm{I}_{r}\Big)\bm{V}^{\top}\right\Vert \leq\left\Vert \bm{{A}}_{1}-\bm{U}\Big(\bm{\Sigma}+\frac{\lambda}{p}\bm{I}_{r}\Big)\bm{V}^{\top}\right\Vert +\left\Vert \bm{{A}}_{2}\right\Vert \lesssim\frac{\lambda}{p}\frac{\kappa^{2}}{n^{5}}\leq\frac{{1}}{2n^{4}}\frac{{\lambda}}{p},
\]
provided that $n\gg\kappa^{2}$. \end{proof}

\subsection{Proof of the inequality~(\ref{eq:low-rank-factor-equivalence})}

Next, we switch attention to the low-rank factors. Our goal is to
demonstrate that $(\bm{X}^{\mathsf{cvx,d}},\bm{Y}^{\mathsf{cvx,d}})$
and $(\bm{X}^{\mathsf{ncvx,d}},\bm{Y}^{\mathsf{ncvx,d}})$ are both
extremely close to $(\bm{U}(\bm{\Sigma}+\frac{\lambda}{p}\bm{I}_{r})^{1/2},\bm{V}(\bm{\Sigma}+\frac{\lambda}{p}\bm{I}_{r})^{1/2})$
modulo some global rotation, which will be established in~(\ref{eq:ncvx-low-rank-dist-to-ref}) and~(\ref{eq:cvx-low-rank-dist-to-ref}) shortly.

We start by justifying the proximity between $(\bm{X}^{\mathsf{ncvx,d}},\bm{Y}^{\mathsf{ncvx,d}})$
and $(\bm{U}(\bm{\Sigma}+\frac{\lambda}{p}\bm{I}_{r})^{1/2},\bm{V}(\bm{\Sigma}+\frac{\lambda}{p}\bm{I}_{r})^{1/2})$.
In view of~(\ref{eq:X-Y-ncvx-near-rotation}), we know that 
\begin{align}
\left\Vert \bm{X}^{\mathsf{ncvx}}-\bm{U}\bm{\Sigma}^{1/2}\bm{U}_{\bm{Q}}\bm{V}_{\bm{Q}}^{\top}\right\Vert  & =\left\Vert \bm{U}\bm{\Sigma}^{1/2}\bm{U}_{\bm{Q}}\bm{\Sigma}_{\bm{Q}}\bm{V}_{\bm{Q}}^{\top}-\bm{U}\bm{\Sigma}^{1/2}\bm{U}_{\bm{Q}}\bm{V}_{\bm{Q}}^{\top}\right\Vert 
  \leq\big\|\bm{\Sigma}^{1/2}\big\|\left\Vert \bm{\Sigma}_{\bm{Q}}-\bm{I}_{r}\right\Vert \nonumber\\
 & \overset{(\text{i})}{\lesssim}\sqrt{\sigma_{\max}}\frac{1}{\sigma_{\min}}\left\Vert \bm{X}^{\top}\bm{X}-\bm{Y}^{\top}\bm{Y}\right\Vert _{\mathrm{F}} \nonumber\\
 & \overset{(\text{ii})}{\lesssim}\sqrt{\sigma_{\max}}\frac{1}{\sigma_{\min}}\frac{1}{n^{5}}\sigma_{\max}\frac{\sigma}{\sigma_{\min}}\sqrt{\frac{n}{p}}
	\overset{(\text{ii})}{\leq} \frac{1}{n^{4}}\frac{\sigma}{\sigma_{\min}}\sqrt{\frac{n}{p}}\cdot\sqrt{\sigma_{\max}}. \label{eq:Xncvx-Xtilde-gap}
\end{align}
Here, (i) depends on the fact that $\|\bm{\Sigma}_{\bm{Q}}-\bm{I}_{r}\|\lesssim\|\bm{\Sigma}_{\bm{Q}}-\bm{\Sigma}_{\bm{Q}}^{-1}\|_{\mathrm{F}}\lesssim\|\bm{X}^{\top}\bm{X}-\bm{Y}^{\top}\bm{Y}\|_{\mathrm{F}}/\sigma_{\min}$
(see \cite[Lemma 20]{chen2019noisy}),  (ii) makes use of the
balancedness assumption~(\ref{eq:X-Y-balance}), whereas (iii) holds if $n\gg \kappa$. Denoting $\tilde{\bm{X}}\triangleq\bm{U}\bm{\Sigma}^{1/2}\bm{U}_{\bm{Q}}\bm{V}_{\bm{Q}}^{\top}$, 
one invokes the triangle inequality to reach
\begin{align*}
 & \left\Vert \bm{X}^{\mathsf{ncvx,d}}-\tilde{\bm{X}}\left(\bm{I}_{r}+\frac{\lambda}{p}\left(\tilde{\bm{X}}^{\top}\tilde{\bm{X}}\right)^{-1}\right)^{1/2}\right\Vert \\
 & \quad\leq\big\|\bm{X}^{\mathsf{ncvx}}-\tilde{\bm{X}}\big\|\left\Vert \Big(\bm{I}_{r}+\frac{\lambda}{p}\left(\bm{X}^{\mathsf{ncvx}\top}\bm{X}^{\mathsf{ncvx}}\right)^{-1}\Big)^{1/2}\right\Vert \\
 & \quad\quad+\big\|\tilde{\bm{X}}\big\|\left\Vert \Big(\bm{I}_{r}+\frac{\lambda}{p}\left(\bm{X}^{\mathsf{ncvx}\top}\bm{X}^{\mathsf{ncvx}}\right)^{-1}\Big)^{1/2}-\left(\bm{I}_{r}+\frac{\lambda}{p}\big(\tilde{\bm{X}}^{\top}\tilde{\bm{X}}\big)^{-1}\right)^{1/2}\right\Vert \\
 & \quad\leq\frac{1}{n^{4}}\frac{\sigma}{\sigma_{\min}}\sqrt{\frac{n}{p}}\cdot\sqrt{\sigma_{\max}}.
\end{align*}
Here the last line arises from \eqref{eq:Xncvx-Xtilde-gap} and the facts $\|\bm{I}_{r}+\frac{\lambda}{p}(\bm{X}^{\mathsf{ncvx}\top}\bm{X}^{\mathsf{ncvx}})^{-1}\|\asymp1$, $\|\tilde{\bm{X}}\|\lesssim \sqrt{\sigma_{\max}}$  
and 
\[
\left\Vert \Big(\bm{I}_{r}+\frac{\lambda}{p}\left(\bm{X}^{\mathsf{ncvx}\top}\bm{X}^{\mathsf{ncvx}}\right)^{-1}\Big)^{1/2}-\Big(\bm{I}_{r}+\frac{\lambda}{p}\big(\tilde{\bm{X}}^{\top}\tilde{\bm{X}}\big)^{-1}\Big)^{1/2}\right\Vert \lesssim\frac{1}{n^{4}}\frac{\sigma}{\sigma_{\min}}\sqrt{\frac{n}{p}}.
\]
The latter bound follows from similar derivations as in~(\ref{eq:debias-correction-close}).
A similar bound holds for $\bm{Y}^{\mathsf{ncvx,d}}$. Recognizing that
\[
\tilde{\bm{X}}\left(\bm{I}_{r}+\frac{\lambda}{p}\big(\tilde{\bm{X}}^{\top}\tilde{\bm{X}}\big)^{-1}\right)^{1/2}=\bm{U}\Big(\bm{\Sigma}+\frac{\lambda}{p}\bm{I}_{r}\Big)^{1/2}\bm{U}_{\bm{Q}}\bm{V}_{\bm{Q}}^{\top},
\]
we have
\begin{align}
	& \min_{\bm{R}\in\mathcal{O}^{r\times r}}\sqrt{\left\Vert \bm{X}^{\mathsf{ncvx,d}}\bm{R}-\bm{U}\Big(\bm{\Sigma}+\frac{\lambda}{p}\bm{I}_{r}\Big)^{1/2}\right\Vert _{\mathrm{F}}^{2}+\left\Vert \bm{Y}^{\mathsf{ncvx,d}}\bm{R}-\bm{V}\Big(\bm{\Sigma}+\frac{\lambda}{p}\bm{I}_{r}\Big)^{1/2}\right\Vert _{\mathrm{F}}^{2}} \nonumber\\
	& \quad \leq \sqrt{\left\Vert \bm{X}^{\mathsf{ncvx,d}} -\bm{U}\Big(\bm{\Sigma}+\frac{\lambda}{p}\bm{I}_{r}\Big)^{1/2} \bm{U}_{\bm{Q}} \bm{V}_{\bm{Q}}^{\top} \right\Vert _{\mathrm{F}}^{2}+\left\Vert \bm{Y}^{\mathsf{ncvx,d}} -\bm{V}\Big(\bm{\Sigma}+\frac{\lambda}{p}\bm{I}_{r}\Big)^{1/2} \bm{U}_{\bm{Q}} \bm{V}_{\bm{Q}}^{\top}  \right\Vert _{\mathrm{F}}^{2}} \nonumber\\
	& \quad \leq \sqrt{r} \sqrt{\left\Vert \bm{X}^{\mathsf{ncvx,d}} -\bm{U}\Big(\bm{\Sigma}+\frac{\lambda}{p}\bm{I}_{r}\Big)^{1/2} \bm{U}_{\bm{Q}} \bm{V}_{\bm{Q}}^{\top} \right\Vert ^{2}+\left\Vert \bm{Y}^{\mathsf{ncvx,d}} -\bm{V}\Big(\bm{\Sigma}+\frac{\lambda}{p}\bm{I}_{r}\Big)^{1/2} \bm{U}_{\bm{Q}} \bm{V}_{\bm{Q}}^{\top}  \right\Vert ^{2}} \nonumber\\
	& \quad \lesssim\frac{\sqrt{r}}{n^{4}}\frac{\sigma}{\sigma_{\min}}\sqrt{\frac{n}{p}}\cdot\sqrt{\sigma_{\max}}.\label{eq:ncvx-low-rank-dist-to-ref}
\end{align}

Next, we establish the connection between $(\bm{X}^{\mathsf{cvx,d}},\bm{Y}^{\mathsf{cvx,d}})$
and $(\bm{U}(\bm{\Sigma}+\frac{\lambda}{p}\bm{I}_{r})^{1/2},\bm{V}(\bm{\Sigma}+\frac{\lambda}{p}\bm{I}_{r})^{1/2})$.
To accomplish this, we first study the relationship between $(\bm{X}^{\mathsf{cvx}},\bm{Y}^{\mathsf{cvx}})$
and $(\bm{U}\bm{\Sigma}^{1/2},\bm{V}\bm{\Sigma}^{1/2})$. Recall that
$\bm{X}^{\mathsf{cvx}}$ and $\bm{Y}^{\mathsf{cvx}}$ constitute a
balanced factorization of $\bm{Z}^{\mathsf{cvx},r}$, while $(\bm{U}\bm{\Sigma}^{1/2},\bm{V}\bm{\Sigma}^{1/2})$
is a balanced one of $\bm{X}^{\mathsf{ncvx}}\bm{Y}^{\mathsf{ncvx}\top}=\bm{U}\bm{\Sigma}\bm{V}^{\top}$.
Hence one can view $\bm{Z}^{\mathsf{cvx},r}$ as a perturbation of
$\bm{X}^{\mathsf{ncvx}}\bm{Y}^{\mathsf{ncvx}\top}=\bm{U}\bm{\Sigma}\bm{V}^{\top}$
and investigate the perturbation bounds on the balanced factorizations.
Going through the same derivations as in \cite[Appendix B.7]{ma2017implicit}
and \cite[Appendix B.2.1]{chen2019nonconvex}, one reaches 
\begin{align*}
\min_{\bm{R}\in\mathcal{O}^{r\times r}}\sqrt{\left\Vert \bm{X}^{\mathsf{cvx}}\bm{R}-\bm{U}\bm{\Sigma}^{1/2}\right\Vert _{\mathrm{F}}^{2}+\left\Vert \bm{Y}^{\mathsf{cvx}}\bm{R}-\bm{V}\bm{\Sigma}^{1/2}\right\Vert _{\mathrm{F}}^{2}} & \lesssim\sqrt{r}\cdot\frac{\kappa^{2}}{\sqrt{\sigma_{\min}}}\left\Vert \bm{Z}^{\mathsf{cvx},r}-\bm{X}^{\mathsf{ncvx}}\bm{Y}^{\mathsf{ncvx}\top}\right\Vert _{\mathrm{F}}\\
 & \lesssim\sqrt{r}\cdot\frac{\kappa^{4}}{\sqrt{\sigma_{\min}}}\cdot\frac{1}{n^{5}}\frac{\lambda}{p}.
\end{align*}
Here the last relation follows from the proximity of the convex
estimator and the nonconvex estimator; see~(\ref{eq:closedness-cvx-ncvx}).
Repeating the same argument as above to translate the bound between
$(\bm{X}^{\mathsf{cvx}},\bm{Y}^{\mathsf{cvx}})$ and $(\bm{U}\bm{\Sigma}^{1/2},\bm{V}\bm{\Sigma}^{1/2})$
to that of $(\bm{X}^{\mathsf{cvx,d}},\bm{Y}^{\mathsf{cvx,d}})$ and
$(\bm{U}(\bm{\Sigma}+\frac{\lambda}{p}\bm{I}_{r})^{1/2},\bm{V}(\bm{\Sigma}+\frac{\lambda}{p}\bm{I}_{r})^{1/2})$,
we conclude that
\begin{equation}
\min_{\bm{R}\in\mathcal{O}^{r\times r}}\sqrt{\left\Vert \bm{X}^{\mathsf{cvx,d}}\bm{R}-\bm{U}\Big(\bm{\Sigma}+\frac{\lambda}{p}\bm{I}_{r}\Big)^{1/2}\right\Vert _{\mathrm{F}}^{2}+\left\Vert \bm{Y}^{\mathsf{cvx,d}}\bm{R}-\bm{V}\Big(\bm{\Sigma}+\frac{\lambda}{p}\bm{I}_{r}\Big)^{1/2}\right\Vert _{\mathrm{F}}^{2}}\lesssim\sqrt{r}\cdot\frac{\kappa^{4}}{n^{5}}\frac{\sigma}{\sigma_{\min}}\sqrt{\frac{n}{p}}\cdot\sqrt{\sigma_{\min}}.
	\label{eq:cvx-low-rank-dist-to-ref}
\end{equation}
This together with~(\ref{eq:ncvx-low-rank-dist-to-ref}) and the assumption $n \gg \kappa^4$ concludes the proof.

\subsection{Proof of the inequality~(\ref{eq:linearize-equivalence})}

We shall focus on proving the claim for the nonconvex estimator $\bm{M}^{\mathsf{ncvx,d}}$ and $\bm{X}^{\mathsf{ncvx}}\bm{Y}^{\mathsf{ncvx}\top}$;
the claim for the convex estimator $\bm{Z}^{\mathsf{cvx}}$ can be treated similarly. 

Recall from~(\ref{eq:Z-d-ncvx-U-sigma-V}) that 
\[
\Big\Vert \bm{M}^{\mathsf{ncvx,d}}-\bm{U}\Big(\bm{\Sigma}+\frac{\lambda}{p}\bm{I}_{r}\Big)\bm{V}^{\top}\Big\Vert _{\mathrm{F}}\leq\frac{{1}}{2n^{4}}\frac{{\lambda}}{p}.
\]
It then suffices to prove that 
\[
\left\Vert \bm{X}^{\mathsf{ncvx}}\bm{Y}^{\mathsf{ncvx}\top}-\frac{1}{p}\mathcal{P}_{T}\mathcal{P}_{\Omega}\left(\bm{X}^{\mathsf{ncvx}}\bm{Y}^{\mathsf{ncvx}\top}-\bm{M}\right)-\bm{U}\Big(\bm{\Sigma}+\frac{\lambda}{p}\bm{I}_{r}\Big)\bm{V}^{\top}\right\Vert _{\mathrm{F}}\leq\frac{{1}}{2n^{4}}\frac{{\lambda}}{p}.
\]
To see this, it has been established in Appendix~\ref{sec:Proof-of-Lemma-connection-1} 
that
\begin{align*}
\bm{X}^{\mathsf{ncvx}}\bm{Y}^{\mathsf{ncvx}\top}-\frac{1}{p}\mathcal{P}_{T}\mathcal{P}_{\Omega}\left(\bm{X}^{\mathsf{ncvx}}\bm{Y}^{\mathsf{ncvx}\top}-\bm{M}\right) & =\bm{U}\bm{\Sigma}\bm{V}^{\top}-\frac{1}{p}\mathcal{P}_{T}\left(-\lambda\bm{U}\bm{V}^{\top}+\bm{R}\right)\\
 & =\bm{U}\bm{\Sigma}\bm{V}^{\top}+\frac{\lambda}{p}\bm{U}\bm{V}^{\top}-\frac{1}{p}\mathcal{P}_{T}\left(\bm{R}\right)\\
 & =\bm{U}\Big(\bm{\Sigma}+\frac{\lambda}{p}\bm{I}_{r}\Big)\bm{V}^{\top}-\frac{1}{p}\mathcal{P}_{T}\left(\bm{R}\right).
\end{align*}
This together with the fact $\|\mathcal{P}_{T}(\bm{R})\|_{\mathrm{F}}\leq\frac{{72\kappa}}{n^{5}}\lambda$
(cf.~(\ref{eq:R-T-T-perp-bound})) and the assumption $n\gg \kappa$ immediately completes the proof.

\section{Analysis of the low-rank factors}

\subsection{Proof of Lemma~\ref{lemma:decomp-low-rank} \label{subsec:Proof-of-Lemma-decomposition-low-rank}}

We concentrate on the factor $\bm{X}^{\mathsf{d}}$; the other factor
$\bm{Y}^{\mathsf{d}}$ can be treated similarly. By definition of
the gradient, one has 
\begin{equation}
\nabla_{\bm{X}}f\left(\bm{X},\bm{Y}\right)=\frac{{1}}{p}\mathcal{P}_{\Omega}\left(\bm{X}\bm{Y}^{\top}-\bm{M}\right)\bm{Y}+\frac{{\lambda}}{p}\bm{X}.\label{eq:first-order-optimality-condition-X}
\end{equation}
Making use of the decomposition 
\begin{equation}
\frac{1}{p}\mathcal{P}_{\Omega}\left(\bm{X}\bm{Y}^{\top}-\bm{M}\right)=\bm{X}\bm{Y}^{\top}-\bm{X}^{\star}\bm{Y}^{\star\top}+\bm{A}-\frac{1}{p}\mathcal{P}_{\Omega}\left(\bm{E}\right)\label{eq:defn-X-rank-r}
\end{equation}
with $\bm{A}$ defined in~(\ref{eq:defn-A}), we can rearrange~(\ref{eq:first-order-optimality-condition-X})
as follows 
\begin{align}
\bm{X}\Big(\bm{Y}^{\top}\bm{Y}+\frac{\lambda}{p}\bm{I}_{r}\Big) & =\bm{X}^{\star}\bm{Y}^{\star\top}\bm{Y}+\frac{1}{p}\mathcal{P}_{\Omega}\left(\bm{E}\right)\bm{Y}-\bm{A}\bm{Y}+\nabla_{\bm{X}}f\left(\bm{X},\bm{Y}\right).\label{eq:first-order-rearrange-X}
\end{align}

By construction, the de-shrunken estimator $\bm{Y}^{\mathsf{d}}$
satisfies the following property 
\begin{align}
\bm{Y}^{\mathsf{d}\top}\bm{Y}^{\mathsf{d}} & =\Big(\bm{{I}}_{r}+\frac{\lambda}{p}\left(\bm{Y}^{\top}\bm{Y}\right)^{-1}\Big)^{1/2}\left(\bm{Y}^{\top}\bm{Y}\right)\Big(\bm{{I}}_{r}+\frac{\lambda}{p}\left(\bm{Y}^{\top}\bm{Y}\right)^{-1}\Big)^{1/2}\nonumber \\
 & =\Big(\bm{{I}}_{r}+\frac{\lambda}{p}\left(\bm{Y}^{\top}\bm{Y}\right)^{-1}\Big)^{1/2}\left(\bm{Y}^{\top}\bm{Y}\right)^{\frac{1}{2}}\left(\bm{Y}^{\top}\bm{Y}\right)^{\frac{1}{2}}\Big(\bm{{I}}_{r}+\frac{\lambda}{p}\left(\bm{Y}^{\top}\bm{Y}\right)^{-1}\Big)^{1/2}\nonumber \\
 & =\bm{Y}^{\top}\bm{Y}+\frac{\lambda}{p}\bm{{I}}_{r},\label{eq:relation-Yd-Y}
\end{align}
where the last identity follows since $(\bm{Y}^{\top}\bm{Y})^{1/2}$
and $(\bm{I}_{r}+\frac{\lambda}{p}(\bm{Y}^{\top}\bm{Y})^{-1})^{1/2}$
commute. Combining~(\ref{eq:first-order-rearrange-X})
with the identity~(\ref{eq:relation-Yd-Y}) gives 
\begin{align}
\bm{X}\left(\bm{Y}^{\mathsf{d}\top}\bm{Y}^{\mathsf{d}}\right) & =\bm{X}^{\star}\bm{Y}^{\star\top}\bm{Y}+\frac{1}{p}\mathcal{P}_{\Omega}\left(\bm{E}\right)\bm{Y}-\bm{A}\bm{Y}+\nabla_{\bm{X}}f\left(\bm{X},\bm{Y}\right).\label{eq:X-equality-first}
\end{align}
Multiplying both sides of~(\ref{eq:X-equality-first}) by $(\bm{I}_{r}+\frac{\lambda}{p}(\bm{Y}^{\top}\bm{Y})^{-1})^{1/2}$
and recalling the definition of $\bm{Y}^{\mathsf{d}}$ in~(\ref{eq:defn-Xd-Yd}), we have 
\begin{align}
 & \bm{X}\left(\bm{Y}^{\mathsf{d}\top}\bm{Y}^{\mathsf{d}}\right)\Big(\bm{{I}}_{r}+\frac{\lambda}{p}\left(\bm{Y}^{\top}\bm{Y}\right)^{-1}\Big)^{1/2}\nonumber \\
 & \quad=\bm{X}^{\star}\bm{Y}^{\star\top}\bm{Y}^{\mathsf{d}}+\frac{1}{p}\mathcal{P}_{\Omega}\left(\bm{E}\right)\bm{Y}^{\mathsf{d}}-\bm{A}\bm{Y}^{\mathsf{d}}+\nabla_{\bm{X}}f\left(\bm{X},\bm{Y}\right)\Big(\bm{{I}}_{r}+\frac{\lambda}{p}\left(\bm{Y}^{\top}\bm{Y}\right)^{-1}\Big)^{1/2}.\label{eq:X-Xd-eqn-1}
\end{align}
Since $\bm{Y}^{\mathsf{d}\top}\bm{Y}^{\mathsf{d}}$ and $(\bm{I}_{r}+\frac{\lambda}{p}(\bm{Y}^{\top}\bm{Y})^{-1})^{1/2}$
also commute, we have 
\begin{align}
\bm{X}\left(\bm{Y}^{\mathsf{d}\top}\bm{Y}^{\mathsf{d}}\right)\Big(\bm{{I}}_{r}+\frac{\lambda}{p}\left(\bm{Y}^{\top}\bm{Y}\right)^{-1}\Big)^{1/2} & =\bm{X}\Big(\bm{{I}}_{r}+\frac{\lambda}{p}\left(\bm{Y}^{\top}\bm{Y}\right)^{-1}\Big)^{1/2}\left(\bm{Y}^{\mathsf{d}\top}\bm{Y}^{\mathsf{d}}\right)\nonumber \\
 & =\bm{X}\Big(\bm{{I}}_{r}+\frac{\lambda}{p}\left(\bm{X}^{\top}\bm{X}\right)^{-1}\Big)^{1/2}\left(\bm{Y}^{\mathsf{d}\top}\bm{Y}^{\mathsf{d}}\right)-\bm{X}\bm{\Delta}_{\mathsf{balancing}}\left(\bm{Y}^{\mathsf{d}\top}\bm{Y}^{\mathsf{d}}\right)\nonumber \\
 & =\bm{X}^{\mathsf{d}}\left(\bm{Y}^{\mathsf{d}\top}\bm{Y}^{\mathsf{d}}\right)-\bm{X}\bm{\Delta}_{\mathsf{balancing}}\left(\bm{Y}^{\mathsf{d}\top}\bm{Y}^{\mathsf{d}}\right),\label{eq:X-Xd-eqn-2-1}
\end{align}
where the last relation uses the definition of $\bm{X}^{\mathsf{d}}$
(see~(\ref{eq:defn-Xd-Yd})).

Substituting the identity~(\ref{eq:X-Xd-eqn-2-1}) back into~(\ref{eq:X-Xd-eqn-1})
and making a few elementary algebraic manipulations yield the desired
decomposition~(\ref{eq:X-d-identity}).

\subsection{Proof of Lemma \ref{lem:Phi1-two-infty}\label{subsec:Proof-of-Lemma-Phi1}}

Recall that $\overline{\bm{Y}}^{\mathsf{d}}=\bm{Y}^{\mathsf{d}}\bm{H}^{\mathsf{d}}$ and similarly define
\[
\overline{\bm{Y}}^{\mathsf{d},(j)}\triangleq\bm{Y}^{\mathsf{d},(j)}\bm{H}^{\mathsf{d},(j)}.
\]
The triangle inequality tells us that for any fixed $1\leq j\leq n$,
\begin{align*}
 \left\Vert \bm{e}_{j}^{\top}\bm{\Phi}_{1}\right\Vert _{2}&\leq\underbrace{\left\Vert \bm{e}_{j}^{\top}\frac{1}{p}\mathcal{P}_{\Omega}\left(\bm{E}\right)\left[\overline{\bm{Y}}^{\mathsf{d},(j)}\big(\overline{\bm{Y}}^{\mathsf{d},(j)\top}\overline{\bm{Y}}^{\mathsf{d},(j)}\big)^{-1}-\bm{Y}^{\star}\left(\bm{Y}^{\star\top}\bm{Y}^{\star}\right)^{-1}\right]\right\Vert _{2}}_{:=\alpha_{1}}\\
 &\quad+\underbrace{\left\Vert \bm{e}_{j}^{\top}\frac{1}{p}\mathcal{P}_{\Omega}\left(\bm{E}\right)\left[\overline{\bm{Y}}^{\mathsf{d}}\big(\overline{\bm{Y}}^{\mathsf{d}\top}\overline{\bm{Y}}^{\mathsf{d}}\big)^{-1}-\overline{\bm{Y}}^{\mathsf{d},(j)}\big(\overline{\bm{Y}}^{\mathsf{d},(j)\top}\overline{\bm{Y}}^{\mathsf{d},(j)}\big)^{-1}\right]\right\Vert _{2}}_{:=\alpha_{2}}.
\end{align*}
In what follows, we shall control $\alpha_{1}$ and $\alpha_{2}$
separately.
\begin{enumerate}
\item To begin with, denoting $\bm{\Delta}^{(j)}\triangleq\overline{\bm{Y}}^{\mathsf{d},(j)}(\overline{\bm{Y}}^{\mathsf{d},(j)\top}\overline{\bm{Y}}^{\mathsf{d},(j)})^{-1}-\bm{Y}^{\star}(\bm{Y}^{\star\top}\bm{Y}^{\star})^{-1}$
results in 
\begin{equation}
\alpha_{1}=\left\Vert \bm{e}_{j}^{\top}\frac{1}{p}\mathcal{P}_{\Omega}\left(\bm{E}\right)\bm{\Delta}^{(j)}\right\Vert _{2}=\bigg\Vert \frac{1}{p}\sum_{k=1}^{n}E_{jk}\delta_{jk}\bm{\Delta}_{k,\cdot}^{(j)}\bigg\Vert _{2}.\label{eq:1st-pert-theta-1}
\end{equation}
Before proceeding, we gather a few useful facts regarding $\bm{\Delta}^{(j)}$,
as summarized in the following claim. \begin{claim}\label{claim:Delta-d-j}With
probability at least $1-O(n^{-11})$, we have 
\begin{align*}
\big\|\bm{\Delta}^{(j)}\big\| & \lesssim\frac{1}{\sqrt{\sigma_{\min}}}\cdot\frac{\sigma}{\sigma_{\min}}\sqrt{\frac{\kappa^{3}n}{p}},\\
\big\|\bm{\Delta}^{(j)}\big\|_{2,\infty} & \lesssim\frac{1}{\sqrt{\sigma_{\min}}}\cdot\frac{\sigma}{\sigma_{\min}}\sqrt{\frac{\kappa^{5}\mu r\log n}{p}}.
\end{align*}
\end{claim}With the bounds on $\|\bm{\Delta}^{(j)}\|$ and $\|\bm{\Delta}^{(j)}\|_{2,\infty}$
in place, we are ready to control $\alpha_{1}$. By construction,
$\bm{\Delta}^{(j)}$ is independent of $\bm{e}_{j}^{\top}\mathcal{P}_{\Omega}\left(\bm{E}\right)$.
Therefore, the vector on the right-hand side of~(\ref{eq:1st-pert-theta-1}),
$\frac{1}{p}\sum_{k=1}^{n}E_{jk}\delta_{jk}\bm{\Delta}_{k,\cdot}^{(j)}$,
is a sum of conditionally independent random vectors. In particular,
conditional on $\bm{\Delta}^{(j)}$ and $\{\delta_{jk}\}_{k:1\leq k\leq n}$,
one has 
\begin{equation}
\frac{1}{p}\sum_{k=1}^{n}E_{jk}\delta_{jk}\bm{\Delta}_{k,\cdot}^{(j)}\,\Big|\,\bm{\Delta}^{(j)},\{\delta_{jk}\}_{k:1\leq k\leq n}\ \sim\ \mathcal{N}\Big(\bm{0},\underbrace{\frac{\sigma^{2}}{p^{2}}\sum_{k=1}^{n}\delta_{jk}\bm{\Delta}_{k,\cdot}^{(j)\top}\bm{\Delta}_{k,\cdot}^{(j)}}_{:=\hat{\bm{\Sigma}}}\Big).\label{eq:1st-pert-cond-gaussian}
\end{equation}
Invoke the concentration inequality for Gaussian random vectors \cite[Proposition 1.1]{MR2994877}
to see that 
\begin{align}
\alpha_{1} & \leq\sqrt{\mathsf{Tr}\big(\hat{\bm{\Sigma}}\big)+2\sqrt{t}\big\|\hat{\bm{\Sigma}}\big\|_{\mathrm{F}}+2\big\|\hat{\bm{\Sigma}}\big\| t}\leq\sqrt{r\big\|\hat{\bm{\Sigma}}\big\|+2\sqrt{rt}\big\|\hat{\bm{\Sigma}}\big\|+2\big\|\hat{\bm{\Sigma}}\big\| t}\nonumber \\
 & \lesssim\sqrt{\big\|\hat{\bm{\Sigma}}\big\|} \left(\sqrt{r}+\sqrt{t}\right)
	\label{eq:upper-bound-theta-1}
\end{align}
with probability at least $1-e^{-t}$. It remains to control $\|\hat{\bm{\Sigma}}\|$,
which we state in the following claim. \begin{claim}\label{claim:1st-pert-sigma-upper}Suppose
that $n^{2}p\gg\kappa^{2}\mu rn\log^{2}n$. Then with probability
exceeding $1-O(n^{-11})$, 
\[
\big\|\hat{\bm{\Sigma}}\big\|\lesssim\frac{\sigma^{2}}{p}\left(\frac{1}{\sqrt{\sigma_{\min}}}\frac{\sigma}{\sigma_{\min}}\sqrt{\frac{\kappa^{3}n}{p}}\right)^{2}.
\]
\end{claim}Combine the upper bound on $\|\hat{\bm{\Sigma}}\|$ with
(\ref{eq:upper-bound-theta-1}) and choose $t\asymp\log n$ to arrive
at 
\[
\alpha_{1}\lesssim\sqrt{\big\|\hat{\bm{\Sigma}}\big\|}\left(\sqrt{r}+\sqrt{\log n}\right)\lesssim\frac{\sigma}{\sqrt{p}}\frac{1}{\sqrt{\sigma_{\min}}}\frac{\sigma}{\sigma_{\min}}\sqrt{\frac{\kappa^{3}rn\log n}{p}}
\]
with probability exceeding $1-O(n^{-11})$.
\item We move on to bounding $\alpha_{2}$, for which we have 
\begin{align}
\alpha_{2} & \leq\frac{1}{p}\left\Vert \mathcal{P}_{\Omega}\left(\bm{E}\right)\right\Vert \left\Vert \overline{\bm{Y}}^{\mathsf{d}}\big(\overline{\bm{Y}}^{\mathsf{d}\top}\overline{\bm{Y}}^{\mathsf{d}}\big)^{-1}-\overline{\bm{Y}}^{\mathsf{d},(j)}\big(\overline{\bm{Y}}^{\mathsf{d},(j)\top}\overline{\bm{Y}}^{\mathsf{d},(j)}\big)^{-1}\right\Vert \nonumber \\
 & \overset{(\text{i})}{\lesssim}\sigma\sqrt{\frac{n}{p}}\frac{1}{\sigma_{\min}}\big\|\overline{\bm{Y}}^{\mathsf{d}}-\overline{\bm{Y}}^{\mathsf{d},(j)}\big\|\nonumber \\
 & \overset{(\text{ii})}{\lesssim}\sigma\sqrt{\frac{n}{p}}\frac{1}{\sigma_{\min}}\kappa\frac{\sigma}{\sigma_{\min}}\sqrt{\frac{n\log n}{p}}\left\Vert \bm{Y}^{\star}\right\Vert _{2,\infty}\label{eq:useful_entry_theta2'}\\
 & \lesssim\sigma\sqrt{\frac{n}{p}}\frac{1}{\sqrt{\sigma_{\min}}}\frac{\sigma}{\sigma_{\min}}\sqrt{\frac{\kappa^{3}\mu r\log n}{p}}.\nonumber 
\end{align}
Here (i) uses the fact that $\|\mathcal{P}_{\Omega}(\bm{E})\|\lesssim\sigma\sqrt{np}$
(see \cite[Lemma 3]{chen2019noisy}), the perturbation bounds for
pseudo-inverses (see~Lemma~\ref{lemma:pseudo-inverse}) and~(\ref{eq:F-norm-upper-bound}); the penultimate
inequality (ii) comes from the fact that $\|\overline{\bm{Y}}^{\mathsf{d}}-\overline{\bm{Y}}^{\mathsf{d},(j)}\|\lesssim\kappa\frac{\sigma}{\sigma_{\min}}\sqrt{\frac{n\log n}{p}}\|\bm{Y}^{\star}\|_{2,\infty}$
(see~(\ref{eq:F-d-j-F-d-dist})) and last one uses the incoherence condition $\|\bm{Y}^{\star}\|_{2,\infty}\leq \sqrt{\mu r \sigma_{\max} / n}$ (cf.~(\ref{eq:incoherence-X})).
\end{enumerate}
Combine the bounds on $\alpha_{1}$ and $\alpha_{2}$ to reach 
\begin{align}
 \left\Vert \bm{e}_{j}^{\top}\bm{\Phi}_{1}\right\Vert _{2}&\lesssim\frac{\sigma}{\sqrt{p}}\frac{1}{\sqrt{\sigma_{\min}}}\frac{\sigma}{\sigma_{\min}}\sqrt{\frac{\kappa^{3}rn\log n}{p}}+\sigma\sqrt{\frac{n}{p}}\frac{1}{\sqrt{\sigma_{\min}}}\frac{\sigma}{\sigma_{\min}}\sqrt{\frac{\kappa^{3}\mu r\log n}{p}}\nonumber \\
 & \lesssim\frac{\sigma}{\sqrt{p\sigma_{\min}}}\cdot\frac{\sigma}{\sigma_{\min}}\sqrt{\frac{\kappa^{3}\mu rn\log n}{p}}.\label{eq:1st-pert-final-bound}
\end{align}
Taking the maximum over $1\leq j\leq n$ establishes our bound on
$\|\bm{\Phi}_{1}\|_{2,\infty}$.

\begin{proof}[Proof of Claim \ref{claim:Delta-d-j}]Apply the perturbation
bound for pseudo-inverses (see~Lemma~\ref{lemma:pseudo-inverse})
to obtain 
\begin{align}
\big\|\bm{\Delta}^{(j)}\big\| & \lesssim\max\left\{ \big\|\bm{Y}^{\star}\left(\bm{Y}^{\star\top}\bm{Y}^{\star}\right)^{-1}\big\|^{2},\big\|\overline{\bm{Y}}^{\mathsf{d},(j)}\big(\overline{\bm{Y}}^{\mathsf{d},(j)\top}\overline{\bm{Y}}^{\mathsf{d},(j)}\big)^{-1}\big\|^{2}\right\} \big\|\overline{\bm{Y}}^{\mathsf{d},(j)}-\bm{Y}^{\star}\big\|\nonumber \\
 & \lesssim\frac{1}{\sigma_{\min}}\kappa\frac{\sigma}{\sigma_{\min}}\sqrt{\frac{n}{p}}\left\Vert \bm{X}^{\star}\right\Vert \asymp\frac{1}{\sqrt{\sigma_{\min}}}\frac{\sigma}{\sigma_{\min}}\sqrt{\frac{\kappa^{3}n}{p}},\label{eq:useful_entry_theta2}
\end{align}
Here we have utilized the facts that $\|\overline{\bm{Y}}^{\mathsf{d},(j)}-\bm{Y}^{\star}\|\lesssim\kappa\frac{\sigma}{\sigma_{\min}}\sqrt{\frac{n}{p}}\|\bm{X}^{\star}\|$
(see~(\ref{eq:F-d-j-op})) and a simple consequence of~(\ref{eq:F-norm-upper-bound}), viz. 
\[
\max\left\{ \big\|\bm{Y}^{\star}\left(\bm{Y}^{\star\top}\bm{Y}^{\star}\right)^{-1}\big\|^{2},\big\|\overline{\bm{Y}}^{\mathsf{d},(j)}\big(\overline{\bm{Y}}^{\mathsf{d},(j)\top}\overline{\bm{Y}}^{\mathsf{d},(j)}\big)^{-1}\big\|^{2}\right\} \lesssim\frac{1}{\sigma_{\min}}.
\]
Moreover, the triangle inequality tells us that 
\begin{align*}
\big\|\bm{\Delta}^{(j)}\big\|_{2,\infty} & \leq\left\Vert \overline{\bm{Y}}^{\mathsf{d},(j)}\left[\big(\overline{\bm{Y}}^{\mathsf{d},(j)\top}\overline{\bm{Y}}^{\mathsf{d},(j)}\big)^{-1}-\left(\bm{Y}^{\star\top}\bm{Y}^{\star}\right)^{-1}\right]\right\Vert _{2,\infty}+\left\Vert \big(\overline{\bm{Y}}^{\mathsf{d},(j)}-\bm{Y}^{\star}\big)\left(\bm{Y}^{\star\top}\bm{Y}^{\star}\right)^{-1}\right\Vert _{2,\infty}\\
 & \leq \big\Vert \overline{\bm{Y}}^{\mathsf{d},(j)}\big\Vert _{2,\infty} \left\Vert \big(\overline{\bm{Y}}^{\mathsf{d},(j)\top}\overline{\bm{Y}}^{\mathsf{d},(j)}\big)^{-1}-\left(\bm{Y}^{\star\top}\bm{Y}^{\star}\right)^{-1}\right\Vert 
	+ \big\Vert \overline{\bm{Y}}^{\mathsf{d},(j)}-\bm{Y}^{\star}\big\Vert _{2,\infty} \big\Vert \left(\bm{Y}^{\star\top}\bm{Y}^{\star}\right)^{-1}\big\Vert \\
 & \lesssim\frac{1}{\sigma_{\min}}\kappa^{2}\frac{\sigma}{\sigma_{\min}}\sqrt{\frac{n}{p}}\left\Vert \bm{F}^{\star}\right\Vert _{2,\infty}+\frac{1}{\sigma_{\min}}\frac{\sigma}{\sigma_{\min}}\sqrt{\frac{\kappa^{2}n\log n}{p}}\left\Vert \bm{F}^{\star}\right\Vert _{2,\infty}\\
 & \lesssim\frac{1}{\sqrt{\sigma_{\min}}}\frac{\sigma}{\sigma_{\min}}\sqrt{\frac{\kappa^{5}\mu r\log n}{p}},
\end{align*}
where the penultimate inequality follows from the facts that $\|\overline{\bm{Y}}^{\mathsf{d},(j)}\|_{2,\infty}\leq2\|\bm{F}^{\star}\|_{2,\infty}$,
$\|\overline{\bm{Y}}^{\mathsf{d},(j)}-\bm{Y}^{\star}\|_{2,\infty}\lesssim\kappa\frac{\sigma}{\sigma_{\min}}\sqrt{\frac{n\log n}{p}}\|\bm{F}^{\star}\|_{2,\infty}$
(see~(\ref{eq:F-d-j-2-infty})) and that 
\begin{align*}
\left\Vert \big(\overline{\bm{Y}}^{\mathsf{d},(j)\top}\overline{\bm{Y}}^{\mathsf{d},(j)}\big)^{-1}-\left(\bm{Y}^{\star\top}\bm{Y}^{\star}\right)^{-1}\right\Vert  & \leq\left\Vert \big(\overline{\bm{Y}}^{\mathsf{d},(j)\top}\overline{\bm{Y}}^{\mathsf{d},(j)}\big)^{-1}\right\Vert \left\Vert \overline{\bm{Y}}^{\mathsf{d},(j)\top}\overline{\bm{Y}}^{\mathsf{d},(j)}-\bm{Y}^{\star\top}\bm{Y}^{\star}\right\Vert \big\Vert \left(\bm{Y}^{\star\top}\bm{Y}^{\star}\right)^{-1}\big\Vert \\
 & \lesssim\frac{1}{\sigma_{\min}^{2}}\left\Vert \bm{F}^{\mathsf{d},(j)}\bm{H}^{\mathsf{{d}},(j)}-\bm{F}^{\star}\right\Vert \big\Vert \bm{F}^{\star}\big\Vert \\
 & \lesssim\frac{1}{\sigma_{\min}}\kappa^{2}\frac{\sigma}{\sigma_{\min}}\sqrt{\frac{n}{p}}.
\end{align*}
Here the penultimate inequality follows from~(\ref{eq:F-norm-upper-bound}). The proof of the claim is then complete. \end{proof}

\begin{proof}[Proof of Claim \ref{claim:1st-pert-sigma-upper}]Conditional
on $\bm{{\Delta}}^{(j)}$, using Bernstein's inequality and the fact
that $\bm{\Delta}^{(j)}$ and $\{\delta_{jk}\}_{k:1\leq k\leq n}$ are
independent, we arrive at that with probability exceeding $1-O(n^{-11})$,
\[
\left\Vert \hat{\bm{\Sigma}}-\frac{\sigma^{2}}{p}\bm{\Delta}^{(j)\top}\bm{\Delta}^{(j)}\right\Vert \lesssim\frac{\sigma^{2}}{p^{2}}\left(\sqrt{V\log n}+B\log n\right),
\]
where 
\begin{align*}
B & \triangleq\max_{1\leq k\leq n}\left\Vert \left(\delta_{jk}-p\right)\bm{\Delta}_{k,\cdot}^{(j)\top}\bm{\Delta}_{k,\cdot}^{(j)}\right\Vert \leq\big\|\bm{\Delta}^{(j)}\big\|_{2,\infty}^{2},\\
V & \triangleq\left\Vert \sum_{k=1}^{n}\mathbb{E}\left(\delta_{jk}-p\right)^{2}\bm{\Delta}_{k,\cdot}^{(j)\top}\bm{\Delta}_{k,\cdot}^{(j)}\bm{\Delta}_{k,\cdot}^{(j)\top}\bm{\Delta}_{k,\cdot}^{(j)}\right\Vert \leq p\big\|\bm{\Delta}^{(j)}\big\|_{2,\infty}^{2}\big\|\bm{\Delta}^{(j)}\big\|^{2}.
\end{align*}
As a result, with probability at least $1-O(n^{-11})$, we have 
\begin{align*}
\left\Vert \hat{\bm{\Sigma}}-\frac{\sigma^{2}}{p}\bm{\Delta}^{(j)\top}\bm{\Delta}^{(j)}\right\Vert  & \lesssim\frac{\sigma^{2}}{p^{2}}\big\|\bm{\Delta}^{(j)}\big\|_{2,\infty}\left(\sqrt{p\log n}\big\|\bm{\Delta}^{(j)}\big\|+\big\|\bm{\Delta}^{(j)}\big\|_{2,\infty}\log n\right)\\
 & \lesssim\frac{\sigma^{2}}{p^{2}}\big\|\bm{\Delta}^{(j)}\big\|_{2,\infty}\left(\sqrt{p\log n}\frac{1}{\sqrt{\sigma_{\min}}}\cdot\frac{\sigma}{\sigma_{\min}}\sqrt{\frac{\kappa^{3}n}{p}}+\frac{1}{\sqrt{\sigma_{\min}}}\cdot\frac{\sigma}{\sigma_{\min}}\sqrt{\frac{\kappa^{5}\mu r\log n}{p}}\log n\right)\\
 & \lesssim\frac{\sigma^{2}}{p^{2}}\big\|\bm{\Delta}^{(j)}\big\|_{2,\infty}\frac{1}{\sqrt{\sigma_{\min}}}\frac{\sigma}{\sigma_{\min}}\sqrt{\kappa^{3}n\log n},
\end{align*}
as long as $np\gg\kappa^{2}\mu r\log^{2}n$. Here the middle inequality
uses Claim~\ref{claim:Delta-d-j}. In view of the triangle inequality,
\begin{align*}
\big\|\hat{\bm{\Sigma}}\big\| & \leq\left\Vert \frac{\sigma^{2}}{p}\bm{\Delta}^{(j)\top}\bm{\Delta}^{(j)}\right\Vert +O\left(\frac{\sigma^{2}}{p^{2}}\big\|\bm{\Delta}^{(j)}\big\|_{2,\infty}\frac{1}{\sqrt{\sigma_{\min}}}\frac{\sigma}{\sigma_{\min}}\sqrt{\kappa^{3}n\log n}\right)\\
 & \lesssim\frac{\sigma^{2}}{p}\left(\big\|\bm{\Delta}^{(j)}\big\|^{2}+\frac{1}{p}\big\|\bm{\Delta}^{(j)}\big\|_{2,\infty}\frac{1}{\sqrt{\sigma_{\min}}}\frac{\sigma}{\sigma_{\min}}\sqrt{\kappa^{3}n\log n}\right)\\
 & \lesssim\frac{\sigma^{2}}{p}\left(\big\|\bm{\Delta}^{(j)}\big\|^{2}+\frac{1}{p}\big\|\bm{\Delta}^{(j)}\big\|_{2,\infty}\frac{1}{\sqrt{\sigma_{\min}}}\frac{\sigma}{\sigma_{\min}}\sqrt{\kappa^{3}n\log n}\right)\\
 & \lesssim\frac{\sigma^{2}}{p}\left(\frac{1}{\sqrt{\sigma_{\min}}}\frac{\sigma}{\sigma_{\min}}\sqrt{\frac{\kappa^{3}n}{p}}\right)^{2},
\end{align*}
with the proviso that $n^{2}p\gg\kappa^{2}\mu rn\log^{2}n$. Again,
the last line makes use of Claim~\ref{claim:Delta-d-j}. This concludes
the proof of the claim. \end{proof}

\subsection{Proof of Lemma \ref{lem:Phi2-two-infty}\label{subsec:Proof-of-Lemma-Phi2}}

Recall that $\overline{\bm{Y}}^{\mathsf{d}}=\bm{Y}^{\mathsf{d}}\bm{H}^{\mathrm{d}}$.
The sub-multiplicativity of the operator norm gives that for any $1\leq j\leq n$, 
\begin{align}
\left\Vert \bm{e}_{j}^{\top}\bm{\Phi}_{2}\right\Vert _{2} & =\left\Vert \bm{e}_{j}^{\top}\bm{X}^{\star}\left[\bm{Y}^{\star\top}\overline{\bm{Y}}^{\mathsf{d}}\big(\overline{\bm{Y}}^{\mathsf{d}\top}\overline{\bm{Y}}^{\mathsf{d}}\big)^{-1}-\overline{\bm{Y}}^{\mathsf{d}\top}\overline{\bm{Y}}^{\mathsf{d}}\big(\overline{\bm{Y}}^{\mathsf{d}\top}\overline{\bm{Y}}^{\mathsf{d}}\big)^{-1}\right]\right\Vert _{2}\nonumber \\
 & \leq\left\Vert \bm{e}_{j}^{\top}\bm{X}^{\star}\right\Vert _{2}\left\Vert \big(\bm{Y}^{\star}-\overline{\bm{Y}}^{\mathsf{d}}\big)^{\top}\overline{\bm{Y}}^{\mathsf{d}}\right\Vert \left\Vert \big(\overline{\bm{Y}}^{\mathsf{d}\top}\overline{\bm{Y}}^{\mathsf{d}}\big)^{-1}\right\Vert \\
 & \lesssim\sqrt{\frac{\mu r \sigma_{\max}}{n}}\frac{1}{\sigma_{\min}}\big\Vert \big(\bm{Y}^{\star}-\overline{\bm{Y}}^{\mathsf{d}}\big)^{\top}\overline{\bm{Y}}^{\mathsf{d}}\big\Vert \nonumber \\
 & \asymp\sqrt{\frac{\kappa\mu r}{n}}\frac{1}{\sqrt{\sigma_{\min}}}\big\Vert \big(\bm{Y}^{\star}-\overline{\bm{Y}}^{\mathsf{d}}\big)^{\top}\overline{\bm{Y}}^{\mathsf{d}}\big\Vert ,\label{eq:2nd-pert-master}
\end{align}
where the second inequality follows from the incoherence assumption
that $\|\bm{e}_{j}^{\top}\bm{X}^{\star}\|_{2}\leq\|\bm{X}^{\star}\|_{2,\infty}\leq\sqrt{\mu r\sigma_{\max}/n}$
(cf.~(\ref{eq:incoherence-X})) and the fact that $\|(\overline{\bm{Y}}^{\mathsf{d}\top}\overline{\bm{Y}}^{\mathsf{d}})^{-1}\|\lesssim1/\sigma_{\min}$, a simple consequence of (\ref{eq:F-norm-upper-bound}).

It remains to control $\|(\overline{\bm{Y}}^{\mathsf{d}}-\bm{Y}^{\star})^{\top}\overline{\bm{Y}}^{\mathsf{d}}\|$.
To simplify notation hereafter, define $\bm{\Delta}_{\bm{X}}\triangleq\overline{\bm{X}}^{\mathsf{d}}-\bm{X}^{\star}$
and $\bm{\Delta}_{\bm{Y}}\triangleq\overline{\bm{Y}}^{\mathsf{d}}-\bm{Y}^{\star}$.
First, observe that 
\begin{equation}
\big(\bm{Y}^{\star}-\overline{\bm{Y}}^{\mathsf{d}}\big)^{\top}\overline{\bm{Y}}^{\mathsf{d}}=\bm{\Delta}_{\bm{Y}}^{\top}\bm{Y}^{\star}+\bm{\Delta}_{\bm{Y}}^{\top}\bm{\Delta}_{\bm{Y}}.\label{eq:2nd-pert-relation-1}
\end{equation}
Second, in view of the decomposition of $\bm{Y}^{\mathsf{d}}$ given
in~(\ref{eq:Y-d-identity-1}), we have 
\begin{align}
\overline{\bm{Y}}^{\mathsf{d}} & =\bm{Y}^{\star}\bm{X}^{\star\top}\overline{\bm{X}}^{\mathsf{d}}\big(\overline{\bm{X}}^{\mathsf{d}\top}\overline{\bm{X}}^{\mathsf{d}}\big)^{-1}+\frac{1}{p}\left[\mathcal{P}_{\Omega}\left(\bm{E}\right)\right]^{\top}\overline{\bm{X}}^{\mathsf{d}}\big(\overline{\bm{X}}^{\mathsf{d}\top}\overline{\bm{X}}^{\mathsf{d}}\big)^{-1}-\bm{A}^{\top}\overline{\bm{X}}^{\mathsf{d}}\big(\overline{\bm{X}}^{\mathsf{d}\top}\overline{\bm{X}}^{\mathsf{d}}\big)^{-1}\nonumber \\
 & \quad+\nabla_{\bm{Y}}f\left(\bm{X},\bm{Y}\right)\Big(\bm{I}_{r}+\frac{\lambda}{p}(\bm{X}^{\top}\bm{X})^{-1}\Big)^{1/2}(\bm{X}^{\mathsf{d}\top}\bm{X}^{\mathsf{d}})^{-1}\bm{H}^{\mathsf{d}}-\bm{Y}\bm{\Delta}_{\mathsf{balancing}}\bm{H}^{\mathsf{d}}.\label{eq:Y-d-identity}
\end{align}
As a result, one obtains 
\begin{align}
\bm{\Delta}_{\bm{Y}}^{\top}\bm{Y}^{\star} & =\left\{ \bm{Y}^{\star}\left(\bm{X}^{\star\top}\overline{\bm{X}}^{\mathsf{d}}\big(\overline{\bm{X}}^{\mathsf{d}\top}\overline{\bm{X}}^{\mathsf{d}}\big)^{-1}-\bm{I}_{r}\right)+\frac{1}{p}\left[\mathcal{P}_{\Omega}\left(\bm{E}\right)\right]^{\top}\overline{\bm{X}}^{\mathsf{d}}\left(\overline{\bm{X}}^{\mathsf{d}\top}\overline{\bm{X}}^{\mathsf{d}}\right)^{-1}-\bm{A}^{\top}\overline{\bm{X}}^{\mathsf{d}}\left(\overline{\bm{X}}^{\mathsf{d}\top}\overline{\bm{X}}^{\mathsf{d}}\right)^{-1}\right\} ^{\top}\bm{Y}^{\star}\nonumber \\
 & \quad+\left\{ \nabla_{\bm{Y}}f\left(\bm{X},\bm{Y}\right)\Big(\bm{I}_{r}+\frac{\lambda}{p}(\bm{X}^{\top}\bm{X})^{-1}\Big)^{1/2}(\bm{X}^{\mathsf{d}\top}\bm{X}^{\mathsf{d}})^{-1}\bm{H}^{\mathsf{d}}-\bm{Y}\bm{\Delta}_{\mathsf{balancing}}\bm{H}^{\mathsf{d}}\right\} ^{\top}\bm{Y}^{\star}\nonumber \\
 & =-\big(\overline{\bm{X}}^{\mathsf{d}\top}\overline{\bm{X}}^{\mathsf{d}}\big)^{-1}\overline{\bm{X}}^{\mathsf{d}\top}\bm{\Delta}_{\bm{X}}\bm{\Sigma}^{\star}+\big(\overline{\bm{X}}^{\mathsf{d}\top}\overline{\bm{X}}^{\mathsf{d}}\big)^{-1}\overline{\bm{X}}^{\mathsf{d}\top}\frac{1}{p}\mathcal{P}_{\Omega}\left(\bm{E}\right)\bm{Y}^{\star}-\big(\overline{\bm{X}}^{\mathsf{d}\top}\overline{\bm{X}}^{\mathsf{d}}\big)^{-1}\overline{\bm{X}}^{\mathsf{d}\top}\bm{A}\bm{Y}^{\star}\nonumber \\
 & \quad+\left\{ \nabla_{\bm{Y}}f\left(\bm{X},\bm{Y}\right)\Big(\bm{I}_{r}+\frac{\lambda}{p}(\bm{X}^{\top}\bm{X})^{-1}\Big)^{1/2}(\bm{X}^{\mathsf{d}\top}\bm{X}^{\mathsf{d}})^{-1}\bm{H}^{\mathsf{d}}-\bm{Y}\bm{\Delta}_{\mathsf{balancing}}\bm{H}^{\mathsf{d}}\right\} ^{\top}\bm{Y}^{\star}\nonumber \\
 & =-\big(\overline{\bm{X}}^{\mathsf{d}\top}\overline{\bm{X}}^{\mathsf{d}}\big)^{-1}\bm{X}^{\star\top}\bm{\Delta}_{\bm{X}}\bm{\Sigma}^{\star}+\bm{S},\label{eq:2nd-pert-relation-2}
\end{align}
where we have used 
\[
\bm{X}^{\star\top}\overline{\bm{X}}^{\mathsf{d}}\big(\overline{\bm{X}}^{\mathsf{d}\top}\overline{\bm{X}}^{\mathsf{d}}\big)^{-1}-\bm{I}_{r}=\bm{X}^{\star\top}\overline{\bm{X}}^{\mathsf{d}}\big(\overline{\bm{X}}^{\mathsf{d}\top}\overline{\bm{X}}^{\mathsf{d}}\big)^{-1}-\overline{\bm{X}}^{\mathsf{d}\top}\overline{\bm{X}}^{\mathsf{d}}\big(\overline{\bm{X}}^{\mathsf{d}\top}\overline{\bm{X}}^{\mathsf{d}}\big)^{-1}=-\bm{\Delta}_{\bm{X}}^{\top}\overline{\bm{X}}^{\mathsf{d}}\big(\overline{\bm{X}}^{\mathsf{d}\top}\overline{\bm{X}}^{\mathsf{d}}\big)^{-1}.
\]
Here, we define $\bm{S}$ to be 
\begin{align}
\bm{S} & \triangleq-\big(\overline{\bm{X}}^{\mathsf{d}\top}\overline{\bm{X}}^{\mathsf{d}}\big)^{-1}\bm{\Delta}_{\bm{X}}^{\top}\bm{\Delta}_{\bm{X}}\bm{\Sigma}^{\star}+\big(\overline{\bm{X}}^{\mathsf{d}\top}\overline{\bm{X}}^{\mathsf{d}}\big)^{-1}\overline{\bm{X}}^{\mathsf{d}\top}\frac{1}{p}\mathcal{P}_{\Omega}\left(\bm{E}\right)\bm{Y}^{\star}-\big(\overline{\bm{X}}^{\mathsf{d}\top}\overline{\bm{X}}^{\mathsf{d}}\big)^{-1}\overline{\bm{X}}^{\mathsf{d}\top}\bm{A}\bm{Y}^{\star}\nonumber \\
 & \quad+\left\{ \nabla_{\bm{Y}}f\left(\bm{X},\bm{Y}\right)\Big(\bm{I}_{r}+\frac{\lambda}{p}(\bm{X}^{\top}\bm{X})^{-1}\Big)^{1/2}(\bm{X}^{\mathsf{d}\top}\bm{X}^{\mathsf{d}})^{-1}\bm{H}^{\mathsf{d}}-\bm{Y}\bm{\Delta}_{\mathsf{balancing}}\bm{H}^{\mathsf{d}}\right\} ^{\top}\bm{Y}^{\star}.\label{eq:defn-S}
\end{align}
The following claim connects $\bm{\Delta}_{\bm{Y}}^{\top}\bm{Y}^{\star}$
with $\bm{X}^{\star\top}\bm{\Delta}_{\bm{X}}$. \begin{claim}\label{claim:connection-delta-x-delta-y}The
following identity holds true:
\[
\bm{\Delta}_{\bm{Y}}^{\top}\bm{Y}^{\star}-\bm{X}^{\star\top}\bm{\Delta}_{\bm{X}}=\frac{1}{2}\left(\bm{\Delta}_{\bm{X}}^{\top}\bm{\Delta}_{\bm{X}}-\bm{\Delta}_{\bm{Y}}^{\top}\bm{\Delta}_{\bm{Y}}\right)+\underbrace{\frac{{1}}{2}\bm{{H}}^{\mathsf{{d}}\top}\big(\bm{{Y}}^{\mathsf{d}\top}\bm{{Y}}^{\mathsf{d}}-\bm{{X}}^{\mathsf{d}\top}\bm{{X}}^{\mathsf{d}}\big)\bm{{H}}^{\mathsf{{d}}}}_{:=\bm{{\Delta}}_{\bm{{XY}}}^{\mathsf{{d}}}}.
\]
\end{claim}This relation together with~(\ref{eq:2nd-pert-relation-2})
yields
\[
\bm{\Delta}_{\bm{Y}}^{\top}\bm{Y}^{\star}=-\big(\overline{\bm{X}}^{\mathsf{d}\top}\overline{\bm{X}}^{\mathsf{d}}\big)^{-1}\left[\bm{\Delta}_{\bm{Y}}^{\top}\bm{Y}^{\star}-\frac{1}{2}\left(\bm{\Delta}_{\bm{X}}^{\top}\bm{\Delta}_{\bm{X}}-\bm{\Delta}_{\bm{Y}}^{\top}\bm{\Delta}_{\bm{Y}}\right)-\bm{{\Delta}}_{\bm{{XY}}}^{\mathsf{{d}}}\right]\bm{\Sigma}^{\star}+\bm{S}.
\]
A little algebraic manipulation then gives 
\[
\overline{\bm{X}}^{\mathsf{d}\top}\overline{\bm{X}}^{\mathsf{d}}\bm{\Delta}_{\bm{Y}}^{\top}\bm{Y}^{\star}+\bm{\Delta}_{\bm{Y}}^{\top}\bm{Y}^{\star}\bm{\Sigma}^{\star}=\overline{\bm{X}}^{\mathsf{d}\top}\overline{\bm{X}}^{\mathsf{d}}\bm{S}+\frac{1}{2}\left(\bm{\Delta}_{\bm{X}}^{\top}\bm{\Delta}_{\bm{X}}-\bm{\Delta}_{\bm{Y}}^{\top}\bm{\Delta}_{\bm{Y}}\right)\bm{\Sigma}^{\star}+\bm{{\Delta}}_{\bm{{XY}}}^{\mathsf{{d}}}\bm{{\Sigma}}^{\star}.
\]
It is easy to check from~(\ref{eq:F-norm-upper-bound}) that $0.25\sigma_{\min}\bm{I}_{r}\preceq\overline{\bm{X}}^{\mathsf{d}\top}\overline{\bm{X}}^{\mathsf{d}},\bm{\Sigma}^{\star}\preceq4\sigma_{\max}\bm{I}_{r}$.
Hence one can invoke Lemma~\ref{lemma:AX+XB=00003DC} with $\bm{X}=\bm{\Delta}_{\bm{Y}}^{\top}\bm{Y}^{\star}$,
$\bm{A}=\bm{\Sigma}^{\star}$, $\bm{B}=\overline{\bm{X}}^{\mathsf{d}\top}\overline{\bm{X}}^{\mathsf{d}}$
and $\bm{C}=\overline{\bm{X}}^{\mathsf{d}\top}\overline{\bm{X}}^{\mathsf{d}}\bm{S}+0.5(\bm{\Delta}_{\bm{X}}^{\top}\bm{\Delta}_{\bm{X}}-\bm{\Delta}_{\bm{Y}}^{\top}\bm{\Delta}_{\bm{Y}})\bm{\Sigma}^{\star} +\bm{{\Delta}}_{\bm{{XY}}}^{\mathsf{{d}}}\bm{\Sigma}^{\star}$
to obtain 
\begin{align*}
\left\Vert \bm{\Delta}_{\bm{Y}}^{\top}\bm{Y}^{\star}\right\Vert  & \lesssim\frac{1}{\sigma_{\min}}\left\Vert \overline{\bm{X}}^{\mathsf{d}\top}\overline{\bm{X}}^{\mathsf{d}}\bm{S}+\frac{1}{2}\left(\bm{\Delta}_{\bm{X}}^{\top}\bm{\Delta}_{\bm{X}}-\bm{\Delta}_{\bm{Y}}^{\top}\bm{\Delta}_{\bm{Y}}\right)\bm{\Sigma}^{\star}+\bm{{\Delta}}_{\bm{{XY}}}^{\mathsf{{d}}}\bm{{\Sigma}}^{\star}\right\Vert \\
 & \leq\frac{1}{\sigma_{\min}}\left\Vert \overline{\bm{X}}^{\mathsf{d}\top}\frac{1}{p}\mathcal{P}_{\Omega}\left(\bm{E}\right)\bm{Y}^{\star}-\overline{\bm{X}}^{\mathsf{d}\top}\bm{A}\bm{Y}^{\star}-\frac{1}{2}\left(\bm{\Delta}_{\bm{X}}^{\top}\bm{\Delta}_{\bm{X}}+\bm{\Delta}_{\bm{Y}}^{\top}\bm{\Delta}_{\bm{Y}}\right)\bm{\Sigma}^{\star}\right\Vert \\
 & \quad+\frac{1}{\sigma_{\min}}\left\Vert \bm{H}^{\mathsf{d}\top}\Big(\bm{I}_{r}+\frac{\lambda}{p}(\bm{X}^{\top}\bm{X})^{-1}\Big)^{1/2}\left[\nabla_{\bm{Y}}f\left(\bm{X},\bm{Y}\right)\right]^{\top}\bm{Y}^{\star}+\overline{\bm{X}}^{\mathsf{d}\top}\overline{\bm{X}}^{\mathsf{d}}\bm{H}^{\mathsf{d}\top}\bm{\Delta}_{\mathsf{balancing}}\bm{Y}^{\top}\bm{Y}^{\star}+\bm{{\Delta}}_{\bm{{XY}}}^{\mathsf{{d}}}\bm{{\Sigma}}^{\star}\right\Vert ,
\end{align*}
where we have plugged in the definition of $\bm{S}$ (see~(\ref{eq:defn-S}))
and used the identity $\overline{\bm{X}}^{\mathsf{d}\top}\overline{\bm{X}}^{\mathsf{d}}\bm{H}^{\mathsf{d}\top}(\bm{X}^{\mathsf{d}\top}\bm{X}^{\mathsf{d}})^{-1}=\bm{H}^{\mathsf{d}\top}$.
Combine the above inequality with~(\ref{eq:2nd-pert-relation-1})
to obtain
\begin{align}
 & \left\Vert \big(\overline{\bm{Y}}^{\mathsf{d}}-\bm{Y}^{\star}\big)^{\top}\overline{\bm{Y}}^{\mathsf{d}}\right\Vert \leq\left\Vert \bm{\Delta}_{\bm{Y}}^{\top}\bm{Y}^{\star}\right\Vert +\left\Vert \bm{\Delta}_{\bm{Y}}^{\top}\bm{\Delta}_{\bm{Y}}\right\Vert \nonumber \\
 & \quad\lesssim\frac{1}{\sigma_{\min}}\underbrace{\left\Vert \overline{\bm{X}}^{\mathsf{d}\top}\frac{1}{p}\mathcal{P}_{\Omega}\left(\bm{E}\right)\bm{Y}^{\star}\right\Vert }_{:=\alpha_{1}}+\frac{1}{\sigma_{\min}}\underbrace{\big\Vert \overline{\bm{X}}^{\mathsf{d}\top}\bm{A}\bm{Y}^{\star}\big\Vert }_{:=\alpha_{2}}+\kappa\underbrace{\left(\left\Vert \bm{\Delta}_{\bm{X}}^{\top}\bm{\Delta}_{\bm{X}}\right\Vert +\left\Vert \bm{\Delta}_{\bm{Y}}^{\top}\bm{\Delta}_{\bm{Y}}\right\Vert \right)}_{:=\alpha_{3}}\nonumber \\
 & \quad\quad+\frac{1}{\sigma_{\min}}\underbrace{\left\Vert \bm{H}^{\mathsf{d}\top}\Big(\bm{I}_{r}+\frac{\lambda}{p}(\bm{X}^{\top}\bm{X})^{-1}\Big)^{1/2}\left[\nabla_{\bm{Y}}f\left(\bm{X},\bm{Y}\right)\right]^{\top}\bm{Y}^{\star}-\overline{\bm{X}}^{\mathsf{d}\top}\overline{\bm{X}}^{\mathsf{d}}\bm{H}^{\mathsf{d}\top}\bm{\Delta}_{\mathsf{balancing}}\bm{Y}^{\top}\bm{Y}^{\star}+\bm{{\Delta}}_{\bm{{XY}}}^{\mathsf{{d}}}\bm{{\Sigma}}^{\star}\right\Vert }_{:=\alpha_{4}}.\label{eq:2nd-pert-master-upper-bound}
\end{align}
It then boils down to controlling the above terms $\alpha_{1},\alpha_{2},\alpha_{3}$
and $\alpha_{4}$.
\begin{enumerate}
\item First, the term $\alpha_{4}$ can be upper bounded by 
\begin{align*}
\alpha_{4} & \leq\left\Vert \Big(\bm{I}_{r}+\frac{\lambda}{p}(\bm{X}^{\top}\bm{X})^{-1}\Big)^{1/2}\right\Vert \left\Vert \nabla_{\bm{Y}}f\left(\bm{X},\bm{Y}\right)\right\Vert _{\mathrm{F}}\|\bm{Y}^{\star}\|+\big\|\overline{\bm{X}}^{\mathsf{d}\top}\overline{\bm{X}}^{\mathsf{d}}\big\|\|\bm{\Delta}_{\mathsf{balancing}}\|\|\bm{Y}^{\top}\bm{Y}^{\star}\|+\|\bm{{\Delta}}_{\bm{{XY}}}^{\mathsf{{d}}}\|\|\bm{\Sigma}^{\star}\|\\
 & \lesssim\frac{1}{n^{5}}\frac{\lambda}{p}\sqrt{\sigma_{\min}}\sqrt{\sigma_{\max}}+\sigma_{\max}^{2}\frac{1}{n^{5}}\frac{\lambda}{p}\frac{\kappa}{\sigma_{\min}}+\sigma_{\max}\frac{\kappa}{n^{5}}\frac{\sigma}{\sigma_{\min}}\sqrt{\frac{n}{p}}\sigma_{\max}\\
 & \asymp\frac{\kappa}{n^{5}}\frac{\sigma}{\sigma_{\min}}\sqrt{\frac{n}{p}}\sigma_{\max}^{2}.
\end{align*}
Here, the second line utilizes the facts that $\|(\bm{I}_{r}+\lambda(\bm{X}^{\top}\bm{X})^{-1}/p)^{1/2}\|\lesssim1$,
$\|\overline{\bm{X}}^{\mathsf{d}\top}\overline{\bm{X}}^{\mathsf{d}}\|\asymp\|\bm{Y}^{\top}\bm{Y}^{\star}\|\lesssim\sigma_{\max}$
and the results in~(\ref{eq:small-gradient}),~(\ref{eq:X-d-Y-d-balance})
and~(\ref{eq:debias-correction-close}).
\item Moving on to $\alpha_{3}$, we recall from~(\ref{eq:F-d-op-quality-H-d})
that 
\[
\max\left\{ \left\Vert \bm{\Delta}_{\bm{X}}\right\Vert ,\left\Vert \bm{\Delta}_{\bm{Y}}\right\Vert \right\} \lesssim\kappa\frac{\sigma}{\sigma_{\min}}\sqrt{\frac{n}{p}}\left\Vert \bm{X}^{\star}\right\Vert .
\]
Therefore one arrives at
\[
\alpha_{3}\lesssim\left(\kappa\frac{\sigma}{\sigma_{\min}}\sqrt{\frac{n}{p}}\right)^{2}\sigma_{\max}.
\]
\item Regarding the term $\alpha_{2}$, we have 
\begin{align*}
\alpha_{2} & \lesssim\bigl\Vert\overline{\bm{X}}^{\mathsf{d}}\bigr\Vert\left\Vert \bm{{A}}\right\Vert \left\Vert \bm{{X}}^{\star}\right\Vert \lesssim\sigma_{\max}\sigma\sqrt{\frac{n}{p}}\cdot\sqrt{\frac{\kappa^{4}\mu^{2}r^{2}\log n}{np}},
\end{align*}
where we utilize the bound in~(\ref{eq:A-norm})
\[
\left\Vert \bm{A}\right\Vert \lesssim\sigma\sqrt{\frac{n}{p}}\cdot\sqrt{\frac{\kappa^{4}\mu^{2}r^{2}\log n}{np}}.
\]
\item Finally, for the term $\alpha_{1}$, by the triangle inequality one
has 
\begin{align*}
\alpha_{1} & \leq\left\Vert \bm{X}^{\star\top}\frac{1}{p}\mathcal{P}_{\Omega}\left(\bm{E}\right)\bm{Y}^{\star}\right\Vert +\left\Vert \bm{\Delta}_{\bm{X}}^{\top}\frac{1}{p}\mathcal{P}_{\Omega}\left(\bm{E}\right)\bm{Y}^{\star}\right\Vert .
\end{align*}
Note that 
\begin{align}
\left\Vert \bm{X}^{\star\top}\frac{1}{p}\mathcal{P}_{\Omega}\left(\bm{E}\right)\bm{Y}^{\star}\right\Vert  & \leq\left\Vert \bm{X}^{\star\top}\frac{1}{p}\mathcal{P}_{\Omega}\left(\bm{E}\right)\bm{Y}^{\star}\right\Vert _{\mathrm{F}}=\sqrt{\sum_{i=1}^{r}\sum_{j=1}^{r}\left|\left(\bm{X}_{\cdot,i}^{\star}\right)^{\top}\frac{1}{p}\mathcal{P}_{\Omega}\left(\bm{E}\right)\bm{{Y}}_{\cdot,j}^{\star}\right|^{2}},\label{eq:2nd-pert-upper-1}
\end{align}
Observe that conditional on $\{\delta_{jk}\}_{1\leq j,k\leq n}$ one
has 
\[
\left(\bm{X}_{\cdot,i}^{\star}\right)^{\top}\frac{1}{p}\mathcal{P}_{\Omega}\left(\bm{E}\right)\bm{{Y}}_{\cdot,j}^{\star}=\left\langle \bm{E},\frac{1}{p}\mathcal{P}_{\Omega}\left(\bm{{X}}_{\cdot,i}^{\star}\left(\bm{{Y}}_{\cdot,j}^{\star}\right)^{\top}\right)\right\rangle \sim\mathcal{N}\left(0,\sigma^{2}\left\Vert \frac{1}{p}\mathcal{P}_{\Omega}\left(\bm{{X}}_{\cdot,i}^{\star}\left(\bm{{Y}}_{\cdot,j}^{\star}\right)^{\top}\right)\right\Vert _{\mathrm{F}}^{2}\right)
\]
As a result, we obtain that with probability at least $1-O(n^{-10})$,
\begin{align}
\left|\left(\bm{X}_{\cdot,i}^{\star}\right)^{\top}\frac{1}{p}\mathcal{P}_{\Omega}\left(\bm{E}\right)\bm{{Y}}_{\cdot,j}^{\star}\right| & \lesssim\sigma\left\Vert \frac{1}{p}\mathcal{P}_{\Omega}\left(\bm{{X}}_{\cdot,i}^{\star}\left(\bm{{Y}}_{\cdot,j}^{\star}\right)^{\top}\right)\right\Vert _{\mathrm{F}}\sqrt{\log n}\nonumber \\
 & \lesssim\sigma\sqrt{\frac{\log n}{p}}\left\Vert \bm{{X}}_{\cdot,i}^{\star}\left(\bm{{Y}}_{\cdot,j}^{\star}\right)^{\top}\right\Vert _{\mathrm{F}}.\label{eq:2nd-term-upper-2}
\end{align}
Here, the second relation uses the fact that 
\[
\left\Vert \frac{1}{\sqrt{p}}\mathcal{P}_{\Omega}\left(\bm{{X}}_{\cdot,i}^{\star}\left(\bm{{Y}}_{\cdot,j}^{\star}\right)^{\top}\right)\right\Vert _{\mathrm{F}}\asymp \big\Vert \bm{{X}}_{\cdot,i}^{\star}\left(\bm{{Y}}_{\cdot,j}^{\star}\right)^{\top}\big\Vert _{\mathrm{F}}
\]
with probability at least $1-O(n^{-10})$ as long as $n^{2}p\gg\mu rn\log n$,
which follows from \cite[Lemma 38]{ma2017implicit} or \cite[Section 4.2]{ExactMC09}
by observing that $\bm{{X}}_{\cdot,i}^{\star}(\bm{{Y}}_{\cdot,j}^{\star})^{\top}$
lies in the tangent space of $\bm{{M}}^{\star}$. Take~(\ref{eq:2nd-pert-upper-1})
and~(\ref{eq:2nd-term-upper-2}) collectively to reach 
\begin{align*}
\left\Vert \bm{X}^{\star\top}\frac{1}{p}\mathcal{P}_{\Omega}\left(\bm{E}\right)\bm{Y}^{\star}\right\Vert  & \lesssim\sigma\sqrt{\frac{\log n}{p}}\sqrt{\sum_{i=1}^{r}\sum_{j=1}^{r}\left\Vert \bm{{X}}_{\cdot,i}^{\star}\left(\bm{{Y}}_{\cdot,j}^{\star}\right)^{\top}\right\Vert _{\mathrm{F}}^{2}}\leq\sigma\sqrt{\frac{\log n}{p}}\left\Vert \bm{X}^{\star}\right\Vert _{\mathrm{F}}\left\Vert \bm{Y}^{\star}\right\Vert _{\mathrm{F}}\lesssim\sigma\sqrt{\frac{\log n}{p}}r\sigma_{\max}.
\end{align*}
In addition, we have 
\begin{align*}
\left\Vert \bm{\Delta}_{\bm{X}}^{\top}\frac{1}{p}\mathcal{P}_{\Omega}\left(\bm{E}\right)\bm{Y}^{\star}\right\Vert  & \leq\left\Vert \bm{\Delta}_{\bm{X}}\right\Vert \left\Vert \frac{1}{p}\mathcal{P}_{\Omega}\left(\bm{E}\right)\right\Vert \left\Vert \bm{Y}^{\star}\right\Vert \lesssim\kappa\frac{{\sigma}}{\sigma_{\min}}\sqrt{{\frac{{n}}{p}}}\left\Vert \bm{{X}}^{\star}\right\Vert \sigma\sqrt{{\frac{{n}}{p}}}\left\Vert \bm{{Y}}^{\star}\right\Vert \lesssim\left(\kappa\sigma\sqrt{{\frac{{n}}{p}}}\right)^{2}.
\end{align*}
Combine these two bounds to reach 
\begin{align*}
\alpha_{1} & \lesssim\sigma\sqrt{\frac{r^{2}\log n}{p}}\sigma_{\max}+\left(\kappa\sigma\sqrt{{\frac{{n}}{p}}}\right)^{2}.
\end{align*}
\end{enumerate}
Substituting the bounds on $\alpha_{1}$, $\alpha_{2}$, $\alpha_{3}$
and $\alpha_{4}$ back to~(\ref{eq:2nd-pert-master-upper-bound})
results in 
\begin{align*}
\left\Vert \big(\overline{\bm{Y}}^{\mathsf{d}}-\bm{Y}^{\star}\big)^{\top}\overline{\bm{Y}}^{\mathsf{d}}\right\Vert  & \lesssim\frac{1}{\sigma_{\min}}\alpha_{1}+\frac{1}{\sigma_{\min}}\alpha_{2}+\kappa\alpha_{3}+\frac{1}{\sigma_{\min}}\alpha_{4}\\
 & \lesssim\frac{1}{\sigma_{\min}}\left(\sigma\sqrt{\frac{r^{2}\log n}{p}}\sigma_{\max}+\left(\kappa\sigma\sqrt{{\frac{{n}}{p}}}\right)^{2}+\sigma_{\max}\sigma\sqrt{\frac{n}{p}}\cdot\sqrt{\frac{\kappa^{4}\mu^{2}r^{2}\log n}{np}}\right)\\
 & \quad+\kappa\left(\kappa\frac{\sigma}{\sigma_{\min}}\sqrt{\frac{n}{p}}\right)^{2}\sigma_{\max}+\frac{1}{\sigma_{\min}}\frac{\kappa}{n^{5}}\frac{\sigma}{\sigma_{\min}}\sqrt{\frac{n}{p}}\sigma_{\max}^{2}\\
 & \asymp\kappa\sigma_{\max}\left(\kappa\frac{\sigma}{\sigma_{\min}}\sqrt{\frac{n}{p}}\right)^{2}+\sigma_{\max}\frac{\sigma}{\sigma_{\min}}\sqrt{\frac{n}{p}}\cdot\sqrt{\frac{\kappa^{4}\mu^{2}r^{2}\log n}{np}},
\end{align*}
which together with~(\ref{eq:2nd-pert-master}) yields 
\[
\left\Vert \bm{e}_{j}^{\top}\bm{X}^{\star}\left[\bm{Y}^{\star\top}\overline{\bm{Y}}^{\mathsf{d}}\big(\overline{\bm{Y}}^{\mathsf{d}\top}\overline{\bm{Y}}^{\mathsf{d}}\big)^{-1}-\bm{I}_{r}\right]\right\Vert _{2}\lesssim\frac{\sigma}{\sqrt{p\sigma_{\min}}}\left(\kappa\frac{\sigma}{\sigma_{\min}}\sqrt{\frac{\kappa^{7}\mu rn}{p}}+\sqrt{\frac{\kappa^{7}\mu^{3}r^{3}\log n}{np}}\right).
\]
Taking the maximum over $1\leq j\leq n$ leads to the desired result.

Finally, we are left with proving Claim~\ref{claim:connection-delta-x-delta-y}.

\begin{proof}[Proof of Claim \ref{claim:connection-delta-x-delta-y}]First,
by $\bm{X}^{\star\top}\bm{X}^{\star}=\bm{Y}^{\star\top}\bm{Y}^{\star}$,
one can obtain 
\begin{align*}
\bm{\Delta}_{\bm{Y}}^{\top}\bm{Y}^{\star}-\bm{X}^{\star\top}\bm{\Delta}_{\bm{X}} & =\left(\bm{Y}^{\mathsf{{d}}}\bm{H}^{\mathsf{{d}}}-\bm{Y}^{\star}\right)^{\top}\bm{Y}^{\star}-\bm{X}^{\star\top}\left(\bm{X}^{\mathsf{{d}}}\bm{H}^{\mathsf{{d}}}-\bm{X}^{\star}\right)\\
 & =\left(\bm{Y}^{\mathsf{{d}}}\bm{H}^{\mathsf{{d}}}\right)^{\top}\bm{Y}^{\star}-\bm{X}^{\star\top}\left(\bm{X}^{\mathsf{{d}}}\bm{H}^{\mathsf{{d}}}\right)\\
 & =\left(\bm{Y}^{\mathsf{{d}}}\bm{H}^{\mathsf{{d}}}\right)^{\top}\big(\bm{Y}^{\star}-\bm{{Y}}^{\mathsf{{d}}}\bm{{H}}^{\mathsf{{d}}}\big)+\bm{{H}}^{\mathsf{{d}}\top}\bm{{Y}}^{\mathsf{{d}}\top}\bm{{Y}}^{\mathsf{{d}}}\bm{{H}}-\bm{X}^{\star\top}\left(\bm{X}^{\mathsf{{d}}}\bm{H}^{\mathsf{{d}}}\right)\\
 & =-\left(\bm{Y}^{\mathsf{{d}}}\bm{H}^{\mathsf{{d}}}\right)^{\top}\bm{{\Delta}}_{\bm{{Y}}}+\bm{{\Delta}}_{\bm{{X}}}^{\top}\left(\bm{X}^{\mathsf{{d}}}\bm{H}^{\mathsf{{d}}}\right)+\bm{{H}}^{\mathsf{{d}}\top}\big(\bm{{Y}}^{\mathsf{{d}}\top}\bm{{Y}}^{\mathsf{{d}}}-\bm{{X}}^{\mathsf{{d}}\top}\bm{{X}}^{\mathsf{{d}}}\big)\bm{{H}}^{\mathsf{{d}}}.
\end{align*}
We can further decompose it as 
\begin{equation}
\bm{\Delta}_{\bm{Y}}^{\top}\bm{Y}^{\star}-\bm{X}^{\star\top}\bm{\Delta}_{\bm{X}}=\bm{\Delta}_{\bm{X}}^{\top}\bm{X}^{\star}+\bm{\Delta}_{\bm{X}}^{\top}\bm{\Delta}_{\bm{X}}-\bm{Y}^{\star\top}\bm{\Delta}_{\bm{Y}}-\bm{\Delta}_{\bm{Y}}^{\top}\bm{\Delta}_{\bm{Y}}+\bm{{H}}^{\mathsf{{d}}\top}\big(\bm{{Y}}^{\mathsf{{d}}\top}\bm{{Y}}^{\mathsf{{d}}}-\bm{{X}}^{\mathsf{{d}}\top}\bm{{X}}^{\mathsf{{d}}}\big)\bm{{H}}^{\mathsf{{d}}}.\label{eq:first}
\end{equation}
Second, since $\bm{{H}}^{\mathsf{{d}}}$ is the best rotation matrix
to align $(\bm{X}^{\mathsf{{d}}},\bm{Y}^{\mathsf{{d}}})$ and $(\bm{X}^{\star},\bm{Y}^{\star})$,
we know from \cite[Lemma 35]{ma2017implicit} that 
\[
\left(\bm{X}^{\mathsf{{d}}}\bm{H}^{\mathsf{{d}}}\right)^{\top}\bm{X}^{\star}+\left(\bm{Y}^{\mathsf{{d}}}\bm{H}^{\mathsf{{d}}}\right)^{\top}\bm{Y}^{\star}\succeq\bm{0},
\]
which implies 
\[
\left(\bm{X}^{\mathsf{{d}}}\bm{H}^{\mathsf{{d}}}-\bm{X}^{\star}\right)^{\top}\bm{X}^{\star}+\left(\bm{Y}^{\mathsf{{d}}}\bm{H}^{\mathsf{{d}}}-\bm{Y}^{\star}\right)^{\top}\bm{Y}^{\star}=\bm{\Delta}_{\bm{X}}^{\top}\bm{X}^{\star}+\bm{\Delta}_{\bm{Y}}^{\top}\bm{Y}^{\star}
\]
is a symmetric matrix, i.e. 
\[
\bm{\Delta}_{\bm{X}}^{\top}\bm{X}^{\star}+\bm{\Delta}_{\bm{Y}}^{\top}\bm{Y}^{\star}=\bm{X}^{\star\top}\bm{\Delta}_{\bm{X}}+\bm{Y}^{\star\top}\bm{\Delta}_{\bm{Y}}.
\]
This is equivalent to 
\begin{equation}
\bm{\Delta}_{\bm{Y}}^{\top}\bm{Y}^{\star}-\bm{X}^{\star\top}\bm{\Delta}_{\bm{X}}=\bm{Y}^{\star\top}\bm{\Delta}_{\bm{Y}}-\bm{\Delta}_{\bm{X}}^{\top}\bm{X}^{\star}.\label{eq:second}
\end{equation}
Combine~(\ref{eq:first}) and~(\ref{eq:second}) to arrive at 
\[
\bm{\Delta}_{\bm{X}}^{\top}\bm{X}^{\star}+\bm{\Delta}_{\bm{X}}^{\top}\bm{\Delta}_{\bm{X}}-\bm{Y}^{\star\top}\bm{\Delta}_{\bm{Y}}-\bm{\Delta}_{\bm{Y}}^{\top}\bm{\Delta}_{\bm{Y}}+\bm{{H}}^{\mathsf{{d}}\top}\big(\bm{{Y}}^{\mathsf{{d}}\top}\bm{{Y}}^{\mathsf{{d}}}-\bm{{X}}^{\mathsf{{d}}\top}\bm{{X}}^{\mathsf{{d}}}\big)\bm{{H}}^{\mathsf{{d}}}=\bm{Y}^{\star\top}\bm{\Delta}_{\bm{Y}}-\bm{\Delta}_{\bm{X}}^{\top}\bm{X}^{\star},
\]
which results in 
\[
\bm{\Delta}_{\bm{Y}}^{\top}\bm{Y}^{\star}-\bm{X}^{\star\top}\bm{\Delta}_{\bm{X}}=\bm{Y}^{\star\top}\bm{\Delta}_{\bm{Y}}-\bm{\Delta}_{\bm{X}}^{\top}\bm{X}^{\star}=\frac{1}{2}\left(\bm{\Delta}_{\bm{X}}^{\top}\bm{\Delta}_{\bm{X}}-\bm{\Delta}_{\bm{Y}}^{\top}\bm{\Delta}_{\bm{Y}}\right)+\frac{{1}}{2}\bm{{H}}^{\mathsf{{d}}\top}\left(\bm{{Y}}^{\mathsf{{d}}\top}\bm{{Y}}^{\mathsf{{d}}}-\bm{{X}}^{\mathsf{{d}}\top}\bm{{X}}^{\mathsf{{d}}}\right)\bm{{H}}^{\mathsf{{d}}}.
\]
This completes the proof of the claim. \end{proof}

\subsection{Proof of Lemma \ref{lem:Phi3-two-infty}\label{subsec:Proof-of-Lemma-Phi3}}

Recall that 
\[
\bm{A}=\frac{1}{p}\mathcal{P}_{\Omega}\left(\bm{X}\bm{Y}^{\top}-\bm{M}^{\star}\right)-\left(\bm{X}\bm{Y}^{\top}-\bm{M}^{\star}\right)\qquad\text{and}\qquad\bm{\Phi}_{3}=-\bm{A}\overline{\bm{Y}}^{\mathsf{d}}\big(\overline{\bm{Y}}^{\mathsf{d}\top}\overline{\bm{Y}}^{\mathsf{d}}\big)^{-1}
\]
with $\overline{\bm{Y}}^{\mathsf{d}}=\bm{Y}^{\mathsf{d}}\bm{H}^{\mathrm{d}}$.
For any $1\leq j\leq n$, we have 
\begin{align*}
\left\Vert \bm{e}_{j}^{\top}\bm{A}\overline{\bm{Y}}^{\mathsf{d}}\big(\overline{\bm{Y}}^{\mathsf{d}\top}\overline{\bm{Y}}^{\mathsf{d}}\big)^{-1}\right\Vert _{2} & \leq\big\|\bm{e}_{j}^{\top}\bm{A}\overline{\bm{Y}}^{\mathsf{d}}\big\|_{2}\big\Vert \big(\overline{\bm{Y}}^{\mathsf{d}\top}\overline{\bm{Y}}^{\mathsf{d}}\big)^{-1}\big\Vert \\
 & \overset{(\text{i})}{=}\left\Vert \bm{e}_{j}^{\top}\bm{A}\bm{Y}^{\mathsf{d}}\right\Vert _{2}\big\Vert \big(\overline{\bm{Y}}^{\mathsf{d}\top}\overline{\bm{Y}}^{\mathsf{d}}\big)^{-1}\big\Vert \\
 & \overset{(\text{ii})}{=}\left\Vert \bm{e}_{j}^{\top}\bm{A}\bm{Y}\Big(\bm{I}_{r}+\frac{\lambda}{p}\left(\bm{Y}^{\top}\bm{Y}\right)^{-1}\Big)^{1/2}\right\Vert _{2}\big\Vert \big(\overline{\bm{Y}}^{\mathsf{d}\top}\overline{\bm{Y}}^{\mathsf{d}}\big)^{-1}\big\Vert \\
 & \leq\left\Vert \bm{e}_{j}^{\top}\bm{A}\bm{Y}\right\Vert _{2}\left\Vert \Big(\bm{I}_{r}+\frac{\lambda}{p}\left(\bm{Y}^{\top}\bm{Y}\right)^{-1}\Big)^{1/2}\right\Vert \big\Vert \big(\overline{\bm{Y}}^{\mathsf{d}\top}\overline{\bm{Y}}^{\mathsf{d}}\big)^{-1}\big\Vert \\
 & \overset{(\text{iii})}{\lesssim}\frac{1}{\sigma_{\min}}\left\Vert \bm{e}_{j}^{\top}\bm{A}\bm{Y}\right\Vert _{2}\overset{(\text{iv})}{=}\frac{1}{\sigma_{\min}}\left\Vert \bm{e}_{j}^{\top}\bm{A}\bm{Y}\bm{H}\right\Vert _{2}.
\end{align*}
Here (i) and (iv) rely on the unitary invariance of the operator norm,
(ii) uses the definition of $\bm{Y}^{\mathsf{d}}$ (see~(\ref{eq:defn-Xd-Yd}))
and (iii) follows from the choice
$\lambda\lesssim\sigma\sqrt{np}$ and immediate consequences of~(\ref{eq:F-norm-upper-bound})
\[
\left\Vert \Big(\bm{I}_{r}+\frac{\lambda}{p}\left(\bm{Y}^{\top}\bm{Y}\right)^{-1}\Big)^{1/2}\right\Vert \asymp1\qquad\text{and}\qquad\big\Vert \big(\overline{\bm{Y}}^{\mathsf{d}\top}\overline{\bm{Y}}^{\mathsf{d}}\big)^{-1}\big\Vert \lesssim\frac{1}{\sigma_{\min}}.
\]
Therefore, it suffices to control $\|\bm{e}_{j}^{\top}\bm{A}\bm{Y}\bm{H}\|_{2}$.
To this end, we have the following decomposition 
\begin{align}
\bm{e}_{j}^{\top}\bm{A}\bm{Y}\bm{H} & =\bm{e}_{j}^{\top}\left[\frac{1}{p}\mathcal{P}_{\Omega}\left(\bm{X}\bm{Y}^{\top}-\bm{M}^{\star}\right)-\left(\bm{X}\bm{Y}^{\top}-\bm{M}^{\star}\right)\right]\bm{Y}\bm{H}\nonumber \\
 & =\bm{e}_{j}^{\top}\left[\frac{1}{p}\mathcal{P}_{\Omega}\left(\bm{X}^{(j)}\bm{Y}^{(j)\top}-\bm{M}^{\star}\right)-\left(\bm{X}^{(j)}\bm{Y}^{(j)\top}-\bm{M}^{\star}\right)\right]\bm{Y}^{(j)}\bm{H}^{(j)}+\bm{\Delta}_{2},\label{eq:ej-AYH-2-terms}
\end{align}
where we define 
\begin{align*}
\bm{\Delta}_{2} & \triangleq\bm{e}_{j}^{\top}\left[\frac{1}{p}\mathcal{P}_{\Omega}\left(\bm{X}\bm{Y}^{\top}-\bm{M}^{\star}\right)-\left(\bm{X}\bm{Y}^{\top}-\bm{M}^{\star}\right)\right]\bm{Y}\bm{H}\\
 & \quad-\bm{e}_{j}^{\top}\left[\frac{1}{p}\mathcal{P}_{\Omega}\big(\bm{X}^{(j)}\bm{Y}^{(j)\top}-\bm{M}^{\star}\big)-\big(\bm{X}^{(j)}\bm{Y}^{(j)\top}-\bm{M}^{\star}\big)\right]\bm{Y}^{(j)}\bm{H}^{(j)}.
\end{align*}
Denoting 
\[
\bm{v}=[v_{1,}\cdots,v_{n}]\triangleq\bm{e}_{j}^{\top}\big(\bm{X}^{(j)}\bm{Y}^{(j)\top}-\bm{M}^{\star}\big),
\]
we can rewrite the first term of~(\ref{eq:ej-AYH-2-terms}) as 
\begin{align*}
\bm{e}_{j}^{\top}\left[\frac{1}{p}\mathcal{P}_{\Omega}\big(\bm{X}^{(j)}\bm{Y}^{(j)\top}-\bm{M}^{\star}\big)-\big(\bm{X}^{(j)}\bm{Y}^{(j)\top}-\bm{M}^{\star}\big)\right]\bm{Y}^{(j)}\bm{H}^{(j)} & =\frac{1}{p}\sum_{k=1}^{n}\left(\delta_{jk}-p\right)v_{k}\left[\bm{Y}^{(j)}\bm{H}^{(j)}\right]_{k,\cdot}.
\end{align*}
Since $(\bm{X}^{(j)},\bm{Y}^{(j)})$ is independent of $\{\delta_{jk}\}_{1\leq k\leq n},$
the right hand side of the above equation can be viewed as a sum of
independent random vectors, conditional on $(\bm{X}^{(j)},\bm{Y}^{(j)})$.
Invoke Bernstein's inequality to see that 
\[
\left\Vert \frac{1}{p}\sum_{k=1}^{n}\left(\delta_{jk}-p\right)v_{k}\big[\bm{Y}^{(j)}\bm{H}^{(j)}\big]_{k,\cdot}\right\Vert _{2}\lesssim\frac{1}{p}\left(\sqrt{V\log n}+B\log n\right)
\]
holds with probability at least $1-O(n^{-10})$. Here, we denote 
\begin{align*}
V & \triangleq\left\Vert \sum\nolimits _{k=1}^{n}\mathbb{E}\left[\left(\delta_{jk}-p\right)^{2}\right]v_{k}^{2}\big[\bm{Y}^{(j)}\bm{H}^{(j)}\big]_{k,\cdot}\big[\bm{Y}^{(j)}\bm{H}^{(j)}\big]_{k,\cdot}^{\top}\right\Vert \le p\left\Vert \bm{v}\right\Vert _{\infty}^{2}\big\|\bm{Y}^{(j)}\big\|_{\mathrm{F}}^{2},\\
B & \triangleq\max_{1\leq k\leq n}\left\Vert \left(\delta_{jk}-p\right)v_{k}\big[\bm{Y}^{(j)}\bm{H}^{(j)}\big]_{k,\cdot}\right\Vert _{2}\leq\left\Vert \bm{v}\right\Vert _{\infty}\big\|\bm{Y}^{(j)}\big\|_{2,\infty}.
\end{align*}
As a result, we obtain 
\begin{align*}
\left\Vert \frac{1}{p}\sum_{k=1}^{n}\left(\delta_{jk}-p\right)v_{k}\left[\bm{Y}^{(j)}\bm{H}^{(j)}\right]_{k,\cdot}\right\Vert _{2} & \lesssim\frac{1}{p}\left(\sqrt{p\log n}\left\Vert \bm{v}\right\Vert _{\infty}\big\|\bm{Y}^{(j)}\big\|_{\mathrm{F}}+\left\Vert \bm{v}\right\Vert _{\infty}\big\|\bm{Y}^{(j)}\big\|_{2,\infty}\log n\right)\\
 & \lesssim\frac{\left\Vert \bm{v}\right\Vert _{\infty}}{p}\left(\sqrt{pr\sigma_{\max}\log n}+\sqrt{\frac{\mu r}{n}\sigma_{\max}\log^{2}n}\right)\\
 & \asymp\left\Vert \bm{v}\right\Vert _{\infty}\sqrt{\frac{r\log n}{p}\sigma_{\max}}
\end{align*}
with the proviso that $np\gg\mu\log n$. Here the middle line depends
on $\|\bm{Y}^{(j)}\|_{\mathrm{F}}\lesssim\sqrt{r\sigma_{\max}}$ and
$\|\bm{Y}^{(j)}\|_{2,\infty}\lesssim\sqrt{\mu r\sigma_{\max}/n}$.
Additionally,
\begin{align*}
\left\Vert \bm{v}\right\Vert _{\infty} & \leq\left\Vert \bm{X}^{(j)}\bm{Y}^{(j)\top}-\bm{M}^{\star}\right\Vert _{\infty}\leq\left\Vert \left(\bm{X}^{(j)}\bm{{R}}^{(j)}-\bm{{X}}^{\star}\right)\bm{{R}}^{(j)\top}\bm{{Y}}^{(j)\top}+\bm{{X}}^{\star}\left(\bm{{Y}}^{(j)}\bm{{R}}^{(j)}-\bm{{Y}}^{\star}\right)^{\top}\right\Vert _{\infty}\\
 & \leq\left\Vert \bm{X}^{(j)}\bm{{R}}^{(j)}-\bm{{X}}^{\star}\right\Vert _{2,\infty}\left\Vert \bm{{Y}}^{(j)}\right\Vert _{2,\infty}+\left\Vert \bm{{X}}^{\star}\right\Vert _{2,\infty}\left\Vert \bm{{Y}}^{(j)}\bm{{R}}^{(j)}-\bm{{Y}}^{\star}\right\Vert _{2,\infty}\\
 & \lesssim\kappa\frac{\sigma}{\sigma_{\min}}\sqrt{\frac{n\log n}{p}}\frac{\mu r}{n}\sigma_{\max}\lesssim\kappa^{2}\sigma\sqrt{\frac{\mu^{2}r^{2}\log n}{np}}.
\end{align*}
Here the penultimate inequality uses~(\ref{eq:F-j-F-star-2-infty-dist})
and the bound $\|\bm{Y}^{(j)}\|_{2,\infty}\lesssim\sqrt{\mu r\sigma_{\max}/n}$.
We arrive at the conclusion that: with probability exceeding $1-O(n^{-10})$,
\[
\left\Vert \frac{1}{p}\sum_{k=1}^{n}\left(\delta_{jk}-p\right)v_{k}\big(\bm{X}^{(j)}\bm{H}^{(j)}\big)_{k,\cdot}\right\Vert _{2}\lesssim\sigma\sqrt{{\frac{{\sigma_{\max}}}{p}}}\cdot\sqrt{\frac{\kappa^{4}\mu^{2}r^{3}\log^{2}n}{np}}.
\]

Next, we move on to the second term $\bm{\Delta}_{2}$ of~(\ref{eq:ej-AYH-2-terms}),
which can be further decomposed as follows
\begin{align*}
\bm{\Delta}_{2} & =\underbrace{\bm{e}_{j}^{\top}\left[\frac{1}{p}\mathcal{P}_{\Omega}\left(\bm{X}\bm{Y}^{\top}-\bm{M}^{\star}\right)-\left(\bm{X}\bm{Y}^{\top}-\bm{M}^{\star}\right)\right]\left(\bm{Y}\bm{H}-\bm{Y}^{(j)}\bm{H}^{(j)}\right)}_{:=\bm{\theta}_{1}}\\
 & \quad+\underbrace{\bm{e}_{j}^{\top}\left[\frac{1}{p}\mathcal{P}_{\Omega}\left(\bm{X}\bm{Y}^{\top}-\bm{X}^{(j)}\bm{Y}^{(j)\top}\right)-\left(\bm{X}\bm{Y}^{\top}-\bm{X}^{(j)}\bm{Y}^{(j)\top}\right)\right]\bm{Y}^{(j)}\bm{H}^{(j)}}_{:=\bm{\theta}_{2}}.
\end{align*}
In what follows, we bound $\bm{\theta}_{1}$ and $\bm{\theta}_{2}$
sequentially.
\begin{enumerate}
\item Regarding $\bm{\theta}_{1}$, using the definition of $\bm{A}$ we
obtain
\begin{align*}
\left\Vert \bm{\theta}_{1}\right\Vert _{2} & \leq\left\Vert \bm{A}\right\Vert \left\Vert \bm{Y}\bm{H}-\bm{Y}^{(j)}\bm{H}^{(j)}\right\Vert _{\mathrm{F}}\lesssim\sigma\sqrt{\frac{n}{p}}\cdot\sqrt{\frac{\kappa^{4}\mu^{2}r^{2}\log n}{np}}\cdot\text{\ensuremath{\kappa}}\frac{\sigma}{\sigma_{\min}}\sqrt{\frac{\mu r\log n}{p}}\sqrt{\sigma_{\max}}\\
 & \asymp\sigma\sqrt{{\frac{{\sigma_{\max}}}{p}}}\cdot\sqrt{\frac{\kappa^{4}\mu^{2}r^{2}\log n}{np}}\cdot\frac{\sigma}{\sigma_{\min}}\sqrt{\frac{\kappa^{2}\mu rn\log n}{p}},
\end{align*}
where the second relation holds due to~(\ref{eq:F-j-F-dist}) and
the fact that $\|\bm{A}\|\lesssim\sigma\sqrt{\frac{n}{p}}\sqrt{\frac{\kappa^{4}\mu^{2}r^{2}\log n}{np}}$~(cf.~(\ref{eq:A-norm})).
\item Moving on to $\bm{\theta}_{2}$, we can utilize the identity 
\[
\bm{X}\bm{Y}^{\top}-\bm{X}^{(j)}\bm{Y}^{(j)\top}=\big(\bm{X}\bm{H}-\bm{X}^{(j)}\bm{H}^{(j)}\big)\big(\bm{Y}^{(j)}\bm{H}^{(j)}\big)^{\top}+\bm{X}\bm{H}\big(\bm{Y}\bm{H}-\bm{Y}^{(j)}\bm{H}^{(j)}\big)^{\top}
\]
to deduce that 
\begin{align*}
\left\Vert \bm{\theta}_{2}\right\Vert _{2} & \leq\left\Vert \bm{e}_{j}^{\top}\left[\frac{1}{p}\mathcal{P}_{\Omega}\left[\big(\bm{X}\bm{H}-\bm{X}^{(j)}\bm{H}^{(j)}\big)\big(\bm{Y}^{(j)}\bm{H}^{(j)}\big)^{\top}\right]-\left(\bm{X}\bm{H}-\bm{X}^{(j)}\bm{H}^{(j)}\right)\big(\bm{Y}^{(j)}\bm{H}^{(j)}\big)^{\top}\right]\bm{Y}^{(j)}\bm{H}^{(j)}\right\Vert _{2}\\
 & \quad+\left\Vert \bm{e}_{j}^{\top}\left[\frac{1}{p}\mathcal{P}_{\Omega}\left[\bm{X}\bm{H}\left(\bm{Y}\bm{H}-\bm{Y}^{(j)}\bm{H}^{(j)}\right)^{\top}\right]-\bm{X}\bm{H}\left(\bm{Y}\bm{H}-\bm{Y}^{(j)}\bm{H}^{(j)}\right)^{\top}\right]\bm{Y}^{(j)}\bm{H}^{(j)}\right\Vert _{2}\\
	& =\underbrace{\left\Vert \big(\bm{X}\bm{H}-\bm{X}^{(j)}\bm{H}^{(j)}\big)_{j,\cdot}\frac{1}{p}\sum_{k=1}^{n}\left(\delta_{jk}-p\right)\big(\bm{Y}^{(j)}\bm{H}^{(j)}\big)_{k,\cdot}^{\top}\big(\bm{Y}^{(j)}\bm{H}^{(j)}\big)_{k,\cdot}\right\Vert _{2}}_{:=\alpha_{1}}\\
 & \quad+\underbrace{\left\Vert \big(\bm{X}\bm{H}\big)_{j,\cdot}\frac{1}{p}\sum_{k=1}^{n}\left(\delta_{jk}-p\right)\left(\bm{Y}\bm{H}-\bm{Y}^{(j)}\bm{H}^{(j)}\right)_{k,\cdot}^{\top}\left(\bm{Y}^{(j)}\bm{H}^{(j)}\right)_{k,\cdot}\right\Vert _{2}}_{:=\alpha_{2}}.
\end{align*}
With regards to $\alpha_{1}$, we have by Bernstein's inequality and
(\ref{eq:F-j-F-dist}) that 
\begin{align*}
\alpha_{1} & \leq\left\Vert \bm{X}\bm{H}-\bm{X}^{(j)}\bm{H}^{(j)}\right\Vert _{\mathrm{F}}\left\Vert \frac{1}{p}\sum_{k=1}^{n}\left(\delta_{jk}-p\right)\left(\bm{Y}^{(j)}\bm{H}^{(j)}\right)_{k,\cdot}^{\top}\left(\bm{Y}^{(j)}\bm{H}^{(j)}\right)_{k,\cdot}\right\Vert \\
 & \lesssim\text{\ensuremath{\kappa\frac{{\sigma}}{\sigma_{\min}}}}\sqrt{\frac{n\log n}{p}}\sqrt{\frac{\mu r}{n}\sigma_{\max}}\cdot\frac{1}{p}\left(\sqrt{V_{2}\log n}+B_{2}\log n\right)
\end{align*}
holds with probability exceeding $1-O(n^{-10})$. Here, we define 
\begin{align*}
V_{2} & \triangleq\left\Vert \sum_{k=1}^{n}\mathbb{E}\left(\delta_{jk}-p\right)^{2}\bigl(\bm{Y}^{(j)}\bm{H}^{(j)}\bigr)_{k,\cdot}^{\top}\bigl(\bm{Y}^{(j)}\bm{H}^{(j)}\bigr)_{k,\cdot}\bigl(\bm{Y}^{(j)}\bm{H}^{(j)}\bigr)_{k,\cdot}^{\top}\bigl(\bm{Y}^{(j)}\bm{H}^{(j)}\bigr)_{k,\cdot}\right\Vert \\
 & \leq p\bigl\Vert\bm{Y}^{(j)}\bigr\Vert_{2,\infty}^{2}\bigl\Vert\bm{Y}^{(j)\top}\bm{Y}^{(j)}\bigr\Vert,\\
B_{2} & \triangleq\max_{1\leq k\leq n}\left\Vert \left(\delta_{jk}-p\right)\big(\bm{Y}^{(j)}\bm{H}^{(j)}\big)_{k,\cdot}^{\top}\big(\bm{Y}^{(j)}\bm{H}^{(j)}\big)_{k,\cdot}\right\Vert \leq\bigl\Vert\bm{Y}^{(j)}\bigr\Vert_{2,\infty}^{2}.
\end{align*}
As a result, we can obtain 
\begin{align*}
\alpha_{1} & \lesssim\kappa\frac{{\sigma}}{\sigma_{\min}}\sqrt{\frac{\mu r\log n}{p}\sigma_{\max}}\cdot\frac{1}{p}\left(\sqrt{p\sigma_{\max}\log n}\left\Vert \bm{Y}^{(j)}\right\Vert _{2,\infty}+\left\Vert \bm{Y}^{(j)}\right\Vert _{2,\infty}^{2}\log n\right)\\
 & \lesssim\sigma\sqrt{{\frac{{\sigma_{\max}}}{p}}}\cdot\sqrt{\frac{\kappa^{4}\mu^{2}r^{2}\log^{2}n}{np}},
\end{align*}
provided that $np\gg\mu r\log n$. Here we apply the bounds $\|\bm{Y}^{(j)}\|\lesssim\sqrt{\sigma_{\max}}$
and $\|\bm{Y}^{(j)}\|_{2,\infty}\lesssim\sqrt{\mu r\sigma_{\max}/n}$ (see~(\ref{eq:F-norm-upper-bound}) and the following remarks).
In the end, we turn to the term $\alpha_{2}$, which obeys 
\begin{align*}
\alpha_{2} & \leq\frac{1}{p}\left\Vert \bm{X}\right\Vert _{2,\infty}\sum_{k=1}^{n}\left|\delta_{jk}-p\right|\left\Vert \big(\bm{Y}\bm{H}-\bm{Y}^{(j)}\bm{H}^{(j)}\big)_{k\cdot}\right\Vert _{2}\big\Vert \big(\bm{Y}^{(j)}\bm{H}^{(j)}\big)_{k\cdot}\big\Vert _{2}\\
 & \leq\frac{1}{p}\sqrt{\frac{\mu r}{n}}\sqrt{\sigma_{\max}}\cdot\sqrt{\sum_{k=1}^{n}\left(\delta_{jk}-p\right)^{2}}\cdot\sqrt{\sum_{k=1}^{n}\left\Vert \left(\bm{Y}\bm{H}-\bm{Y}^{(j)}\bm{H}^{(j)}\right)_{k\cdot}\right\Vert _{2}^{2}\left\Vert \left(\bm{Y}^{(j)}\bm{H}^{(j)}\right)_{k\cdot}\right\Vert _{2}^{2}}\\
 & \lesssim\frac{1}{p}\sqrt{\frac{\mu r}{n}}\sqrt{\sigma_{\max}}\cdot\sqrt{np}\cdot\left\Vert \bm{Y}\bm{H}-\bm{Y}^{(j)}\bm{H}^{(j)}\right\Vert _{\mathrm{F}}\big\|\bm{Y}^{(j)}\big\|_{2,\infty}\\
 & \lesssim\sqrt{\frac{\mu r}{p}}\sqrt{\sigma_{\max}}\cdot\kappa\frac{\sigma}{\sigma_{\min}}\sqrt{\frac{n\log n}{p}}\frac{\mu r}{n}\sigma_{\max}\asymp\sigma\sqrt{{\frac{{\sigma_{\max}}}{p}}}\cdot\sqrt{\frac{\kappa^{4}\mu^{3}r^{3}\log n}{np}},
\end{align*}
where the second line arises from the Cauchy-Schwarz inequality.
\end{enumerate}
Take the previous bounds collectively to arrive at 
\begin{align*}
\left\Vert \bm{\Delta}_{2}\right\Vert _{2} & \lesssim\sigma\sqrt{\frac{\sigma_{\max}}{p}}\left\{ \sqrt{\frac{\kappa^{4}\mu^{2}r^{2}\log n}{np}}\cdot\frac{\sigma}{\sigma_{\min}}\sqrt{\frac{\kappa^{2}\mu rn\log n}{p}}+\sqrt{\frac{\kappa^{4}\mu^{2}r^{2}\log^{2}n}{np}}+\sqrt{\frac{\kappa^{4}\mu^{3}r^{3}\log n}{np}}\right\} \\
 & \lesssim\sigma\sqrt{\frac{\sigma_{\max}}{p}}\cdot\sqrt{\frac{\kappa^{4}\mu^{3}r^{3}\log^{2}n}{np}}
\end{align*}
as long as $\frac{\sigma}{\sigma_{\min}}\sqrt{\frac{\kappa^{2}n\log n}{p}}\ll1$.
Finally, we conclude that
\begin{align}
\left\Vert \bm{e}_{j}^{\top}\bm{A}\overline{\bm{Y}}^{\mathsf{d}}\big(\overline{\bm{Y}}^{\mathsf{d}\top}\overline{\bm{Y}}^{\mathsf{d}}\big)^{-1}\right\Vert _{2} & \lesssim\frac{1}{\sigma_{\min}}\left\Vert \bm{e}_{j}^{\top}\bm{A}\bm{Y}\bm{H}\right\Vert _{2}\nonumber \\
 & \lesssim\frac{1}{\sigma_{\min}}\left(\sigma\sqrt{\frac{\sigma_{\max}}{p}}\sqrt{\frac{\kappa^{4}\mu^{2}r^{3}\log^{2}n}{np}}+\sigma\sqrt{\frac{\sigma_{\max}}{p}}\sqrt{\frac{\kappa^{4}\mu^{3}r^{3}\log^{2}n}{np}}\right)\nonumber \\
 & \asymp\frac{\sigma}{\sqrt{p\sigma_{\min}}}\cdot\sqrt{\frac{\kappa^{5}\mu^{3}r^{3}\log^{2}n}{np}},\label{eq:3nd-pert-final-bound}
\end{align}
thus concluding the proof.

\subsection{Proof of Lemma \ref{lem:Phi4-two-infty}\label{subsec:Proof-of-Lemma-Phi4}}

First, it is straightforward to verify that 
\begin{align}
 & \left\Vert \nabla_{\bm{X}}f\left(\bm{X},\bm{Y}\right)\Big(\bm{I}_{r}+\frac{\lambda}{p}\left(\bm{Y}^{\top}\bm{Y}\right)^{-1}\Big)^{1/2}\left(\bm{Y}^{\mathsf{d}\top}\bm{Y}^{\mathsf{d}}\right)^{-1}\bm{H}^{\mathsf{d}}\right\Vert _{2,\infty}\nonumber \\
 & \quad\leq\left\Vert \nabla_{\bm{X}}f\left(\bm{X},\bm{Y}\right)\right\Vert _{\mathrm{F}}\left\Vert \Big(\bm{I}_{r}+\frac{\lambda}{p}\left(\bm{Y}^{\top}\bm{Y}\right)^{-1}\Big)^{1/2}\right\Vert \big\|\big(\bm{Y}^{\mathsf{d}\top}\bm{Y}^{\mathsf{d}}\big)^{-1}\big\|\nonumber \\
 & \quad\lesssim\frac{1}{n^{5}}\frac{\lambda}{p}\sqrt{\sigma_{\min}}\cdot\frac{1}{\sigma_{\min}}\lesssim\frac{\sigma}{\sqrt{p\sigma_{\min}}}\cdot\frac{1}{n^{4}},\label{eq:Phi4-term1}
\end{align}
where the last line arises from~(\ref{eq:small-gradient}), the choice
$\lambda\lesssim\sigma\sqrt{np}$ (cf.~(\ref{eq:lambda-condition})), and the bounds
\[
\left\Vert \Big(\bm{I}_{r}+\frac{\lambda}{p}\left(\bm{Y}^{\top}\bm{Y}\right)^{-1}\Big)^{1/2}\right\Vert \asymp1\qquad\text{and}\qquad\big\|\big(\bm{Y}^{\mathsf{d}\top}\bm{Y}^{\mathsf{d}}\big)^{-1}\big\|\lesssim\frac{1}{\sigma_{\min}}.
\]
Here the latter two are immediate consequences of~(\ref{eq:F-norm-upper-bound}).
Second, with regards to the term involving $\bm{\Delta}_{\mathsf{balancing}}$,
we have 
\begin{align}
\left\Vert \bm{X}\bm{\Delta}_{\mathsf{balancing}}\bm{H}^{\mathsf{d}}\right\Vert _{2,\infty} & \leq\left\Vert \bm{X}\right\Vert _{2,\infty}\left\Vert \bm{\Delta}_{\mathsf{balancing}}\right\Vert \nonumber \\
 & \lesssim\sqrt{\frac{\mu r}{n}\sigma_{\max}}\left\Vert \Big(\bm{I}_{r}+\frac{\lambda}{p}\left(\bm{X}^{\top}\bm{X}\right)^{-1}\Big)^{1/2}-\Big(\bm{I}_{r}+\frac{\lambda}{p}\left(\bm{Y}^{\top}\bm{Y}\right)^{-1}\Big)^{1/2}\right\Vert \nonumber \\
 & \lesssim\sqrt{\frac{\mu r}{n}\sigma_{\max}}\cdot\frac{1}{n^{5}}\frac{\lambda}{p}\frac{{\kappa}}{\sigma_{\min}}\asymp\frac{\sigma}{\sqrt{p\sigma_{\min}}}\sqrt{\frac{\kappa^{3}\mu r}{n^{10}}},\label{eq:Phi4-term2}
\end{align}
where the middle line uses (\ref{eq:F-norm-upper-bound}) and the last one follows from~(\ref{eq:debias-correction-close}).

Combine~(\ref{eq:Phi4-term1}),~(\ref{eq:Phi4-term2}) and the triangle
inequality to establish the advertised result, with the proviso that
$n^{2}\gg\kappa^{3}\mu r$.

\subsection{Proof of Lemma \ref{lemma:approx-gaussian-leading-term} \label{subsec:Proof-of-Lemma-approx-gaussian-low-rank}}

We invoke the identity $\bm{Y}^{\star}(\bm{Y}^{\star\top}\bm{Y}^{\star})^{-1}=\bm{V}^{\star}(\bm{\Sigma}^{\star})^{-1/2}$
(since $\bm{Y}^{\star}=\bm{V}^{\star}(\bm{\Sigma}^{\star})^{1/2}$)
to see that for any $1\leq i\leq n$,
\begin{align}
\left(\frac{1}{p}\mathcal{P}_{\Omega}\left(\bm{E}\right)\bm{Y}^{\star}(\bm{Y}^{\star\top}\bm{Y}^{\star})^{-1}\right)^{\top}\bm{e}_{i} =\left(\frac{1}{p}\mathcal{P}_{\Omega}\left(\bm{E}\right)\bm{V}^{\star}\left(\bm{\Sigma}^{\star}\right)^{-1/2}\right)^{\top}\bm{e}_{i}=\sum_{k=1}^{n}\frac{1}{p}E_{ik}\delta_{ik}\left(\bm{\Sigma}^{\star}\right)^{-1/2}\left(\bm{V}_{k,\cdot}^{\star}\right)^{\top}\label{eq:PE-Y-identity}
\end{align}
consists of a sum of independent random vectors, where we recall that
$\delta_{ik}=\ind\{(i,k)\in\Omega\}$. In addition, the right-hand
side of the above formula is conditionally Gaussian, namely, 
\[
\sum_{k=1}^{n}\frac{1}{p}E_{ik}\delta_{ik}\left(\bm{\Sigma}^{\star}\right)^{-1/2}\left(\bm{V}_{k,\cdot}^{\star}\right)^{\top}\,\Big|\,\left\{ \delta_{ik}\right\} _{k:1\leq k\leq n}\,\sim\,\mathcal{N}\Bigg(\bm{0},\underbrace{\frac{\sigma^{2}}{p^{2}}\sum_{k=1}^{n}\delta_{ik}\left(\bm{\Sigma}^{\star}\right)^{-1/2}\big(\bm{V}_{k,\cdot}^{\star}\big)^{\top}\bm{V}_{k,\cdot}^{\star}\left(\bm{\Sigma}^{\star}\right)^{-1/2}}_{:=\bm{S}}\Bigg).
\]
Note that $\bm{S}$ depends on the index $i$ through $\{\delta_{ik}\}_{k:1\leq k\leq n}$. Denote by $\bm{S}^{\star}$ the expectation of $\bm{S}$, that is,
\[
\bm{S}^{\star}\triangleq\mathbb{E}\big[\bm{S}\big]=p^{-1}\sigma^{2}\left(\bm{\Sigma}^{\star}\right)^{-1}\succeq\sigma^{2}/(p\sigma_{\max})\cdot\bm{I}_{r},
\]
and introduce the following event 
\[
\mathcal{E}\triangleq\left\{ \left\Vert \bm{S}-\bm{S}^{\star}\right\Vert \lesssim\frac{\sigma^{2}}{p\sigma_{\min}}\sqrt{\frac{\mu r\log n}{np}}\right\} .
\]
Clearly, when $np\gg\kappa^{2}\mu r\log n$, one has $\bm{S}\succ\bm{0}$
on the event $\mathcal{E}$ and hence $\bm{S}^{-1/2}$ is well-defined. As a result, on the event $\mathcal{E}$,
we have 
\begin{equation}
\left(\bm{S}^{\star}\right)^{1/2}\bm{S}^{-1/2}\sum_{k=1}^{n}\frac{1}{p}E_{ik}\delta_{ik}\left(\bm{\Sigma}^{\star}\right)^{-1/2}\left(\bm{V}_{k,\cdot}^{\star}\right)^{\top}\,\Big|\,\left\{ \delta_{ik}\right\} _{k:1\leq k\leq n}\ \sim\ \mathcal{N}\left(\bm{0},\bm{S}^{\star}\right).\label{eq:Sstar-S-conditional-Gaussian}
\end{equation}
In view of this relation, we can define the $i$th row of $\bm{Z}_{\bm{X}}\in\mathbb{R}^{n\times r}$ to be 
\begin{equation}
\bm{e}_{i}^\top\bm{Z}_{\bm{X}}\triangleq\begin{cases}
\frac{1}{p}\bm{e}_{i}^\top\mathcal{P}_{\Omega}\left(\bm{E}\right)\bm{Y}^{\star}(\bm{Y}^{\star\top}\bm{Y}^{\star})^{-1}\bm{S}^{-1/2}\left(\bm{S}^{\star}\right)^{1/2}, & \text{on the event }\mathcal{E},\\
\bm{e}_{i}^\top\bm{G}_{\bm{X}}, & \text{on the event }\mathcal{E}^{c},
\end{cases}\label{eq:ZX-construction}
\end{equation}
where $\bm{G}_{\bm{X}}\in\mathbb{R}^{n\times r}$ is an independently
generated random matrix satisfying 
\[
\bm{G}_{\bm{X}}^{\top}\bm{e}_{i}\overset{\text{i.i.d.}}{\sim}\mathcal{N}\left(\bm{0},\frac{\sigma^{2}}{p}\left(\bm{\Sigma}^{\star}\right)^{-1}\right)\qquad\text{for}\quad1\leq i\leq n.
\]
As can be easily seen from~(\ref{eq:PE-Y-identity}) and~(\ref{eq:Sstar-S-conditional-Gaussian}),
each row of $\bm{Z}_{\bm{X}}$ follows the Gaussian distribution
\[
\bm{Z}_{\bm{X}}^{\top}\bm{e}_{i}\overset{\text{i.i.d.}}{\sim}\mathcal{N}\left(\bm{0},\frac{\sigma^{2}}{p}\left(\bm{\Sigma}^{\star}\right)^{-1}\right)\qquad\text{for}\quad1\leq i\leq n.
\]
It remains to show that, with high probability, 
\begin{align*}
\bm{\Delta}_{\bm{X}} & \triangleq\frac{1}{p}\mathcal{P}_{\Omega}\left(\bm{E}\right)\bm{Y}^{\star}(\bm{Y}^{\star\top}\bm{Y}^{\star})^{-1}-\bm{Z}_{\bm{X}}=\frac{1}{p}\mathcal{P}_{\Omega}\left(\bm{E}\right)\bm{V}^{\star}\left(\bm{\Sigma}^{\star}\right)^{-1/2}-\bm{Z}_{\bm{X}}
\end{align*}
is small when measured by the $\ell_{2,\infty}$ norm. To this end,
observe that on the event $\mathcal{E}$, 
\begin{align*}
\bm{e}_{i}^{\top}\bm{\Delta}_{\bm{X}} & =\frac{1}{p}\bm{e}_{i}^{\top}\mathcal{P}_{\Omega}\left(\bm{E}\right)\bm{V}^{\star}\left(\bm{\Sigma}^{\star}\right)^{-1/2}-\frac{1}{p}\bm{e}_{i}^{\top}\mathcal{P}_{\Omega}\left(\bm{E}\right)\bm{V}^{\star}\left(\bm{\Sigma}^{\star}\right)^{-1/2}\bm{S}^{-1/2}\left(\bm{S}^{\star}\right)^{1/2}\\
 & =\frac{1}{p}\bm{e}_{i}^{\top}\mathcal{P}_{\Omega}\left(\bm{E}\right)\bm{V}^{\star}\left(\bm{\Sigma}^{\star}\right)^{-1/2}\left[\bm{I}_{r}-\bm{S}^{-1/2}\left(\bm{S}^{\star}\right)^{1/2}\right],
\end{align*}
and therefore, we have 
\begin{align*}
\left\Vert \bm{e}_{i}^{\top}\bm{\Delta}_{\bm{X}}\right\Vert _{2} &\leq\frac{1}{p}\Big\|\bm{e}_{i}^{\top}\mathcal{P}_{\Omega}\left(\bm{E}\right)\bm{V}^{\star}\Big\|_{2}\big\|\left(\bm{\Sigma}^{\star}\right)^{-1/2}\big\|\left\Vert \bm{I}_{r}-\bm{S}^{-1/2}\left(\bm{S}^{\star}\right)^{1/2}\right\Vert \\
 & =\frac{1}{p\sqrt{\sigma_{\min}}}\Big\|\sum\nolimits _{k}E_{ik}\delta_{ik}\bm{V}_{k,\cdot}^{\star}\Big\|_{2}\left\Vert \bm{I}_{r}-\bm{S}^{-1/2}\left(\bm{S}^{\star}\right)^{1/2}\right\Vert .
\end{align*}
In what follows, we shall bound the two terms on the right-hand side
of the above display sequentially.
\begin{enumerate}
\item First, observe that $\sum_{k=1}^{n}E_{ik}\delta_{ik}\bm{V}_{k,\cdot}^{\star}$
involves a sum of independent random vectors with 
\[
\left\Vert \left\Vert E_{ik}\delta_{ik}\bm{V}_{k,\cdot}^{\star}\right\Vert _{2}\right\Vert _{\psi_{1}}\leq\left\Vert \bm{V}_{k,\cdot}^{\star}\right\Vert _{2}\left\Vert E_{jk}\delta_{jk}\right\Vert _{\psi_{1}}\lesssim\sigma\sqrt{{\mu r}/{n}},
\]
where $\|\cdot\|_{\psi_{1}}$ denotes the sub-exponential norm \cite{vershynin2016high}.
One can then apply the matrix Bernstein inequality \cite[Theorem 2.7]{Koltchinskii2011oracle}
to conclude that with probability at least $1-O(n^{-20})$, 
\[
\left\Vert \sum\nolimits _{k}E_{ik}\delta_{ik}\bm{V}_{k,\cdot}^{\star}\right\Vert _{2}\lesssim\sqrt{V_{1}\log n}+\max_{1\leq k\leq n}\left\Vert \left\Vert E_{ik}\delta_{ik}\bm{V}_{k,\cdot}^{\star}\right\Vert _{2}\right\Vert _{\psi_{1}}\log^{2}n,
\]
where we denote 
\[
V_{1}\triangleq\left\Vert \mathbb{E}\left[\sum\nolimits _{k=1}^{n}E_{jk}^{2}\delta_{jk}^{2}\bm{V}_{k,\cdot}^{\star}\left(\bm{V}_{k,\cdot}^{\star}\right)^{\top}\right]\right\Vert =\sigma^{2}p\left\Vert \bm{V}^{\star}\right\Vert _{\mathrm{F}}^{2}=\sigma^{2}pr.
\]
As a result, we arrive at 
\begin{align}
\left\Vert \sum\nolimits _{k}E_{ik}\delta_{ik}\bm{V}_{k,\cdot}^{\star}\right\Vert _{2} & \lesssim\sqrt{\sigma^{2}pr\log n}+\sigma\sqrt{\frac{\mu r}{n}}\log^{2}n\lesssim\sigma\sqrt{pr\log n}\label{eq:useful-decouple}
\end{align}
as long as $np\gg \mu\log^{3}n$.
\item Next, we move on to $\|\bm{I}_{r}-\bm{S}^{-1/2}(\bm{S}^{\star})^{1/2}\|$.
Recall that on the event $\mathcal{E}$, one has
\[
\left\Vert \bm{S}-\bm{S}^{\star}\right\Vert \lesssim\frac{\sigma^{2}}{p\sigma_{\min}}\sqrt{\frac{\mu r\log n}{np}}.
\]
This together with the fact that $\sigma^{2}/(p\sigma_{\max})\leq\lambda_{\min}(\bm{S}^{\star})\leq\lambda_{\max}(\bm{S}^{\star})\leq\sigma^{2}/(p\sigma_{\min})$
gives
\begin{equation}
\frac{\sigma^{2}}{2p\sigma_{\max}}\leq\lambda_{\min}\left(\bm{S}\right)\leq\lambda_{\max}\left(\bm{S}\right)\leq\frac{2\sigma^{2}}{p\sigma_{\min}},\ \ \sqrt{\frac{\sigma^{2}}{2p\sigma_{\max}}}\leq\lambda_{\min}\big(\bm{S}^{1/2}\big)\leq\lambda_{\max}\big(\bm{S}^{1/2}\big)\leq\sqrt{\frac{2\sigma^{2}}{p\sigma_{\min}}},\label{eq:useful-decouple-2}
\end{equation}
with the proviso that $np\gg\kappa^{2}\mu r\log n$. Therefore, straightforward
calculations yield 
\begin{align*}
\big\|\bm{I}_{r}-\bm{S}^{-1/2}\left(\bm{S}^{\star}\right)^{1/2}\big\| & \leq\big\|\bm{S}^{-1/2}\big\|\cdot\big\|\bm{S}^{1/2}-\left(\bm{S}^{\star}\right)^{1/2}\big\|\\
 & \leq\big\|\bm{S}^{-1/2}\big\|\frac{1}{\lambda_{\min}\big(\bm{S}^{1/2}\big)+\lambda_{\min}\big(\left(\bm{S}^{\star}\right)^{1/2}\big)}\big\|\bm{S}-\bm{S}^{\star}\big\|\\
 & \lesssim\sqrt{\frac{p\sigma_{\max}}{\sigma^{2}}}\cdot\frac{1}{\sqrt{\frac{\sigma^{2}}{p\sigma_{\max}}}}\cdot\frac{\sigma^{2}}{p\sigma_{\min}}\sqrt{\frac{\mu r\log n}{np}}\asymp\sqrt{\frac{\kappa^{2}\mu r\log n}{np}}.
\end{align*}
Here the second relation is the perturbation bound for the matrix
square roots (see~Lemma~\ref{lemma:matrix-sqrt}).
\end{enumerate}
Combine the above two bounds to conclude that
\begin{align*}
\left\Vert \bm{e}_{i}^{\top}\bm{\Delta}_{\bm{X}}\right\Vert _{2} & =\frac{1}{p}\left\Vert \bm{e}_{i}^{\top}\mathcal{P}_{\Omega}\left(\bm{E}\right)\bm{V}^{\star}\left(\bm{\Sigma}^{\star}\right)^{-1/2}\left[\bm{I}_{r}-\bm{S}^{-1/2}\left(\bm{S}^{\star}\right)^{1/2}\right]\right\Vert _{2}\\
 & \lesssim\frac{1}{p}\cdot\sigma\sqrt{pr\log n}\cdot\frac{1}{\sqrt{\sigma_{\min}}}\cdot\sqrt{\frac{\kappa^{2}\mu r\log n}{np}}\asymp\frac{\sigma}{\sqrt{p\sigma_{\min}}}\cdot\sqrt{\frac{\kappa^{2}\mu r^{2}\log^{2}n}{np}}.
\end{align*}

Finally, we are left with demonstrating that $\mathbb{P}(\mathcal{E}^{c})=O(n^{-10})$.
To see this, by definition one has
\begin{align*}
\left\Vert \bm{S}-\bm{S}^{\star}\right\Vert  & =\frac{\sigma^{2}}{p}\left\Vert \frac{1}{p}\sum_{k=1}^{n}\delta_{ik}\left(\bm{\Sigma}^{\star}\right)^{-1/2}\big(\bm{V}_{k,\cdot}^{\star}\big)^{\top}\bm{V}_{k,\cdot}^{\star}\left(\bm{\Sigma}^{\star}\right)^{-1/2}-\left(\bm{\Sigma}^{\star}\right)^{-1}\right\Vert \\
 & \leq\frac{\sigma^{2}}{p^{2}\sigma_{\min}}\left\Vert \sum\nolimits _{k}\delta_{ik}\big(\bm{V}_{k,\cdot}^{\star}\big)^{\top}\bm{V}_{k,\cdot}^{\star}-p\bm{I}_{r}\right\Vert \\
 & \lesssim\frac{\sigma^{2}}{p^{2}\sigma_{\min}}\left(\sqrt{V_{2}\log n}+B_{2}\log n\right)
\end{align*}
with probability at least $1-O(n^{-10})$. Here the last line utilizes
the matrix Bernstein inequality, where
\begin{align*}
B_{2} & \triangleq\max_{1\leq k\leq n}\left\Vert \left(\delta_{jk}-p\right)\left(\bm{V}_{k,\cdot}^{\star}\right)^{\top}\bm{V}_{k,\cdot}^{\star}\right\Vert \leq\frac{\mu r}{n},\\
V_{2} & \triangleq\left\Vert \mathbb{E}\Big[\sum\nolimits _{k}\left(\delta_{jk}-p\right)^{2}\big(\bm{V}_{k,\cdot}^{\star}\big)^{\top}\bm{V}_{k,\cdot}^{\star}\big(\bm{V}_{k,\cdot}^{\star}\big)^{\top}\bm{V}_{k,\cdot}^{\star}\Big]\right\Vert \leq p\frac{\mu r}{n}\left\Vert \bm{V}^{\star\top}\bm{V}^{\star}\right\Vert =\frac{\mu rp}{n}.
\end{align*}
Consequently with probability exceeding $1-O(n^{-10})$ one has 
\[
\left\Vert \bm{S}-\bm{S}^{\star}\right\Vert \lesssim\frac{\sigma^{2}}{p^{2}\sigma_{\min}}\left(\sqrt{\frac{\mu rp\log n}{n}}+\frac{\mu r}{n}\log n\right)\asymp\frac{\sigma^{2}}{p\sigma_{\min}}\sqrt{\frac{\mu r\log n}{np}}
\]
as long as $np\gtrsim\mu r\log n$. This means that $\mathbb{P}(\mathcal{E}^{c})=O(n^{-10})$ and taking the union bounds over $1\leq i\leq n$ concludes the proof.

\section{Analysis of the entries of the matrix \label{sec:Analysis-of-entry}}

\subsection{Proof of Lemma \ref{lemma:entry-residual}}

The term $\Lambda_{ij}$ can be naturally split into two terms, namely
\[
\bm{e}_{i}^{\top}\bm{\Psi}_{\bm{X}}\bm{Y}^{\star\top}\bm{e}_{j}+\bm{e}_{i}^{\top}\bm{X}^{\star}\bm{\Psi}_{\bm{Y}}^{\top}\bm{e}_{j}\qquad\text{and}\qquad\bm{e}_{i}^{\top}\big(\overline{\bm{X}}^{\mathsf{d}}-\bm{X}^{\star}\big)\big(\overline{\bm{Y}}^{\mathsf{d}}-\bm{Y}^{\star}\big)^{\top}\bm{e}_{j}.
\]
In what follows, we shall bound each term individually.
\begin{enumerate}
\item Regarding the first term, one sees from Theorem~\ref{thm:low-rank-factor-master-bound}
that with probability exceeding $1-O(n^{-10})$
\[
\max\bigl\{\left\Vert \bm{\Psi}_{\bm{X}}\right\Vert _{2,\infty},\left\Vert \bm{\Psi}_{\bm{Y}}\right\Vert _{2,\infty}\bigr\}\lesssim\frac{\sigma}{\sqrt{p\sigma_{\min}}}\left(\frac{\sigma}{\sigma_{\min}}\sqrt{\frac{\kappa^{7}\mu rn\log n}{p}}+\sqrt{\frac{\kappa^{7}\mu^{3}r^{3}\log^{2}n}{np}}\right).
\]
As a result, we obtain 
\begin{align*}
\left|\bm{e}_{i}^{\top}\bm{\Psi}_{\bm{X}}\bm{Y}^{\star\top}\bm{e}_{j}+\bm{e}_{i}^{\top}\bm{X}^{\star}\bm{\Psi}_{\bm{Y}}^{\top}\bm{e}_{j}\right| & \leq\left\Vert \bm{\Psi}_{\bm{X}}\right\Vert _{2,\infty}\left\Vert \bm{Y}_{j,\cdot}^{\star}\right\Vert _{2}+\left\Vert \bm{X}_{i,\cdot}^{\star}\right\Vert _{2}\left\Vert \bm{\Psi}_{\bm{Y}}\right\Vert _{2,\infty}\\
 & \lesssim\left(\left\Vert \bm{U}_{i,\cdot}^{\star}\right\Vert _{2}+\left\Vert \bm{V}_{j,\cdot}^{\star}\right\Vert _{2}\right)\frac{\sigma}{\sqrt{p}}\left(\frac{\sigma}{\sigma_{\min}}\sqrt{\frac{\kappa^{8}\mu rn\log n}{p}}+\sqrt{\frac{\kappa^{8}\mu^{3}r^{3}\log^{2}n}{np}}\right),
\end{align*}
where the last line follows since $\|\bm{X}_{i,\cdot}^{\star}\|_{2}\leq\sqrt{\sigma_{\max}}\|\bm{U}_{i,\cdot}^{\star}\|_{2}$
and $\|\bm{Y}_{j,\cdot}^{\star}\|_{2}\leq\sqrt{\sigma_{\max}}\|\bm{V}_{j,\cdot}^{\star}\|_{2}$.
\item Turning to the second term, we have by the Cauchy-Schwarz inequality
that 
\begin{align*}
\left|\bm{e}_{i}^{\top}\big(\overline{\bm{X}}^{\mathsf{d}}-\bm{X}^{\star}\big)\big(\overline{\bm{Y}}^{\mathsf{d}}-\bm{Y}^{\star}\big)^{\top}\bm{e}_{j}\right| & \leq\big\|\overline{\bm{X}}^{\mathsf{d}}-\bm{X}^{\star}\big\|_{2,\infty}\big\|\overline{\bm{Y}}^{\mathsf{d}}-\bm{Y}^{\star}\big\|_{2,\infty}\lesssim\left(\kappa\frac{\sigma}{\sigma_{\min}}\sqrt{\frac{n\log n}{p}}\left\Vert \bm{F}^{\star}\right\Vert _{2,\infty}\right)^{2}\\
 & \lesssim\left(\frac{\sigma}{\sqrt{\sigma_{\min}}}\sqrt{\frac{\kappa^{3}\mu r\log n}{p}}\right)^{2},
\end{align*}
where the penultimate inequality uses (\ref{eq:F-d-2-infty}) and
the last one depends on the incoherence assumption that $\|\bm{F}^{\star}\|_{2,\infty}\leq\sqrt{\mu r\sigma_{\max}/n}$
(see~(\ref{eq:incoherence-X})).
\end{enumerate}
Take collectively the above two bounds to complete the proof.

\subsection{Proof of Lemma \ref{lemma:entry-normal}}

If $\bm{Z}_{\bm{X}}^{\top}\bm{e}_{i}$ and $\bm{Z}_{\bm{Y}}^{\top}\bm{e}_{j}$
were independent, then clearly one would have
\[
\bm{e}_{i}^{\top}\bm{Z}_{\bm{X}}\bm{Y}^{\star\top}\bm{e}_{j}+\bm{e}_{i}^{\top}\bm{X}^{\star}\bm{Z}_{\bm{Y}}^{\top}\bm{e}_{j}\sim\mathcal{N}\left(0,v_{ij}^{\star}\right).
\]
As such, the main ingredient of the proof boils down to demonstrating
that $\bm{Z}_{\bm{X}}^{\top}\bm{e}_{i}$ and $\bm{Z}_{\bm{Y}}^{\top}\bm{e}_{j}$
are nearly independent.

To begin with, we remind the readers of the way we construct $\bm{e}_{i}^\top \bm{Z}_{\bm{X}}$
and $\bm{e}_{j}^\top\bm{Z}_{\bm{Y}}$ in Appendix~\ref{subsec:Proof-of-Lemma-approx-gaussian-low-rank}:
there exist events $\mathcal{E}$ and $\widetilde{\mathcal{E}}$ with
$\mathbb{P}(\mathcal{E}^{\mathrm{c}}\cup\widetilde{\mathcal{E}}^{\mathrm{c}})\lesssim n^{-10}$
such that 
\begin{align*}
\bm{e}_{i}^{\top}\bm{Z}_{\bm{X}} & \triangleq\frac{1}{p}\bm{e}_{i}^{\top}\mathcal{P}_{\Omega}\left(\bm{E}\right)\bm{Y}^{\star}(\bm{Y}^{\star\top}\bm{Y}^{\star})^{-1}\bm{S}^{-1/2}\left(\bm{S}^{\star}\right)^{1/2}\qquad\quad\quad\text{on the event }\mathcal{E}\\
\bm{e}_{j}^{\top}\bm{Z}_{\bm{Y}} & \triangleq\frac{1}{p}\bm{e}_{j}^{\top}\big(\mathcal{P}_{\Omega}\left(\bm{E}\right)\big)^{\top}\bm{X}^{\star}(\bm{X}^{\star\top}\bm{X}^{\star})^{-1}\tilde{\bm{S}}^{-1/2}\left(\bm{S}^{\star}\right)^{1/2}\qquad\text{on the event }\widetilde{\mathcal{E}}
\end{align*}
where the randomness of $\bm{S}$ only comes from $\{\delta_{ik}\}_{k:1\leq k\leq n}$,
and the randomness of $\tilde{\bm{S}}$ only comes from $\{\delta_{kj}\}_{k:1\leq k\leq n}$.
In addition, the events $\mathcal{E}$ and $\widetilde{\mathcal{E}}$
depend only on $\{\delta_{ik}\}_{k:1\leq k\leq n}$ and $\{\delta_{kj}\}_{k:1\leq k\leq n}$,
respectively. As a result, $\bm{Z}_{\bm{X}}^{\top}\bm{e}_{i}$ depends
only on $\{\delta_{ik}, E_{ik}\}_{k:1\leq k\leq n}$ and $\bm{Z}_{\bm{Y}}^{\top}\bm{e}_{j}$
relies only on $\{\delta_{kj}, E_{kj}\}_{k:1\leq k\leq n}$. This tells us
that: the only common randomness underlying $\bm{Z}_{\bm{X}}^\top \bm{e}_{i}$
and $\bm{Z}_{\bm{Y}}^{\top}\bm{e}_{j}$ lies in $\delta_{ij}$ and
$E_{ij}$.

Fortunately, this weak dependency can be easily decoupled, for which
we have the following claim.

\begin{claim}\label{claim:independence}Suppose that $np\gg\kappa^{2}\mu r^{2}\log^{2}n$.
One has the decomposition 
\[
\bm{Z}_{\bm{X}}^{\top}\bm{e}_{i}=\widetilde{\bm{Z}}_{\bm{X}}^{\top}\bm{e}_{i}+\bm{\Delta}_{i},
\]
where $\widetilde{\bm{Z}}_{\bm{X}}^{\top}\bm{e}_{i}\sim\mathcal{N}(\bm{0},\sigma^{2}(\bm{\Sigma}^{\star})^{-1}/p)$
and is independent of $\{\delta_{kj}, E_kj\}_{k:1\leq k\leq n}$ and hence
of $\bm{Z}_{\bm{Y}}^{\top}\bm{e}_{j}$. In addition, with probability
at least $1-O(n^{-10})$ one has
\[
\|\bm{\Delta}_{i}\|_{2}\lesssim\frac{\sigma}{\sqrt{p\sigma_{\min}}}\sqrt{\frac{\kappa\mu r\log n}{np}}.
\]
\end{claim}

The desired result follows immediately from Claim~\ref{claim:independence},
since
\[
\bm{e}_{i}^{\top}\bm{Z}_{\bm{X}}\bm{Y}^{\star\top}\bm{e}_{j}+\bm{e}_{i}^{\top}\bm{X}^{\star}\bm{Z}_{\bm{Y}}^{\top}\bm{e}_{j}=\underset{\sim\,\mathcal{N}\left(0,v_{ij}^{\star}\right)}{\underbrace{\bm{e}_{i}^{\top}\widetilde{\bm{Z}}_{\bm{X}}\bm{Y}^{\star\top}\bm{e}_{j}+\bm{e}_{i}^{\top}\bm{X}^{\star}\bm{Z}_{\bm{Y}}^{\top}\bm{e}_{j}}}+\bm{\Delta}_{i}^{\top}\bm{Y}^{\star\top}\bm{e}_{j},
\]
where
\[
\big|\bm{\Delta}_{i}^{\top}\bm{Y}^{\star\top}\bm{e}_{j}\big|\leq\|\bm{\Delta}_{i}\|_{2}\|\bm{Y}_{j,\cdot}^{\star}\|_{2,\infty}\lesssim\frac{\sigma}{\sqrt{p\sigma_{\min}}}\sqrt{\frac{\kappa\mu r\log n}{np}}\sqrt{\sigma_{\max}}\|\bm{V}_{j,\cdot}^{\star}\|_{2,\infty}\asymp\frac{\sigma}{\sqrt{p}}\sqrt{\frac{\kappa^{2}\mu r\log n}{np}}\|\bm{V}_{j,\cdot}^{\star}\|_{2,\infty}.
\]

Similarly, repeating the same argument above, we can also show that
$\bm{e}_{i}^{\top}\bm{Z}_{\bm{X}}\bm{Y}^{\star\top}\bm{e}_{j}+\bm{e}_{i}^{\top}\bm{X}^{\star}\bm{Z}_{\bm{Y}}^{\top}\bm{e}_{j}$
can be decomposed as a Gaussian random variable $\mathcal{N}\left(0,v_{ij}^{\star}\right)$
as well as a residual term bounded above by $({\sigma}/{\sqrt{p}})\sqrt{({\kappa^{2}\mu r\log n})/({np})}\|\bm{U}_{i,\cdot}^{\star}\|_{2,\infty}$
with high probability. These together finish the proof.

\begin{proof}[Proof of Claim~\ref{claim:independence}] Instate
the notation used in Appendix~\ref{subsec:Proof-of-Lemma-approx-gaussian-low-rank}.
Recall that 
\[
\bm{Z}_{\bm{X}}^{\top}\bm{e}_{i}=\begin{cases}
\left(\bm{S}^{\star}\right)^{1/2}\bm{S}^{-1/2}\sum_{k=1}^{n}\frac{1}{p}E_{ik}\delta_{ik}\left(\bm{\Sigma}^{\star}\right)^{-1/2}\big(\bm{V}_{k,\cdot}^{\star}\big)^{\top}, & \text{on the event }\mathcal{E},\\
\bm{G}_{\bm{X}}^{\top}\bm{e}_{i}, & \text{on the event }\mathcal{E}^{c}.
\end{cases}
\]
To remove the effect of ${\delta_{ij},E_{ij}}$ on $\bm{Z}_{\bm{X}}^{\top}\bm{e}_{i}$,
we construct an auxiliary random matrix $\widetilde{\bm{Z}}_{\bm{X}}$
as follows 
\[
\widetilde{\bm{Z}}_{\bm{X}}^{\top}\bm{e}_{i}=\begin{cases}
\left(\bm{S}^{\star}\right)^{1/2}\bm{S}_{-j}^{-1/2}\sum_{k:k\neq j}\frac{1}{p}E_{ik}\delta_{ik}\left(\bm{\Sigma}^{\star}\right)^{-1/2}\big(\bm{V}_{k,\cdot}^{\star}\big)^{\top}, & \text{on the event }\mathcal{E}_{-j},\\
\bm{G}_{\bm{X}}^{\top}\bm{e}_{i}, & \text{on the event }\mathcal{E}_{-j}^{c},
\end{cases}
\]
where $\bm{S}^{\star}=p^{-1}\sigma^{2}\left(\bm{\Sigma}^{\star}\right)^{-1}$,
\[
\bm{S}_{-j}\triangleq\frac{\sigma^{2}}{p^{2}}\sum_{k:k\neq j}\delta_{ik}\left(\bm{\Sigma}^{\star}\right)^{-1/2}\big(\bm{V}_{k,\cdot}^{\star}\big)^{\top}\bm{V}_{k,\cdot}^{\star}\left(\bm{\Sigma}^{\star}\right)^{-1/2}\qquad\text{and}\qquad\mathcal{E}_{-j}\triangleq\left\{ \left\Vert \bm{S}_{-j}-\bm{S}^{\star}\right\Vert \lesssim\frac{\sigma^{2}}{p\sigma_{\min}}\sqrt{\frac{\mu r\log n}{np}}\right\} .
\]
It is easily seen that $\widetilde{\bm{Z}}_{\bm{X}}^{\top}\bm{e}_{i}\sim\mathcal{N}(\bm{0},\sigma^{2}(\bm{\Sigma}^{\star})^{-1}/p)$;
more importantly $\widetilde{\bm{Z}}_{\bm{X}}^{\top}\bm{e}_{i}$ is
independent of $\{\delta_{kj},E_{kj}\}_{1\leq k\leq n}$ and hence of $\bm{Z}_{\bm{Y}}^{\top}\bm{e}_{j}$.

We still need to verify the closeness between $\widetilde{\bm{Z}}_{\bm{X}}^{\top}\bm{e}_{i}$
and $\bm{Z}_{\bm{X}}^{\top}\bm{e}_{i}$. Towards this, we first repeat
the proof in Appendix~\ref{subsec:Proof-of-Lemma-approx-gaussian-low-rank}
to obtain $\mathbb{P}(\mathcal{E}_{-j})\geq1-O(n^{-10})$. Therefore
on the high probability event $\mathcal{E}\cap\mathcal{\mathcal{E}}_{-j}$,
one has
\begin{align*}
\big\Vert \widetilde{\bm{Z}}_{\bm{X}}^{\top}\bm{e}_{i}-\bm{Z}_{\bm{X}}^{\top}\bm{e}_{i}\big\Vert _{2} & \leq\big\Vert \left(\bm{S}^{\star}\right)^{1/2}\big\Vert \Bigg\|\bm{S}^{-1/2}\left(\bm{\Sigma}^{\star}\right)^{-1/2}\sum_{k=1}^{n}\frac{1}{p}E_{ik}\delta_{ik}\left(\bm{V}_{k,\cdot}^{\star}\right)^{\top}-\bm{S}_{-j}^{-1/2}\left(\bm{\Sigma}^{\star}\right)^{-1/2}\sum_{k:k\neq j}\frac{1}{p}E_{ik}\delta_{ik}\left(\bm{V}_{k,\cdot}^{\star}\right)^{\top}\Bigg\|,
\end{align*}
which together with the triangle inequality and the fact $\|\bm{S}^{\star}\|=\sigma^{2}/(p\sigma_{\min})$
yields 
\begin{align*}
\sqrt{\frac{p\sigma_{\min}}{\sigma^{2}}}\left\Vert \widetilde{\bm{Z}}_{\bm{X}}^{\top}\bm{e}_{i}-\bm{Z}_{\bm{X}}^{\top}\bm{e}_{i}\right\Vert _{2} & \leq\left\Vert \bm{S}^{-1/2}-\bm{S}_{-j}^{-1/2}\right\Vert \left\Vert \left(\bm{\Sigma}^{\star}\right)^{-1/2}\right\Vert \Bigg\|\sum_{k:k\neq j}\frac{1}{p}E_{ik}\delta_{ik}\bm{V}_{k,\cdot}^{\star}\Bigg\|_{2}\\
 & \quad+\left\Vert \bm{S}^{-1/2}\right\Vert \left\Vert \left(\bm{\Sigma}^{\star}\right)^{-1/2}\right\Vert \left\Vert \frac{1}{p}E_{ij}\delta_{ij}\bm{V}_{j,\cdot}^{\star}\right\Vert _{2}\\
 & \lesssim\left\Vert \bm{S}^{-1/2}-\bm{S}_{-j}^{-1/2}\right\Vert \frac{\sigma}{\sqrt{\sigma_{\min}}}\sqrt{\frac{r\log n}{p}}+\sqrt{\frac{p\sigma_{\max}}{\sigma^{2}}}\frac{1}{\sqrt{\sigma_{\min}}}\frac{\sigma\sqrt{\log n}}{p}\sqrt{\frac{\mu r}{n}}.
\end{align*}
Here we have used the results in (\ref{eq:useful-decouple}) and (\ref{eq:useful-decouple-2}).
We are left with bounding $\|\bm{S}^{-1/2}-\bm{S}_{-j}^{-1/2}\|$,
for which we have 
\[
\left\Vert \bm{S}-\bm{S}_{-j}\right\Vert =\frac{\sigma^{2}}{p^{2}}\left\Vert \delta_{ij}\left(\bm{\Sigma}^{\star}\right)^{-1/2}\left(\bm{V}_{j,\cdot}^{\star}\right)^{\top}\bm{V}_{j,\cdot}^{\star}\left(\bm{\Sigma}^{\star}\right)^{-1/2}\right\Vert \leq\frac{\sigma^{2}}{p^{2}\sigma_{\min}}\frac{\mu r}{n}.
\]
Take the above bound collectively with (\ref{eq:useful-decouple-2})
to yield 
\[
\left\Vert \bm{S}_{-j}^{-1/2}\right\Vert \lesssim\sqrt{\frac{p\sigma_{\max}}{\sigma^{2}}},
\]
as long as $np\gg\kappa\mu r$. As a result, we have 
\begin{align*}
\left\Vert \bm{S}^{-1/2}-\bm{S}_{-j}^{-1/2}\right\Vert  & \leq\left\Vert \bm{S}^{-1/2}\right\Vert \left\Vert \bm{S}^{1/2}-\bm{S}_{-j}^{1/2}\right\Vert \left\Vert \bm{S}_{-j}^{-1/2}\right\Vert \\
 & \lesssim\frac{p\sigma_{\max}}{\sigma^{2}}\cdot\frac{1}{\lambda_{\min}(\bm{S}^{1/2})+\lambda_{\min}(\bm{S}_{-j}^{1/2})}\left\Vert \bm{S}-\bm{S}_{-j}\right\Vert \\
 & \lesssim\frac{p\sigma_{\max}}{\sigma^{2}}\cdot\frac{1}{\sqrt{\frac{\sigma^{2}}{p\sigma_{\max}}}}\cdot\frac{\sigma^{2}}{p^{2}\sigma_{\min}}\frac{\mu r}{n}\asymp\frac{\kappa\mu r}{np}\sqrt{\frac{p\sigma_{\max}}{\sigma^{2}}},
\end{align*}
where the middle line relies on the perturbation of matrix square
roots; see Lemma~\ref{lemma:matrix-sqrt}. Combining all, we arrive
at 
\[
\sqrt{\frac{p\sigma_{\min}}{\sigma^{2}}}\left\Vert \widetilde{\bm{Z}}_{\bm{X}}^{\top}\bm{e}_{i}-\bm{Z}_{\bm{X}}^{\top}\bm{e}_{i}\right\Vert _{2}\lesssim\frac{\kappa\mu r}{np}\cdot\sqrt{\kappa r\log n}+\sqrt{\frac{\kappa\mu r\log n}{np}}\asymp\sqrt{\frac{\kappa\mu r\log n}{np}},
\]
with the proviso that $np\gg\kappa^{2}\mu r^{2}\log^{2}n$. This finishes
the proof. \end{proof}

\section{Proof of Corollary \ref{coro:confidence-interval} \label{sec:Proof-of-Corollary-confidence-interval}}

This section is dedicated to establishing the following result, which subsumes Corollary \ref{coro:confidence-interval} as a special case.

\begin{corollary}
\label{coro:confidence-interval-augment} 
Suppose that the conditions (\ref{subeq:entry-condition}) hold, and recall the notation in  Corollary \ref{coro:confidence-interval}. 
Then one has 
\begin{align*}
	& \sup_{0<\alpha<1}\Big|\mathbb{P}\Big\{ M_{ij}^{\star}\in\big[M_{ij}^{\mathsf{{d}}}\pm\Phi^{-1}\left(1-\alpha/2\right) \sqrt{v_{ij}}\big]\Big\} -(1-\alpha)\Big|\\
	& \quad\lesssim   \frac{\sigma}{\sigma_{\min}}\sqrt{\frac{\kappa^{8}\mu rn\log n}{p}}+\sqrt{\frac{\kappa^{8}\mu^{3}r^{3}\log^{2}n}{np}}+\left(\left\Vert \bm{U}_{i,\cdot}^{\star}\right\Vert _{2}+\left\Vert \bm{V}_{j,\cdot}^{\star}\right\Vert _{2}\right)^{-1}\sqrt{\frac{r}{n}}\frac{\sigma}{\sigma_{\min}}\sqrt{\frac{\kappa^{10}\mu^{2}rn\log^{2}n}{p}}.
\end{align*}
\end{corollary}

Before entering the main proof of Corollary \ref{coro:confidence-interval-augment}, we make a simple observation that
\begin{equation}
\max\left\{ \bigl\Vert\overline{\bm{X}}_{i,\cdot}^{\mathsf{d}}-\bm{X}_{i,\cdot}^{\star}\bigr\Vert_{2},\bigl\Vert\overline{\bm{Y}}_{j,\cdot}^{\mathsf{d}}-\bm{Y}_{j,\cdot}^{\star}\bigr\Vert_{2}\right\} \lesssim\kappa\frac{\sigma}{\sigma_{\min}}\sqrt{\frac{n\log n}{p}}\sqrt{\frac{\mu r\sigma_{\max}}{n}}\leq\frac{\sqrt{\sigma_{\max}}}{\kappa^{2}}\left(\left\Vert \bm{U}_{i,\cdot}^{\star}\right\Vert _{2}+\left\Vert \bm{V}_{j,\cdot}^{\star}\right\Vert _{2}\right),\label{eq:relative-bound}
\end{equation}
where we recall that $\overline{\bm{X}}^{\mathsf{d}}=\bm{X}^{\mathsf{d}}\bm{H}^{\mathsf{d}}$
and $\overline{\bm{Y}}^{\mathsf{d}}=\bm{Y}^{\mathsf{d}}\bm{H}^{\mathsf{d}}$.
Here, the first inequality arises from (\ref{eq:F-d-2-infty}) and
the second one uses the assumption on $\|\bm{U}_{i,\cdot}^{\star}\|_{2}+\|\bm{V}_{j,\cdot}^{\star}\|_{2}$
(i.e.~(\ref{eq:entry-inference-lower-bound})). A simple consequence
of~(\ref{eq:relative-bound}) is that 
\begin{equation}
\max\left\{ \bigl\Vert\overline{\bm{X}}_{i,\cdot}^{\mathsf{d}}\bigr\Vert_{2},\bigl\Vert\overline{\bm{Y}}_{j,\cdot}^{\mathsf{d}}\bigr\Vert_{2}\right\} \leq2\sqrt{\sigma_{\max}}\left(\left\Vert \bm{U}_{i,\cdot}^{\star}\right\Vert _{2}+\left\Vert \bm{V}_{j,\cdot}^{\star}\right\Vert _{2}\right).\label{eq:absolute-bound}
\end{equation}

Turning to the main proof, we define
\begin{equation}
\Delta_{V}\triangleq\frac{M_{ij}^{\mathsf{d}}-M_{ij}^{\star}}{\sqrt{v_{ij}}}-\frac{M_{ij}^{\mathsf{d}}-M_{ij}^{\star}}{\sqrt{v_{ij}^{\star}}},\label{eq:defn-Delta-V}
\end{equation}
which in conjunction with Theorem~\ref{thm:entries-master-decomposition-augment}
yields the following decomposition 
\[
\frac{M_{ij}^{\mathsf{d}}-M_{ij}^{\star}}{\sqrt{v_{ij}}}=\frac{M_{ij}^{\mathsf{d}}-M_{ij}^{\star}}{\sqrt{v_{ij}^{\star}}}+\Delta_{V}=\frac{g_{ij}}{\sqrt{v_{ij}^{\star}}}+\frac{\Delta_{ij}}{\sqrt{v_{ij}^{\star}}}+\Delta_{V}.
\]
With this decomposition at hand, we have that for any $\varepsilon>0$,
\begin{align*}
\mathbb{P}\left(\frac{M_{ij}^{\mathsf{d}}-M_{ij}^{\star}}{\sqrt{v_{ij}}}\leq t\right)-\Phi\left(t\right) & =\mathbb{P}\left(\frac{g_{ij}}{\sqrt{v_{ij}^{\star}}}+\frac{\Delta_{ij}}{\sqrt{v_{ij}^{\star}}}+\Delta_{V}\leq t\right)-\Phi\left(t\right)\\
 & \leq\mathbb{P}\left(\frac{g_{ij}}{\sqrt{v_{ij}^{\star}}}\leq t+\varepsilon\right)+\mathbb{P}\left(\frac{\left|\Delta_{ij}\right|}{\sqrt{v_{ij}^{\star}}}+\left|\Delta_{V}\right|\geq\varepsilon\right)-\Phi\left(t\right)\\
 & \overset{(\text{i})}{=}\Phi\left(t+\varepsilon\right)-\Phi\left(t\right)+\mathbb{P}\left(\left|\Delta_{ij}\right|+\left|\Delta_{V}\right|\sqrt{v_{ij}^{\star}}\geq\varepsilon\sqrt{v_{ij}^{\star}}\right)\\
 & \leq \, \varepsilon+\mathbb{P}\left(\left|\Delta_{ij}\right|+\left|\Delta_{V}\right|\sqrt{v_{ij}^{\star}}\geq\varepsilon\sqrt{v_{ij}^{\star}}\right),
\end{align*}
where $\Phi(\cdot)$ is the CDF of $\mathcal{N}(0,1)$.
Here, the relation (i) uses the fact that $g_{ij} \sim\mathcal{N}(0,v_{ij}^{\star})$. 
It then suffices to upper bound the right-hand side $\varepsilon+\mathbb{P}(|\Delta_{ij}|+|\Delta_{V}|\sqrt{v_{ij}^{\star}}\geq\varepsilon\sqrt{v_{ij}^{\star}})$.
Our goal is to demonstrate that for a particular choice of $\varepsilon>0$,
this quantity is well controlled. In view of Theorem~\ref{thm:entries-master-decomposition-augment},
we know that $|\Delta_{ij}|$ is small with high probability. We are
still in need of a high probability bound on the term $|\Delta_{V}|$, 
which we obtain through the following claim.

\begin{claim}
\label{claim:Delta-V}
With probability exceeding $1-O(n^{-10})$, the term $\Delta_{V}$ obeys 
\[
\left|\Delta_{V}\right|\lesssim\left(\left\Vert \bm{U}_{i,\cdot}^{\star}\right\Vert _{2}+\left\Vert \bm{V}_{j,\cdot}^{\star}\right\Vert _{2}\right)^{-1} \frac{\sigma}{\sigma_{\min}}\sqrt{\frac{\kappa^{10}\mu^{2}rn\log^{2}n}{p}}\sqrt{\frac{r}{n}}.
\]
\end{claim}

With Claim~\ref{claim:Delta-V} at hand, we are ready to take 
\[
\varepsilon\asymp\frac{\sigma}{\sigma_{\min}}\sqrt{\frac{\kappa^{8}\mu rn\log n}{p}}+\sqrt{\frac{\kappa^{8}\mu^{3}r^{3}\log^{2}n}{np}}+\left(\left\Vert \bm{U}_{i,\cdot}^{\star}\right\Vert _{2}+\left\Vert \bm{V}_{j,\cdot}^{\star}\right\Vert _{2}\right)^{-1} \frac{\sigma}{\sigma_{\min}}\sqrt{\frac{\kappa^{10}\mu^{2}rn\log^{2}n}{p}}\sqrt{\frac{r}{n}}
\]
and arrive at the upper bound 
\begin{align*}
	& \mathbb{P}\left(\frac{M_{ij}^{\mathsf{d}}-M_{ij}^{\star}}{\sqrt{v_{ij}}}\leq t\right)-\Phi\left(t\right) \leq ~\varepsilon + n^{-3}.
\end{align*}
A similar argument yields the lower bound on $\mathbb{P}(M_{ij}^{\mathsf{d}}-M_{ij}^{\star}\leq t\sqrt{v_{ij}})-\Phi(t)$. As a result, one has
\begin{align*}
	& \left| \mathbb{P}\left(\frac{M_{ij}^{\mathsf{d}}-M_{ij}^{\star}}{\sqrt{v_{ij}}}\leq t\right)-\Phi\left(t\right) \right| \lesssim \varepsilon + n^{-3}\\
 & \quad \asymp \frac{\sigma}{\sigma_{\min}}\sqrt{\frac{\kappa^{8}\mu rn\log n}{p}}+\sqrt{\frac{\kappa^{8}\mu^{3}r^{3}\log^{2}n}{np}}+\left(\left\Vert \bm{U}_{i,\cdot}^{\star}\right\Vert _{2}+\left\Vert \bm{V}_{j,\cdot}^{\star}\right\Vert _{2}\right)^{-1} \frac{\sigma}{\sigma_{\min}}\sqrt{\frac{\kappa^{10}\mu^{2}rn\log^{2}n}{p}}\sqrt{\frac{r}{n}}
\end{align*}
for any $t$.  This immediately establishes Corollary \ref{coro:confidence-interval-augment}.

\begin{proof}[Proof of Claim~\ref{claim:Delta-V}]
Recall that
\[
\Delta_{V}=\left(M_{ij}^{\mathsf{d}}-M_{ij}^{\star}\right)\left[\left(v_{ij}\right)^{-1/2}-\left(v_{ij}^{\star}\right)^{-1/2}\right]=\left(M_{ij}^{\mathsf{d}}-M_{ij}^{\star}\right)\frac{v_{ij}^{\star}-v_{ij}}{\sqrt{v_{ij}^{\star}}\sqrt{v_{ij}}}\frac{1}{\sqrt{v_{ij}^{\star}}+\sqrt{v_{ij}}}.
\]
Suppose for the moment that $|v_{ij}-v_{ij}^{\star}|\leq c v_{ij}^{\star}$
for some $c\leq1/2$. Then it follows immediately that 
\[
\left|\Delta_{V}\right|\lesssim c\frac{\left|M_{ij}^{\mathsf{d}}-M_{ij}^{\star}\right|}{\sqrt{v_{ij}^{\star}}}.
\]
Therefore if suffices to control $|M_{ij}^{\mathsf{d}}-M_{ij}^{\star}|$
and $|v_{ij}^{\star}-v_{ij}|$ (i.e.~obtaining the quantity $c$).
\begin{itemize}
\item First, expand $M_{ij}^{\mathsf{d}}$ and $M_{ij}^{\star}$ to see
\begin{align*}
\left|M_{ij}^{\mathsf{d}}-M_{ij}^{\star}\right| & =\left|\overline{\bm{X}}_{i,\cdot}^{\mathsf{d}}\big(\overline{\bm{Y}}_{j,\cdot}^{\mathsf{d}}\big)^{\top}-\bm{X}_{i,\cdot}^{\star}(\bm{Y}_{j,\cdot}^{\star})^{\top}\right|\leq\bigl\Vert\overline{\bm{X}}_{i,\cdot}^{\mathsf{d}}-\bm{X}_{i,\cdot}^{\star}\bigr\Vert_{2}\bigl\Vert\overline{\bm{Y}}_{j,\cdot}^{\mathsf{d}}\bigr\Vert_{2}+\left\Vert \bm{X}_{i,\cdot}^{\star}\right\Vert _{2}\bigl\Vert\overline{\bm{Y}}_{j,\cdot}^{\mathsf{d}}-\bm{Y}_{j,\cdot}^{\star}\bigr\Vert_{2}\\
 & \lesssim\kappa\frac{\sigma}{\sigma_{\min}}\sqrt{\frac{n\log n}{p}}\sqrt{\frac{\mu r}{n}}\sigma_{\max}\left(\left\Vert \bm{U}_{i,\cdot}^{\star}\right\Vert _{2}+\left\Vert \bm{V}_{j,\cdot}^{\star}\right\Vert _{2}\right)\\
 & \lesssim\kappa^{2}\sqrt{\mu r\log n}\sqrt{v_{ij}^{\star}},
\end{align*}
where the middle line depends on (\ref{eq:relative-bound}) and (\ref{eq:absolute-bound}), 
and the last inequality arises since $\sigma(\|\bm{U}_{i,\cdot}^{\star}\|_{2}+\|\bm{V}_{j,\cdot}^{\star}\|_{2})/\sqrt{p}\lesssim\sqrt{v_{ij}^{\star}}$.

\item Now we move on to $|v_{ij}^{\star}-v_{ij}|$. By the definition of
$v_{ij}$, one has 
\begin{align*}
\left|v_{ij}^{\star}-v_{ij}\right| & \leq\frac{\sigma^{2}}{p}\left|\bm{X}_{i,\cdot}^{\star}\left(\bm{X}^{\star\top}\bm{X}^{\star}\right)^{-1}(\bm{X}_{i,\cdot}^{\star})^{\top}-\bm{X}_{i,\cdot}^{\mathsf{d}}\left(\bm{X}^{\mathsf{d}\top}\bm{X}^{\mathsf{d}}\right)^{-1}(\bm{X}_{i,\cdot}^{\mathsf{d}})^{\top}\right|\\
 & \quad+\frac{\sigma^{2}}{p}\left|\bm{Y}_{j,\cdot}^{\star}\left(\bm{Y}^{\star\top}\bm{Y}^{\star}\right)^{-1}(\bm{Y}_{j,\cdot}^{\star})^{\top}-\bm{Y}_{j,\cdot}^{\mathsf{d}}\left(\bm{Y}^{\mathsf{d}\top}\bm{Y}^{\mathsf{d}}\right)^{-1}(\bm{Y}_{j,\cdot}^{\mathsf{d}})^{\top}\right|.
\end{align*}
Focusing on the $\bm{X}$ factor, we have --- with probability at least
$1-O(n^{-10})$ --- that
\begin{align}
 & \left|\bm{X}_{i,\cdot}^{\star}\left(\bm{X}^{\star\top}\bm{X}^{\star}\right)^{-1}(\bm{X}_{i,\cdot}^{\star})^{\top}-\bm{X}_{i,\cdot}^{\mathsf{d}}\left(\bm{X}^{\mathsf{d}\top}\bm{X}^{\mathsf{d}}\right)^{-1}(\bm{X}_{i,\cdot}^{\mathsf{d}})^{\top}\right|\nonumber \\
 & \quad=\left|\bm{X}_{i,\cdot}^{\star}\left(\bm{X}^{\star\top}\bm{X}^{\star}\right)^{-1}(\bm{X}_{i,\cdot}^{\star})^{\top}-\overline{\bm{X}}_{i,\cdot}^{\mathsf{d}}\big(\overline{\bm{X}}^{\mathsf{d}\top}\overline{\bm{X}}^{\mathsf{d}}\big)^{-1}(\overline{\bm{X}}_{i,\cdot}^{\mathsf{d}})^{\top}\right|\nonumber \\
 & \quad\leq\left\Vert \bm{X}_{i,\cdot}^{\star}\left(\bm{X}^{\star\top}\bm{X}^{\star}\right)^{-1}\right\Vert _{2}\bigl\Vert\bm{X}_{i,\cdot}^{\star}-\overline{\bm{X}}_{i,\cdot}^{\mathsf{d}}\bigr\Vert_{2}+\left\Vert \bm{X}_{i,\cdot}^{\star}\right\Vert _{2}\bigl\Vert\left(\bm{X}^{\star\top}\bm{X}^{\star}\right)^{-1}-\big(\overline{\bm{X}}^{\mathsf{d}\top}\overline{\bm{X}}^{\mathsf{d}}\big)^{-1}\bigr\Vert\bigl\Vert\overline{\bm{X}}_{i,\cdot}^{\mathsf{d}}\bigr\Vert_{2}\nonumber \\
 & \quad\quad+\bigl\Vert\bm{X}_{i,\cdot}^{\star}-\overline{\bm{X}}_{i,\cdot}^{\mathsf{d}}\bigr\Vert_{2}\bigl\Vert\big(\overline{\bm{X}}^{\mathsf{d}\top}\overline{\bm{X}}^{\mathsf{d}}\big)^{-1}\bigr\Vert\bigl\Vert\overline{\bm{X}}_{i,\cdot}^{\mathsf{d}}\bigr\Vert_{2}.
\label{eq:V_ij-upper-bound}
\end{align}
Here, the first relation comes from the identity $\bm{X}_{i,\cdot}^{\mathsf{d}}(\bm{X}^{\mathsf{d}\top}\bm{X}^{\mathsf{d}})^{-1}(\bm{X}_{i,\cdot}^{\mathsf{d}})^{\top}=\overline{\bm{X}}_{i,\cdot}^{\mathsf{d}}(\overline{\bm{X}}^{\mathsf{d}\top}\overline{\bm{X}}^{\mathsf{d}})^{-1}(\overline{\bm{X}}_{i,\cdot}^{\mathsf{d}})^{\top}$, 
and the inequality arises from the triangle inequality. Notice that $\|(\overline{\bm{X}}^{\mathsf{d}\top}\overline{\bm{X}}^{\mathsf{d}})^{-1}\|\lesssim1/\sigma_{\min}$
and that 
\begin{align*}
\bigl\Vert\left(\bm{X}^{\star\top}\bm{X}^{\star}\right)^{-1}-\big(\overline{\bm{X}}^{\mathsf{d}\top}\overline{\bm{X}}^{\mathsf{d}}\big)^{-1}\bigr\Vert & \leq\bigl\Vert\left(\bm{X}^{\star\top}\bm{X}^{\star}\right)^{-1}\bigr\Vert\bigl\Vert\bm{X}^{\star\top}\bm{X}^{\star}-\overline{\bm{X}}^{\mathsf{d}\top}\overline{\bm{X}}^{\mathsf{d}}\bigr\Vert\bigl\Vert\big(\overline{\bm{X}}^{\mathsf{d}\top}\overline{\bm{X}}^{\mathsf{d}}\big)^{-1}\bigr\Vert\\
 & \lesssim\frac{1}{\sigma_{\min}^{2}}\bigl\Vert\overline{\bm{X}}^{\mathsf{d}}-\bm{X}^{\star}\bigr\Vert\left\Vert \bm{X}^{\star}\right\Vert \lesssim\kappa^{2}\frac{\sigma}{\sigma_{\min}}\sqrt{\frac{n}{p}}\cdot\frac{1}{\sigma_{\min}},
\end{align*}
where the last inequality follows from (\ref{eq:F-d-op-quality-H-d}).
Using the bounds (\ref{eq:relative-bound}) and (\ref{eq:absolute-bound}),
we continue the upper bound in (\ref{eq:V_ij-upper-bound}) as follows
\begin{align*}
 & \left|\bm{X}_{i,\cdot}^{\star}\left(\bm{X}^{\star\top}\bm{X}^{\star}\right)^{-1}(\bm{X}_{i,\cdot}^{\star})^{\top}-\bm{X}_{i,\cdot}^{\mathsf{d}}\left(\bm{X}^{\mathsf{d}\top}\bm{X}^{\mathsf{d}}\right)^{-1}(\bm{X}_{i,\cdot}^{\mathsf{d}})^{\top}\right|\\
 & \quad\lesssim\frac{1}{\sqrt{\sigma_{\min}}}\left\Vert \bm{U}_{i,\cdot}^{\star}\right\Vert _{2}\cdot\kappa\frac{\sigma}{\sigma_{\min}}\sqrt{\frac{n\log n}{p}}\sqrt{\frac{\mu r\sigma_{\max}}{n}}\\
 & \quad\quad+\left\Vert \bm{U}_{i,\cdot}^{\star}\right\Vert _{2}\sqrt{\sigma_{\max}}\cdot\kappa^{2}\frac{\sigma}{\sigma_{\min}}\sqrt{\frac{n}{p}}\frac{1}{\sigma_{\min}}\cdot\sqrt{\sigma_{\max}}\left(\left\Vert \bm{U}_{i,\cdot}^{\star}\right\Vert _{2}+\left\Vert \bm{V}_{j,\cdot}^{\star}\right\Vert _{2}\right)\\
 & \quad\quad+\kappa\frac{\sigma}{\sigma_{\min}}\sqrt{\frac{n\log n}{p}}\sqrt{\frac{\mu r\sigma_{\max}}{n}}\cdot\frac{1}{\sigma_{\min}}\cdot\sqrt{\sigma_{\max}}\left(\left\Vert \bm{U}_{i,\cdot}^{\star}\right\Vert _{2}+\left\Vert \bm{V}_{j,\cdot}^{\star}\right\Vert _{2}\right)\\
 & \quad\lesssim\kappa^{3}\frac{\sigma}{\sigma_{\min}}\sqrt{\frac{n\log n}{p}}\sqrt{\frac{\mu r}{n}}\left(\left\Vert \bm{U}_{i,\cdot}^{\star}\right\Vert _{2}+\left\Vert \bm{V}_{j,\cdot}^{\star}\right\Vert _{2}\right).
\end{align*}
A similar bound holds for the factor $\bm{Y}$. Therefore, with high
probability we have 
\begin{align*}
\left|v_{ij}^{\star}-v_{ij}\right| & \lesssim\frac{\sigma^{2}}{p}\kappa^{3}\frac{\sigma}{\sigma_{\min}}\sqrt{\frac{n\log n}{p}}\sqrt{\frac{\mu r}{n}}\left(\left\Vert \bm{U}_{i,\cdot}^{\star}\right\Vert _{2}+\left\Vert \bm{V}_{j,\cdot}^{\star}\right\Vert _{2}\right)\\
 & \lesssim\left(\left\Vert \bm{U}_{i,\cdot}^{\star}\right\Vert _{2}+\left\Vert \bm{V}_{j,\cdot}^{\star}\right\Vert _{2}\right)^{-1}\kappa^{3}\frac{\sigma}{\sigma_{\min}}\sqrt{\frac{n\log n}{p}}\sqrt{\frac{\mu r}{n}}V_{ij}^{\star}\leq\frac{1}{2}v_{ij}^{\star},
\end{align*}
where the last relation results from the condition on $\|\bm{U}_{i,\cdot}^{\star}\|_{2}+\|\bm{V}_{j,\cdot}^{\star}\|_{2}$
(cf.~(\ref{eq:entry-inference-lower-bound})).
\end{itemize}
Combine the bounds on $|M_{ij}^{\mathsf{d}}-M_{ij}^{\star}|$ and
$|v_{ij}^{\star}-v_{ij}|$ to see that with probability exceeding
$1-O(n^{-10})$,
\begin{align*}
\left|\Delta_{V}\right| & \lesssim\left(\left\Vert \bm{U}_{i,\cdot}^{\star}\right\Vert _{2}+\left\Vert \bm{V}_{j,\cdot}^{\star}\right\Vert _{2}\right)^{-1}\kappa^{3}\frac{\sigma}{\sigma_{\min}}\sqrt{\frac{n\log n}{p}}\sqrt{\frac{\mu r}{n}}\cdot\kappa^{2}\sqrt{\mu r\log n}\\
 & \lesssim\left(\left\Vert \bm{U}_{i,\cdot}^{\star}\right\Vert _{2}+\left\Vert \bm{V}_{j,\cdot}^{\star}\right\Vert _{2}\right)^{-1} \frac{\sigma}{\sigma_{\min}}\sqrt{\frac{\kappa^{10}n\mu^{2}r\log^{2}n}{p}}\sqrt{\frac{r}{n}}.
\end{align*}
This establishes the desired upper bound on $|\Delta_{V}|$. 
\end{proof}

\section{Proof of Theorem \ref{thm:estimation-error}}\label{sec:estimation-error} 
As we have argued in Section~\ref{subsec:Approximate-equivalence}, it suffices to prove the claim for $\bm{M}^{\mathsf{d}}=\bm{X}^{\mathsf{d}}\bm{X}^{\mathsf{d}\top}= \overline{\bm{X}}^{\mathsf{d}}\overline{\bm{Y}}^{\mathsf{d}\top}$. 
For simplicity of notation, we define 
\[
\bm{\Gamma}_{\bm{X}}\triangleq\overline{\bm{X}}^{\mathsf{d}}-\bm{X}^{\star}\qquad\text{and}\qquad\bm{\Gamma}_{\bm{Y}}\triangleq\overline{\bm{Y}}^{\mathsf{d}}-\bm{Y}^{\star}.
\]
Apply the decompositions in Theorem~\ref{thm:low-rank-factor-master-bound}
to obtain 
\begin{align}
\bm{M}^{\mathsf{d}}-\bm{M}^{\star} & =\overline{\bm{X}}^{\mathsf{d}}\overline{\bm{Y}}^{\mathsf{d}\top}-\bm{X}^{\star}\bm{Y}^{\star\top} \nonumber\\
 & =\bm{\Gamma}_{\bm{X}}\bm{Y}^{\star\top}+\bm{X}^{\star}\bm{\Gamma}_{\bm{Y}}^{\top}+\bm{\Gamma}_{\bm{X}}\bm{\Gamma}_{\bm{Y}}^{\top} \nonumber\\
	& =\bm{Z}_{\bm{X}}\bm{Y}^{\star\top}+\bm{X}^{\star}\bm{Z}_{\bm{Y}}^{\top}+\underbrace{\bm{\Psi}_{\bm{X}}\bm{Y}^{\star\top}+\bm{X}^{\star}\bm{\Psi}_{\bm{Y}}^{\top}+\bm{\Gamma}_{\bm{X}}\bm{\Gamma}_{\bm{Y}}^{\top}}_{ \triangleq \bm{\Theta}},  \label{eq:Md-M-Theta}
\end{align}
where $\bm{\Psi}_{\bm{X}}$ and $\bm{\Psi}_{\bm{Y}}$ are defined in Theorem~\ref{thm:low-rank-factor-master-bound}. 
Further, expand $\|\bm{M}^{\mathsf{d}}-\bm{M}^{\star}\|_{\mathrm{F}}^{2}$
to obtain 
\[
\left\Vert \bm{M}^{\mathsf{d}}-\bm{M}^{\star}\right\Vert _{\mathrm{F}}^{2}=\left\Vert \bm{Z}_{\bm{X}}\bm{Y}^{\star\top}\right\Vert _{\mathrm{F}}^{2}+\left\Vert \bm{X}^{\star}\bm{Z}_{\bm{Y}}^{\top}\right\Vert _{\mathrm{F}}^{2}+\mathsf{rem},
\]
where we define the remainder term as
\begin{align*}
\mathsf{rem} & \triangleq2\mathsf{Tr}\left(\bm{Z}_{\bm{X}}\bm{Y}^{\star\top}\bm{Z}_{\bm{Y}}\bm{X}^{\star\top}\right)+\left\Vert \bm{\Theta}\right\Vert _{\mathrm{F}}^{2}+2\mathsf{Tr}\left(\bm{Z}_{\bm{X}}\bm{Y}^{\star\top}\bm{\Theta}^{\top}\right)+2\mathsf{Tr}\left(\bm{X}^{\star}\bm{Z}_{\bm{Y}}^{\top}\bm{\Theta}^{\top}\right).
\end{align*}
In what follows, we aim to demonstrate that $\|\bm{Z}_{\bm{X}}\bm{Y}^{\star}\|_{\mathrm{F}}^{2}+\|\bm{X}^{\star}\bm{Z}_{\bm{Y}}^{\top}\|_{\mathrm{F}}^{2}$, which can be shown to
sharply concentrate around its mean, is the dominant term, and
the remainder term $\mathsf{rem}$ is much smaller in magnitude with high probability.

\begin{itemize}
\item We begin with the term $\|\bm{Z}_{\bm{X}}\bm{Y}^{\star}\|_{\mathrm{F}}^{2}+\|\bm{X}^{\star}\bm{Z}_{\bm{Y}}^{\top}\|_{\mathrm{F}}^{2}$.
We shall focus on bounding $\|\bm{Z}_{\bm{X}}\bm{Y}^{\star}\|_{\mathrm{F}}^{2}$
since the other term $\|\bm{X}^{\star}\bm{Z}_{\bm{Y}}^{\top}\|_{\mathrm{F}}^{2}$
can be treated analogously. To this end, we first have the identity
\[
\frac{p}{\sigma^{2}}\left\Vert \bm{Z}_{\bm{X}}\bm{Y}^{\star\top}\right\Vert _{\mathrm{F}}^{2}=\frac{p}{\sigma^{2}}\mathsf{Tr}\left(\bm{Z}_{\bm{X}}\bm{Y}^{\star\top}\bm{Y}^{\star}\bm{Z}_{\bm{X}}^{\top}\right)=\frac{p}{\sigma^{2}}\mathsf{Tr}\left(\bm{Z}_{\bm{X}}\bm{\Sigma}^{\star}\bm{Z}_{\bm{X}}^{\top}\right)=\sum_{i=1}^{n}\Bigl\Vert\frac{\sqrt{p}}{\sigma}\left(\bm{\Sigma}^{\star}\right)^{1/2}\bm{Z}_{\bm{X}}^{\top}\bm{e}_{i}\Bigr\Vert_{2}^{2},
\]
where we use the fact that $\bm{Y}^{\star\top}\bm{Y}^{\star}=\bm{\Sigma}^{\star}$.
Theorem~\ref{thm:low-rank-factor-master-bound} tells us
that $\bm{Z}_{\bm{X}}^{\top}\bm{e}_{i}\overset{\text{i.i.d.}}{\sim}\mathcal{N}(\bm{0},\sigma^{2}(\bm{\Sigma}^{\star})^{-1}/p)$,
which further implies
\[
\frac{\sqrt{p}}{\sigma}\left(\bm{\Sigma}^{\star}\right)^{1/2}\bm{Z}_{\bm{X}}^{\top}\bm{e}_{i}\overset{\text{i.i.d.}}{\sim}\mathcal{N}\left(\bm{0},\bm{I}_{r}\right).
\]
Therefore, the quantity $p\|\bm{Z}_{\bm{X}}\bm{Y}^{\star}\|_{\mathrm{F}}^{2}/\sigma^{2}$
follows the chi-squared distribution with $nr$ degrees of freedom. Standard concentration inequalities
 \cite[Equation (2.19)]{wainwright2019high}
reveals that with probability at least $1-O(n^{-10})$,
\[
\left|\frac{p}{\sigma^{2}}\left\Vert \bm{Z}_{\bm{X}}\bm{Y}^{\star\top}\right\Vert _{\mathrm{F}}^{2}-nr\right|\lesssim\sqrt{nr\log n}.
\]
Repeating the above argument for $\|\bm{X}^{\star}\bm{Z}_{\bm{Y}}^{\top}\|_{\mathrm{F}}^{2}$, we conclude that with probability at least $1-O(n^{-10})$,
\[
\left\Vert \bm{Z}_{\bm{X}}\bm{Y}^{\star\top}\right\Vert _{\mathrm{F}}^{2}+\left\Vert \bm{X}^{\star}\bm{Z}_{\bm{Y}}^{\top}\right\Vert _{\mathrm{F}}^{2}
		= 2\frac{\sigma^{2}nr}{p}+O\left(\frac{\sigma^{2}}{p}\sqrt{nr\log n}\right)
		= (2+o(1))\frac{\sigma^{2}nr}{p}.
\]

\item Now we turn to the term $\mathsf{rem}$, for which we have the following
two claims.

\begin{claim}\label{claim:rem-2}With probability at least $1-O(n^{-10})$,
one has 
\[
\left|\left\Vert \bm{\Theta}\right\Vert _{\mathrm{F}}^{2}+2\mathsf{Tr}\left(\bm{Z}_{\bm{X}}\bm{Y}^{\star\top}\bm{\Theta}^{\top}\right)+2\mathsf{Tr}\left(\bm{X}^{\star}\bm{Z}_{\bm{Y}}^{\top}\bm{\Theta}^{\top}\right)\right|=o\left(\frac{\sigma^{2}nr}{p}\right).
\]
\end{claim}

\begin{claim}\label{claim:rem-1}With probability exceeding $1-O(n^{-10})$,
we have 
\[
\left|\mathsf{Tr}\left(\bm{Z}_{\bm{X}}\bm{Y}^{\star\top}\bm{Z}_{\bm{Y}}\bm{X}^{\star\top}\right)\right|=o\left(\frac{\sigma^{2}nr}{p}\right).
\]
\end{claim}
\end{itemize}
Combine all of the above bounds to yield the desired result.

\begin{proof}[Proof of Claim~\ref{claim:rem-2}]Use triangle inequality
and the bound $|\mathsf{Tr}(\bm{A}\bm{B})|\leq\|\bm{A}\|_{\mathrm{F}}\|\bm{B}\|_{\mathrm{F}}$
to obtain 
\begin{align}
\left|\left\Vert \bm{\Theta}\right\Vert _{\mathrm{F}}^{2}+2\mathsf{Tr}\left(\bm{Z}_{\bm{X}}\bm{Y}^{\star\top}\bm{\Theta}^{\top}\right)+2\mathsf{Tr}\left(\bm{X}^{\star}\bm{Z}_{\bm{Y}}^{\top}\bm{\Theta}^{\top}\right)\right| & \leq\left\Vert \bm{\Theta}\right\Vert _{\mathrm{F}}^{2}+2\left\Vert \bm{Z}_{\bm{X}}\bm{Y}^{\star\top}\right\Vert _{\mathrm{F}}\left\Vert \bm{\Theta}\right\Vert _{\mathrm{F}}+2\left\Vert \bm{X}^{\star}\bm{Z}_{\bm{Y}}^{\top}\right\Vert _{\mathrm{F}}\left\Vert \bm{\Theta}\right\Vert _{\mathrm{F}} \nonumber\\
	& =\left(\left\Vert \bm{\Theta}\right\Vert _{\mathrm{F}}+2\left\Vert \bm{Z}_{\bm{X}}\bm{Y}^{\star\top}\right\Vert _{\mathrm{F}}+2\left\Vert \bm{X}^{\star}\bm{Z}_{\bm{Y}}^{\top}\right\Vert _{\mathrm{F}}\right)\left\Vert \bm{\Theta}\right\Vert _{\mathrm{F}}.  \label{eq:Theta-upper-bound-2}
\end{align}
Plug in the definition of $\bm{\Theta}$ (cf.~\eqref{eq:Md-M-Theta}) and invoke the triangle inequality
again to see that
\begin{align*}
\left\Vert \bm{\Theta}\right\Vert _{\mathrm{F}} & \leq\left\Vert \bm{\Psi}_{\bm{X}}\bm{Y}^{\star\top}\right\Vert _{\mathrm{F}}+\left\Vert \bm{X}^{\star}\bm{\Psi}_{\bm{Y}}^{\top}\right\Vert _{\mathrm{F}}+\left\Vert \bm{\Gamma}_{\bm{X}}\bm{\Gamma}_{\bm{Y}}^{\top}\right\Vert _{\mathrm{F}}\\
 & \leq\left\Vert \bm{\Psi}_{\bm{X}}\right\Vert _{\mathrm{F}}\sqrt{\sigma_{\max}}+\sqrt{\sigma_{\max}}\left\Vert \bm{\Psi}_{\bm{Y}}\right\Vert _{\mathrm{F}}+\left\Vert \bm{\Gamma}_{\bm{X}}\right\Vert _{\mathrm{F}}\left\Vert \bm{\Gamma}_{\bm{Y}}\right\Vert _{\mathrm{F}}\\
 & \leq\sqrt{n\sigma_{\max}}(\left\Vert \bm{\Psi}_{\bm{X}}\right\Vert _{2,\infty}+\left\Vert \bm{\Psi}_{\bm{Y}}\right\Vert _{2,\infty})+\left\Vert \bm{\Gamma}_{\bm{X}}\right\Vert _{\mathrm{F}}\left\Vert \bm{\Gamma}_{\bm{Y}}\right\Vert _{\mathrm{F}}.
\end{align*}
	Combine Theorem~\ref{thm:low-rank-factor-master-bound} and the fact $\max\{\|\bm{\Gamma}_{\bm{X}}\|_{\mathrm{F}},\|\bm{\Gamma}_{\bm{Y}}\|_{\mathrm{F}}\}\lesssim(\sigma/\sigma_{\min})\sqrt{n/p}\|\bm{X}^{\star}\|_{\mathrm{F}}$ (see~(\ref{eq:F-d-fro-quality-H-d}))
to conclude that with probability at least $1-O(n^{-3})$
\begin{align*}
\left\Vert \bm{\Theta}\right\Vert _{\mathrm{F}} & \lesssim\sqrt{n\sigma_{\max}}\frac{\sigma}{\sqrt{p\sigma_{\min}}}\left(\,\frac{\sigma}{\sigma_{\min}}\sqrt{\frac{\kappa^{7}\mu rn\log n}{p}}+\sqrt{\frac{\kappa^{7}\mu^{3}r^{3}\log^{2}n}{np}}\,\right)+\left(\frac{\sigma}{\sigma_{\min}}\sqrt{\frac{n}{p}}\sqrt{r\sigma_{\max}}\right)^{2}\\
 & =o\left(\sigma\sqrt{\frac{nr}{p}}\right).
\end{align*}
Here the last relation depends on the assumption (\ref{eq:requirement-entry}).
Second, we have already established in this section that 
\begin{equation*}
	 \|\bm{Z}_{\bm{X}}\bm{Y}^{\star}\|_{\mathrm{F}}  + 
	 \|\bm{X}^{\star}\bm{Z}_{\bm{Y}}^{\top}\|_{\mathrm{F}}  = O(\sigma\sqrt{nr/p})
\end{equation*}
with probability exceeding $1-O(n^{-10})$. Substitute the above two
	facts into (\ref{eq:Theta-upper-bound-2}) to arrive at 
\[
\left|\left\Vert \bm{\Theta}\right\Vert _{\mathrm{F}}^{2}+2\mathsf{Tr}\left(\bm{Z}_{\bm{X}}\bm{Y}^{\star\top}\bm{\Theta}^{\top}\right)+2\mathsf{Tr}\left(\bm{X}^{\star}\bm{Z}_{\bm{Y}}^{\top}\bm{\Theta}^{\top}\right)\right|\lesssim\sigma\sqrt{\frac{nr}{p}}\left\Vert \bm{\Theta}\right\Vert _{\mathrm{F}}=o\left(\frac{\sigma^{2}nr}{p}\right).
\]
This concludes the proof.
\end{proof}

\begin{proof}[Proof of Claim \ref{claim:rem-1}]

According to Lemma \ref{lemma:approx-gaussian-leading-term}, one
can write
\[
\bm{Z}_{\bm{X}}=\underset{\triangleq\bm{Z}_{\bm{X},\bm{E}}}{\underbrace{\frac{1}{p}\mathcal{P}_{\Omega}\left(\bm{E}\right)\bm{Y}^{\star}\left(\bm{Y}^{\star\top}\bm{Y}^{\star}\right)^{-1}}}-\bm{\Delta}_{\bm{X}},\qquad\bm{Z}_{\bm{Y}}=\underset{\triangleq\bm{Z}_{\bm{Y},\bm{E}}}{\underbrace{\frac{1}{p}\left[\mathcal{P}_{\Omega}\left(\bm{E}\right)\right]^\top\bm{X}^{\star}\left(\bm{X}^{\star\top}\bm{X}^{\star}\right)^{-1}}}-\bm{\Delta}_{\bm{Y}},
\]
where $\max\big\{\|\bm{\Delta}_{\bm{X}}\|_{2,\infty},\|\bm{\Delta}_{\bm{Y}}\|_{2,\infty}\big\}\lesssim\frac{\sigma}{\sqrt{p\sigma_{\min}}}\sqrt{\frac{\kappa^{2}\mu r^{2}\log^{2}n}{np}}$
and hence
\begin{equation}
\max\left\{ \|\bm{\Delta}_{\bm{X}}\|_{\mathrm{F}},\|\bm{\Delta}_{\bm{Y}}\|_{\mathrm{F}}\right\} \lesssim\sqrt{n}\max\left\{ \|\bm{\Delta}_{\bm{X}}\|_{2,\infty},\|\bm{\Delta}_{\bm{Y}}\|_{2,\infty}\right\} \lesssim\frac{\sigma}{\sqrt{p\sigma_{\min}}}\sqrt{\frac{\kappa^{2}\mu r^{2}\log^{2}n}{p}}.\label{eq:Delta_X-Delta_Y-F-bound}
\end{equation}
Consequently, use the triangle inequality and Cauchy-Schwarz to verify
that
\begin{align}
 & \left|\mathsf{Tr}\big(\bm{Z}_{\bm{X}}\bm{Y}^{\star\top}\bm{Z}_{\bm{Y}}\bm{X}^{\star\top}\big)-\mathsf{Tr}\big(\bm{Z}_{\bm{X},\bm{E}}\bm{Y}^{\star\top}\bm{Z}_{\bm{Y},\bm{E}}\bm{X}^{\star\top}\big)\right|\nonumber \\
 & \quad\leq\left|\mathsf{Tr}\big(\bm{\Delta}_{\bm{X}}\bm{Y}^{\star\top}\bm{Z}_{\bm{Y}}\bm{X}^{\star\top}\big)\right|+\left|\mathsf{Tr}\big(\bm{Z}_{\bm{X}}\bm{Y}^{\star\top}\bm{\Delta}_{\bm{Y}}\bm{X}^{\star\top}\big)\right|+\left|\mathsf{Tr}\big(\bm{\Delta}_{\bm{X}}\bm{Y}^{\star\top}\bm{\Delta}_{\bm{Y}}\bm{X}^{\star\top}\big)\right|\nonumber \\
 & \quad\leq\|\bm{\Delta}_{\bm{X}}\|_{\mathrm{F}}\|\bm{Y}^{\star}\|\|\bm{Z}_{\bm{Y}}\bm{X}^{\star\top}\|_{\mathrm{F}}+\|\bm{\Delta}_{\bm{Y}}\|_{\mathrm{F}}\|\bm{X}^{\star}\|\|\bm{Z}_{\bm{X}}\bm{Y}^{\star\top}\|_{\mathrm{F}}+\|\bm{X}^{\star}\|\|\bm{Y}^{\star}\|\|\bm{\Delta}_{\bm{X}}\|_{\mathrm{F}}\|\bm{\Delta}_{\bm{Y}}\|_{\mathrm{F}}\nonumber \\
 & \quad\overset{(\text{i})}{\leq}\|\bm{\Delta}_{\bm{X}}\|_{\mathrm{F}}\|\bm{Y}^{\star}\|\|\bm{Z}_{\bm{Y}}\bm{\Sigma}^{\star1/2}\|_{\mathrm{F}}+\|\bm{\Delta}_{\bm{Y}}\|_{\mathrm{F}}\|\bm{X}^{\star}\|\|\bm{Z}_{\bm{X}}\bm{\Sigma}^{\star1/2}\|_{\mathrm{F}}+\|\bm{X}^{\star}\|\|\bm{Y}^{\star}\|\|\bm{\Delta}_{\bm{X}}\|_{\mathrm{F}}\|\bm{\Delta}_{\bm{Y}}\|_{\mathrm{F}}\\
 & \quad\overset{(\text{ii})}{\leq}\frac{\sigma}{\sqrt{p}}\sqrt{\frac{\kappa^{3}\mu r^{2}\log^{2}n}{p}}\|\bm{Z}_{\bm{Y}}\bm{\Sigma}^{\star1/2}\|_{\mathrm{F}}+\frac{\sigma}{\sqrt{p}}\sqrt{\frac{\kappa^{3}\mu r^{2}\log^{2}n}{p}}\|\bm{Z}_{\bm{Y}}\bm{\Sigma}^{\star1/2}\|_{\mathrm{F}}+\frac{\sigma^{2}}{p}\cdot\frac{\kappa^{3}\mu r^{2}\log^{2}n}{p},\label{eq:trace-decompose-2}
\end{align}
where (i) follows since $\bm{X}^{\star}=\bm{U}^{\star}\bm{\Sigma}^{\star1/2}$
and $\|\bm{U}^{\star}\|=1$, and (ii) makes use of (\ref{eq:Delta_X-Delta_Y-F-bound})
as well as the facts $\|\bm{Y}^{\star}\|,\|\bm{X}^{\star}\|=\sqrt{\sigma_{\max}}$
. In addition, invoke Lemma \ref{lemma:approx-gaussian-leading-term}
to see that $\bm{Z}_{\bm{X}}\bm{\Sigma}^{\star1/2}$ and $\bm{Z}_{\bm{Y}}\bm{\Sigma}^{\star1/2}$
are both Gaussian matrices with i.i.d. $\mathcal{N}(0,\sigma^{2}/p)$
entries, which together with standard concentration results implies
that
\[
\|\bm{Z}_{\bm{X}}\bm{\Sigma}^{\star1/2}\|_{\mathrm{F}}=(1+o(1))\sigma\sqrt{nr/p};\qquad\|\bm{Z}_{\bm{Y}}\bm{\Sigma}^{\star1/2}\|_{\mathrm{F}}=(1+o(1))\sigma\sqrt{nr/p}.
\]
Substituting it into (\ref{eq:trace-decompose-2}) gives
\begin{align*}
\left|\mathsf{Tr}\big(\bm{Z}_{\bm{X}}\bm{Y}^{\star\top}\bm{Z}_{\bm{Y}}\bm{X}^{\star\top}\big)-\mathsf{Tr}\big(\bm{Z}_{\bm{X},\bm{E}}\bm{Y}^{\star\top}\bm{Z}_{\bm{Y},\bm{E}}\bm{X}^{\star\top}\big)\right| & \lesssim\frac{\sigma^{2}}{p}\sqrt{\frac{\kappa^{3}\mu nr^{3}\log^{2}n}{p}}+\frac{\sigma^{2}}{p}\frac{\kappa^{3}\mu r^{2}\log^{2}n}{p}\asymp o\left(\frac{\sigma^{2}nr}{p}\right),
\end{align*}
with the proviso that $np\gtrsim\kappa^{3}\mu r\log^{3}n$. This means
that, with high probability,
\begin{equation}
\mathsf{Tr}\big(\bm{Z}_{\bm{X}}\bm{Y}^{\star\top}\bm{Z}_{\bm{Y}}\bm{X}^{\star\top}\big)=\mathsf{Tr}\big(\bm{Z}_{\bm{X},\bm{E}}\bm{Y}^{\star\top}\bm{Z}_{\bm{Y},\bm{E}}\bm{X}^{\star\top}\big)+o\left({\sigma^{2}nr} / p\right).\label{eq:cross-term-bridge}
\end{equation}

Everything then boils down to controlling $\mathsf{Tr}\big(\bm{Z}_{\bm{X},\bm{E}}\bm{Y}^{\star\top}\bm{Z}_{\bm{Y},\bm{E}}\bm{X}^{\star\top}\big)$.
Towards this end, we first note that
\[
	\bm{Z}_{\bm{X},\bm{E}}\bm{Y}^{\star\top} =  p^{-1} \mathcal{P}_{\Omega}\left(\bm{E}\right)\bm{Y}^{\star}\left(\bm{Y}^{\star\top}\bm{Y}^{\star}\right)^{-1}\bm{Y}^{\star\top}= p^{-1} \mathcal{P}_{\Omega}\left(\bm{E}\right)\bm{V}^{\star}\bm{V}^{\star\top}.
\]
Similarly, $\bm{Z}_{\bm{Y},\bm{E}}\bm{X}^{\star\top}=[\mathcal{P}_{\Omega}(\bm{E})]^{\top}\bm{U}^{\star}\bm{U}^{\star\top}/p$.
These identities allow us to derive 
\begin{align*}
\mathsf{Tr}\left(\bm{Z}_{\bm{X},\bm{E}}\bm{Y}^{\star\top}\bm{Z}_{\bm{Y},\bm{E}}\bm{X}^{\star\top}\right) & =\frac{1}{p^{2}}\mathsf{Tr}\left(\mathcal{P}_{\Omega}\left(\bm{E}\right)\bm{V}^{\star}\bm{V}^{\star\top}\left[\mathcal{P}_{\Omega}\left(\bm{E}\right)\right]^{\top}\bm{U}^{\star}\bm{U}^{\star\top}\right)\\
 & =\frac{1}{p^{2}}\mathsf{Tr}\left(\bm{U}^{\star\top}\mathcal{P}_{\Omega}\left(\bm{E}\right)\bm{V}^{\star}\bm{V}^{\star\top}\left[\mathcal{P}_{\Omega}\left(\bm{E}\right)\right]^{\top}\bm{U}^{\star}\right)\\
 & =\Big\Vert \bm{U}^{\star\top}\frac{1}{p}\mathcal{P}_{\Omega}\left(\bm{E}\right)\bm{V}^{\star}\Big\Vert _{\mathrm{F}}^{2}.
\end{align*}
Apply the same arguments in controlling~(\ref{eq:2nd-pert-upper-1})
to obtain that with probability at least $1-O(n^{-10})$, 
\[
\Big\Vert \bm{U}^{\star\top}\frac{1}{p}\mathcal{P}_{\Omega}\left(\bm{E}\right)\bm{V}^{\star}\Big\Vert _{\mathrm{F}}^{2}\lesssim\sigma^{2}\frac{\log n}{p}\left\Vert \bm{U}^{\star}\right\Vert _{\mathrm{F}}^2\left\Vert \bm{V}^{\star}\right\Vert _{\mathrm{F}}^2\asymp\frac{\sigma^{2}r^2\log n}{p}=o\left(\frac{\sigma^{2}nr}{p}\right),
\]
as long as $n\gtrsim r \log^2 n$. This combined with~(\ref{eq:cross-term-bridge}) yields the desired claim.
\end{proof}

\section{Proof of lower bounds}

\subsection{Proof of Lemma \ref{lemma:optimal-low-rank-variance} \label{sec:Proof-of-Lemma-optimal-low-rank}}

Fix any $\varepsilon>0$. It suffices to prove that the matrix $\mathsf{CRLB}(\bm{X}_{i,\cdot}^{\star}\mid\Omega)$
defined in (\ref{eq:CRLB-Xi}) satisfies
\begin{equation}
\left\Vert \frac{p}{\sigma^{2}}\mathsf{CRLB}(\bm{X}_{i,\cdot}^{\star}\mid\Omega)-\left(\bm{\Sigma}^{\star}\right)^{-1}\right\Vert \leq\frac{\varepsilon}{\sigma_{\max}}\label{eq:low-rank-optimality-sufficient}
\end{equation}
with probability at least $1-O(n^{-10})$, provided that $np\geq C_{0}\varepsilon^{-2}\kappa^{4}\mu r$.
Towards this end, we first compute

\[
\mathsf{CRLB}(\bm{X}_{i,\cdot}^{\star}\mid\Omega)=\sigma^{2}\,\Big(\sum_{k:(i,k)\in\Omega}(\bm{Y}_{k,\cdot}^{\star})^{\top}\bm{Y}_{k,\cdot}^{\star}\Big)^{-1}=\frac{\sigma^{2}}{p}\Bigg(\,\underset{:=\bm{A}}{\underbrace{\frac{1}{p}\sum_{k=1}^{n}\delta_{ik}(\bm{Y}_{k,\cdot}^{\star})^{\top}\bm{Y}_{k,\cdot}^{\star}}}\,\Bigg)^{-1},
\]
where we recall that $\delta_{ik}=\ind\{(i,k)\in\Omega\}$.
Next, define the following event 
\[
\mathcal{E}\triangleq\biggl\{\big\|\bm{A}-\bm{\Sigma}^{\star}\big\|\leq C\sqrt{\frac{\mu r\log n}{np}}\sigma_{\max}\biggr\},
\]
where $C>0$ is some large absolute constant. On the event $\mathcal{E}$,
in view of the fact $\sigma_{\min}\bm{I}_{r}\preceq\bm{\Sigma}^{\star}\preceq\sigma_{\max}\bm{I}_{r}$,
one has 
\[
0.5\sigma_{\min}\bm{I}_{r}\preceq\bm{A}\preceq2\sigma_{\max}\bm{I}_{r},
\]
with the proviso that $np\geq4C^{2}\kappa^{2}\mu r\log n$. This further
implies that
\begin{align*}
\left\Vert \frac{p}{\sigma^{2}}\mathsf{CRLB}(\bm{X}_{i,\cdot}^{\star}\mid\Omega)-\left(\bm{\Sigma}^{\star}\right)^{-1}\right\Vert  & =\big\|\bm{A}^{-1}-\left(\bm{\Sigma}^{\star}\right)^{-1}\big\|\leq\|\bm{A}-\bm{\Sigma}^{\star}\|\cdot\|\bm{A}^{-1}\|\cdot\big\|\left(\bm{\Sigma}^{\star}\right)^{-1}\big\|\\
 & \leq\frac{2C}{\sigma_{\min}}\sqrt{\frac{\kappa^{2}\mu r\log n}{np}}
\end{align*}
on the event $\mathcal{E}$. Clearly, the requirement (\ref{eq:low-rank-optimality-sufficient})
holds true if $np\geq C_{0}\varepsilon^{-2}\kappa^{4}\mu r\log n$
with $C_{0}=4C^{2}$.

To finish up, we are left with proving that $\mathcal{E}$ occurs
with probability at least $1-O(n^{-10})$. Invoke the matrix Bernstein
inequality to show that 
\begin{align*}
\big\|\bm{A}-\bm{\Sigma}^{\star}\big\| & =\frac{1}{p}\biggl\Vert\sum_{k=1}^{n}\left(\delta_{ik}-p\right)(\bm{Y}_{k,\cdot}^{\star})^{\top}\bm{Y}_{k,\cdot}^{\star}\biggr\Vert\lesssim\frac{1}{p}\left(\sqrt{V\log n}+B\log n\right)
\end{align*}
holds with probability at least $1-O(n^{-10})$, where we define 
\begin{align*}
B & \triangleq\max_{1\leq k\leq n}\left\Vert \left(\delta_{ik}-p\right)(\bm{Y}_{k,\cdot}^{\star})^{\top}\bm{Y}_{k,\cdot}^{\star}\right\Vert \leq\left\Vert \bm{Y}^{\star}\right\Vert _{2,\infty}^{2}\leq\mu r\sigma_{\max}/n,\\
V & \triangleq\biggl\Vert\sum_{k=1}^{n}\mathbb{E}\left[\left(\delta_{ik}-p\right)^{2}(\bm{Y}_{k,\cdot}^{\star})^{\top}\bm{Y}_{k,\cdot}^{\star}(\bm{Y}_{k,\cdot}^{\star})^{\top}\bm{Y}_{k,\cdot}^{\star}\right]\biggr\Vert\leq p\biggl\Vert\sum_{k=1}^{n}(\bm{Y}_{k,\cdot}^{\star})^{\top}\bm{Y}_{k,\cdot}^{\star}(\bm{Y}_{k,\cdot}^{\star})^{\top}\bm{Y}_{k,\cdot}^{\star}\biggr\Vert\\
 & \leq p\left\Vert \bm{Y}^{\star}\right\Vert _{2,\infty}^{2}\left\Vert \bm{Y}^{\star\top}\bm{Y}^{\star}\right\Vert \leq\mu rp\sigma_{\max}^{2}/n.
\end{align*}
Here we have used the incoherence condition~(\ref{eq:incoherence-X}).
Consequently, one reaches the conclusion that with probability exceeding
$1-O(n^{-10})$, 
\[
\big\|\bm{A}-\bm{\Sigma}^{\star}\big\|\lesssim\frac{1}{p}\left(\sqrt{\frac{\mu rp\sigma_{\max}^{2}}{n}\log n}+\frac{\mu r\sigma_{\max}}{n}\log n\right)\asymp\sqrt{\frac{\mu r\log n}{np}}\sigma_{\max}
\]
as long as $np\gg\mu r\log n$, thus concluding the proof.

\subsection{Proof of Lemma \ref{lemma:optimal-entry-variance} \label{sec:Proof-of-Lemma-optimal-entry}}

The proof strategy is similar to the one used in proving Lemma~\ref{lemma:optimal-low-rank-variance}
(cf.~Appendix~\ref{sec:Proof-of-Lemma-optimal-low-rank}). Fix any
$\varepsilon>0$. It is sufficient to establish the following inequality
\begin{equation}
\frac{p}{\sigma^{2}}\left|\mathsf{CRLB}(M_{ij}^{\star}\mid\Omega)-v_{ij}^{\star}\right|\leq\varepsilon\frac{p}{\sigma^{2}}v_{ij}^{\star},\label{eq:entry-optimality-condition}
\end{equation}
where the scalar $\mathsf{CRLB}(M_{ij}^{\star}\mid\Omega)$ is defined in (\ref{eq:CRLB-Mij}) and
$v_{ij}^{\star}$ is defined in Theorem \ref{thm:entries-master-decomposition}.
Expand the left-hand side to reach 
\begin{align*}
\frac{p}{\sigma^{2}}\left|\mathsf{CRLB}(M_{ij}^{\star}\mid\Omega)-v_{ij}^{\star}\right| & \leq\biggl|\bm{Y}_{j,\cdot}^{\star}\Big(\,\underset{:=\bm{A}_{Y}}{\underbrace{\frac{1}{p}\sum_{k:k\neq j,(i,k)\in\Omega}(\bm{Y}_{k,\cdot}^{\star})^{\top}\bm{Y}_{k,\cdot}^{\star}}}\,\Big)^{-1}(\bm{Y}_{j,\cdot}^{\star})^{\top}-\bm{Y}_{j,\cdot}^{\star}\left(\bm{\Sigma}^{\star}\right)^{-1}(\bm{Y}_{j,\cdot}^{\star})^{\top}\biggr|\\
 & \quad+\biggl|\bm{X}_{i,\cdot}^{\star}\Big(\,\underset{:=\bm{A}_{X}}{\underbrace{\frac{1}{p}\sum_{k:k\neq i,(k,j)\in\Omega}(\bm{X}_{k,\cdot}^{\star})^{\top}\bm{X}_{k,\cdot}^{\star}}}\,\Big)^{-1}(\bm{X}_{i,\cdot}^{\star})^{\top}-\bm{X}_{j,\cdot}^{\star}\left(\bm{\Sigma}^{\star}\right)^{-1}(\bm{X}_{j,\cdot}^{\star})^{\top}\biggr|\\
 & \leq\left\Vert \bm{V}_{j,\cdot}^{\star}\right\Vert _{2}^{2}\sigma_{\max}\big\|\bm{A}_{Y}^{-1}-\left(\bm{\Sigma}^{\star}\right)^{-1}\big\|+\left\Vert \bm{U}_{i,\cdot}^{\star}\right\Vert _{2}^{2}\sigma_{\max}\big\|\bm{A}_{X}^{-1}-\left(\bm{\Sigma}^{\star}\right)^{-1}\big\|,
\end{align*}
where the last line follows from the observations that $\|\bm{Y}_{j,\cdot}^{\star}\|_{2}\leq\sqrt{\sigma_{\max}}\|\bm{V}_{j,\cdot}^{\star}\|_{2}$
and $\|\bm{X}_{i,\cdot}^{\star}\|_{2}\leq\sqrt{\sigma_{\max}}\|\bm{U}_{i,\cdot}^{\star}\|_{2}$.

Define the following event 
\[
\mathcal{E}_{2}\triangleq\biggl\{\max\big\{\big\|\bm{A}_{Y}-\bm{\Sigma}^{\star}\big\|,\big\|\bm{A}_{X}-\bm{\Sigma}^{\star}\big\|\big\}\leq C\sqrt{\frac{\mu r\log n}{np}}\sigma_{\max}\biggr\},
\]
where $C>0$ is some large universal constant. 
Two observations are sufficient to derive the desired the result (\ref{eq:entry-optimality-condition}).
First, the event $\mathcal{E}_{2}$ happens with probability at least
$1-O(n^{-10})$ --- an easy consequence of the proof of Lemma~\ref{lemma:optimal-low-rank-variance}
(cf.~Appendix~\ref{sec:Proof-of-Lemma-optimal-low-rank}). Second,
on the event $\mathcal{E}_{2}$, repeating the same proof of Lemma~\ref{lemma:optimal-low-rank-variance}
(cf.~Appendix~\ref{sec:Proof-of-Lemma-optimal-low-rank}), one can
deduce that
\begin{equation}
\frac{p}{\sigma^{2}}\left|\mathsf{CRLB}(M_{ij}^{\star}\mid\Omega)-v_{ij}^{\star}\right|\leq\left(\left\Vert \bm{U}_{i,\cdot}^{\star}\right\Vert _{2}^{2}+\left\Vert \bm{V}_{j,\cdot}^{\star}\right\Vert _{2}^{2}\right)\sigma_{\max}\cdot\frac{2C}{\sigma_{\min}}\sqrt{\frac{\kappa^{2}\mu r\log n}{np}}.\label{eq:entry-optimality-upper}
\end{equation}
Comparing (\ref{eq:entry-optimality-condition}) and (\ref{eq:entry-optimality-upper}),
one arrives at the desired result as long as $np\geq4C^{2}\varepsilon^{-2}\kappa^{4}\mu r\log n$.

\section{Proofs in Section \ref{subsec:Preliminaries} \label{sec:Proofs-in-Section-prelim}}
\subsection{Proof of the inequalities (\ref{subeq:F-d-property})}

We start with (\ref{eq:F-d-op-quality-H}). Invoke the triangle inequality
to get
\begin{equation}
\left\Vert \bm{F}^{\mathsf{d}}\bm{H}-\bm{F}^{\star}\right\Vert \leq\left\Vert \bm{F}^{\mathsf{d}}\bm{H}-\bm{F}\bm{H}\right\Vert +\left\Vert \bm{F}\bm{H}-\bm{F}^{\star}\right\Vert =\left\Vert \bm{F}^{\mathsf{d}}-\bm{F}\right\Vert +O\left(\frac{\sigma}{\sigma_{\min}}\sqrt{\frac{n}{p}}\left\Vert \bm{X}^{\star}\right\Vert \right),\label{eq:Fd-Fstar-spectral-bound}
\end{equation}
where the last relation depends on the unitary invariance of the operator
norm and (\ref{eq:F-op-quality}). It then boils down to controlling
$\|\bm{F}^{\mathsf{d}}-\bm{F}\|$. Notice that 
\begin{align*}
\left\Vert \bm{F}^{\mathsf{d}}-\bm{F}\right\Vert  & \leq\left\Vert \bm{F}\Big(\bm{I}_{r}+\frac{\lambda}{p}\left(\bm{X}^{\top}\bm{X}\right)^{-1}\Big)^{1/2}-\bm{F}\right\Vert +\left\Vert \bm{Y}\left[\Big(\bm{I}_{r}+\frac{\lambda}{p}\left(\bm{Y}^{\top}\bm{Y}\right)^{-1}\Big)^{1/2}-\Big(\bm{I}_{r}+\frac{\lambda}{p}\left(\bm{X}^{\top}\bm{X}\right)^{-1}\Big)^{1/2}\right]\right\Vert \\
 & \leq\left\Vert \bm{F}\right\Vert \left\Vert \Big(\bm{I}_{r}+\frac{\lambda}{p}\left(\bm{X}^{\top}\bm{X}\right)^{-1}\Big)^{1/2}-\bm{I}_{r}\right\Vert +\left\Vert \bm{Y}\right\Vert \left\Vert \Big(\bm{I}_{r}+\frac{\lambda}{p}\left(\bm{Y}^{\top}\bm{Y}\right)^{-1}\Big)^{1/2}-\Big(\bm{I}_{r}+\frac{\lambda}{p}\left(\bm{X}^{\top}\bm{X}\right)^{-1}\Big)^{1/2}\right\Vert \\
 & \leq\left\Vert \bm{F}\right\Vert \left\Vert \Big(\bm{I}_{r}+\frac{\lambda}{p}\left(\bm{X}^{\top}\bm{X}\right)^{-1}\Big)^{1/2}-\bm{I}_{r}\right\Vert +O\left(\frac{\sigma}{\sigma_{\min}}\sqrt{\frac{n}{p}}\left\Vert \bm{X}^{\star}\right\Vert \right),
\end{align*}
where the last inequality uses $\|\bm{Y}\|\leq \|\bm{F}\| \leq 2\|\bm{X}^\star\|$ (cf.~(\ref{eq:F-norm-upper-bound})),
the fact that $\lambda\lesssim\sigma\sqrt{np}$ (see~(\ref{eq:lambda-condition})), the bound (\ref{eq:debias-correction-close}) and the condition $n^{5}\gg\kappa$. Apply the perturbation bound
for matrix square roots (see Lemma~\ref{lemma:matrix-sqrt}) to obtain
that 
\begin{align*}
\left\Vert \Big(\bm{I}_{r}+\frac{\lambda}{p}\left(\bm{X}^{\top}\bm{X}\right)^{-1}\Big)^{1/2}-\bm{I}_{r}\right\Vert  & \leq\frac{\lambda/p}{\lambda_{\min}\left(\bm{I}_{r}\right)+\lambda_{\min}\left[\Big(\bm{I}_{r}+\frac{\lambda}{p}\left(\bm{X}^{\top}\bm{X}\right)^{-1}\Big)^{1/2}\right]}\left\Vert \left(\bm{X}^{\top}\bm{X}\right)^{-1}\right\Vert \\
 & \overset{(\text{i})}{\lesssim}\frac{\lambda}{p\sigma_{\min}}\overset{(\text{ii})}{\lesssim}\frac{\sigma}{\sigma_{\min}}\sqrt{\frac{n}{p}}.
\end{align*}
Here, (i) uses the facts that $\|(\bm{X}^{\top}\bm{X})^{-1}\|\lesssim1/\sigma_{\min}$
and that $\lambda_{\min}[(\bm{I}_{r}+\lambda/p(\bm{X}^{\top}\bm{X})^{-1})^{1/2}]\geq1$,
and (ii) follows from the condition that $\lambda\lesssim\sigma\sqrt{np}$
(see~(\ref{eq:lambda-condition})). Combine the above two bounds
with $\|\bm{F}\|\leq2\|\bm{X}^{\star}\|$ (cf.~(\ref{eq:F-norm-upper-bound}))
to reach 
\begin{equation}
\left\Vert \bm{F}^{\mathsf{d}}-\bm{F}\right\Vert \lesssim\frac{\sigma}{\sigma_{\min}}\sqrt{\frac{n}{p}}\left\Vert \bm{X}^{\star}\right\Vert .\label{eq:F-d-F-op-dist}
\end{equation}
Substitution into (\ref{eq:Fd-Fstar-spectral-bound}) gives
\begin{equation}
\left\Vert \bm{F}^{\mathsf{d}}\bm{H}-\bm{F}^{\star}\right\Vert \lesssim\frac{\sigma}{\sigma_{\min}}\sqrt{\frac{n}{p}}\left\Vert \bm{X}^{\star}\right\Vert .\label{eq:FdH-Fstar-final-spectral}
\end{equation}
Analogous arguments yield 
\begin{equation*}
\left\Vert \bm{F}^{\mathsf{d}}\bm{H}^{\mathsf{d}}-\bm{F}^{\star}\right\Vert_{\mathrm{F}} \leq \left\Vert \bm{F}^{\mathsf{d}}\bm{H}-\bm{F}^{\star}\right\Vert_{\mathrm{F}} \lesssim\frac{\sigma}{\sigma_{\min}}\sqrt{\frac{n}{p}}\left\Vert \bm{X}^{\star}\right\Vert_{\mathrm{F}},
\end{equation*}
which is the claim in (\ref{eq:F-d-fro-quality-H-d}).

Moving on to~(\ref{eq:F-d-op-quality-H-d}), we apply the triangle
inequality and (\ref{eq:FdH-Fstar-final-spectral}) to see that 
\[
\left\Vert \bm{F}^{\mathsf{d}}\bm{H}^{\mathsf{d}}-\bm{F}^{\star}\right\Vert \leq\left\Vert \bm{F}^{\mathsf{d}}\bm{H}^{\mathsf{d}}-\bm{F}^{\mathsf{d}}\bm{H}\right\Vert +\left\Vert \bm{F}^{\mathsf{d}}\bm{H}-\bm{F}^{\star}\right\Vert \leq\left\Vert \bm{F}^{\mathsf{d}}\right\Vert \left\Vert \bm{H}^{\mathsf{d}}-\bm{H}\right\Vert +O\left(\frac{\sigma}{\sigma_{\min}}\sqrt{\frac{n}{p}}\left\Vert \bm{X}^{\star}\right\Vert \right).
\]
In order to control $\|\bm{H}^{\mathsf{d}}-\bm{H}\|$, we leverage
\cite[Lemma 36]{ma2017implicit} to get 
\begin{align}
\left\Vert \bm{H}^{\mathsf{d}}-\bm{H}\right\Vert  & \leq\frac{1}{\sigma_{\min}\left(\bm{F}^{\top}\bm{F}^{\star}\right)}\left\Vert \bm{F}^{\mathsf{d}\top}\bm{F}^{\star}-\bm{F}^{\top}\bm{F}^{\star}\right\Vert \lesssim\frac{1}{\sigma_{\min}}\left\Vert \bm{F}^{\mathsf{d}}-\bm{F}\right\Vert \left\Vert \bm{F}^{\star}\right\Vert \lesssim\kappa\frac{\sigma}{\sigma_{\min}}\sqrt{\frac{n}{p}},\label{eq:H-d-H-op-dist}
\end{align}
where the last relation uses (\ref{eq:F-d-F-op-dist}) and $\left\Vert \bm{F}^{\star}\right\Vert \asymp\left\Vert \bm{X}^{\star}\right\Vert \asymp\sqrt{\sigma_{\max}}$.
Taking these bounds collectively yields 
\[
\left\Vert \bm{F}^{\mathsf{d}}\bm{H}^{\mathsf{d}}-\bm{F}^{\star}\right\Vert \lesssim\kappa\frac{\sigma}{\sigma_{\min}}\sqrt{\frac{n}{p}}\left\Vert \bm{X}^{\star}\right\Vert .
\]

Now we turn attention to (\ref{eq:F-d-2-infty}). Observe that 
\begin{align}
\left\Vert \bm{F}^{\mathsf{d}}\bm{H}^{\mathsf{d}}-\bm{F}^{\star}\right\Vert _{2,\infty} & \leq\left\Vert \bm{F}^{\mathsf{d}}\bm{H}^{\mathsf{d}}-\bm{F}^{\mathsf{d}}\bm{H}\right\Vert _{2,\infty}+\left\Vert \bm{F}^{\mathsf{d}}\bm{H}-\bm{F}\bm{H}\right\Vert _{2,\infty}+\left\Vert \bm{F}\bm{H}-\bm{F}^{\star}\right\Vert _{2,\infty}\nonumber \\
 & \leq\left\Vert \bm{F}^{\mathsf{d}}\right\Vert _{2,\infty}\left\Vert \bm{H}^{\mathsf{d}}-\bm{H}\right\Vert +\left\Vert \bm{F}^{\mathsf{d}}-\bm{F}\right\Vert _{2,\infty}+O\Bigl(\kappa\frac{\sigma}{\sigma_{\min}}\sqrt{\frac{n\log n}{p}}\left\Vert \bm{F}^{\star}\right\Vert _{2,\infty}\Bigr),\label{eq:upper-bound-F-d-H-d-infty}
\end{align}
where the last bound arises from (\ref{eq:F-2-inf-quality}). Going
through the same calculation as in bounding $\|\bm{F}^{\mathsf{d}}-\bm{F}\|$,
we arrive at
\[
\left\Vert \bm{F}^{\mathsf{d}}-\bm{F}\right\Vert _{2,\infty}\lesssim\frac{\sigma}{\sigma_{\min}}\sqrt{\frac{n}{p}}\left\Vert \bm{F}^{\star}\right\Vert _{2,\infty}\qquad\text{and}\qquad\left\Vert \bm{F}^{\mathsf{d}}\right\Vert _{2,\infty}\leq2\left\Vert \bm{F}^{\star}\right\Vert _{2,\infty}
\]
as long as $\sigma\sqrt{n/p}\ll\sigma_{\min}$. We can thus continue
the upper bound in (\ref{eq:upper-bound-F-d-H-d-infty}) to derive
\begin{align*}
\left\Vert \bm{F}^{\mathsf{d}}\bm{H}^{\mathsf{d}}-\bm{F}^{\star}\right\Vert _{2,\infty} & \lesssim\left\Vert \bm{F}^{\star}\right\Vert _{2,\infty}\left\Vert \bm{H}^{\mathsf{d}}-\bm{H}\right\Vert +\kappa\frac{\sigma}{\sigma_{\min}}\sqrt{\frac{n\log n}{p}}\left\Vert \bm{F}^{\star}\right\Vert _{2,\infty}\\
 & \lesssim\kappa\frac{\sigma}{\sigma_{\min}}\sqrt{\frac{n}{p}}\left\Vert \bm{F}^{\star}\right\Vert _{2,\infty}+\kappa\frac{\sigma}{\sigma_{\min}}\sqrt{\frac{n\log n}{p}}\left\Vert \bm{F}^{\star}\right\Vert _{2,\infty}\\
 & \asymp\kappa\frac{\sigma}{\sigma_{\min}}\sqrt{\frac{n\log n}{p}}\left\Vert \bm{F}^{\star}\right\Vert _{2,\infty}.
\end{align*}
Here, the second line results from (\ref{eq:H-d-H-op-dist}).

Finally, we deal with (\ref{eq:X-d-Y-d-balance}). From the definition
of the de-shrunken estimator (\ref{eq:defn-Xd-Yd}), we have 
\begin{align*}
\bm{X}^{\mathsf{d}\top}\bm{X}^{\mathsf{d}}-\bm{Y}^{\mathsf{d}\top}\bm{Y}^{\mathsf{d}} & =\Big(\bm{I}_{r}+\frac{\lambda}{p}\left(\bm{X}^{\top}\bm{X}\right)^{-1}\Big)^{1/2}\bm{X}^{\top}\bm{X}\Big(\bm{I}_{r}+\frac{\lambda}{p}\left(\bm{X}^{\top}\bm{X}\right)^{-1}\Big)^{1/2}\\
 & \quad-\Big(\bm{I}_{r}+\frac{\lambda}{p}\left(\bm{Y}^{\top}\bm{Y}\right)^{-1}\Big)^{1/2}\bm{Y}^{\top}\bm{Y}\Big(\bm{I}_{r}+\frac{\lambda}{p}\left(\bm{Y}^{\top}\bm{Y}\right)^{-1}\Big)^{1/2}.
\end{align*}
This combined with the triangle inequality reveals that 
\begin{align*}
 & \left\Vert \bm{X}^{\mathsf{d}\top}\bm{X}^{\mathsf{d}}-\bm{Y}^{\mathsf{d}\top}\bm{Y}^{\mathsf{d}}\right\Vert \\
 & \quad\leq\left\Vert \Big(\bm{I}_{r}+\frac{\lambda}{p}\left(\bm{X}^{\top}\bm{X}\right)^{-1}\Big)^{1/2}\right\Vert \left\Vert \bm{X}^{\top}\bm{X}-\bm{Y}^{\top}\bm{Y}\right\Vert \left\Vert \Big(\bm{I}_{r}+\frac{\lambda}{p}\left(\bm{X}^{\top}\bm{X}\right)^{-1}\Big)^{1/2}\right\Vert \\
 & \quad\quad+\left\Vert \Big(\bm{I}_{r}+\frac{\lambda}{p}\left(\bm{X}^{\top}\bm{X}\right)^{-1}\Big)^{1/2}-\Big(\bm{I}_{r}+\frac{\lambda}{p}\left(\bm{Y}^{\top}\bm{Y}\right)^{-1}\Big)^{1/2}\right\Vert \left\Vert \bm{Y}^{\top}\bm{Y}\right\Vert \left\Vert \Big(\bm{I}_{r}+\frac{\lambda}{p}\left(\bm{X}^{\top}\bm{X}\right)^{-1}\Big)^{1/2}\right\Vert \\
 & \quad\quad+\left\Vert \Big(\bm{I}_{r}+\frac{\lambda}{p}\left(\bm{Y}^{\top}\bm{Y}\right)^{-1}\Big)^{1/2}\right\Vert \left\Vert \bm{Y}^{\top}\bm{Y}\right\Vert \left\Vert \Big(\bm{I}_{r}+\frac{\lambda}{p}\left(\bm{X}^{\top}\bm{X}\right)^{-1}\Big)^{1/2}-\Big(\bm{I}_{r}+\frac{\lambda}{p}\left(\bm{Y}^{\top}\bm{Y}\right)^{-1}\Big)^{1/2}\right\Vert .
\end{align*}
Making use of (\ref{eq:X-Y-balance}) and (\ref{eq:debias-correction-close})
allows us to establish the claim.

\subsection{Proof of the inequalities (\ref{subeq:F-d-j-property})}

The proofs of~(\ref{eq:F-d-j-op}) and~(\ref{eq:F-d-j-2-infty})
are the same as those of (\ref{eq:F-d-op-quality-H-d}) and (\ref{eq:F-d-2-infty}),
and are hence omitted for conciseness. We are left with~(\ref{eq:F-d-j-F-d-dist}).
Denoting 
\[
\bm{F}_{0}\triangleq\bm{F}^{\star},\qquad\bm{F}_{1}\triangleq\bm{F}^{\mathsf{d}}\bm{H}\qquad\text{and}\qquad\bm{F}_{2}\triangleq\bm{F}^{\mathsf{d},(j)}\bm{R}^{(j)},
\]
one has 
\begin{align*}
\left\Vert \bm{F}_{1}-\bm{F}_{0}\right\Vert \left\Vert \bm{F}_{0}\right\Vert  & =\left\Vert \bm{F}^{\mathsf{d}}\bm{H}-\bm{F}^{\star}\right\Vert \left\Vert \bm{F}^{\star}\right\Vert \lesssim\frac{\sigma}{\sigma_{\min}}\sqrt{\frac{n}{p}}\sigma_{\max}\leq\sigma_{\min}=\frac{\sigma_{r}^{2}\left(\bm{F}_{0}\right)}{2},
\end{align*}
as long as $\sigma\sqrt{n/p}\ll\sigma_{\min}/\kappa$. Here the first
inequality follows from (\ref{eq:F-d-op-quality-H}). In addition,
we have 
\begin{align}
 & \left\Vert \bm{F}_{1}-\bm{F}_{2}\right\Vert \left\Vert \bm{F}_{0}\right\Vert =\big\|\bm{F}^{\mathsf{d}}\bm{H}-\bm{F}^{\mathsf{d},(j)}\bm{R}^{(j)}\big\|\left\Vert \bm{F}^{\star}\right\Vert \nonumber \\
 & \quad\leq\left\Vert \bm{F}\Big(\bm{I}_{r}+\frac{\lambda}{p}\left(\bm{Y}^{\top}\bm{Y}\right)^{-1}\Big)^{1/2}\bm{H}-\bm{F}^{(j)}\Bigl(\bm{I}_{r}+\frac{\lambda}{p}\left(\bm{Y}^{(j)\top}\bm{Y}^{(j)}\right)^{-1}\Bigr)^{1/2}\bm{R}^{(j)}\right\Vert \left\Vert \bm{F}^{\star}\right\Vert +\theta,\nonumber \\
 & \quad=\left\Vert \bm{F}\bm{H}\Bigl(\bm{I}_{r}+\frac{\lambda}{p}\left(\bm{H}^{\top}\bm{Y}^{\top}\bm{Y}\bm{H}\right)^{-1}\Bigr)^{1/2}-\bm{F}^{(j)}\bm{R}^{(j)}\Bigl(\bm{I}_{r}+\frac{\lambda}{p}\left(\bm{R}^{(j)\top}\bm{Y}^{(j)\top}\bm{Y}^{(j)}\bm{R}^{(j)}\right)^{-1}\Bigr)^{1/2}\right\Vert \left\Vert \bm{F}^{\star}\right\Vert +\theta,\label{eq:pre-1}
\end{align}
where $\theta$ is defined to be 
\begin{align*}
\theta & \triangleq\left\Vert \bm{X}\left[\Big(\bm{I}_{r}+\frac{\lambda}{p}\left(\bm{Y}^{\top}\bm{Y}\right)^{-1}\Big)^{1/2}-\Big(\bm{I}_{r}+\frac{\lambda}{p}\left(\bm{X}^{\top}\bm{X}\right)^{-1}\Big)^{1/2}\right]\right\Vert \left\Vert \bm{F}^{\star}\right\Vert \\
 & \quad+\left\Vert \bm{X}^{(j)}\left[\Big(\bm{I}_{r}+\frac{\lambda}{p}\left(\bm{Y}^{(j)\top}\bm{Y}^{(j)}\right)^{-1}\Big)^{1/2}-\Big(\bm{I}_{r}+\frac{\lambda}{p}\left(\bm{X}^{(j)\top}\bm{X}^{(j)}\right)^{-1}\Big)^{1/2}\right]\right\Vert \left\Vert \bm{F}^{\star}\right\Vert .
\end{align*}
Regarding $\theta$, one can apply the bound (\ref{eq:debias-correction-close})
for $(\bm{X},\bm{Y})$ and a similar bound for $(\bm{X}^{(j)},\bm{Y}^{(j)})$
to obtain 
\[
\theta\lesssim\sigma_{\max}\cdot\frac{\kappa}{n^{5}}\frac{\sigma}{\sigma_{\min}}\sqrt{\frac{n}{p}}.
\]

Returning to (\ref{eq:pre-1}), one has by the triangle inequality
that 
\begin{align*}
 & \left\Vert \bm{F}\bm{H}\Bigl(\bm{I}_{r}+\frac{\lambda}{p}\left(\bm{H}^{\top}\bm{Y}^{\top}\bm{Y}\bm{H}\right)^{-1}\Bigr)^{1/2}-\bm{F}^{(j)}\bm{R}^{(j)}\Bigl(\bm{I}_{r}+\frac{\lambda}{p}\left(\bm{R}^{(j)\top}\bm{Y}^{(j)\top}\bm{Y}^{(j)}\bm{R}^{(j)}\right)^{-1}\Bigr)^{1/2}\right\Vert \\
 & \quad\leq\left\Vert \left(\bm{F}\bm{H}-\bm{F}^{(j)}\bm{R}^{(j)}\right)\Bigl(\bm{I}_{r}+\frac{\lambda}{p}\left(\bm{H}^{\top}\bm{Y}^{\top}\bm{Y}\bm{H}\right)^{-1}\Bigr)^{1/2}\right\Vert \\
 & \quad\quad+\left\Vert \bm{F}^{(j)}\bm{R}^{(j)}\left[\Bigl(\bm{I}_{r}+\frac{\lambda}{p}\left(\bm{H}^{\top}\bm{Y}^{\top}\bm{Y}\bm{H}\right)^{-1}\Bigr)^{1/2}-\Bigl(\bm{I}_{r}+\frac{\lambda}{p}\left(\bm{R}^{(j)\top}\bm{Y}^{(j)\top}\bm{Y}^{(j)}\bm{R}^{(j)}\right)^{-1}\Bigr)^{1/2}\right]\right\Vert \\
 & \quad\leq\left\Vert \bm{F}\bm{H}-\bm{F}^{(j)}\bm{R}^{(j)}\right\Vert _{\mathrm{F}}\left\Vert \Bigl(\bm{I}_{r}+\frac{\lambda}{p}\left(\bm{H}^{\top}\bm{Y}^{\top}\bm{Y}\bm{H}\right)^{-1}\Bigr)^{1/2}\right\Vert \\
 & \quad\quad+\left\Vert \bm{F}^{(j)}\bm{R}^{(j)}\right\Vert \left\Vert \Bigl(\bm{I}_{r}+\frac{\lambda}{p}\left(\bm{H}^{\top}\bm{Y}^{\top}\bm{Y}\bm{H}\right)^{-1}\Bigr)^{1/2}-\Bigl(\bm{I}_{r}+\frac{\lambda}{p}\left(\bm{R}^{(j)\top}\bm{Y}^{(j)\top}\bm{Y}^{(j)}\bm{R}^{(j)}\right)^{-1}\Bigr)^{1/2}\right\Vert .
\end{align*}
Recognizing that 
\[
\lambda_{\min}\left[\Bigl(\bm{I}_{r}+\frac{\lambda}{p}\left(\bm{H}^{\top}\bm{Y}^{\top}\bm{Y}\bm{H}\right)^{-1}\Bigr)^{1/2}\right]\geq 1\qquad\text{and}\qquad\lambda_{\min}\left[\Bigl(\bm{I}_{r}+\frac{\lambda}{p}\left(\bm{R}^{(j)\top}\bm{Y}^{(j)\top}\bm{Y}^{(j)}\bm{R}^{(j)}\right)^{-1}\Bigr)^{1/2}\right]\geq 1,
\]
we can apply the perturbation bound for matrix square roots (see~Lemma~\ref{lemma:matrix-sqrt})
to obtain 
\begin{align*}
 & \left\Vert \Bigl(\bm{I}_{r}+\frac{\lambda}{p}\left(\bm{H}^{\top}\bm{Y}^{\top}\bm{Y}\bm{H}\right)^{-1}\Bigr)^{1/2}-\Bigl(\bm{I}_{r}+\frac{\lambda}{p}\left(\bm{R}^{(j)\top}\bm{Y}^{(j)\top}\bm{Y}^{(j)}\bm{R}^{(j)}\right)^{-1}\Bigr)^{1/2}\right\Vert \\
 & \quad\lesssim\frac{\lambda}{p}\left\Vert \left(\bm{H}^{\top}\bm{Y}^{\top}\bm{Y}\bm{H}\right)^{-1}-\left(\bm{R}^{(j)\top}\bm{Y}^{(j)\top}\bm{Y}^{(j)}\bm{R}^{(j)}\right)^{-1}\right\Vert \\
 & \quad\lesssim\frac{\lambda}{p}\left\Vert \left(\bm{H}^{\top}\bm{Y}^{\top}\bm{Y}\bm{H}\right)^{-1}\right\Vert \left\Vert \bm{H}^{\top}\bm{Y}^{\top}\bm{Y}\bm{H}-\bm{R}^{(j)\top}\bm{Y}^{(j)\top}\bm{Y}^{(j)}\bm{R}^{(j)}\right\Vert \Big\|\left(\bm{R}^{(j)\top}\bm{Y}^{(j)\top}\bm{Y}^{(j)}\bm{R}^{(j)}\right)^{-1}\Big\|\\
 & \quad\lesssim\frac{\lambda}{p}\frac{1}{\sigma_{\min}^{2}}\left\Vert \bm{H}^{\top}\bm{Y}^{\top}\bm{Y}\bm{H}-\bm{R}^{(j)\top}\bm{Y}^{(j)\top}\bm{Y}^{(j)}\bm{R}^{(j)}\right\Vert \lesssim\frac{\lambda}{p}\frac{1}{\sigma_{\min}^{2}}\sqrt{\sigma_{\max}}\big\Vert \bm{F}\bm{H}-\bm{F}^{(j)}\bm{R}^{(j)}\big\Vert _{\mathrm{F}}\\
 & \quad\lesssim\frac{\sigma}{\sigma_{\min}}\sqrt{\frac{n}{p}}\frac{\sqrt{\sigma_{\max}}}{\sigma_{\min}}\big\Vert \bm{F}\bm{H}-\bm{F}^{(j)}\bm{R}^{(j)}\big\Vert _{\mathrm{F}}.
\end{align*}
Collect the pieces to arrive at 
\begin{align*}
\left\Vert \bm{F}_{1}-\bm{F}_{2}\right\Vert \left\Vert \bm{F}_{0}\right\Vert  
 & \lesssim\sqrt{\sigma_{\max}}\left(\big\Vert \bm{F}\bm{H}-\bm{F}^{(j)}\bm{R}^{(j)}\big\Vert _{\mathrm{F}}+\kappa\frac{\sigma}{\sigma_{\min}}\sqrt{\frac{n}{p}}\big\Vert \bm{F}\bm{H}-\bm{F}^{(j)}\bm{R}^{(j)}\big\Vert _{\mathrm{F}}\right)+\sigma_{\max}\cdot\frac{\kappa}{n^{5}}\frac{\sigma}{\sigma_{\min}}\sqrt{\frac{n}{p}}\\
 & \lesssim\sqrt{\sigma_{\max}}\big\Vert \bm{F}\bm{H}-\bm{F}^{(j)}\bm{R}^{(j)}\big\Vert _{\mathrm{F}}+\sigma_{\max}\cdot\frac{\kappa}{n^{5}}\frac{\sigma}{\sigma_{\min}}\sqrt{\frac{n}{p}}\\
 & \lesssim\sqrt{\sigma_{\max}}\frac{\sigma}{\sigma_{\min}}\sqrt{\frac{n\log n}{p}}\big\Vert \bm{F}^{\star}\big\Vert _{2,\infty}\ll\frac{\sigma_{r}^{2}\left(\bm{F}_{0}\right)}{4},
\end{align*}
where the penultimate relation uses (\ref{eq:F-j-F-best-rotate})
as well as the fact that $\|\bm{F}^{\star}\|_{2,\infty}\geq\sqrt{\sigma_{\min}r/n}$.

With the above bound in place, we are ready to invoke \cite[Lemma 22]{chen2019noisy}
to obtain 
\begin{align*}
\big\|\bm{F}^{\mathsf{d}}\bm{H}^{\mathsf{d}}-\bm{F}^{\mathsf{d},(j)}\bm{H}^{\mathsf{d},(j)}\big\| & \lesssim\kappa\big\|\bm{F}^{\mathsf{d}}\bm{H}-\bm{F}^{\mathsf{d},(j)}\bm{R}^{(j)}\big\|\lesssim\kappa\big\|\bm{F}\bm{H}-\bm{F}^{(j)}\bm{R}^{(j)}\big\|_{\mathrm{F}}\\
 & \lesssim\kappa\frac{\sigma}{\sigma_{\min}}\sqrt{\frac{n\log n}{p}}\big\Vert \bm{F}^{\star}\big\Vert _{2,\infty},
\end{align*}
where the last line comes from (\ref{eq:F-j-F-best-rotate}). This
concludes the proof.

\section{Technical lemmas}

This section collects a few useful matrix perturbation bounds. The
first one is concerned with the perturbation of pseudo-inverses.

\begin{lemma}[\textsf{Perturbation of pseudo-inverses}]\label{lemma:pseudo-inverse}Let
$\bm{A}^{\dagger}$ (resp.~$\bm{B}^{\dagger}$) be the pseudo-inverse
(i.e.~Moore--Penrose inverse) of $\bm{A}$ (resp.~$\bm{B}$). Then
we have 
\[
\|\bm{B}^{\dagger}-\bm{A}^{\dagger}\|\leq3\max\left\{ \|\bm{A}^{\dagger}\|^{2},\|\bm{B}^{\dagger}\|^{2}\right\} \left\Vert \bm{B}-\bm{A}\right\Vert .
\]
\end{lemma}\begin{proof}See \cite[Theorem 3.3]{stewart1977perturbation}.
\end{proof}

The next lemma focuses on the perturbation bound for matrix square
roots.

\begin{lemma}[\textsf{Perturbation of matrix square roots}]\label{lemma:matrix-sqrt}Consider
two symmetric matrices obeying $\bm{A}_{1}\succeq\mu_{1}\bm{I}$ and
$\bm{A}_{2}\succeq\mu_{2}\bm{I}$ for some $\mu_{1},\mu_{2}>0$. Let
$\bm{R}_{1}\succeq\bm{0}$ (resp.~$\bm{R}_{2}\succeq\bm{0}$) be
the (principal) matrix square root of $\bm{A}_{1}$ (resp.~$\bm{A}_{2}$). Then
one has 
\[
\left\Vert \bm{R}_{1}-\bm{R}_{2}\right\Vert \leq\frac{1}{\sqrt{\mu_{1}}+\sqrt{\mu_{2}}}\left\Vert \bm{A}_{1}-\bm{A}_{2}\right\Vert .
\]
\end{lemma}\begin{proof}See \cite[Lemma 2.1]{MR1176461}. \end{proof}

The following lemma concerns the perturbation of top-$r$ components of matrices.

\begin{lemma}[\textsf{Perturbation of top-$r$ components}]
\label{lemma:SVD-pert}
Consider two matrices $\bm{M}, \bm{M}+\bm{E} \in \mathbb{R}^{n\times n}$. Suppose that $\|\bm{E}\|\leq\|\bm{M}\|$
and $\sigma_{r}(\bm{M})>\sigma_{r+1}(\bm{M}+\bm{E})$. Let $\bm{U}\bm{\Sigma}\bm{V}^{\top}$
(resp.~$\hat{\bm{U}}\hat{\bm{\Sigma}}\hat{\bm{V}}^{\top}$) be the
rank-$r$ SVD of $\bm{M}$ (resp.~$\bm{M}+\bm{E}$).
Then one has 
\[
\big\Vert \bm{U}\bm{\Sigma}\bm{V}^{\top}-\hat{\bm{U}}\hat{\bm{\Sigma}}\hat{\bm{V}}^{\top}\big\Vert_{\mathrm{F}} \leq\left(\frac{12\left\Vert \bm{\Sigma}\right\Vert }{\sigma_{r}\left(\bm{M}\right)-\sigma_{r+1}\left(\bm{M}+\bm{E}\right)}+1\right)\left\Vert \bm{E}\right\Vert _{\mathrm{F}}.
\]
\end{lemma}
\begin{proof} From Wedin's $\sin\bm{\Theta}$ theorem~\cite{wedin1972perturbation},
there exist  orthonormal matrices $\bm{R}_{1},\bm{R}_{2}\in\mathcal{O}^{r\times r}$
such that 
\begin{equation}
\max\left\{ \|\hat{\bm{U}}\bm{R}_{1}-\bm{U}\|_{\mathrm{F}},\|\hat{\bm{V}}\bm{R}_{2}-\bm{V}\|_{\mathrm{F}}\right\} \leq\frac{{2}}{\sigma_{r}\left(\bm{M}\right)-\sigma_{r+1}\left(\bm{M}+\bm{E}\right)}\left\Vert \bm{E}\right\Vert _{\mathrm{F}}.\label{eq:vector-pert}
\end{equation}
In addition, Weyl's inequality tells us that 
\begin{equation}
	\big\Vert \bm{\Sigma}-\hat{\bm{\Sigma}} \big \Vert \leq \big\Vert \bm{E}\big\Vert \qquad\text{and hence}\qquad\big\Vert \hat{\bm{\Sigma}}\big\Vert \leq2\big\Vert \bm{\Sigma} \big\Vert .
	\label{eq:value-pert}
\end{equation}
Here, the second inequality follows from the triangle inequality and
the assumption that $\|\bm{E}\|\leq\|\bm{M}\|=\|\bm{\Sigma}\|$. Expand
$\bm{U}\bm{\Sigma}\bm{V}^{\top}-\hat{\bm{U}}\hat{\bm{\Sigma}}\hat{\bm{V}}^{\top}$
and apply the triangle inequality to obtain 
\begin{align*}
\big\Vert \bm{U}\bm{\Sigma}\bm{V}^{\top}-\hat{\bm{U}}\hat{\bm{\Sigma}}\hat{\bm{V}}^{\top}\big\Vert _{\mathrm{F}} & =\big\Vert \bm{U}\bm{\Sigma}\bm{V}^{\top}-\hat{\bm{U}}\bm{R}_{1}\bm{R}_{1}^{\top}\hat{\bm{\Sigma}}\bm{R}_{2}\bm{R}_{2}^{\top}\hat{\bm{V}}^{\top}\big\Vert _{\mathrm{F}}\\
 & \leq\big\Vert \big(\bm{U}-\hat{\bm{U}}\bm{R}_{1}\big)\bm{\Sigma}\bm{V}^{\top}\big\Vert _{\mathrm{F}}+\big\Vert \hat{\bm{U}}\bm{R}_{1}\big(\bm{\Sigma}-\bm{R}_{1}^{\top}\hat{\bm{\Sigma}}\bm{R}_{2}\big)\bm{V}^{\top}\big\Vert _{\mathrm{F}}\\
 & \quad+\big\Vert \hat{\bm{U}}\bm{R}_{1}\bm{R}_{1}^{\top}\hat{\bm{\Sigma}}\bm{R}_{2}\big(\bm{V}-\hat{\bm{V}}\bm{R}_{2}\big)^{\top}\big\Vert _{\mathrm{F}},
\end{align*}
which further implies that
\begin{align}
\big\Vert \bm{U}\bm{\Sigma}\bm{V}^{\top}-\hat{\bm{U}}\hat{\bm{\Sigma}}\hat{\bm{V}}^{\top}\big\Vert _{\mathrm{F}} 
	& \leq \big\Vert \bm{U}-\hat{\bm{U}}\bm{R}_{1}\big\Vert _{\mathrm{F}}\big\Vert \bm{\Sigma}\big\Vert +\big\Vert \bm{\Sigma}-\bm{R}_{1}^{\top}\hat{\bm{\Sigma}}\bm{R}_{2}\big\Vert _{\mathrm{F}}+\big\Vert \hat{\bm{\Sigma}}\big\Vert \big\Vert \bm{V}-\hat{\bm{V}}\bm{R}_{2}\big\Vert _{\mathrm{F}}\nonumber \\
 & \leq\frac{6\left\Vert \bm{\Sigma}\right\Vert }{\sigma_{r}\left(\bm{M}\right)-\sigma_{r+1}\left(\bm{M}+\bm{E}\right)}\left\Vert \bm{E}\right\Vert _{\mathrm{F}}+\big\Vert \bm{\Sigma}-\bm{R}_{1}^{\top}\hat{\bm{\Sigma}}\bm{R}_{2}\big\Vert _{\mathrm{F}}.\label{eq:SVD-1}
\end{align}
Here, the last line arises from (\ref{eq:vector-pert}) and (\ref{eq:value-pert}).
It then boils down to controlling $\|\bm{\Sigma}-\bm{R}_{1}^{\top}\hat{\bm{\Sigma}}\bm{R}_{2}\|_{\mathrm{F}}$.
Recognizing that $\bm{\Sigma}=\bm{U}^{\top}\bm{M}\bm{V}$ and $\hat{\bm{\Sigma}}=\hat{\bm{U}}^{\top}(\bm{M}+\bm{E})\hat{\bm{V}}$, 
we obtain 
\begin{align*}
\big\Vert \bm{\Sigma}-\bm{R}_{1}^{\top}\hat{\bm{\Sigma}}\bm{R}_{2}\big\Vert _{\mathrm{F}} & =\left\Vert \bm{U}^{\top}\bm{M}\bm{V}-\bm{R}_{1}^{\top}\hat{\bm{U}}^{\top}\left(\bm{M}+\bm{E}\right)\hat{\bm{V}}\bm{R}_{2}\right\Vert _{\mathrm{F}}\\
 & \leq\big\Vert \big(\bm{U}-\hat{\bm{U}}\bm{R}_{1}\big)^{\top}\bm{M}\bm{V}\big\Vert _{\mathrm{F}}+\big\Vert \bm{R}_{1}^{\top}\hat{\bm{U}}^{\top}\bm{E}\bm{V}\big\Vert _{\mathrm{F}} 
	+\big\Vert \bm{R}_{1}^{\top}\hat{\bm{U}}^{\top}\left(\bm{M}+\bm{E}\right)\big(\bm{V}-\hat{\bm{V}}\bm{R}_{2}\big)\big\Vert _{\mathrm{F}}\\
 & \leq\big\Vert \bm{U}-\hat{\bm{U}}\bm{R}_{1}\big\Vert _{\mathrm{F}}\left\Vert \bm{\Sigma}\right\Vert +\left\Vert \bm{E}\right\Vert _{\mathrm{F}}
	+ \big\Vert \hat{\bm{\Sigma}}\big\Vert \big\Vert \bm{V}-\hat{\bm{V}}\bm{R}_{2}\big\Vert _{\mathrm{F}}.
\end{align*}
Once again,
employ (\ref{eq:vector-pert}) and (\ref{eq:value-pert}) to arrive at
\begin{equation}
\big\Vert \bm{\Sigma}-\bm{R}_{1}^{\top}\hat{\bm{\Sigma}}\bm{R}_{2}\big\Vert _{\mathrm{F}}\leq\frac{6\left\Vert \bm{\Sigma}\right\Vert }{\sigma_{r}\left(\bm{M}\right)-\sigma_{r+1}\left(\bm{M}+\bm{E}\right)}\left\Vert \bm{E}\right\Vert _{\mathrm{F}}+\left\Vert \bm{E}\right\Vert _{\mathrm{F}}.\label{eq:SVD-2}
\end{equation}
Combining (\ref{eq:SVD-1}) and (\ref{eq:SVD-2}), we reach 
\[
\big\Vert \bm{U}\bm{\Sigma}\bm{V}^{\top}-\hat{\bm{U}}\hat{\bm{\Sigma}}\hat{\bm{V}}^{\top}\big\Vert _{\mathrm{F}}\leq\left(\frac{12\left\Vert \bm{\Sigma}\right\Vert }{\sigma_{r}\left(\bm{M}\right)-\sigma_{r+1}\left(\bm{M}+\bm{E}\right)}+1\right)\left\Vert \bm{E}\right\Vert _{\mathrm{F}}
\]
as claimed.
\end{proof}

The last bound centers around the well-known Sylvester equation $\bm{X}\bm{A}+\bm{B}\bm{X}=\bm{C}$.

\begin{lemma}[\textsf{The Sylvester equation}]\label{lemma:AX+XB=00003DC}Suppose
$\bm{X}\in\mathbb{R}^{r\times r}$ satisfies the matrix equation $\bm{X}\bm{A}+\bm{B}\bm{X}=\bm{C}$
for some matrices $\bm{A}\in\mathbb{R}^{r\times r}$,$\bm{B}\in\mathbb{R}^{r\times r}$
and $\bm{C}\in\mathbb{R}^{r\times r}$. Then one has 
\[
\left\Vert \bm{X}\right\Vert \leq(2\lambda_{\min})^{-1}\left\Vert \bm{C}\right\Vert ,
\]
as long as $\lambda_{\min}\bm{I}_{r}\preceq\bm{A}\preceq\lambda_{\max}\bm{I}_{r}$
and $\lambda_{\min}\bm{I}_{r}\preceq\bm{B}\preceq\lambda_{\max}\bm{I}_{r}$
for some $\lambda_{\max}\geq\lambda_{\min}>0$. \end{lemma}\begin{proof}To
begin with, we intend to show that under the condition $\lambda_{\min}\bm{I}_{r}\preceq\bm{A},\bm{B}\preceq\lambda_{\max}\bm{I}_{r}$
for some $\lambda_{\max}\geq\lambda_{\min}>0$, there is a unique
solution to the matrix equation $\bm{X}\bm{A}+\bm{B}\bm{X}=\bm{C}$.
Use the notation of Kronecker product to obtain an equivalent form
of $\bm{X}\bm{A}+\bm{B}\bm{X}=\bm{C}$ as follows
\[
\mathsf{vec}\left(\bm{X}\bm{A}+\bm{B}\bm{X}\right)=\left(\bm{A}^{\top}\otimes\bm{I}_{r}+\bm{I}_{r}\otimes\bm{B}\right)\cdot\mathsf{vec}\left(\bm{X}\right)=\mathsf{vec}\left(\bm{C}\right),
\]
where $\otimes$ denotes the Kronecker product and $\mathsf{vec}(\bm{A})$
stands for the vectorization of the matrix $\bm{A}$. Given that $\bm{A}\succ\bm{0}$
and $\bm{B}\succ\bm{0}$, it is straightforward to see that $\bm{A}^{\top}\otimes\bm{I}_{r}+\bm{I}_{r}\otimes\bm{B}$
is invertible, thus justifying the uniqueness of $\bm{X}$.

The next step is to characterize $\bm{X}$ explicitly. The argument
herein is adapted from \cite{smith1968matrix} and \cite{MR1176461}.
Specifically, it has been shown in \cite{smith1968matrix} that the
equation $\bm{X}\bm{A}+\bm{B}\bm{X}=\bm{C}$ is equivalent to 
\[
\bm{X}-\bm{U}\bm{X}\bm{V}=\bm{W},
\]
where $\bm{U}=(q\bm{I}_{r}+\bm{B})^{-1}(q\bm{I}_{r}-\bm{B})$, $\bm{V}=(q\bm{I}_{r}-\bm{A})(q\bm{I}_{r}+\bm{A})^{-1}$
and $\bm{W}=2q(q\bm{I}_{r}+\bm{B})^{-1}\bm{C}(q\bm{I}_{r}+\bm{A})^{-1}$,
for any $q>0$. In particular, when $q>\lambda_{\min}$, the matrix
\begin{equation}
\bm{X}=\sum_{k=1}^{\infty}\bm{U}^{k-1}\bm{W}\bm{V}^{k-1}\label{eq:unique-representation}
\end{equation}
is the unique solution to $\bm{X}-\bm{U}\bm{X}\bm{V}=\bm{W}$ and
hence to $\bm{X}\bm{A}+\bm{B}\bm{X}=\bm{C}$. To show this, it suffices
to verify that the matrix series is convergent. Note that when $q>\lambda_{\min}$,
one has 
\[
\left\Vert \bm{U}\right\Vert \leq\|\left(q\bm{I}_{r}+\bm{B}\right)^{-1}\|\left\Vert q\bm{I}_{r}-\bm{B}\right\Vert \leq\frac{q-\lambda_{\min}}{q+\lambda_{\min}}<1,
\]
and similarly $\|\bm{V}\|\leq(q-\lambda_{\min})/(q+\lambda_{\max})<1$.
These two bounds taken together immediately establish the convergence
of the matrix series (\ref{eq:unique-representation}).

In the end, the explicit representation (\ref{eq:unique-representation})
allows us to upper bound $\|\bm{X}\|$. A little algebra reveals that
\begin{align*}
\left\Vert \bm{X}\right\Vert  & \leq\sum_{k=1}^{\infty}\left\Vert \bm{U}^{k-1}\bm{W}\bm{V}^{k-1}\right\Vert \leq\left\Vert \bm{W}\right\Vert \sum_{k=1}^{\infty}\left\Vert \bm{U}\right\Vert ^{k-1}\left\Vert \bm{V}\right\Vert ^{k-1}\leq\frac{\left\Vert \bm{W}\right\Vert }{1-\left\Vert \bm{U}\right\Vert \left\Vert \bm{V}\right\Vert },
\end{align*}
where we make use of the fact $\|\bm{U}\|\|\bm{V}\|<1$. In addition,
from the definition of $\bm{W}$ we know that
\[
\left\Vert \bm{W}\right\Vert \leq2q\big\|\left(q\bm{I}_{r}+\bm{B}\right)^{-1}\big\|\left\Vert \bm{C}\right\Vert \big\|\left(q\bm{I}_{r}+\bm{A}\right)^{-1}\big\|\leq\left\Vert \bm{C}\right\Vert \frac{2q}{\left(q+\lambda_{\min}\right)^{2}},
\]
provided that $q>0$. Combine this with the bounds on $\|\bm{U}\|$
and $\|\bm{V}\|$ to reach 
\[
\left\Vert \bm{X}\right\Vert \leq\frac{\left\Vert \bm{C}\right\Vert \frac{2q}{\left(q+\lambda_{\min}\right)^{2}}}{1-\left(\frac{q-\lambda_{\min}}{q+\lambda_{\min}}\right)^{2}}=\frac{2q\left\Vert \bm{C}\right\Vert }{\left(q+\lambda_{\min}\right)^{2}-\left(q-\lambda_{\min}\right)^{2}}=\frac{\left\Vert \bm{C}\right\Vert }{2\lambda_{\min}}
\]
as claimed. \end{proof}
 \bibliographystyle{alpha}
\bibliography{bibfileNonconvex}

\end{document}